\documentclass[twoside,11pt]{article}

\usepackage[preprint]{jmlr2e}
\usepackage{natbib,subfigure}
\usepackage{pifont}
\usepackage{amsmath}
\usepackage[english]{babel}

\usepackage{lmodern}
\usepackage{inconsolata}
\usepackage{xspace}
\usepackage{setspace}

\usepackage{amsmath}
\usepackage{amssymb}
\usepackage{amsfonts}
\usepackage{mathrsfs}
\usepackage{bm}
\usepackage{mathtools}
\usepackage{dsfont}
\usepackage{nicefrac}

\usepackage{graphicx}
\usepackage{xcolor}
\usepackage{color}
\usepackage{booktabs}
\usepackage{multirow}
\usepackage{makecell}
\usepackage{caption2}
\usepackage{rotating}
\usepackage{pdflscape}
\usepackage{standalone}
\usepackage{preview}
\usepackage{tcolorbox}
\usepackage{lipsum}

\usepackage{paralist}
\usepackage{enumerate}
\usepackage{enumitem}
\usepackage{prettyref}
\usepackage{appendix}

\let\proof\relax

\usepackage{amsthm}

\usepackage{algorithm}
\usepackage{algorithmic}

\usepackage{hyperref}
\usepackage{url}

\hypersetup{
    colorlinks,
    breaklinks,
    linkcolor = blue,
    citecolor = blue,
    urlcolor  = blue,
}
\allowdisplaybreaks[2]

\renewcommand{\tilde}{\widetilde}
\renewcommand{\hat}{\widehat}

\let\norm\undefined
\newcommand\norm[1]{\left\| #1 \right\|}
\newcommand\abs[1]{\left| #1 \right|}
\newcommand\inner[2]{\langle #1, #2 \rangle}
\newcommand\ceil[1]{\lceil #1 \rceil}

\newcommand\given[1][]{\:#1\vert\:}

\newcommand\term[1]{\textsc{term}~(\textsc{#1})}
\newcommand\sbr[1]{\left( #1 \right)}
\newcommand\mbr[1]{\left[ #1 \right]}
\newcommand\bbr[1]{\left\{ #1 \right\}}
\newcommand\seq[1]{\{#1\}_{t=1}^T}

\DeclareMathOperator*{\argmax}{arg\,max}
\DeclareMathOperator*{\argmin}{arg\,min}
\DeclareMathOperator*{\poly}{poly}

\newcounter{romancounter}
\newcommand{\rom}[1]{\setcounter{romancounter}{#1}\textit{(\roman{romancounter})}}

\definecolor{myblue}{RGB}{0,112,192}
\definecolor{myred}{RGB}{192,0,1}
\definecolor{wine_red}{RGB}{228,48,64}
\definecolor{DSgray}{cmyk}{0,1,0,0}

\newcommand{\black}[1]{{\color{black}{#1}}}

\newcommand{\cmark}{\ding{51}}
\newcommand{\xmark}{\ding{55}}
\def \Yes {\textcolor{red}{\cmark}}
\def \No {\textcolor{blue}{\xmark}}

\def \A {\mathcal{A}}
\def \B {\mathcal{B}}

\def \D {\mathcal{D}}
\def \E {\mathbb{E}}
\def \F {\mathcal{F}}

\def \H {\mathcal{H}}

\def \M {\mathcal{M}}

\def \O {\mathcal{O}}

\def \R {\mathbb{R}}

\def \X {\mathcal{X}}
\def \Y {\mathcal{Y}}

\def \Bb {\mathbb{B}}

\def \a {\boldsymbol{a}}
\def \abf {\mathbf{a}}
\def \b {\boldsymbol{b}}
\def \c {\boldsymbol{c}}

\def \e {\boldsymbol{e}}

\def \g {\mathbf{g}}

\def \m {\boldsymbol{m}}
\def \mb {\mathbf{m}}

\def \p {\boldsymbol{p}}
\def \q {\boldsymbol{q}}
\def \r {\boldsymbol{r}}

\def \uu {\boldsymbol{u}}
\def \v {\mathbf{v}}
\def \w {\boldsymbol{w}}
\def \x {\mathbf{x}}
\def \y {\mathbf{y}}
\def \z {\mathbf{z}}

\def \ph {\hat{p}}

\def \qh {\hat{q}}

\def \xh {\hat{\x}}

\def \xt {\tilde{\x}}

\def \pbh {\boldsymbol{\ph}}
\def \qbh {\boldsymbol{\qh}}

\def \epsilon {\varepsilon}
\def \sumtau {\sum_{t=1}^\tau}
\def \sumT {\sum_{t=1}^T}
\def \sumTT {\sum_{t=2}^T}

\def \sumM {\sum_{j=1}^M}
\def \sumN {\sum_{i=1}^N}

\def \Reg {\textnormal{\textsc{Reg}}}

\def \meta {\textsc{Meta-Reg}}
\def \base {\textsc{Base-Reg}}
\def \ellb {\boldsymbol{\ell}}
\def \Ot {\tilde{\O}}
\def \KL {\textnormal{KL}}

\def \ith {{i\text{-th}}\xspace}

\def \xs {\x^\star}

\def \is {{i^\star}}
\def \js {{j^\star}}

\def \define {\triangleq}
\def \les {\lesssim}

\let\proof\relax

\newenvironment{proof}{\par\noindent{\textbf{Proof}\ }}{\hfill\BlackBox\\[2mm]}

\def \endenv {\hfill\raisebox{1pt}{$\triangleleft$}\smallskip}

\newtheorem{myThm}{Theorem}

\newtheorem{myLemma}{Lemma}
\newtheorem{myCor}{Corollary}
\newtheorem{myProp}{Proposition}

\theoremstyle{definition}
\newtheorem{myAssum}{Assumption}
\newtheorem{myDef}{Definition}

\newtheorem{myRemark}{Remark}

\newcommand{\pref}[1]{\prettyref{#1}}

\newcommand{\savehyperref}[2]{\texorpdfstring{\hyperref[#1]{#2}}{#2}}
\newrefformat{eq}{\savehyperref{#1}{\black{Eq.}~\textup{(\ref*{#1})}}}
\newrefformat{eqn}{\savehyperref{#1}{\black{Eq.}~(\ref*{#1})}}
\newrefformat{lem}{\savehyperref{#1}{\black{Lemma}~\ref*{#1}}}
\newrefformat{lemma}{\savehyperref{#1}{\black{Lemma}~\ref*{#1}}}
\newrefformat{def}{\savehyperref{#1}{\black{Definition}~\ref*{#1}}}
\newrefformat{line}{\savehyperref{#1}{\black{Line}~\ref*{#1}}}
\newrefformat{thm}{\savehyperref{#1}{\black{Theorem}~\ref*{#1}}}
\newrefformat{table}{\savehyperref{#1}{\black{Table}~\ref*{#1}}}
\newrefformat{corr}{\savehyperref{#1}{\black{Corollary}~\ref*{#1}}}
\newrefformat{cor}{\savehyperref{#1}{\black{Corollary}~\ref*{#1}}}
\newrefformat{sec}{\savehyperref{#1}{\black{Section}~\ref*{#1}}}
\newrefformat{subsec}{\savehyperref{#1}{\black{Section}~\ref*{#1}}}
\newrefformat{app}{\savehyperref{#1}{\black{Appendix}~\ref*{#1}}}
\newrefformat{appendix}{\savehyperref{#1}{\black{Appendix}~\ref*{#1}}}
\newrefformat{assum}{\savehyperref{#1}{\black{Assumption}~\ref*{#1}}}
\newrefformat{ex}{\savehyperref{#1}{\black{Example}~\ref*{#1}}}
\newrefformat{fig}{\savehyperref{#1}{\black{Figure}~\ref*{#1}}}
\newrefformat{alg}{\savehyperref{#1}{\black{Algorithm}~\ref*{#1}}}
\newrefformat{remark}{\savehyperref{#1}{\black{Remark}~\ref*{#1}}}
\newrefformat{conj}{\savehyperref{#1}{\black{Conjecture}~\ref*{#1}}}
\newrefformat{prop}{\savehyperref{#1}{\black{Proposition}~\ref*{#1}}}
\newrefformat{proto}{\savehyperref{#1}{\black{Protocol}~\ref*{#1}}}
\newrefformat{prob}{\savehyperref{#1}{\black{Problem}~\ref*{#1}}}
\newrefformat{claim}{\savehyperref{#1}{\black{Claim}~\ref*{#1}}}
\newrefformat{que}{\savehyperref{#1}{\black{Question}~\ref*{#1}}}
\newrefformat{op}{\savehyperref{#1}{\black{Open Problem}~\ref*{#1}}}
\newrefformat{fn}{\savehyperref{#1}{\black{Footnote}~\ref*{#1}}}
\newrefformat{require}{\savehyperref{#1}{\black{Requirement}~\ref*{#1}}}

\numberwithin{equation}{section}

\def \Vs {V_\star}
\def \Vb {\bar{V}}
\def \hexp {h^{\textnormal{exp}}}
\def \hsc {h^{\textnormal{sc}}}
\def \hc {h^{\textnormal{c}}}
\def \mub {\boldsymbol{\mu}}
\newcommand{\Bottomcoef}{\kappa}

\newcommand{\UniGrad}{\textnormal{\textsf{\mbox{UniGrad}}}\xspace}
\newcommand{\UniGradpp}{\textnormal{\textsf{\mbox{UniGrad++}}}\xspace}
\newcommand{\Correct}{\textnormal{\textsf{UniGrad.Correct}}\xspace}
\newcommand{\Bregman}{\textnormal{\textsf{UniGrad.Bregman}}\xspace}
\newcommand{\Correctpp}{\textnormal{\textsf{UniGrad++.Correct}}\xspace}
\newcommand{\Bregmanpp}{\textnormal{\textsf{UniGrad++.Bregman}}\xspace}

\def \OOMD {\textnormal{\textsf{OOMD}}\xspace}
\def \OOGD {\textnormal{\textsf{OOGD}}\xspace}

\def \msmwc {\textnormal{\textsf{MsMwC}}\xspace}
\def \MSMWCtop {\mbox{\textnormal{\textsf{MoM-Top}}}\xspace}
\def \MSMWCmid {\mbox{\textnormal{\textsf{MoM-Mid}}}\xspace}
\newcommand{\msoms}{\mbox{\textnormal{\textsf{MoM}}}\xspace}
\def \mlprod {\textsf{Adapt-ML-Prod}\xspace}
\def \omlprod {\textsf{Optimistic-Adapt-ML-Prod}\xspace}
\def \Top {\textnormal{\textsc{top}}}
\def \Mid {\textnormal{\textsc{mid}}}

\def \Gsc {G_\scvx}
\def \Dpsi {\D_{\psi}}
\def \scvx {\textnormal{sc}}
\def \exp {\textnormal{exp}}
\def \cvx {\textnormal{c}}

\usepackage{lastpage}
\jmlrheading{1}{2025}{1-\pageref{LastPage}}{11/25}{10/00}{Zhao et al}{Peng Zhao, Yu-Hu Yan, Hang Yu, and Zhi-Hua Zhou}

\ShortHeadings{Adaptivity and Universality: Problem-dependent Universal Regret for OCO}{Zhao, Yan, Yu, Zhou}
\firstpageno{1}

\usepackage{multirow}
\usepackage{colortbl}  

\makeatletter
\renewenvironment{cases}{%
  \renewcommand{\arraystretch}{1.3}
  \left\lbrace
  \def\\{\cr}%
  \array{@{}l@{\quad}l@{}}%
}{%
  \endarray\right.%
}
\makeatother

\begin{document}

\title{Adaptivity and Universality: Problem-dependent Universal Regret for Online Convex Optimization}

\author{\name Peng Zhao     \email zhaop@lamda.nju.edu.cn \\
       \name Yu-Hu Yan      \email yanyh@lamda.nju.edu.cn \\
       \name Hang Yu        \email yuhang@lamda.nju.edu.cn \\
       \name Zhi-Hua Zhou \email zhouzh@lamda.nju.edu.cn \\
       \addr National Key Laboratory for Novel Software Technology, Nanjing University, China\\
   School of Artificial Intelligence, Nanjing University, China}

\editor{My Editor}

\maketitle

\begin{abstract}
    Universal online learning aims to achieve optimal regret guarantees without requiring prior knowledge of the curvature of online functions. 
    Existing methods have established minimax-optimal regret bounds for universal online learning, where a \emph{single} algorithm can simultaneously attain $\O(\sqrt{T})$ regret for convex functions, $\O(d \log T)$ for exp-concave functions, and $\O(\log T)$ for strongly convex functions, where $T$ is the number of rounds and $d$ is the dimension of the feasible domain.
    However, these methods still lack problem-dependent adaptivity.
    In particular, no universal method provides regret bounds that scale with the \emph{gradient variation} $V_T$, a key quantity that plays a crucial role in applications such as stochastic optimization and fast-rate convergence in games. 
    In this work, we introduce \textsf{UniGrad}, a novel approach that achieves both universality and adaptivity, with two distinct realizations: \Correct and \Bregman.
    Both methods achieve universal regret guarantees that adapt to gradient variation, simultaneously attaining $\O(\log V_T)$ regret for strongly convex functions and $\O(d \log V_T)$ regret for exp-concave functions.
    For convex functions, the regret bounds differ: \Correct achieves an $\O(\sqrt{V_T \log V_T})$ bound while preserving the RVU property that is crucial for fast convergence in online games, whereas \Bregman achieves the optimal $\O(\sqrt{V_T})$ regret bound through a novel design.
    Both methods employ a meta algorithm with $\O(\log T)$ base learners, which naturally requires $\O(\log T)$ gradient queries per round.
    To further enhance computational efficiency, we introduce \textsf{UniGrad++}, which retains the regret guarantees while reducing the gradient query requirement to just $1$ per round via a surrogate optimization technique. 
    Our results advance the state-of-the-art in universal online learning, with immediate implications and applications, including small-loss and gradient-variance bounds, novel guarantees for the stochastically extended adversarial model, and faster convergence in online games.
    Additionally, as an extension, we provide an anytime variant of our method.
\end{abstract}

\section{Introduction}
\label{sec:intro}
Online convex optimization (OCO) is a versatile and powerful framework for modeling the interaction between a learner and the environment over time~\citep{book'16:Hazan-OCO, book'22:FO-book}.
In each round $t \in [T]$, the learner selects a decision $\x_t$ from a convex compact set $\X \subseteq \R^d$, while the environment simultaneously chooses a convex online function $f_t: \X \rightarrow \R$.
The learner then incurs a loss $f_t(\x_t)$ and receives information about the online function to update the decision to $\x_{t+1}$, with the goal of optimizing the game-theoretic performance metric known as \emph{regret}~\citep{book'06:PLG-Bianchi}, whose definition is given by
\begin{equation}
  \label{eq:regret}
  \Reg_T = \sumT f_t(\x_t) - \min_{\x \in \X} \sumT f_t(\x).
\end{equation}
The regret quantifies the learner's excess loss relative to the best offline decision in hindsight.
In OCO, it is well-established that the type and curvature of online functions significantly influence the minimax regret bounds.
For convex functions, Online Gradient Descent (OGD) can achieve an $\O(\sqrt{T})$ regret~\citep{ICML'03:zinkevich}.
For $\alpha$-exp-concave functions, Online Newton Step (ONS), with prior knowledge of the curvature coefficient $\alpha$, attains an $\O(\frac{d}{\alpha} \log T)$ regret \citep{MLJ'07:Hazan-logT}.
For $\lambda$-strongly convex functions, OGD with a suitable step size configuration relating to the curvature coefficient $\lambda$ enjoys an $\O(\frac{1}{\lambda} \log T)$ regret~\citep{MLJ'07:Hazan-logT}.
These regret rates have been proved to be minimax optimal~\citep{MOR'98:exp-concave-lowerbound,COLT'08:OCO-lowerbound}.

Notably, there are two caveats in the above results.
First, to achieve the optimal regret bounds, these algorithms require prior knowledge of the function type and the parameter characterizing the curvature as the algorithmic input.
Moreover, these algorithms are only optimal in the worst case and lack adaptivity to problem-dependent hardness.
Therefore, modern online learning research strengthens these results in two key aspects: \rom{1} \emph{universality}, to handle the unknown types and curvatures; \rom{2} \emph{adaptivity}, to adapt to the problem-dependent hardness.
In the following, we discuss each aspect in detail.

\subsection{Universality: Unknown Curvature of Online Functions}
\label{subsec:universal}
Traditionally, to attain the minimax optimality, the learner must select the ``correct'' algorithm with well-tuned parameters tailored to each specific class of online functions, which requires the prior knowledge of the curvature information and can be burdensome in practice.
\emph{Universal online learning} aims to develop a single algorithm agnostic to the specific function type and curvature while achieving the same regret guarantees as if they were known~\citep{NIPS'16:MetaGrad,NIPS'17:Ashok-universal,UAI'19:Maler, COLT'19:Lipschitz-MetaGrad,NeurIPS'21:dual-adaptivity,ICML'22:universal,NeurIPS'23:universal,NeurIPS'24:OptimalGV,NeurIPS'24:universalProjection}.
The pioneering work of \citet{NIPS'16:MetaGrad} proposes the \textsf{MetaGrad} algorithm that achieves an $\O(\sqrt{T})$ regret for convex functions and an $\O(\frac{d}{\alpha} \log T)$ regret for exp-concave/strongly convex functions, leaving a gap to the optimal $\O(\frac{1}{\lambda} \log T)$ regret for strongly convex functions, which is later closed by \citet{UAI'19:Maler}. 

The above methods rely on a meta-base two-layer structure to handle the uncertainty of the function curvature.
Specifically, for each function family, the online learner maintains a set of base learners with different configurations, with a meta algorithm running on top to combine their outputs.
To achieve the desired universality, MetaGrad and its variants require the base learners to optimize over heterogeneous surrogate loss functions customized to each function family, and use a complex meta algorithm to adaptively combine these heterogeneous base learners.
The entire structure can be complex and technically challenging to analyze due to the heterogeneity of the base learners. 

To enhance flexibility, \citet{ICML'22:universal} introduce a simple and general framework that uses a meta algorithm equipped with a second-order regret bound, directly operating the base learners on the original online functions.
The method still achieves optimal regret bounds universally: $\O(\sqrt{T})$ for convex functions, $\O(\frac{d}{\alpha} \log T)$ for $\alpha$-exp-concave functions, and $\O(\frac{1}{\lambda} \log T)$ for $\lambda$-strongly convex functions.
It requires $\O(\log T)$ gradient queries per round, as $\O(\log T)$ base learners are maintained in parallel with different configurations.

\subsection{Adaptivity: Unknown Niceness of Online Environments}
\label{subsec:adaptivity}
Within a specific function family, the algorithm's performance is also influenced by the problem-dependent hardness.
Ideally, a well-designed online algorithm should be robust to the worst-case scenarios while simultaneously delivering faster-rate guarantees in more benign and easier environments.
\emph{Problem-dependent online learning} aims to design algorithms with regret guarantees scaling with the problem-dependent quantities~\citep{NIPS'10:smooth,COLT'12:VT,AISTATS'12:Orabona,COLT'18:adaptive-bandits,ICML'19:Adaptive-Smooth,NeurIPS'20:Ashok-matrix-regret,JMLR'24:Sword++}.
There are several different notions of problem-dependent adaptivity, including the small-loss bound, the gradient-variance bound, and the gradient-variation bound.
Among these, we focus on \emph{gradient-variation regret}~\citep{COLT'12:VT, MLJ'14:variation-Yang}, due to its fundamental importance in modern online learning and its strong connections to game theory and optimization.
The gradient-variation regret replaces the dependence on the time horizon $T$ with the gradient-variation quantity $V_T$ defined as  
\begin{equation}
  \label{eq:VT}
  V_T \define \sumTT \sup_{\x \in \X} \|\nabla f_t(\x) - \nabla f_{t-1}(\x)\|^2,
\end{equation}  
which measures the cumulative change of gradients across consecutive functions.
When the online functions change gradually, $V_T$ can be very small, even as low as $\O(1)$ when the functions are basically the same.
On the other hand, under the standard bounded-gradient assumption, $V_T$ is at most $\O(T)$ in the worst case.
Under smoothness, more adaptive algorithms can improve the minimax regret to $\O(\sqrt{V_T})$, $\O(\frac{d}{\alpha} \log V_T)$, and $\O(\frac{1}{\lambda} \log V_T)$ for convex, $\alpha$-exp-concave, and $\lambda$-strongly convex functions, respectively. 

These results are important since they safeguard the minimax worst-case rates and can be much better when the environment is easier such as $V_T = \O(1)$.
Moreover, as demonstrated by \citet{JMLR'24:Sword++}, the gradient-variation regret is more fundamental than other well-known problem-dependent quantities like the small loss $F_T = \min_{\x \in \X} \sumT f_t(\x)$~\citep{NIPS'10:smooth, AISTATS'12:Orabona}, since gradient-variation regret can imply small-loss bounds directly in analysis.
Furthermore, gradient variation plays a crucial role in bridging adversarial and stochastic optimization~\citep{NeurIPS'22:SEA,JMLR'24:OMD4SEA}, enabling fast rates in online games~\citep{NIPS'15:fast-rate-game, ICML'22:TVgame}, and facilitating acceleration in smooth offline optimization~\citep{LectureNote:AOpt25,NeurIPS'25:Holder}.
As a result, there has been a surge of recent interest in achieving gradient-variation regret bounds in various online learning problems~\citep{NeurIPS'20:Sword,NeurIPS'22:SEA,ICML'22:TVgame,AAAI'23:GV-constraints,NeurIPS'23:portfolio,JMLR'24:Sword++,AISTATS'24:online-bilevel,NeurIPS'24:Xie}.

\subsection{Our Contributions and Techniques}
\label{subsec:contributions}
Motivated by the above progress of modern online learning, a natural question arises: \emph{Is it possible to design a single algorithm that achieves both universality and adaptivity?}
More concretely, the goal is to design a universal algorithm with gradient-variation regret bounds across different function families: a single online algorithm simultaneously achieving an $\O(\sqrt{V_T})$ regret for convex functions, an $\O(\frac{d}{\alpha} \log V_T)$ regret for $\alpha$-exp-concave functions, and an $\O(\frac{1}{\lambda} \log V_T)$ regret for $\lambda$-strongly convex functions, respectively.

\begin{table}[!t]
  \centering
  \caption{Comparison with existing results (including prior works and our conference version~\citep{NeurIPS'23:universal}). The second column shows the regret bounds for strongly convex, exp-concave, and convex functions. The third column indicates the computational efficiency, where \mbox{``\# Grad.''} is the number of gradient queries in each round and \mbox{``\# Base''} is the number of maintained base learners. The last column indicates whether the method can support the RVU property that is vital for the fast-convergence of online games.} 
  \label{table:main}
  \vspace{2mm}
  \renewcommand*{\arraystretch}{1.6}
  \resizebox{0.999\textwidth}{!}{
  \begin{tabular}{c|ccc|c|c|c}
    \hline

    \hline
    \multirow{2}{*}{\textbf{Method}} & \multicolumn{3}{c|}{\textbf{Regret Bounds}} & \multicolumn{2}{c|}{\textbf{Efficiency}} & \multirow{2}{*}{\textbf{RVU}} \\ \cline{2-6}                 
    & Strongly Convex & Exp-concave & Convex & \# Grad. & \# Base \\ \hline
    \citet{NIPS'16:MetaGrad} & $\O(d \log T)$ & $\O(d \log T)$ & $\O(\sqrt{T})$ & $1$ & $\O(\log T)$ & \No\\ \hline
    \citet{UAI'19:Maler} & $\O(\log T)$ & $\O(d \log T)$ & $\O(\sqrt{T})$ & $1$ & $\O(\log T)$ & \No\\ \hline
    \citet{ICML'22:universal} & $\O(\log V_T)$ & $\O(d \log V_T)$ & $\O(\sqrt{T})$ & $\O(\log T)$ & $\O(\log T)$ & \No\\ \hline  \hline
    \rowcolor{gray!13}\citet{NeurIPS'23:universal} & $\O(\log V_T)$ & $\O(d \log V_T)$ & $\O(\sqrt{V_T \log V_T})$ & $\O((\log T)^2)$ & $\O((\log T)^2)$ & \Yes\\    \hline
    \rowcolor{gray!13}\Correct (\pref{thm:unigrad-correct}) & $\O(\log V_T)$ & $\O(d \log V_T)$ & $\O(\sqrt{V_T \log V_T})$ & $\O(\log T)$ & $\O(\log T)$ & \Yes\\    \hline
    \rowcolor{gray!13}\textsf{UniGrad.Bregman} (\pref{thm:unigrad-bregman}) & $\O(\log V_T)$ & $\O(d \log V_T)$ & $\O(\sqrt{V_T})$ & $\O(\log T)$ & $\O(\log T)$ & \No\\ \hline \hline
    \rowcolor{gray!13}\textsf{UniGrad++.Correct} (\pref{thm:unigrad-correct-1grad}) & $\O(\log V_T)$ & $\O(d \log V_T)$ & $\O(\sqrt{V_T \log V_T})$ & $1$ & $\O(\log T)$ & \Yes\\
    \hline
    \rowcolor{gray!13}\textsf{UniGrad++.Bregman} (\pref{thm:unigrad-bregman-1grad}) & $\O(\log V_T)$ & $\O(d \log V_T)$ & $\O(\sqrt{V_T})$ & $1$ & $\O(\log T)$ & \No\\ \hline

    \hline
  \end{tabular}}
\end{table}

\citet{ICML'22:universal} obtain partial results with regret bounds of $\O(\sqrt{T})$, $\O(\frac{d}{\alpha} \log V_T)$, and $\O(\frac{1}{\lambda} \log V_T)$ for convex, $\alpha$-exp-concave, and $\lambda$-strongly convex functions, simultaneously.
Their result, however, falls short in the convex case, which is arguably the most important: the improvement from $T$ to $V_T$ is polynomial for convex functions, whereas logarithmic for the other cases.
This gap was left open in their work. 

In this paper, we resolve the open problem and obtain the desired gradient-variation regret bounds for all three types of functions.
Our approach builds on the \emph{optimistic online ensemble} framework developed for gradient-variation dynamic regret~\citep{JMLR'24:Sword++}.
However, significant new ingredients are required for universal online learning, particularly in properly encoding gradient-variation adaptivity across all three types of functions.
We propose a novel approach called \UniGrad (short for ``\underline{Uni}versal \underline{Grad}ient-variation Online Learning''), consisting of two distinct realizations based on fundamentally different ideas.
\begin{itemize}[leftmargin=*]
    \item \textbf{Method 1 (\textsf{UniGrad.Correct}): Online Ensemble with Injected Corrections}. We propose the \Correct algorithm, which consists of a three-layer ensemble structure and incorporates \emph{injected corrections} to ensure stability cancellation within the online ensemble. The algorithm simultaneously achieves regret bounds of $\O(\frac{1}{\lambda} \log V_T)$ for $\lambda$-strongly convex functions, $\O(\frac{d}{\alpha} \log V_T)$ for $\alpha$-exp-concave functions, and $\O(\sqrt{V_T \log V_T})$ for convex functions.
    \item \textbf{Method 2 (\textsf{UniGrad.Bregman}): Online Ensemble with Extracted Bregman Divergence}. We propose the \Bregman algorithm, which leverages a negative term from the \emph{extracted Bregman divergence} in linearization and employs a novel analysis that bypasses stability arguments to handle gradient variation. The algorithm simultaneously achieves regret bounds of $\O(\frac{1}{\lambda} \log V_T)$ for $\lambda$-strongly convex functions, $\O(\frac{d}{\alpha} \log V_T)$ for $\alpha$-exp-concave functions, and $\O(\sqrt{V_T})$ for convex functions.
\end{itemize}
Although \Correct exhibits slight suboptimality in the convex case compared to \Bregman, the two methods rely on fundamentally different principles and enjoy their own merits.
\pref{table:main} provides a comparison of our results with prior works and our conference version~\citep{NeurIPS'23:universal}, highlighting that we are the first to achieve gradient-variation regret bounds across all three types of functions simultaneously.
Both \Correct and \Bregman require $\O(\log T)$ gradient queries per round, as they maintain this many base learners in parallel.
To improve efficiency, we further develop \UniGradpp, which matches the same regret guarantees with only $1$ gradient query per round.
The key improvement is a careful deployment of the ``surrogate optimization'' technique, which broadcasts global gradient information to all base learners while effectively handling bias.
Additionally, we extend our method to an anytime variant, eliminating the requirement of the time horizon $T$ in advance and preserving the same guarantees.

\paragraph{Technical Contributions.}
To achieve gradient-variation adaptivity, it is typical to first establish regret guarantees with respect to the \emph{empirical} gradient-variation quantity $\bar{V}_T \define \sumTT \|\nabla f_t(\x_t) - \nabla f_{t-1}(\x_{t-1})\|^2$, and then convert them to the desired $V_T$-type bounds (see definition in~\pref{eq:VT}) by handling the additional positive terms.
This requires carefully extracting proper positive terms and canceling them by leveraging negative terms in the regret analysis and algorithm design comprehensively, as well as exploiting additional curvature-induced negative terms in the exp-concave and strongly convex cases.
Crucially, all these considerations must be compatible with the online ensemble structure, which demands careful design and, in some cases, surgical adjustments across meta-base layers.
Below, we discuss key techniques of each method (with further comparisons in \pref{sec:comparison-discussion}), along with \UniGradpp and the anytime variant.
\begin{itemize}[leftmargin=*]
  \item \textbf{Techniques of \textsf{UniGrad.Correct}.} The key challenge here is to handle the \emph{stability} term $\sumTT \|\x_t - \x_{t-1}\|^2$ in universal online learning. This requires the meta algorithm to achieve an optimistic second-order regret while retaining a stability negative term in the analysis. To this end, we employ a two-layer mirror-descent-based meta algorithm, resulting in an overall three-layer online ensemble structure. We develop a cascaded correction mechanism to cancel stability in this three-layer ensemble and design appropriate optimism to ensure adaptivity across all three function classes. Owing to this explicit stability cancellation, \Correct preserves the \textsf{RVU} (Regret bounded by Variations in Utility) property, which is essential for achieving fast convergence in online games~\citep{NIPS'15:fast-rate-game}.
  \item \textbf{Techniques of \textsf{UniGrad.Bregman}.} We employ a fundamentally different approach by conducting a novel smoothness analysis of gradient variations. This enables handling a positive term unrelated to stability, thereby bypassing the need for stability-induced negative terms in the meta algorithm. Combined with a new \emph{Bregman-divergence negative term} extracted from the linearization, the algorithm avoids stability arguments entirely and employs a simple two-layer structure, achieving the optimal $\O(\sqrt{V_T})$ regret for convex functions. This optimal universal rate directly yields optimal guarantees for the stochastically extended adversarial (SEA) model~\citep{JMLR'24:OMD4SEA}. 
  \item \textbf{Techniques of \textsf{UniGrad++}.} To improve gradient query efficiency, instead of relying on the multi-gradient information $\{\nabla f_t(\x_{t,i})\}_{i=1}^N$, \UniGradpp adopts surrogate optimization by using only the global gradient $\nabla f_t(\x_t)$ to construct different surrogate functions that incorporate curvature information and then feed them into the meta and base updates. Consequently, it is crucial to address the bias introduced by the surrogate function and to handle the additional positive term arising from gradient variations with respect to the surrogate functions.
  \item \textbf{Techniques of Anytime variant.}  For the anytime variant, since the algorithm does not know the time horizon $T$ in advance, it is impossible to predefine the number of base learners, which is originally set as $\O(\log T)$. Moreover, the doubling trick cannot be employed in this case, as it would introduce $\poly(\log T)$ regret degradations, thereby ruining the desired gradient-variation bounds for exp-concave and strongly convex functions. We design a dynamic online ensemble framework where the number of base learners is adjusted dynamically based on certain monitoring metrics.
\end{itemize}

\paragraph{Implication and Applications.}
We demonstrate the importance and generality of our results through several implications and applications.
\rom{1} The obtained gradient-variation regret bounds not only safeguard worst-case guarantees~\citep{NIPS'16:MetaGrad,UAI'19:Maler} but also directly imply the small-loss bounds of~\citet{ICML'22:universal} and the gradient-variance bounds of~\citet{MLJ'10:Hazan-variance-bound} in analysis.
\rom{2} Gradient variation is shown to play an essential role in the stochastically extended adversarial (SEA) model~\citep{NeurIPS'22:SEA,JMLR'24:OMD4SEA}, an interpolation between stochastic and adversarial convex optimization.
Our approach positively resolves a major open problem left in \citet{JMLR'24:OMD4SEA} on whether it is possible to develop a single algorithm with universal guarantees across strongly convex, exp-concave, and convex functions in the SEA model.
\rom{3} In game theory, gradient variation captures changes in other players' actions and facilitates fast convergence to the Nash equilibrium with stability cancellation arguments~\citep{NIPS'13:optimism-games, NIPS'15:fast-rate-game, ICML'22:TVgame}, and we apply UniGrad.Correct to two-player zero-sum games to illustrate its universality.

\paragraph{Comparison to Conference Version.}
This journal extension significantly improves upon our earlier conference papers~\citep{NeurIPS'23:universal,NeurIPS'24:OptimalGV} in algorithm design, regret analysis, presentation, and experimental evaluation.
Specifically, while \Correct still employs a three-layer online ensemble to achieve the desired gradient-variation regret, the initial algorithm of~\citet{NeurIPS'23:universal} required maintaining $\O((\log T)^2)$ base learners.
In contrast, the new design reduces the number of base learners to $\O(\log T)$, substantially improving computational efficiency.
This improvement is enabled by a sharper understanding of the three-layer online ensemble, leading to a new construction of correction terms injected into the meta algorithm's feedback loss.
More comparisons are discussed in~\pref{subsec:comparison-correct-conference}.
Additionally, we develop an anytime variant of~\citet{NeurIPS'24:OptimalGV} that eliminates the need for $T$ in advance, which is achieved through a novel dynamic online ensemble framework that adjusts the number of base learners based on monitoring metrics.
Lastly, we have made substantial improvements to the presentation, introducing a more systematic and unified framework for the two methods and a modular structure for technical proofs.
We then present the one-gradient version, with the critical role of surrogate loss design highlighted.
We also include additional implications and applications to broaden the scope and significance of our methods and conduct empirical evaluation to validate their effectiveness.

\paragraph{Organization.}
In the following, we first formally state the problem setup and review a general framework for universal online learning in~\pref{sec:preliminary}.
Next, we provide the main technical results in \pref{sec:method1-correct} and \pref{sec:method2-Bregman}, where two methods with universal gradient-variation regret bounds are developed.
\pref{sec:one-gradient} enhances the efficiency by ensuring only $1$ gradient query per round.
Then, we discuss the implications, applications, and extensions of the obtained gradient-variation universal regret bounds in \pref{sec:applications}.
Based on the technical details and provided applications, \pref{sec:comparison-discussion} offers detailed discussions of the two methods and the extension over the conference version.
\pref{sec:experiments} reports the experiments.
Finally, \pref{sec:conclusion} concludes the paper.
All proofs and omitted details are deferred to appendices.

\section{Problem Setup and Preliminaries}
\label{sec:preliminary}
In this section, we introduce preliminaries, including the problem setup, assumptions, the optimistic online mirror descent, and a general universal online learning framework. 

\paragraph{Notations.}
We use $\|\cdot\|$ for $\|\cdot\|_2$ by default.
We represent the $i$-th out of $d$ dimensions of the bold vector $\v$ (or $\boldsymbol{v}$) using the corresponding regular font $v_i$, i.e., $\v$ (or $\boldsymbol{v}$) $= (v_1, v_2, \dots, v_d)^\top$.
$\|\x\|_{U} \define \sqrt{\x^\top U \x}$ refers to the matrix norm for any $\x$, where $U$ is a positive semi-definite matrix.
For a strictly convex and differentiable function $\psi:\X\rightarrow \R$, the induced Bregman divergence is defined as $\D_\psi(\x, \y) \define \psi(\x) - \psi(\y) - \inner{\nabla \psi(\y)}{\x - \y}$.
We use $\Delta_d$ to represent a $d$-dimensional simplex and denote the $i$-th basis vector by $\e_i$.
We adopt the asymptotic notations $a \les b$ or $a = \O(b)$ to denote that there exists a constant $C < \infty$ such that $a\le C b$.
We use the $\O(\cdot)$-notation to highlight the dependence on $T$ and $V_T$ while treating the iterated logarithmic factors as a constant following previous work~\citep{ALT'12:closer-adaptive-regret,COLT'15:Luo-AdaNormalHedge,JMLR'24:Sword++}.

\subsection{Problem Setup}
\label{subsec:problem-setup}
The protocol of online convex optimization (OCO) is as follows: at each round $t\in [T]$, the learner will select a decision $\x_t \in \X \subseteq \R^d$, while the environment simultaneously chooses a convex function $f_t: \X\rightarrow \R$.
The learner then incurs a loss $f_t(\x_t)$ and observes the gradient information of the online function $f_t(\cdot)$.
Following~\citet{JMLR'24:Sword++}, the OCO setting can be further refined based on the type of gradient information accessible to the learner:
\begin{enumerate}
    \item[\rom{1}] \textbf{multi-gradient feedback}: the learner can access multiple gradients of the online function, that is, $f_t(\cdot)$ at round $t \in [T]$;
    \item[\rom{2}] \textbf{one-gradient feedback}: the learner can only access one gradient at the decision point, that is, $\nabla f_t(\x_t)$ at round $t \in [T]$.
\end{enumerate}
We will first address the multi-gradient feedback model (in~\pref{sec:method1-correct} and~\pref{sec:method2-Bregman}) and then improve our results to the more challenging one-gradient feedback model in~\pref{sec:one-gradient}.

The goal of the online learner is to minimize the regret measure defined in~\pref{eq:regret}.
It is now well-established that the regret rates differ significantly depending on the type of online functions and their curvature coefficients.
In fact, there are three main classes of online functions: strongly convex, exponentially concave (abbreviated as exp-concave), and convex functions.
The formal definitions are as follows (with convex functions omitted).
\begin{myDef}[Strong Convexity]
    \label{def:sconvex}
    A function $f(\cdot)$ is $\lambda$-strongly convex if $f(\x) - f(\y) \le \inner{\nabla f(\x)}{\x - \y} - \frac{\lambda}{2} \cdot \|\x - \y\|^2$ holds for any $\x,\y \in \X$.
\end{myDef}
\begin{myDef}[Exp-Concavity]
    \label{def:expconcave}
    A function $f(\cdot)$ is $\alpha$-exponentially concave (abbreviated as \emph{exp-concave}),\footnote{The formal definition of $\beta$-exp-concavity is that $\exp(-\beta f(\cdot))$ is concave. Under Assumptions~\ref{assum:domain-boundedness} and~\ref{assum:gradient-boundedness} (see \pref{subsec:notation-assumption-definition}), $\beta$-exp-concavity implies \pref{def:expconcave} with $\alpha = \frac{1}{2} \cdot \min\{1/(4GD), \beta\}$~{\citep[Lemma~4.3]{book'16:Hazan-OCO}}. For clarity and simplicity, we adopt \pref{def:expconcave} as an alternative of exp-concavity.} if $f(\x) - f(\y) \le \inner{\nabla f(\x)}{\x - \y} - \frac{\alpha}{2} \cdot \inner{\nabla f(\x)}{\x - \y}^2$ holds for any $\x,\y \in \X$.
\end{myDef}
We refer to the $- \frac{\lambda}{2} \cdot \|\x - \y\|^2$ term in $\lambda$-strongly convex functions and the $- \frac{\alpha}{2} \cdot \inner{\nabla f(\x)}{\x - \y}^2$ term in $\alpha$-exp-concave functions as the \emph{curvature-induced negative terms}, which play a crucial role in achieving improved regret bounds compared to convex functions.
For the problem-independent regret bounds, it is known that the minimax rates are $\O(\frac{1}{\lambda}\log T)$, $\O(\frac{d}{\alpha}\log T)$, and $\O(\sqrt{T})$ for $\lambda$-strongly convex, $\alpha$-exp-concave, and convex functions, respectively~\citep{MOR'98:exp-concave-lowerbound,COLT'08:OCO-lowerbound}.
For the more adaptive gradient-variation regret bounds, it is known that different algorithms can be designed for each class of functions to achieve the corresponding regret bounds: $\O(\frac{1}{\lambda}\log V_T)$ regret for $\lambda$-strongly convex functions, $\O(\frac{d}{\alpha}\log V_T)$ regret for $\alpha$-exp-concave functions, and $\O(\sqrt{V_T})$ regret for convex functions~\citep{COLT'12:VT,ICML'22:universal}.

\paragraph{Universal Online Learning.}
As can be observed from the above discussions, the curvature information is crucial for the regret rate (no matter for the problem-independent or gradient-variation regret), and thus it is crucial for the online learner to choose the correct algorithm with well-tuned parameters for each class of functions.
However, this clearly burdens the learner with the prior knowledge of the function type and the parameter characterizing the curvature, hence prohibiting more applications in practice.
Given this background, \emph{universal online learning} aims to design a \emph{single} algorithm that can achieve the optimal regret bound for all three classes of online functions simultaneously.

Mathematically, for a sequence of online functions $\{f_t\}_{t=1}^T$ that may belong to one of the three classes -- $\F^\lambda_\scvx$ (for $\lambda$-strongly convex functions), $\F^\alpha_{\text{ec}}$ (for $\alpha$-exp-concave functions), and $\F_{\text{c}}$ (for convex functions), universal online learning algorithm $\A$ aims to attain the following \emph{universal regret} satisfying:
\begin{equation}
    \label{eq:universal-goal}
    \Reg_T(\A, \{f_t\}_{t=1}^T) \lesssim 
    \begin{cases}
        \Reg_T(\A_\scvx, \F^\lambda_\scvx), & \text{when } \{f_t\}_{t=1}^T \text{ belongs to } \F^\lambda_\scvx, \\[2mm]
        \Reg_T(\A_{\text{ec}}, \F^\alpha_{\text{ec}}), & \text{when } \{f_t\}_{t=1}^T \text{ belongs to } \F^\alpha_{\text{ec}}, \\[2mm]
        \Reg_T(\A_{\text{c}}, \F_{\text{c}}), & \text{when } \{f_t\}_{t=1}^T \text{ belongs to } \F_{\text{c}}, \\[2mm]
    \end{cases}
\end{equation}
where $\A_\scvx$, $\A_{\text{ec}}$, $\A_{\text{c}}$ are the (optimal) algorithms designed for $\F^\lambda_\scvx$, $\F^\alpha_{\text{ec}}$, and $\F_{\text{c}}$, respectively.
The corresponding regret bounds are denoted as $\Reg_T(\A_\scvx,\F^\lambda_\scvx)$, $\Reg_T(\A_{\text{ec}},\F^\alpha_{\text{ec}})$, and $\Reg_T(\A_{\text{c}},\F_{\text{c}})$.
For problem-independent regret, the respective rates are $\O(\frac{1}{\lambda}\log T)$, $\O(\frac{d}{\alpha}\log T)$, and $\O(\sqrt{T})$, as achieved by~\citet{ICML'22:universal}.
Furthermore, when adapting to gradient variations, the regret improves to $\O(\frac{1}{\lambda}\log V_T)$, $\O(\frac{d}{\alpha}\log V_T)$, and $\O(\sqrt{V_T})$.

\subsection{Assumptions and Optimistic Online Mirror Descent}
\label{subsec:notation-assumption-definition}
In this subsection, we first present several standard assumptions commonly used in online convex optimization, and then introduce the algorithmic framework of optimistic online mirror descent (\OOMD)~\citep{COLT'12:VT,COLT'13:optimistic}, which serves not only the foundation of many (adaptive) online learning algorithms, but also the basis of our proposed methods for universal online learning.
\begin{myAssum}[Domain Boundedness]
    \label{assum:domain-boundedness}
    For any $\x,\y \in \X \subseteq \R^d$, the domain diameter satisfies $\|\x - \y\| \le D$. 
\end{myAssum}

\begin{myAssum}[Gradient Boundedness]
    \label{assum:gradient-boundedness}
    For all $t \in [T]$ and any $\x \in \X$, the gradient norm of the online functions is bounded as $\|\nabla f_t(\x)\| \le G$. 
\end{myAssum}

\begin{myAssum}[Smoothness]
    \label{assum:smoothness}
    For each $t \in [T]$, the online function $f_t(\cdot)$ is $L$-smooth, i.e., $\|\nabla f_t(\x) - \nabla f_t(\y)\| \le L \|\x - \y\|$ holds for any $\x, \y \in \R^d$.
\end{myAssum}

The domain boundedness and gradient boundedness are standard assumptions for regret minimization in OCO~\citep{book'12:Shai-OCO,book'16:Hazan-OCO}.
The smoothness assumption on the online functions is necessary for first-order algorithms to achieve the gradient-variation regret~\citep{COLT'12:VT}.
While \pref{assum:smoothness} requires the smoothness on the entire $\R^d$ space here, this assumption can be relaxed to different degrees for our two proposed methods, which will be specified later. 

\paragraph{Optimistic Online Mirror Descent.}
\OOMD applies to the optimistic online learning scenario, where in addition to the standard protocol of OCO, at round $t \in [T]$, the learner also has access to an optimistic estimation of the future loss's gradient $\nabla f_t(\x_t)$ denoted by $M_t \in \R^d$, which is called ``optimistic vector'' or simply ``optimism''.
Based on this information, \OOMD updates in the following way:
\begin{equation}
\label{eq:OOMD-exp-concave}
    \begin{split}
    \x_t =  {}&\argmin_{\x \in \X} \bbr{\eta_t\inner{M_t}{\x} + \D_{\psi_t}(\x, \xh_t)},\\
    \xh_{t+1}  = {}& \argmin_{\x \in \X} \bbr{\eta_t\inner{\nabla f_t(\x_t)}{\x} + \D_{\psi_t}(\x, \xh_t)},
    \end{split}
\end{equation}
where $\psi_t(\cdot)$ is a regularizer to be specified, $\eta_t > 0$ is a time-varying step size, $\xh_t$ is an internal decision.
This framework is highly generic and can recover many existing online learning algorithms through flexible configurations~\citep{JMLR'24:Sword++}.
A notable fact is that \OOMD can achieve an $\O(\sqrt{A_T})$ adaptive bound for convex functions under standard bounded domain and gradient assumptions, where $A_T \triangleq \sumT \norm{\nabla f_t(\x_t) - M_t}^2$~\citep{COLT'13:optimistic}.
Essentially, this design represents how to capture the intrinsic/desired adaptivity in the online learning process: when the optimistic vector $M_t$ accurately predicts the actual gradient $\nabla f_t(\x_t)$, the quantity $A_T$ becomes small, leading to improved regret.

Focusing on the gradient-variation regret and the case of known curvature information, we have the following results.
For convex functions, setting the optimism as the last-round gradient (i.e., $M_t = \nabla f_{t-1}(\x_{t-1})$) and the Euclidean regularizer $\psi_t(\x) = \frac{1}{2}\|\x\|^2_2$, \OOMD recovers the well-known Optimistic Online Gradient Descent (\OOGD)~\citep{COLT'12:VT}:
\begin{equation}
    \label{eq:OOGD}
    \x_t = \Pi_\X \mbr{\xh_t - \eta_t M_t},\quad 
    \xh_{t+1} = \Pi_\X \mbr{\xh_t - \eta_t \nabla f_t(\x_t)},
\end{equation}
where $\Pi_\X[\x] \define \argmin_{\y \in \X} \|\x - \y\|_2$ is the Euclidean projection onto the feasible domain $\X$.
Under standard assumptions, setting the step size as $\eta_t = \min\{D / \sqrt{1 + \Vb_{t-1}}, 1/(2L)\}$, where $\Vb_t \define \sum_{s=1}^t \|\nabla f_s(\x_s)- \nabla f_{s-1}(\x_{s-1})\|^2$, \OOGD enjoys an $\O(\sqrt{V_T})$ gradient-variation regret bound, which is provably optimal~\citep{COLT'12:VT}. 

For $\lambda$-strongly convex functions, using the \OOGD algorithm with $M_t = \nabla f_{t-1}(\x_{t-1})$ and $\eta_t = 2/\lambda t$, we can obtain an $\O(\frac{1}{\lambda}\log V_T)$ gradient-variation regret bound~\citep{COLT'12:VT,ICML'22:universal}.

For $\alpha$-exp-concave functions, setting the optimism as the last-round gradient (i.e., $M_t = \nabla f_{t-1}(\x_{t-1})$) and using the regularizer $\psi_t(\x) = \frac{1}{2}\|\x\|^2_{U_t}$ with $U_t = I + \frac{\alpha G^2}{2} I + \frac{\alpha}{2} \sum_{s=1}^{t-1} \nabla f_s(\x_s) \nabla f_s(\x_s)^\top$, \OOMD recovers the Optimistic Online Newton Step (\textsc{OONS}) algorithm~\citep{MLJ'14:variation-Yang}:
\begin{equation}
    \label{eq:ONNS}
    \x_t = \argmin_{\x \in \X} \big\| \x - (\xh_t - U_t^{-1} M_t)\big\|_{U_t}^2,\quad 
    \xh_{t+1} = \argmin_{\x \in \X} \big\| \x - (\xh_t - U_t^{-1} \nabla f_t(\x_t))\big\|_{U_t}^2.
\end{equation}
OONS achieves an $\O(\frac{d}{\alpha}\log V_T)$ gradient-variation regret bound~\citep{MLJ'14:variation-Yang}.

\subsection{A General Framework for Universal Online Learning}
\label{subsec:framework}
As presented in~\pref{subsec:notation-assumption-definition}, while the same algorithmic template (\OOMD) can be used to achieve gradient-variation regret bounds across different function classes, the specific configurations such as step size tuning and regularization are vastly different.
This requires the online learner to select the ``correct'' algorithm and configuration to ensure the favorable guarantees.
Universal online learning seeks to eliminate this burden by designing a single algorithm that does not require prior knowledge of the function type or curvature, yet still achieves the same regret bounds as if this information were known.

Now we will review a general framework for universal online learning and the key insight of~\citet{ICML'22:universal}, which achieves the minimax optimal regret bounds of $\O(\frac{1}{\lambda}\log T)$ for $\lambda$-strongly convex, $\O(\frac{d}{\alpha}\log T)$ for $\alpha$-exp-concave, and $\O(\sqrt{T})$ for convex functions.
We will also discuss the challenges of adapting this framework to the gradient-variation regret.

\paragraph{Online Ensemble for Universal Online Learning.}
The fundamental challenge in universal online learning lies in the \emph{uncertainty} of the function type and curvature parameters.
A common wisdom is to employ an \emph{online ensemble} with a meta-base two-layer structure, where multiple diverse base learners are deployed to explore the environment and a meta algorithm runs on top to dynamically track the best-performing base learner~\citep{NIPS'16:MetaGrad,JMLR'21:metagrad,ICML'22:universal,NeurIPS'23:universal,NeurIPS'24:OptimalGV}.
Without loss of generality, we can focus on the case where parameters $\alpha, \lambda \in [1/T, 1]$.
If $\alpha, \lambda < 1/T$, even the optimal minimax results---$\O(\frac{d}{\alpha}\log T)$ for exp-concave functions and $\O(\frac{1}{\lambda} \log T)$ for strongly convex functions~\citep{MLJ'07:Hazan-logT}---become linear in $T$, making the regret bounds vacuous.
Conversely, if $\alpha, \lambda > 1$, they can be treated as $\alpha, \lambda = 1$, which only worsens the regret by an ignorable constant factor. 

For the non-degenerated case of $\alpha, \lambda \in [1/T, 1]$, we can discretize the unknown $\alpha$ and $\lambda$ into a candidate pool $\H^{\exp}$ and $\H^{\scvx}$ using an exponential grid, defined as
\begin{equation}
    \label{eq:candidate-pool}
    \H^{\exp} = \H^{\scvx} \triangleq \left\{ \frac{1}{T}, \frac{2}{T}, \frac{2^2}{T}, \cdots, \frac{2^{n-1}}{T} \right\},
\end{equation}
where $n = \ceil{\log_2 T} + 1 = \O(\log T)$ is the number of candidates.
It can be proved that the discretized candidate pool $\H^{\exp}$ and $\H^{\scvx}$ can approximate the continuous value of $\alpha$ and $\lambda$ with only constant errors.
Based on the pool, it is natural to design three distinct groups of base learners, each tailored to handle different curvature properties:
\begin{enumerate}
    \item[\rom{1}] \textit{strongly convex} base learners $\{\B_i^\scvx\}_{i \in [N_\scvx]}$: $|\H^\scvx| = n$ in total. Each base learner $\B_i$ runs the algorithm for strongly convex functions with a guess $\lambda_i \in \H^\scvx$ of the true $\lambda$;
    
    \item[\rom{2}] \textit{exp-concave} base learners $\{\B_i^\exp\}_{i \in [N_\exp]}$: $|\H^\exp| = n$ in total. Each base learner $\B_i$ runs the algorithm for exp-concave functions with a guess $\alpha_i \in \H^\exp$ of the true $\alpha$;
    
    \item[\rom{3}] \textit{convex} base learners $\B^{\cvx}$: only $1$ base learner running an algorithm for convex functions.
\end{enumerate}
In total, there are $N \define 1 + \abs{\H^{\exp}} + \abs{\H^{\scvx}} = 2n+1 = \O(\log T)$ base learners.
The best base learner is the one with the right guess of the curvature type and the closest guess of the curvature coefficient.
For example, suppose the online functions are $\alpha$-exp-concave (while this is unknown to the online learner), then the right guessed coefficient of the best base learner (indexed by $\is$) satisfies $\alpha_\is \le \alpha \le 2 \alpha_\is$.

In addition, there is a meta algorithm running on top of those base learners.
At the $t$-th round, we denote by $\x_{t,i}$ the decision generated by the $i$-th base learner, for $i \in [N]$.
The meta learner will produce the weight vector $\p_t = (p_{t,1}, p_{t,2}, \ldots, p_{t,N})^\top \in \Delta_N$ to combine the base learners adaptively.
The final decision is formed as $\x_t = \sum_{i=1}^N p_{t,i} \x_{t,i}$. 

\paragraph{The key idea of~\citet{ICML'22:universal}.}
The online ensemble framework offers a general recipe for constructing a universal online learning algorithm, but the specific designs for base learners and, more critically, the meta algorithm remain undefined.
The key innovation of~\citet{ICML'22:universal} lies in the design of the meta algorithm.
Their approach starts from the regret decomposition of the two-layer algorithm:
\begin{equation}
  \label{eq:Zhang-decompose}
  \Reg_T = \mbr{\sumT f_t(\x_t) - \sumT f_t(\x_{t,\is})} + \mbr{\sumT f_t(\x_{t,\is}) - \min_{\x \in \X} \sumT f_t(\x)},
\end{equation}
where the \emph{meta regret} (first term) evaluates how well the algorithm tracks the best base learner, and the \emph{base regret} (second term) measures the performance of this base learner.
The best base learner is the one that runs the algorithm matching the ground-truth function type with the most accurate guess of the curvature.
\citet{ICML'22:universal} insightfully observe that ensuring the \emph{second-order regret for the meta algorithm} is pivotal for achieving universality.
Specifically, the meta algorithm should satisfy:
\begin{equation}
    \label{eq:second-order-regret}
    \sumT \inner{\nabla f_t(\x_t)}{\x_t - \x_{t,\is}}  = \sumT \inner{\p_t}{\ellb_t} -  \sumT \ell_{t,\is} = \sumT  r_{t, \is} \le  \O\left( \sqrt{\sumT r_{t, \is}^2} \right)
\end{equation}
where the feedback loss is defined as $\ell_{t,i} \triangleq \inner{\nabla f_t(\x_t)}{\x_{t,i}}$, and hence $r_{t,i} = \inner{\nabla f_t(\x_t)}{\x_t - \x_{t,i}}$ represents the instantaneous regret 
for the meta algorithm for $i \in [N]$.
This condition can be satisfied by advanced prediction-with-expert-advice (PEA) algorithms, such as \mlprod~\citep{COLT'14:ML-Prod}.

By combining this condition with the \emph{curvature-induced negative terms}, \citet{ICML'22:universal} demonstrate that the meta regret can be bounded by a \emph{constant} $\O(1)$ for exp-concave and strongly convex functions, while ensuring $\O(\sqrt{T})$ for convex functions.
Taking $\alpha$-exp-concave functions as an example, by definition, the meta regret can be bounded as 
\begin{equation}
    \label{eq:meta-regret-exp-concave}
    \meta \le  \sumT  r_{t, \is} - \frac{\alpha}{2} \sumT r_{t, \is}^2 \lesssim \sqrt{\sumT r_{t, \is}^2} - \frac{\alpha}{2} \sumT r_{t, \is}^2 \le \O(1),
\end{equation}
where the first inequality follows from the property of exp-concave functions (\pref{def:expconcave}), and the second step holds by the second-order regret~\eqref{eq:second-order-regret} of the meta algorithm.
The last inequality is by the AM-GM inequality (\pref{lem:AM-GM}). 
A similar derivation applies to strongly convex functions.
\pref{eq:second-order-regret} illuminates the importance of both the second-order regret of the meta algorithm and the curvature-induced negative terms in universal online learning.
Meanwhile, for the convex case, the second-order regret in \pref{eq:second-order-regret} still ensures an $\O(\sqrt{T})$ meta regret. 

Therefore, with the meta algorithm achieving the second-order regret bound, \citet{ICML'22:universal} further employ base learners that directly optimize the base regret: using ONS for exp-concave functions leads to an $\O(\frac{d}{\alpha}\log T)$ base regret; using OGD with an appropriate step size for $\lambda$-strongly convex functions yields an $\O(\frac{1}{\lambda}\log T)$ base regret; and using OGD with a proper step size for convex functions results in an $\O(\sqrt{T})$ base regret.
Combining these base regret bounds with the corresponding meta regret bounds (i.e., $\O(1)$ for exp-concave and strongly convex functions, and $\O(\sqrt{T})$ for convex functions) yields the desired minimax optimal guarantees for universal online learning.

\paragraph{Challenges for Gradient-Variation Regret.} 
The meta regret of~\citet{ICML'22:universal} is $\O(1)$ for strongly convex and exp-concave functions, and $\O(\sqrt{T})$ for convex functions.
As a consequence, for gradient-variation regret, one can choose base learners with gradient-variation bounds~\citep{COLT'12:VT} to achieve final regret bounds of $\O(\frac{1}{\lambda}\log V_T)$ for $\lambda$-strongly convex functions and $\O(\frac{d}{\alpha}\log V_T)$ for $\alpha$-exp-concave functions.
However, for the convex case, since the meta regret is $\O(\sqrt{T})$, it will dominate the final regret even if the base regret can be improved to $\O(\sqrt{V_T})$, resulting in an unfavorable $\O(\sqrt{T})$ overall regret for convex functions that is problem-independent.
In the following two sections, we will present novel methods building upon~\citet{ICML'22:universal} to fix the issue and achieve the desired universal gradient-variation regret across all the three function families.

\section{Method I: Online Ensemble with Injected Corrections}
\label{sec:method1-correct}
This section introduces our first method, \Correct, which achieves universal gradient-variation regret bounds for strongly convex, exp-concave, and convex functions.

\subsection{Requirement on Meta Algorithm}
\label{sec:require-meta-correction}
As discussed in \pref{subsec:framework}, the main challenge for the existing universal online learning method~\citep{ICML'22:universal} in achieving gradient-variation regret is the $\O(\sqrt{T})$ meta regret in the convex case.
Therefore, in this subsection, we first analyze the requirements for the meta algorithm and address them in the following subsections.

To achieve adaptivity, we build upon the optimistic online ensemble framework~\citep{JMLR'24:Sword++}, in which it is crucial to introduce the optimistic update in the meta algorithm.
Essentially, the meta algorithm is solving the problem of Prediction with Expert Advice (PEA) involving $N$ experts over $T$ rounds.
At each round $t \in [T]$, in addition to the feedback loss $\ellb_t \in [0,1]^N$ from the environment, optimistic online learning also receives an \emph{optimistic vector} (also called \emph{optimism}), denoted by $\m_{t+1} \in \R^N$, that encodes predictable future information.
Using this hint, the learner updates the weight vector $\p_{t+1} \in \Delta_N$ to minimize cumulative regret: $\sumT \inner{\ellb_t} {\p_t} - \min_{i \in [N]} \sumT \ell_{t,i}$.

As shown in the analysis surrounding Eqs.~\eqref{eq:second-order-regret}--\eqref{eq:meta-regret-exp-concave}, the \emph{second-order regret} of the meta algorithm is crucial for ensuring universality across different function families.
To enjoy gradient-variation adaptivity in the convex case, it is necessary to further incorporate the optimistic update in the meta algorithm.
An example of such an algorithm is the \omlprod algorithm~\citep{NIPS'16:Optimistic-ML-Prod}, which can be viewed as an optimistic variant of \mlprod~\citep{COLT'14:ML-Prod} used in~\citep{ICML'22:universal}.
While we do not present algorithmic details here, we give its \emph{optimistic second-order regret bound} in the following form:
\begin{equation}
  \label{eq:optimistic-second-order-regret}
  \sumT \inner{\ellb_t}{\p_t} -  \sumT \ell_{t,\is} \le  \O\left( \sqrt{\sumT \Big(r_{t, \is} - m_{t, \is}\Big)^2} \right),
\end{equation}
where $r_{t,i} = \inner{\ellb_t}{\p_t} - \ell_{t,i} = \inner{\nabla f_t(\x_t)}{\x_t - \x_{t,i}}$ is the instantaneous regret of the $i$-th base learner, since we set the feedback loss as $\ellb_{t,i} = \inner{\nabla f_t(\x_t)}{\x_{t,i}}$ in the meta update.
We are now in a position to design an optimistic vector $\m_{t} \in \R^N$ to attain a favorable meta regret, particularly to avoid the $\O(\sqrt{T})$ meta regret in the convex case.

Indeed, to achieve gradient-variation adaptivity with $V_T \triangleq \sumTT \sup_{\x \in \X} \|\nabla f_t(\x) - \nabla f_{t-1}(\x)\|^2$, a common approach in the literature~\citep{COLT'12:VT,JMLR'24:Sword++} is to first obtain an upper bound related to \emph{empirical gradient variations}, defined as $\Vb_T \triangleq \sumTT \|\nabla f_t(\x_t) - \nabla f_{t-1}(\x_{t-1})\|^2$, and then account for the additional positive term introduced by the smoothness of the online functions (\pref{assum:smoothness}).
In fact, we can relax the assumption to the following one, which only requires the smoothness over the feasible domain $\X$ rather than the entire space $\R^d$. 
\begin{myAssum}[Smoothness over $\X$]
  \label{assum:smoothness-X}
  For each $t \in [T]$, the online function $f_t(\cdot)$ is $L$-smooth, i.e., $\|\nabla f_t(\x) - \nabla f_t(\y)\| \le L \|\x - \y\|$ holds for any $\x, \y \in \X$.
\end{myAssum}

We then have the following decomposition of the empirical gradient variation.
\begin{myLemma}[Empirical Gradient Variation Conversion]
    \label{lemma:empirical-GV-stability}
    Under~\pref{assum:smoothness-X}, the empirical gradient variation can be upper bounded as follows:
    \begin{equation}
        \label{eq:empirical-gradient-variation}
        \begin{split}
        \Vb_T \le {} & 2\sumTT \|\nabla f_t(\x_t) - \nabla f_{t-1}(\x_t)\|^2 + 2\sumTT \|\nabla f_{t-1}(\x_t) - \nabla f_{t-1}(\x_{t-1})\|^2\\
        \le {} & 2 \sumTT \sup\nolimits_{\x \in \X} \|\nabla f_t(\x) - \nabla f_{t-1}(\x)\|^2 + 2 L^2 \sumTT \|\x_t - \x_{t-1}\|^2.
        \end{split}
    \end{equation}
\end{myLemma}

As a result, we can achieve the favorable $V_T$-type meta regret by eliminating the positive stability term of the final decisions, i.e., $\norm{\x_t - \x_{t-1}}^2$.
Following the analysis in previous work~\citep{JMLR'24:Sword++}, it can be observed that the final decision $\x_t = \sum_{i=1}^N p_{t,i} \x_{t,i}$ admits a meta-base ensemble update, leading to the following decomposition:
\begin{equation}
  \label{eq:decompose}
  \|\x_t - \x_{t-1}\|^2 \lesssim \|\p_t - \p_{t-1}\|_1^2 + \sumN p_{t,i} \|\x_{t,i} - \x_{t-1,i}\|^2,
\end{equation}
where the first part is the meta learner's stability, and the second one is a weighted version of the base learners' stability.
The proof is provided in \pref{lem:decompose}.
For now, we focus on the meta stability term, $\|\p_t - \p_{t-1}\|_1^2$, which typically requires the meta algorithm to contribute a negative regret of the same form to cancel it out, as pioneered in~\citep{NeurIPS'20:Sword} and further developed in the follow-up works~\citep{JMLR'24:Sword++,ICML'22:TVgame}.

Now the requirements for the meta algorithm are clear: it must not only ensure a regret upper bound with a negative stability term, but also provide a concrete optimism that attains the empirical gradient variation across different function types.
Specifically,
\begin{enumerate}
  \item[\rom{1}] \textbf{Regret Bound:} The meta algorithm needs to ensure the following \emph{optimistic second-order meta regret} bound with \emph{negative stability terms}:
  \begin{equation}
    \label{eq:meta-goal}
    \sumT \inner{\ellb_t}{\p_t - \e_{\is}} \le  \O\left( \sqrt{\sumT \Big(r_{t, \is} - m_{t, \is}\Big)^2} - \sumTT \|\p_t - \p_{t-1}\|_1^2 \right),  
  \end{equation}
  or other similar formulations.
  \item[\rom{2}] \textbf{Optimism Design:} The meta algorithm needs a concrete and feasible design for the optimism $\m_t \in \R^N$ that can effectively unify various function types to achieve the desired $\bar{V}_T$-type (empirical gradient variation) bound.
\end{enumerate}

Based on the above, we briefly clarify our choice of meta algorithm.
The meta algorithm should achieve an optimistic second-order regret bound while preserving the negative stability terms as shown in~\pref{eq:meta-goal}.
For this purpose, instead of using the Prod-type update like \omlprod, we focus on the mirror-descent-type update, which is well-studied and proven to enjoy negative stability terms in analysis.
To the best of our knowledge, the \emph{only} one satisfying both requirements so far is the \textsf{Multi-scale Multiplicative-weight with Correction} (\msmwc) proposed by \citet{COLT'21:impossible-tuning}, which updates as follows:
\begin{equation}
  \label{eq:msmwc}
  \p_t = \argmin_{\p \in \Delta_d}\ \bbr{\inner{\m_t}{\p} + \D_{\psi_t}(\p, \pbh_t)},\quad \pbh_{t+1} = \argmin_{\p \in \Delta_d}\ \bbr{\inner{\ellb_t + \b_t}{\p} + \D_{\psi_t}(\p, \pbh_t)},
\end{equation}
where $\psi_t(\p) = \sum_{i=1}^d \varepsilon_{t,i}^{-1} p_i \log p_i$ is the weighted negative entropy regularizer with time-coordinate-varying learning rate $\varepsilon_{t,i}$, $\m_t$ is the optimism, $\ellb_t$ is the loss vector, and $\b_t$ is a bias term, which is key to solving the ``impossible tuning'' issue~\citep{COLT'21:impossible-tuning}.

Next, we analyze the negative terms in \msmwc, which are omitted by the authors in their analysis and turn out to be crucial for our purpose.
In \pref{lem:MsMwC-refine} below, we extend Lemma 1 of \citet{COLT'21:impossible-tuning} by explicitly exhibiting the negative terms in \msmwc.
The proof is deferred to \pref{app:msmwc}.
\begin{myLemma}[\msmwc Regret]
  \label{lem:MsMwC-refine}
  If $\max_{t \in [T], i \in [d]} \{|\ell_{t,i}|, |m_{t,i}|\}\le 1$ and $\varepsilon_i \le 1/32$, then \msmwc in~\pref{eq:msmwc} with time-invariant step sizes (i.e., $\varepsilon_{t,i} = \varepsilon_i$ for any $t \in [T]$)\footnote{We only focus on the proof with fixed learning rate, since it is sufficient for our analysis.} and bias term $b_{t,i} = 16 \varepsilon_{t,i} (\ell_{t,i} - m_{t,i})^2$ enjoys:
  \begin{equation*}
      \begin{aligned}
        \sumT \inner{\ellb_t}{\p_t} - \sumT \ell_{t,\is} \le \frac{1}{\varepsilon_\is} \log \frac{1}{\ph_{1,\is}} + \sum_{i=1}^d \frac{\ph_{1,i}}{\varepsilon_i} - 8 \sumT \sum_{i=1}^d \varepsilon_i p_{t,i} (\ell_{t,i} - m_{t,i})^2 &\\
        + 16 \varepsilon_\is \sumT (\ell_{t,\is} - m_{t,\is})^2 - 4 \sumTT \|\p_t - \p_{t-1}\|_1^2 &.
      \end{aligned}
  \end{equation*}
\end{myLemma}
With proper step size tuning, \pref{lem:MsMwC-refine} derives an optimistic second-order regret bound with negative stability terms, i.e., $\Ot\Big(\sqrt{\sumT (\ell_{t,\is} - m_{t,\is})^2} - \sumTT \|\p_t - \p_{t-1}\|_1^2\Big)$, which is much closer to the desired one in~\pref{eq:meta-goal}, where $\Ot(\cdot)$-notation omits $\poly(\log T)$ factors.
Nonetheless, there are two caveats.
\begin{enumerate}
  \item[\rom{1}] The second-order bound of \msmwc is based on $(\ell_{t,\is} - m_{t,\is})^2$, which differs from $(r_{t,\is} - m_{t,\is})^2$ in \pref{eq:meta-goal} and is in fact stronger. 
  To see this, note that the OMD update of~\citet{COLT'21:impossible-tuning} in~\pref{eq:msmwc} enjoys a \emph{shifting-invariant property}, meaning that adding a constant to \emph{all} entries of the loss vector does not change the update of $\p_t$. 
  Therefore, we can define $\tilde{m}_{t,i} = \inner{\nabla f_t(\x_t)}{\x_t} - m_{t,i}$ for any $i \in [d]$ as the optimism in \msmwc. 
  With this choice, $(\ell_{t,\is} - \tilde{m}_{t,\is})^2 = (r_{t,\is} - m_{t,\is})^2$, while the update of $\p_t$ remains unchanged. 
  In other words, even when using the original optimism $m_{t,i}$, the algorithm still enjoys the second-order regret bound scaling with $(r_{t,\is} - m_{t,\is})^2$.

  \item[\rom{2}] The aforementioned bound of \msmwc omits $\poly(\log T)$ factors.
  Unfortunately, this makes it infeasible for the strongly convex or exp-concave cases, since the target rate is $\O(\log V_T)$, and an $\Ot(1) = \O(\log T)$ meta regret would ruin the desired gradient-variation adaptivity.
\end{enumerate}

In~\pref{subsec:universal-optimism}, we design optimism compatible with various function types, where the shift-invariant property plays a key role.
We then introduce a new meta algorithm in~\pref{subsec:meta-correct}, which builds on \msmwc and further consists of a two-layer structure to eliminate the additional $\O(\log T)$ factor in the meta regret.
Finally, in~\pref{subsec:method1-overall}, we combine these components to present the overall algorithm and its regret guarantees.

\subsection{Optimistic Second-Order Meta Regret: A Universal Optimism Design}
\label{subsec:universal-optimism}

In the following, we will demonstrate that designing an optimism $m_{t,i}$ to effectively unify various function types with the desired adaptivity is non-trivial, necessitating novel ideas.

\paragraph{A First Attempt on Optimism Design.}
Examining the optimistic second-order regret bound of the meta algorithm in~\pref{eq:optimistic-second-order-regret} and the analysis around~\pref{eq:meta-regret-exp-concave}, it is known that the meta regret for base learners associated with exp-concave and strongly convex functions (i.e., $i \in [N_\scvx]$ and $i \in [N_\exp]$, respectively) is bounded by a constant.
Therefore, a natural choice for the optimism $\m_t \in \R^N$ is:
\begin{equation}
  \label{eq:optimism-natural-1}
  m_{t,i} = \inner{\nabla f_{t-1}(\x_{t-1})}{\x_{t-1} - \x_{t,i}} \mbox{ for } i \in [N_{\cvx}], \mbox{ and } m_{t,i} = 0 \mbox{ for } i \in [N_{\exp}] \cup [N_{\scvx}].\footnote{In \pref{subsec:universal-optimism} and \pref{subsec:meta-correct}, we ignore the requirement of $\max_{t \in [T], i \in [d]} \{|\ell_{t,i}|, |m_{t,i}|\}\le 1$ only for clarity. When presenting the final and formal setups of the losses and optimisms of the meta algorithm in \pref{subsec:method1-overall}, we will ensure that this requirement is satisfied using normalization.}
\end{equation}
This essentially keeps the optimism for exp-concave and strongly convex base learners to zero, while approximating the instantaneous regret $r_{t,i} = \inner{\nabla f_t(\x_t)}{\x_t - \x_{t,i}}$ for the convex base learner as closely as possible using the last-round decision $\x_{t-1}$ and the latest base decisions $\{\x_{t,i}\}_{i \in [N]}$.
However, for the non-zero entries, it becomes challenging to quantify the upper bound of the term $(r_{t,i} - m_{t,i})^2 = (\inner{\nabla f_t(\x_t)}{\x_t - \x_{t,i}} - \inner{\nabla f_{t-1}(\x_{t-1})}{\x_{t-1} - \x_{t,i}})^2$ due to a mismatch in indices.

To tackle this challenge, inspired by the literature~\citep{NIPS'16:Optimistic-ML-Prod,COLT'21:impossible-tuning}, one possibility is to make the optimism slightly ``lookahead'', leveraging the shift-invariant property of \msmwc.
Specifically, we can set the optimism vector $\m_t \in \R^N$ as:
\begin{equation}
  \label{eq:optimism-natural-2}
  m_{t,i} = \inner{\nabla f_{t-1}(\x_{t-1})}{ \x_t - \x_{t,i} }, ~~\forall i\in [N].
\end{equation}
Although $\x_t$ is unknown when defining $m_{t,i}$, all entries of $\m_t$ share the same unknown value $\inner{\nabla f_{t-1}(\x_{t-1})}{ \x_t}$, making it equivalent to using $\widetilde{m}_{t,i} = \inner{\nabla f_{t-1}(\x_{t-1})}{\x_{t,i}}$ for $i \in [N]$, and the OMD-type update remains unchanged.
Under the optimism in~\pref{eq:optimism-natural-2}, the second-order optimistic quantity in the meta regret can be bounded as follows:
\begin{equation*}
  (r_{t, \is} - m_{t, \is})^2 \lesssim \left\{ 
    \begin{matrix}
      \begin{aligned}
        & \|\x_t - \x_{t,\is}\|^2, & \textnormal{(strongly convex)}\\[1mm]
        & \inner{\nabla f_{t}(\x_{t}) - \nabla f_{t-1}(\x_{t-1})}{\x_t - \x_{t,\is}}^2, & \textnormal{(exp-concave)}\\[1mm]
        & \|\nabla f_t(\x_t) - \nabla f_{t-1}(\x_{t-1})\|^2. & \textnormal{(convex)}
      \end{aligned}
    \end{matrix}
  \right. 
\end{equation*}

This works well for the convex case, since it yields a $\Vb_T$-type bound that can be converted into the desired $V_T$ bound by addressing the additional positive term later (see~\pref{lemma:empirical-GV-stability}).
It also works for the strongly convex case, where the upper bound is canceled by the curvature-induced negative term $-\norm{\x_t - \x_{t,\is}}^2$ from strong convexity (see \pref{def:sconvex}). 
However, this design~\eqref{eq:optimism-natural-2} would \emph{fail} for exp-concave base learners, because the curvature-induced negative term $- \inner{\nabla f_t(\x_t)}{\x_t - \x_{t,\is}}^2$ from exp-concavity (see~\pref{def:expconcave}) \emph{cannot} cancel the positive term $\inner{\nabla f_{t}(\x_{t}) - \nabla f_{t-1}(\x_{t-1})}{\x_t - \x_{t,\is}}^2$ in the meta regret due to a mismatch.

\paragraph{Our Unifying Optimism Design.}
To unify various types of functions, we propose a simple optimism design: set the optimism as the last-round instantaneous regret, i.e.,
\begin{equation}
  \label{eq:optimism-design}
  m_{t,i} \triangleq r_{t-1,i} = \inner{\nabla f_{t-1}(\x_{t-1})}{ \x_{t-1} - \x_{t,i} },~~\forall i\in[N].
\end{equation}
Unlike the ``lookahead'' design in~\eqref{eq:optimism-natural-2}, we simply use the last-round information.
The key idea is that, although the resulting optimistic second-order regret bound in \eqref{eq:optimistic-second-order-regret} cannot be perfectly canceled by the exp-concavity-induced negative term (i.e., $- r_{t,\is}^2$) on each round, it becomes manageable when \emph{aggregated over the entire horizon}:
\begin{equation}
  \label{eq:optimism-design-summation}
  \sumT (r_{t,\is} - m_{t,\is})^2 \overset{\eqref{eq:optimism-design}}{=} \sumT (r_{t,\is} - r_{t-1,\is})^2 \lesssim 4 \sumT r_{t,\is}^2.
\end{equation}
This holds because $r_{t,\is}$ and $r_{t-1,\is}$ differ only by one step, making the cumulative sum easier to control than the individual terms.
\pref{lem:universal-optimism} shows that this design achieves universality, in particular resolving the failure in the exp-concave case. 
\begin{myLemma}[Universality of Optimism]
  \label{lem:universal-optimism}
  Under \pref{assum:domain-boundedness}--\ref{assum:smoothness}, when setting the optimism as in \pref{eq:optimism-design}, it holds that
  \begin{align*}
    && \sumT (r_{t, \is} - m_{t, \is})^2 \lesssim 
    \left\{ 
      \begin{matrix}
        \begin{aligned}
          & \sumT \norm{\x_t - \x_{t,\is}}^2, & \text{(strongly convex)}\\[1mm]
          & \sumT \inner{\nabla f_t(\x_t)}{\x_t - \x_{t,\is}}^2, & \text{(exp-concave)}\\[1mm]
          & \sumTT \|\nabla f_t(\x_t) - \nabla f_{t-1}(\x_{t-1})\|^2. & \text{\qquad(convex)}
        \end{aligned}
      \end{matrix}
    \right.
  \end{align*}
\end{myLemma}
The proof is in \pref{app:universal-optimism}.
For strongly convex and exp-concave learners, the meta regret is effectively canceled by curvature-induced negative terms, following the same analysis as in~\citet{ICML'22:universal} (see~\pref{eq:meta-regret-exp-concave}).
For the convex case, we achieve a $\bar{V}_T$-type bound scaling with the empirical gradient variation.
As shown in \pref{lemma:empirical-GV-stability}, this can be further reduced to the desired $\O(\sqrt{V_T})$ meta regret by addressing the extra positive stability term of the final decisions, i.e., $\norm{\x_t - \x_{t-1}}^2$, which will be discussed in the next subsection.

\subsection{A New Meta Algorithm: Negative Regret Terms and Injected Corrections}
\label{subsec:meta-correct}
In this part, we present the complete meta algorithm design for this problem.
To motivate this, we recall that the meta algorithm is required to ensure an optimistic second-order regret bound with negative stability terms as in~\pref{eq:meta-goal}.
\msmwc~\citep{COLT'21:impossible-tuning} is the only known algorithm satisfying both requirements, but its regret bound contains an additional $\poly(\log T)$ factor, making it \emph{infeasible} for strongly convex or exp-concave cases.

To address this issue, we design a new meta algorithm termed \textsf{\msoms} (\textsf{MsMwC-over-MsMwC}), which itself is a \emph{two-layer} algorithm using \msmwc as both meta and base learners.
The key insight is that the $\log T$ factor arises from the multi-scale regularizer with \emph{time-varying} learning rates.
While this regularizer is crucial for resolving the ``impossible tuning'' issue addressed in their paper, it introduces undesired $\poly(\log T)$ factors due to the clipping issue, which is intolerable for our setting.
To overcome this, our proposed \msoms meta algorithm uses \msmwc with \emph{time-invariant} learning rates as a base learner, maintaining multiple base learners with different candidate learning rates and dynamically searching for a suitable one to achieve adaptivity.
This approach effectively replaces the additional $\O(\log T)$ factors with $\O(\log \sum_t (\ell_{t,\is} - m_{t,\is})^2)$, a tolerable overhead introduced by the two-layer structure, making it possible to achieve the desired $\O(\log V_T)$ regret for strongly convex and exp-concave cases.

\begin{algorithm}[!t]
    \caption{\textsf{MsMwC-over-MsMwC} (\textsf{\msoms}): Meta algorithm of \Correct}
    \label{alg:2layer-msmwc}
    \begin{algorithmic}[1]
    \REQUIRE Time horizon $T$, hyperparameter $C_0$

    \STATE \textbf{Initialize}: 
    \begin{itemize}[left=5mm, labelsep=5pt, itemsep=1pt, topsep=1pt]
      \item \MSMWCtop with learning rates $\varepsilon^\Top_{t,j} = \varepsilon^\Top_j = \frac{1}{C_0 \cdot 2^j}$ for all $t \in [T]$ and initial decision $q_{1,j}^\Top=\qh_{1,j}^\Top = \frac{(\varepsilon_j^\Top)^2}{\sumM (\varepsilon_j^\Top)^2}$ for $j \in [M]$
      \item \MSMWCmid with learning rates $\varepsilon^\Mid_{t,j,i} = 2 \varepsilon_j^\Top$ for all $t \in [T]$ and initial decision $q_{1,j,i}^\Mid=\qh_{1,j,i}^\Mid = \frac1N$ for $i \in [N]$
      \item Number of {\MSMWCmid}'s $M = \ceil{\log_2 T}$, number of base learners $N = 2 \ceil{\log_2 T} + 1$
    \end{itemize}

    \FOR{$t=1$ {\bfseries to} $T$}
      \STATE Compute the aggregated weight for the next round:
    $\p_{t} = \sumM q_{t,j}^\Top \q_{t,j}^\Mid \in \Delta_N$
        \STATE For all $j \in [M]$, the $j$-th \MSMWCmid updates to $\q_{t+1,j}^\Mid \in \Delta_N$ using $b^\Mid_{t,j,i} = 16 \varepsilon^\Mid_{t,j,i} (\ell^\Mid_{t,j,i} - m^\Mid_{t,j,i})^2$ via the following rule: 
        \begin{equation}
            \label{eq:MSMWC-mid}
            \begin{aligned}
                \qbh_{t+1,j}^\Mid = {} & \argmin_{\q\in\Delta_N} \left\{\inner{\ellb_{t,j}^\Mid + \b_{t,j}^\Mid}{\q} + \D_{\psi_{t,j}^\Mid}(\q, \qbh_{t,j}^\Mid) \right\},\\
                \q_{t+1,j}^\Mid = {} & \argmin_{\q\in\Delta_N} \left\{ \inner{\m_{t+1,j}^\Mid}{\q} + \D_{\psi_{t+1,j}^\Mid}(\q, \qbh_{t+1,j}^\Mid) \right\}.
            \end{aligned}
    \end{equation}

    \STATE \MSMWCtop updates to $\q_{t+1}^\Top \in \Delta_M$ using $b^\Top_{t,j} = 16 \varepsilon^\Top_{t,j} (\ell^\Top_{t,j} - m^\Top_{t,j})^2$ via:
    \begin{equation}
        \label{eq:MSMWC-top}
        \begin{aligned}
            \qbh_{t+1}^\Top = {} & \argmin_{\q\in\Delta_M} \left\{\inner{\ellb_t^\Top + \b_t^\Top}{\q} + \D_{\psi_t^\Top}(\q, \qbh_t^\Top) \right\},\\
            \q_{t+1}^\Top = {} & \argmin_{\q\in\Delta_M} \left\{ \inner{\m_{t+1}^\Top}{\q} + \D_{\psi_{t+1}^\Top}(\q, \qbh_{t+1}^\Top) \right\}.
        \end{aligned}
    \end{equation}

    \ENDFOR
    \end{algorithmic}
\end{algorithm}

\paragraph{Meta Algorithm.} 
\msoms updates in the following way.
The first layer runs a single \msmwc (marked as \MSMWCtop) on $\Delta_M$, whose decision is denoted by $\smash{\q_t^\Top \in \Delta_M}$.
It follows the general update rule of \eqref{eq:msmwc} with its own losses $\seq{\ellb_t^\Top}$, optimisms $\seq{\m_t^\Top}$, bias terms $\smash{\seq{\b_t^\Top}}$, and learning rates $\seq{\{\varepsilon^\Top_{t,j}\}_{j=1}^M}$.
The weighted negative entropy regularizer $\psi_{t}^\Top$ is defined as $\psi_{t}^\Top(\q) = \sum_{j=1}^M ({\varepsilon_{t,j}^{\Top}})^{-1} q_j \log q_j$.
\MSMWCtop further connects with $M$ \mbox{\msmwc}s (marked as \MSMWCmid) in the second layer.
The decision of the $j$-th \MSMWCmid is denoted by $\smash{\q_{t,j}^\Mid \in \Delta_N}$, which is updated via the same update rule as in~\eqref{eq:msmwc} with its own losses $\seq{\ellb_{t,j}^\Mid}$, optimisms $\seq{\m_{t,j}^\Mid}$, bias terms $\seq{\b_{t,j}^\Mid}$, and learning rates $\seq{\{\varepsilon^\Mid_{t,j,i}\}_{i=1}^N}$.
The weighted negative entropy regularizer $\psi_{t,j}^\Mid$ is defined as $\psi_{t,j}^\Mid(\q) = \sum_{i=1}^N ({\varepsilon_{t,j,i}^\Mid})^{-1} q_i \log q_i$.
The final output of \msoms at the $t$-th iteration is:
\begin{equation}
  \label{eq:meta-output}
  \p_t = \sumM q_{t,j}^\Top \q_{t,j}^\Mid \in \Delta_N.
\end{equation}
The details of the two-layer meta algorithm \msoms are described in~\pref{alg:2layer-msmwc}.

Building on \pref{lem:MsMwC-refine}, we provide an analysis for the two-layer meta learner \msoms, which largely follows Theorems 4 and 5 of \citet{COLT'21:impossible-tuning}, but includes additional negative stability terms.
The proof is deferred to \pref{app:two-layer-MsMwC}.
\begin{myLemma}[Two-layer \msoms]
  \label{lem:two-layer-MsMwC}
  If $|\ell^\Top_{t,j}|, |m^\Top_{t,j}|, |\ell^\Mid_{t,j,i}|, |m^\Mid_{t,j,i}| \le 1$ and $(\ell^\Top_{t,j} - m^\Top_{t,j})^2 = \inner{\ellb^\Mid_{t,j} - \m^\Mid_{t,j}}{\q^\Mid_{t,j}}^2$ for any $t \in [T]$, $j \in [M]$, and $i \in [N]$, \msoms (\pref{alg:2layer-msmwc}) satisfies
  \begin{equation*}
    \sumT \inner{\ellb_t^\Top}{\q_t^\Top - \e_\js} + \sumT \inner{\ellb_{t,\js}^\Mid}{\q_{t,\js}^\Mid - \e_\is} \le \frac{1}{\varepsilon^\Top_{\js}} \log \frac{N}{3 C_0^2 (\varepsilon^\Top_{\js})^2} + 32 \varepsilon^\Top_{\js} \Vs - \frac{C_0}{2} S_T^\Top - \frac{C_0}{4} S_{T,\js}^\Mid,
  \end{equation*}
  where the terms are defined as follows:
  \begin{itemize}
    \item $\Vs \define \sumTT (\ell^\Mid_{t,\js,\is} - m^\Mid_{t,\js,\is})^2$ is the second-order quantity;
    \item $S_T^\Top \define \sumTT \|\q_t^\Top - \q_{t-1}^\Top\|_1^2$ measures the stability of \MSMWCtop;
    \item $S_{T,j}^\Mid \define \sumTT \|\q_{t,j}^\Mid - \q_{t-1,j}^\Mid\|_1^2$ measures the stability of \MSMWCmid.
  \end{itemize}
\end{myLemma}
We highlight several important points regarding this result.
First, following \citet{COLT'21:impossible-tuning}, we choose $M = \O(\log T)$ instances of \MSMWCmid to ensure a second-order regret guarantee of $\O(\sqrt{\Vs \log \Vs})$ (Theorem 5 therein).
Second, following \pref{lem:universal-optimism}, we set $m^\Mid_{t,\js,i} = \inner{\nabla f_t(\x_t)}{\x_t} - \inner{\nabla f_{t-1}(\x_{t-1})}{\x_{t-1} - \x_{t-1,i}}$ to unify various function types.
Finally, we note that the condition $(\ell^\Top_{t,j} - m^\Top_{t,j})^2 = \inner{\ellb^\Mid_{t,j} - \m^\Mid_{t,j}}{\q^\Mid_{t,j}}^2$ in \pref{lem:two-layer-MsMwC} is only to make the lemma self-contained.
When using \pref{lem:two-layer-MsMwC} (more specifically, in \pref{thm:unigrad-correct-1grad}), we will verify that this condition is inherently satisfied by our algorithm.

\paragraph{Injected Corrections.}
Note that from \pref{lem:two-layer-MsMwC}, the two-layer \msoms already consists of negative terms of $\|\q_t^\Top - \q_{t-1}^\Top\|_1^2$ and $\|\q_{t,j}^\Mid - \q_{t-1,j}^\Mid\|_1^2$.
However, observing the decomposition in~\pref{eq:decompose}, we can see that those negative terms still \emph{mismatch} with the positive term $\|\p_t - \p_{t-1}\|_1^2$, where $p_{t,i} = \sumM q_{t,j}^\Top q_{t,j,i}^\Mid$ as shown in \pref{eq:meta-output}.
To solve this issue, we decompose the stability term $\|\p_t - \p_{t-1}\|_1^2$ into two parts (with proof in \pref{lem:decompose-simplex}):
\begin{equation}
  \label{eq:meta-stability-de}
  \|\p_t - \p_{t-1}\|_1^2 \le 2 \|\q^\Top_t - \q^\Top_{t-1}\|_1^2 + 2 \sumM q^\Top_{t,j} \|\q^\Mid_{t,j} - \q^\Mid_{t-1,j}\|_1^2.
\end{equation}
The first term on the right-hand side, $\|\q^\Top_t - \q^\Top_{t-1}\|_1^2$, can be directly canceled by the corresponding negative term in the analysis of \msoms, as shown in \pref{lem:two-layer-MsMwC}.
However, the second term, $\sumM q^\Top_{t,j} \|\q^\Mid_{t,j} - \q^\Mid_{t-1,j}\|_1^2$, cannot be canceled in the same way as the negative term $\|\q_{t,j}^\Mid - \q_{t-1,j}^\Mid\|_1^2$ in \pref{lem:two-layer-MsMwC} does not align with it.
This mismatch presents a key challenge in the analysis.
To address this issue, we draw inspiration from the work of~\citet{JMLR'24:Sword++} on the gradient-variation dynamic regret in non-stationary online learning,\footnote{This work proposes an improved dynamic regret minimization algorithm compared to its conference version~\citep{NeurIPS'20:Sword}, which introduces the correction terms to the meta-base online ensemble structure and thus improves the gradient query complexity from $\O(\log T)$ to 1 within each round.} and introduce carefully designed \emph{correction terms} to facilitate effective collaboration between layers, such that the second term in \pref{eq:meta-stability-de} can be canceled under the universal online learning scenario.
This adaptation exhibits more challenges due to the more complicated structure of the employed meta algorithm. 

To see how the correction works, consider a simpler PEA problem with regret $\sum_t \langle \ellb_t, \q_t - \e_{\js} \rangle$.
If we instead optimize the corrected loss $\ellb_t + \c_t$ and obtain a regret bound of $R_T$, i.e., $\sum_t \langle \ellb_t + \c_t, \q_t - \e_{\js} \rangle \le R_T$, then moving the correction terms to the right-hand side, the original regret is at most $\sum_t \langle \ellb_t, \q_t - \e_{\js} \rangle \le R_T - \sum_t \sum_{j} q_{t,j} c_{t,j} + \sum_t c_{t, \js}$, where the \emph{correction-induced negative term} $- \sum_t \sum_{j} q_{t,j} c_{t,j}$ can be used for cancellation.
Meanwhile, the algorithm is required to handle an extra term of $\sum_t c_{t,\js}$, which only relies on the $\js$-th dimension and is thus relatively easier to control within that dimension (or called expert).

To see how the correction scheme works in our case, we can inject the correction terms into the loss of \MSMWCtop as:
\begin{equation}
  \label{eq:top-loss-1}
  \begin{aligned}
    \ell_{t,j}^\Top = {} & \inner{\ellb_t^\Mid}{\q_{t,j}^\Mid} + \gamma^\Top \|\q_{t,j}^\Mid - \q_{t-1,j}^\Mid\|_1^2, \\
    m_{t,j}^\Top = {} & \inner{\m_t^\Mid}{\q_{t,j}^\Mid} + \gamma^\Top \|\q_{t,j}^\Mid - \q_{t-1,j}^\Mid\|_1^2,
  \end{aligned}
\end{equation}
where $\gamma^\Top>0$ is the coefficient of corrections, which will be specified later.
This correction setup is analogous to $c_{t,j} = \gamma^\Top \|\q_{t,j}^\Mid - \q_{t-1,j}^\Mid\|_1^2$ in the simplified example above.
As a result, by choosing the correction coefficient appropriately, we can ensure that the correction-induced negative term $- \sum_{j} q^\Top_{t,j} \|\q_{t,j}^\Mid - \q_{t-1,j}^\Mid\|_1^2$ can be used for cancellation.
Moreover, as shown above, the correction introduces a positive term (the cost of corrections) $c_{t, \js}$, which equals $\gamma^\Top \|\q_{t,\js}^\Mid - \q_{t-1,\js}^\Mid\|_1^2$ in our case.
Note that this cost of corrections can be perfectly handled by the intrinsic negative terms in the analysis of \msoms, as given in \pref{lem:two-layer-MsMwC}.
We further note that the construction of the loss and optimism in \eqref{eq:top-loss-1} satisfies the requirement $(\ell^\Top_{t,j} - m^\Top_{t,j})^2 = \inner{\ellb^\Mid_{t,j} - \m^\Mid_{t,j}}{\q^\Mid_{t,j}}^2$ in \pref{lem:two-layer-MsMwC}, which is crucial for the correctness of the regret analysis.

Finally, to conclude, given a PEA problem with regret $\sum_t \langle \ellb_t, \p_t - \e_\is \rangle$, by leveraging \msoms (\pref{alg:2layer-msmwc}) along with corrected losses in \pref{eq:top-loss-1}, it holds that
\begin{align*}
  & \sumT \inner{\ellb_t}{\p_t - \e_\is} = \sumT \inner{\ellb_t}{\p_t - \q^\Mid_{t,\js}} + \sumT \inner{\ellb_t}{\q^\Mid_{t,\js} - \e_\is}\\
  = {} & \sumT \inner{\ellb^\Top_t}{\q^\Top_t - \e_{\js}} + \sumT \inner{\ellb^\Mid_t}{\q^\Mid_{t,\js} - \e_\is} - \gamma^\Top \sumT \sumM q^\Top_{t,j} \|\q_{t,j}^\Mid - \q_{t-1,j}^\Mid\|_1^2 + \gamma^\Top S_{T,\js}^\Mid\\
  \le {} & \O\sbr{\sqrt{\Vs \log \Vs} - \sumTT \|\q_t^\Top - \q_{t-1}^\Top\|_1^2 - \gamma^\Top \sumT \sumM q^\Top_{t,j} \|\q_{t,j}^\Mid - \q_{t-1,j}^\Mid\|_1^2}\\
  \le {} & \O\sbr{\sqrt{\Vs \log \Vs} - \sumT \|\p_t - \p_{t-1}\|_1^2},
\end{align*}
where $\Vs \define \sumTT (\ell^\Mid_{t,\js,\is} - m^\Mid_{t,\js,\is})^2$ is the second-order quantity, the second step sets $\ellb^\Mid_t = \ellb_t$ and uses the definition of correction terms in \pref{eq:top-loss-1}, the third step leverages \pref{lem:two-layer-MsMwC}, and the last step is due to \pref{eq:meta-stability-de}. This nearly matches our goal in \pref{eq:meta-goal}, up to a logarithmic regret overhead in the second-order optimistic term.

\subsection{Overall Algorithm and Regret Guarantee}
\label{subsec:method1-overall}
The meta algorithm proposed in \pref{subsec:meta-correct} is able to achieve a second-order regret bound of $\O(\sqrt{\Vs \log \Vs})$ with negative stability terms of $\|\p_t - \p_{t-1}\|_1^2$. Recall that we still need to handle $\sumN p_{t,i} \|\x_{t,i} - \x_{t-1,i}\|^2$ as shown in the second term of \pref{eq:decompose}. To this end, we provide a further decomposition of this quantity:
\begin{equation}
  \label{eq:decompose-2}
  \begin{aligned}
    & \sumN p_{t,i} \|\x_{t,i} - \x_{t-1,i}\|^2 = \sumN \sbr{\sumM q_{t,j}^\Top q_{t,j,i}^\Mid} \|\x_{t,i} - \x_{t-1,i}\|^2\\
    = {} & \sumN \sumM q_{t,j}^\Top q_{t,j,i}^\Mid \|\x_{t,i} - \x_{t-1,i}\|^2 = \sumM q_{t,j}^\Top \sumN q_{t,j,i}^\Mid \|\x_{t,i} - \x_{t-1,i}\|^2,
  \end{aligned}
\end{equation}
where the first equality exploits the two-layer structure of the meta algorithm for computing $p_{t,i}$ in~\pref{eq:meta-output}. Thus, we obtain the following decomposition of the overall algorithmic stability $\|\x_t - \x_{t-1}\|^2$.
\begin{myLemma}
  \label{lem:decompose-three-layer}
  For any $t \ge 2$, if $\x_t = \sumN p_{t,i} \x_{t,i}\in \X$ and $\p_t = \sumM q_{t,j}^\Top \q_{t,j}^\Mid \in \Delta_N$, where $\q_t^\Top \in \Delta_M$ and $\q_{t,j}^\Mid \in \Delta_N$ for any $j \in [M]$, then it holds that
  \begin{equation*}
    \|\x_t - \x_{t-1}\|^2 \le 4D^2 \|\q_t^\Top - \q_{t-1}^\Top\|_1^2 + 4D^2 \sumM q_{t,j}^\Top \|\q_{t,j}^\Mid - \q_{t-1,j}^\Mid\|_1^2 + 2 \sumM q_{t,j}^\Top \sumN q_{t,j,i}^\Mid \|\x_{t,i} - \x_{t-1,i}\|^2.
  \end{equation*}
\end{myLemma}
The proof can be directly derived by combining \pref{lem:decompose}, \pref{eq:meta-stability-de}, and \pref{eq:decompose-2}.
It is worth noting that this decomposition differs from the one presented in our conference version~\citep{NeurIPS'23:universal}.
The specific differences and key improvements will be discussed at the end of this section.

\begin{figure*}[!t]
    \centering
    \includegraphics[clip, trim=20mm 19mm 9mm 26mm, width=0.9\textwidth]{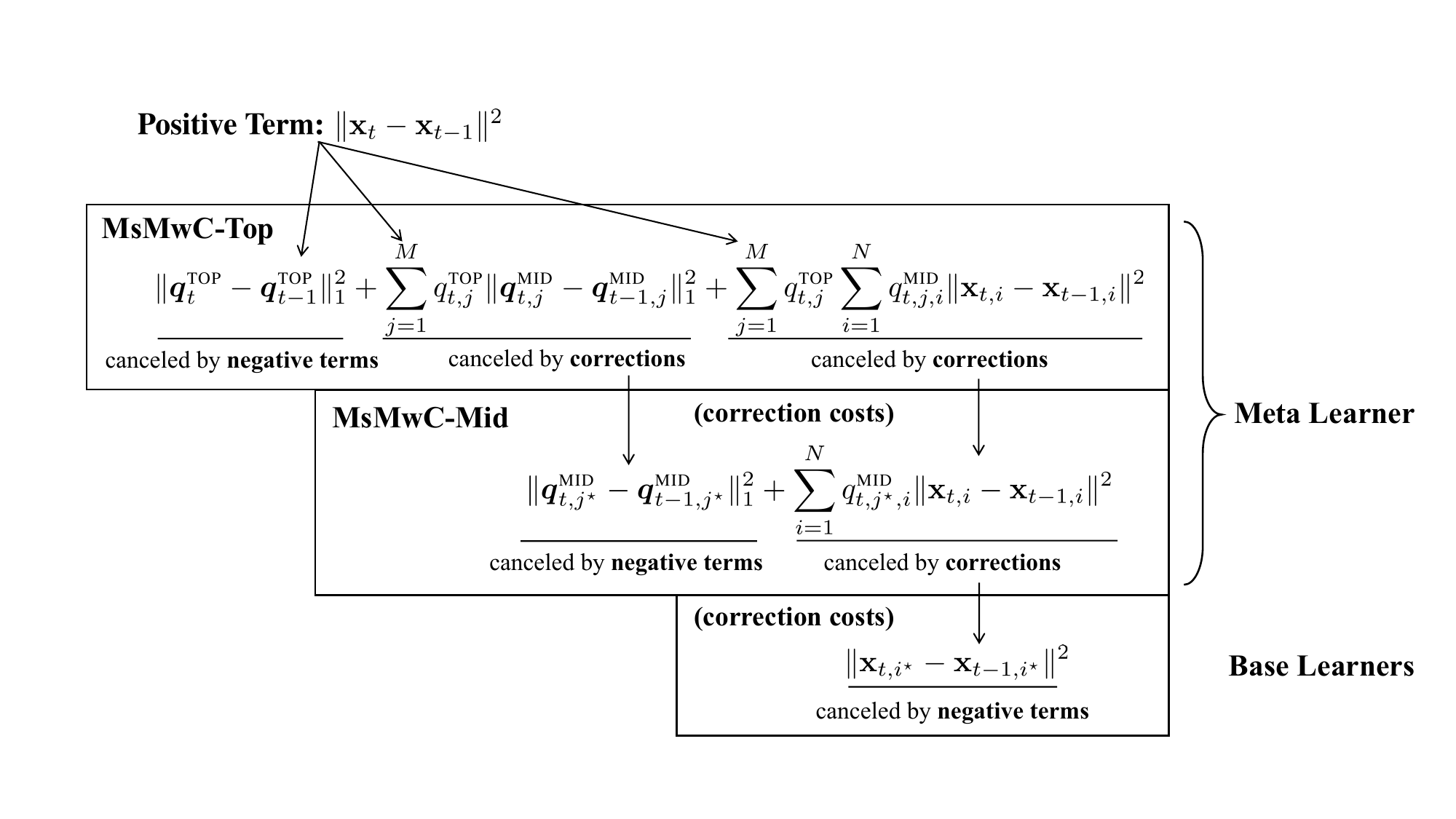}
    \caption{\small{Decomposition of the positive term $\|\x_t - \x_{t-1}\|^2$} and how it is handled by our online ensemble method via intrinsic negative stability terms and injected corrections.}
    \label{fig:correction}
    \vspace*{-4mm}
\end{figure*}

Based on~\pref{lem:decompose-three-layer}, we leverage the idea of \emph{cancellation by corrections}~\citep{JMLR'24:Sword++}, as already employed in the last subsection. Generally, to handle weighted terms like $\sum_j q_{t,j}^\Top c_{t,j}$, we inject correction terms into the loss of \MSMWCtop as $\ell_{t,j}^\Top \leftarrow \ell_{t,j}^\Top + c_{t,j}$. Here the correction term $c_{t,j}$ is designed with two components: $$c_{t,j} \approx \|\q_{t,j}^\Mid - \q_{t-1,j}^\Mid\|_1^2 + \sumN q_{t,j,i}^\Mid \|\x_{t,i} - \x_{t-1,i}\|^2,$$
allowing us to cancel both the second and third terms in \pref{lem:decompose-three-layer} simultaneously, rather than relying on a single correction as in \pref{eq:top-loss-1}. 
Thus, the final feedback loss and optimism of \MSMWCtop are set using two corrections:
\begin{equation}
  \label{eq:top-loss-2}
  \begin{gathered}
    \ell_{t,j}^\Top = \frac{1}{Z}\sbr{\inner{\ellb_t^\Mid}{\q_{t,j}^\Mid} + \gamma^\Top \|\q_{t,j}^\Mid - \q_{t-1,j}^\Mid\|_1^2 + \sumN \gamma^\Mid q_{t,j,i}^\Mid\norm{\x_{t,i} - \x_{t-1,i}}^2}\in[-1,1], \\
    m_{t,j}^\Top = \frac{1}{Z}\sbr{\inner{\m_{t}^\Mid}{\q_{t,j}^\Mid} + \gamma^\Top \|\q_{t,j}^\Mid - \q_{t-1,j}^\Mid\|_1^2 + \sumN \gamma^\Mid q_{t,j,i}^\Mid\norm{\x_{t,i} - \x_{t-1,i}}^2}\in[-1,1],
  \end{gathered}
\end{equation}
where $\gamma^\Top, \gamma^\Mid>0$ are the correction coefficients and $Z>0$ is a to-be-determined normalization factor to ensure that both the feedback loss and the optimism in \pref{eq:top-loss-2} lie in $[-1,1]$.
Note that when injecting a correction into the loss, the same correction must be applied to the optimism.
This is because the second-order bound depends on the difference between the loss and the optimism, i.e., $\ell_{t,j}^\Top - m_{t,j}^\Top$.
Therefore, this preserves the regret guarantees while incorporating corrections.

Recall that the correction scheme comes with a cost. Specifically, when we use the corrections in \pref{eq:top-loss-2}, we need to handle the extra term of $\gamma^\Mid \sumN q_{t,\js,i}^\Mid \norm{\x_{t,i} - \x_{t-1,i}}^2$ for some specific but unknown $\js \in [M]$ in the analysis. To this end, we inject the correction term into the loss of every \MSMWCmid as $\ell_{t,j,i}^\Mid \leftarrow \ell_{t,j,i}^\Mid + c_{t,i}$, where $c_{t,i} = \gamma^\Mid \|\x_{t,i} - \x_{t-1,i}\|^2$. To conclude, the final losses and optimisms of \MSMWCmid are: 
\begin{equation}
  \label{eq:mid-loss}
  \begin{aligned}
    \ell_{t,j,i}^\Mid = {} & \frac{1}{Z}\sbr{\inner{\nabla f_t(\x_t)}{\x_{t,i}} + \gamma^\Mid \|\x_{t,i} - \x_{t-1,i}\|^2}\in[-1,1],\\
    m_{t,j,i}^\Mid = {} & \frac{1}{Z}\sbr{\inner{\nabla f_{t-1}(\x_{t-1})}{\x_{t-1,i}} + \gamma^\Mid \|\x_{t,i} - \x_{t-1,i}\|^2}\in[-1,1],
  \end{aligned}
\end{equation}
where $Z$ is set as $\max\{Z_\Mid, Z_\Top\}$. Specifically, $Z_\Mid \define GD+\gamma^\Mid D^2$ serves as the normalization factor ensuring that $\ell_{t,j,i}^\Mid, m_{t,j,i}^\Mid \in [-1,1]$ for all $t \in [T], j \in [M], i \in [N]$, and $Z_\Top \define 1+\gamma^\Mid D^2+2\gamma^\Top$ is chosen to restrict the range of $\ell_{t,j}^\Top, m_{t,j}^\Top$ for all $t \in [T], j \in [M]$.
We note that setting the same normalization factor for \MSMWCtop and \MSMWCmid is necessary to directly apply \pref{lem:two-layer-MsMwC}.

We now present our final meta algorithm \msoms (\pref{alg:2layer-msmwc}) with carefully designed injected corrections into both \MSMWCtop and \MSMWCmid.
We emphasize that the cost of the corrections in \MSMWCmid is a positive term of $\gamma^\Mid \|\x_{t,\is} - \x_{t-1,\is}\|^2$ for some specific but unknown $\is \in [N]$ in the analysis.
Fortunately, this is exactly the stability term of the $\ith$ base learner, which can be perfectly handled by the intrinsic negative terms in the analysis of classic online mirror descent algorithms.
We summarize the overall correction process in \pref{fig:correction}.
Because \eqref{eq:top-loss-2} and \eqref{eq:mid-loss} both include correction steps, the mechanism exhibits a cascade in which higher-layer corrections propagate to lower layers to cancel negative terms.
Thus, we refer to this as a \emph{cascaded correction} mechanism.

\begin{algorithm}[!t]
    \caption{\Correct: Universal Gradient-variation Regret by Injected Corrections}
    \label{alg:UniGrad-Correct}
    \begin{algorithmic}[1]
    \REQUIRE Base learner configurations $\{\B_i\}_{i \in [N]}\define \{\B_i^\scvx\}_{ i\in [N_\scvx]} \cup \{\B_i^\exp\}_{i \in [N_\exp]} \cup \B^{\cvx}$, algorithm parameters $\gamma^\Mid$, $\gamma^\Top$, and $C_0$
    
    \STATE \textbf{Initialize}: $\M$~---~meta algorithm \msoms as shown in \pref{alg:2layer-msmwc} \\
    \makebox[1.8cm]{} $\{\B_i\}_{i \in [N]}$~---~base learners as specified in \pref{subsec:framework}
    
    \FOR{$t=1$ {\bfseries to} $T$}
        \STATE Submit $\x_t = \sumN p_{t,i} \x_{t,i}$, suffer $f_t(\x_t)$, and observe $\nabla f_t(\cdot)$ \label{line:observe}
        \STATE $\{\B_i\}_{i=1}^N$ update their own decisions to $\{\x_{t+1,i}\}_{i=1}^N$ using $\nabla f_t(\cdot)$ \label{line:base-update}

        \STATE Compute $\{\ellb^\Mid_{t,j}, \m^\Mid_{t+1,j}\}_{j=1}^M$ via~\pref{eq:mid-loss}, send to $\M$, get $\{\q_{t+1,j}^\Mid\}_{j=1}^M \in (\Delta_N)^M$ \label{line:meta-mid-update}

        \STATE Compute $\ellb_t^\Top, \m_{t+1}^\Top$ via~\pref{eq:top-loss-2}, send to $\M$, and obtain $\q_{t+1}^\Top \in \Delta_M$ \label{line:meta-top-update}

        \STATE Aggregate the final meta weights $\p_{t+1} \in \Delta_N$ via \pref{eq:meta-output}
        \label{line:meta-combine}
    \ENDFOR
    \end{algorithmic}
\end{algorithm}

To conclude, because we choose the two-layer \msoms as the meta learner, the overall algorithm results in a \emph{three-layer} online ensemble structure.
The overall update of \Correct is presented in \pref{alg:UniGrad-Correct}.
The base learner configurations are the same as those introduced in \pref{subsec:framework}.
Specifically, in \pref{line:observe}, the learner submits the weighted decision, suffers the corresponding loss, and receives the gradient information of the loss function.
Subsequently, the update is conducted from the bottom to the top.
Concretely, in \pref{line:base-update}, the base learners update their own decisions to $\{\x_{t+1,i}\}_{i=1}^N$ using $\nabla f_t(\cdot)$.
In \pref{line:meta-mid-update}, each \MSMWCmid computes its own losses and optimisms using \pref{eq:mid-loss} and updates its own decisions according to \pref{alg:2layer-msmwc}.
In \pref{line:meta-top-update}, \MSMWCtop computes its loss and optimism using \pref{eq:top-loss-2} and updates accordingly.
Finally, in \pref{line:meta-combine}, the meta learner aggregates the final weights $\p_{t+1} \in \Delta_N$ via \pref{eq:meta-output}.

We next present the main theoretical result. Our proposed UniGrad.Correct can achieve the following gradient-variation regret guarantees, with the proof provided in \pref{app:unigrad-correct}.
\begin{myThm}
    \label{thm:unigrad-correct}
    Under Assumptions \ref{assum:domain-boundedness}, \ref{assum:gradient-boundedness}, \ref{assum:smoothness-X}, by setting $C_0=\max\bbr{1,8D,4\gamma^{\Top},4D^2C_1}$, $\gamma^\Top=C_1$, and $\gamma^{\Mid}=2D^2C_1$, where $C_1 = 128(D^2L^2+G^2)$, \Correct (\pref{alg:UniGrad-Correct}) achieves the following universal gradient-variation regret guarantees: 
    \begin{equation}
        \sumT f_t(\x_t) - \min_{\x \in \X} \sumT f_t(\x) \le  
        \begin{cases}
            \O\left(\frac{1}{\lambda}\log V_T\right), & \text{when } \{f_t\}_{t=1}^T \text{ are $\lambda$-strongly convex}, \\[2mm]
            \O\left(\frac{d}{\alpha}\log V_T\right), & \text{when } \{f_t\}_{t=1}^T \text{ are $\alpha$-exp-concave}, \\[2mm]
            \O(\sqrt{V_T \log V_T}), & \text{when } \{f_t\}_{t=1}^T \text{ are convex},
        \end{cases}
    \end{equation}
\end{myThm}
We emphasize that the regret guarantee is achieved \emph{without} prior knowledge of the function families or curvature information.
For strongly convex functions and exp-concave functions, the regret bound matches the best-known gradient-variation regret bounds that were specifically designed with curvature prior knowledge.
For convex functions, the result exhibits a slight logarithmic gap compared to the target $\O(\sqrt{V_T})$ bound, and this gap will be addressed and closed in the next section.

It is worth mentioning that even when the curvature $\alpha$ (or $\lambda$) is smaller than $1/T$, our algorithm can still guarantee an $\O(\sqrt{V_T \log V_T})$ bound, because exp-concave and strongly convex functions are also convex and thus our convex bound still holds.

\begin{myRemark}[Technique]
    Our method follows the general \emph{optimistic online ensemble} framework proposed by \citet{JMLR'24:Sword++}, which was originally designed for dynamic regret minimization with convex functions.
    In the universal online learning with gradient-variation setting, new ingredients are required.
    Specifically, the goal of handling different function families universally necessitates a two-layer meta algorithm with both second-order regret and negative stability terms, namely, \msoms.
    Furthermore, solving this problem requires new ideas on universal optimism design, the correction scheme, and base-learner sharing, which collectively lead to our three-layer collaborative online ensemble structure, making it significantly different from the two-layer structure used in \citet{JMLR'24:Sword++}.
    \endenv
\end{myRemark}

\begin{myRemark}[Comparison to Conference Version]
  We here briefly mention the differences between \Correct and the conference version~\citep{NeurIPS'23:universal}.
  In the conference version, we have introduced the idea of employing a three-layer online ensemble (with a two-layer \msoms as the meta algorithm and cascaded correction terms to address positive stability terms) to achieve gradient-variation universal regret. 
  However, the current \Correct differs significantly from the conference version---the conference version requires maintaining $\O((\log T)^2)$ base learners that are updated simultaneously, whereas \Correct reduces this complexity to $\O(\log T)$, offering a more efficient solution. Detailed technical comparisons and the improvement will be discussed in \pref{subsec:comparison-correct-conference}.
\endenv
\end{myRemark}
\section{Method II: Online Ensemble with Extracted Bregman Divergence}
\label{sec:method2-Bregman}
This section introduces our second method, named \Bregman, which achieves optimal universal gradient-variation regret bounds.
The new method exhibits a significantly different methodology from the \Correct method presented in \pref{sec:method1-correct}.

\subsection{Online Ensemble with Extracted Bregman Divergence}
\label{subsec:Bregman-algorithm-analysis}

This subsection leverages the Bregman divergence term extracted from the linearization of convex online functions, along with a key property of smooth functions that connects gradient variation to Bregman divergence.

For clarity, we illustrate our idea from the ground level. To obtain the gradient variation $V_T$ defined in~\eqref{eq:VT}, we first need to attain its \emph{empirical} version $\Vb_T \define \sum_{t \le T} \|\nabla f_t(\x_t) - \nabla f_{t-1}(\x_{t-1})\|^2$.
Previous studies decompose this term as shown in \pref{eq:empirical-gradient-variation}, which requires controlling the algorithmic stability $\|\x_t - \x_{t-1}\|^2$.
Consequently, since each decision is a weighted combination of base learners' decisions (i.e., $\x_t = \sum_{i \le N} p_{t,i} \x_{t,i}$), the stability is difficult to control. This requires a very powerful meta algorithm with both second-order regret guarantees and negative stability terms.
This is why \Correct requires a three-layer online ensemble structure and achieves a sub-optimal $\O(\sqrt{V_T \log V_T})$ regret bound for convex functions.

\paragraph{Empirical Gradient Variation Decomposition.}
To address the above problem, we propose a novel decomposition of the empirical gradient variation that avoids the stability term at the meta algorithm level.
Specifically, we decompose the empirical gradient variation into three parts:
\begin{equation}
    \label{eq:emp-VT-de-breg-1}
    \Vb_T \les \|\nabla f_t(\x_t) - \nabla f_t(\x_{t,\is})\|^2 + \|\nabla f_t(\x_{t,\is}) - \nabla f_{t-1}(\x_{t-1,\is})\|^2 + \|\nabla f_{t-1}(\x_{t-1,\is}) - \nabla f_{t-1}(\x_{t-1})\|^2.
\end{equation}
\begin{itemize}[leftmargin=*]
    \item The middle term on the right-hand side measures the empirical gradient variation of a base learner, which can be controlled within the base learner via its intrinsic negative stability terms using optimistic OMD algorithms.
    \item The first and last terms capture the gradient differences between the meta decision $\x_t$ (or $\x_{t-1}$) and the best base learner's decision $\x_{t,\is}$ (or $\x_{t-1,\is}$).
    We show that they can be upper-bounded by \emph{Bregman divergences} using a useful property of smoothness~\citep{book'18:Nesterov-Convex} (restated below), which yields $\|\nabla f_t(\x_t) - \nabla f_t(\x_{t,\is})\|^2 \leq 2L \cdot \D_{f_t}(\x_{t,\is}, \x_t)$.
\end{itemize}
\begin{myProp}[Theorem 2.1.5 of \citet{book'18:Nesterov-Convex}]
    \label{prop:smoothness}
    A function $f(\cdot)$ is $L$-smooth over $\R^d$ if and only if 
    \begin{equation}
        \label{eq:smoothness-property}
        \|\nabla f(\x) - \nabla f(\y)\|^2 \le 2 L \cdot \D_f(\y,\x), \quad \text{for any } \x, \y \in \R^d.
    \end{equation}
\end{myProp}
Compared with the commonly used inequality $\|\nabla f(\x) - \nabla f(\y)\| \le L \|\x - \y\|$, \pref{prop:smoothness} provides a tighter bound for the \emph{squared} gradient changes. Since $\|\nabla f(\x) - \nabla f(\y)\|^2 \le 2 L \D_f(\y,\x) \le L^2 \|\x - \y\|^2$, where the second inequality holds because $\D_f(\y,\x) \le \frac{L}{2} \|\x - \y\|^2$ for any $\x, \y \in \R^d$~{\citep[Theorem~2.1.5]{book'18:Nesterov-Convex}}, this yields a simpler subsequent analysis.   

Combining the above decomposition with the smoothness property yields the following result, which upper-bounds the empirical gradient variation $\Vb_T$ by the Bregman divergence terms and the gradient variation $V_T$.
The proof is deferred to \pref{app:emp-VT-de-breg}.
\begin{myLemma}[Empirical Gradient Variation Conversion - II]
    \label{lem:emp-VT-de-breg}
    Under \pref{assum:smoothness}, the empirical gradient variation $\Vb_T \define \sumTT \|\nabla f_t(\x_t) - \nabla f_{t-1}(\x_{t-1})\|^2$ can be upper bounded as
    \begin{equation}
        \label{eq:emp-VT-de-breg}
        \Vb_T \les 2L \sumT \D_{f_t}(\x_{t,\is}, \x_t) + L^2 \sumTT \|\x_{t,\is} - \x_{t-1,\is}\|^2 + V_T, 
    \end{equation}
    where $\is \in [N]$ can be any base learner index.
\end{myLemma}

\paragraph{Negative Bregman Divergence.}
In \pref{lem:emp-VT-de-breg}, the only term that remains to be handled is the Bregman divergence $\D_{f_t}(\x_{t,\is}, \x_t)$.
Fortunately, this term is canceled by a negative term arising from the linearization of convex functions.
Specifically, the meta regret can be transformed into
\begin{equation}
    \label{eq:negative-Bregman-1}
    \meta = \sumT f_t(\x_t) - \sumT f_t(\x_{t,\is}) = \sumT \inner{\nabla f_t(\x_t)}{\x_t - \x_{t,\is}} - \sumT \D_{f_t}(\x_{t,\is}, \x_t),
\end{equation}
where $\is$ indicates the best base learner.
The last term is a \emph{negative term from linearization}, which represents the compensation when treating a convex function as linear.
Previous studies on gradient-variation regret omit this term, while we show below that this negative term helps achieve a simpler analysis of the empirical gradient variation.

The Bregman divergence negative term in \pref{eq:negative-Bregman-1} cancels the positive term in \pref{eq:emp-VT-de-breg}, and the stability term $\|\x_{t,\is} - \x_{t-1,\is}\|^2$ can be controlled within the base learner via its intrinsic negative stability terms using optimistic algorithms.
Thus, only the gradient variation quantity $V_T$ remains as the desired regret bound.

We emphasize that the Bregman divergence negative term arises from the linearization of convex functions and is thus \emph{algorithm-independent}.
Therefore, we avoid controlling the algorithmic stability for gradient variation regret, in contrast to previous works~\citep{COLT'12:VT,NeurIPS'23:universal}.
To the best of our knowledge, this is the \emph{first} alternative analysis of gradient-variation regret since its introduction.

The inspiration for our negative term from linearization comes from recent advances in stochastic smooth optimization~\citep{ICML'20:simple-acceleration}.
While their work focuses on achieving the $\O(1/T^2)$ function-value convergence rate of Nesterov's accelerated gradient method~\citep{book'18:Nesterov-Convex}, our approach addresses the gradient-variation regret in the universal online (adversarial) convex optimization setting.

\paragraph{About Smoothness Requirement.}
Finally, we discuss the smoothness requirement.
Notice that \pref{prop:smoothness} requires smoothness over the whole $\R^d$, which is a much too strong assumption.
We emphasize that to leverage the useful smoothness property in \pref{prop:smoothness}, we only need $\|\nabla f(\x) - \nabla f(\y)\|^2 \le 2 L \D_f(\y,\x)$ on the feasible domain $\X$.
In this work, we show that this requirement can be satisfied with smoothness over a slightly larger domain than $\X$, formally, the minimal \pref{assum:smoothness-X+} below.
Readers can refer to \pref{lem:smooth-relax} in \pref{app:smooth-relax} for the formal statement and the proof.
\begin{myAssum}[Smoothness over $\X_+$]
    \label{assum:smoothness-X+}
    Under the condition of $\|\nabla f_t(\x)\| \le G$ for any $\x \in \X$ and $t \in [T]$, all online functions are $L$-smooth: $\|\nabla f_t(\x) - \nabla f_t(\y)\| \le L \|\x - \y\|$ over $\X_+$, where $\X_+ \define \{\x + \z \given \x \in \X, \z \in G\Bb/L\}$ and $\Bb \define \{\x \given \|\x\| \le 1\}$ is a unit ball.
\end{myAssum}
We note that \pref{assum:smoothness-X+} is reasonable, as many commonly used functions are globally smooth, such as the squared loss $f(\x) = \frac{1}{2}\|\x\|^2$ and the logistic loss $f(\x) = \log(1 + \exp(- \boldsymbol{\theta}^\top \x))$ for some $\boldsymbol{\theta} \in \R^d$.
While the feasible domains of these functions are constrained to $\X$ (satisfying \pref{assum:smoothness-X}), they remain smooth over larger sets that extend beyond $\X$, thereby satisfying \pref{assum:smoothness-X+}.

Building on the fact that \pref{assum:smoothness-X+} is a sufficient condition for \pref{eq:smoothness-property} on $\X$, it can be directly obtained that \pref{lem:emp-VT-de-breg} also holds under the relaxed \pref{assum:smoothness-X+}.

\subsection{Overall Algorithm and Regret Guarantee}
\label{subsec:method2-overall}
In this section, we present the overall algorithm and its regret guarantees.
Our algorithm adopts a two-layer online ensemble structure.
Leveraging the new decomposition in \pref{lem:emp-VT-de-breg}, the meta algorithm does \emph{not} require explicit control of its algorithmic stability; it only needs an optimistic second-order regret guarantee.
Consequently, we employ \omlprod~\citep{NIPS'16:Optimistic-ML-Prod} as the meta algorithm.
The base learners remain the same as those introduced in \pref{subsec:framework}.

\paragraph{Meta Algorithm.}
\Bregman employs \omlprod~\citep{NIPS'16:Optimistic-ML-Prod} as the meta algorithm to dynamically combine the base learners, in contrast to the complex \msoms used in \Correct.
\omlprod is a simple algorithm that enjoys a second-order optimistic regret bound and admits closed-form weight updates.
Specifically, the weight vector $\p_{t+1} \in \Delta_N$ is updated as follows: $\forall i \in [N]$,
\begin{equation}
    \label{eq:omlprod}
    \begin{gathered}
        \forall t \ge 1, W_{t,i} = \sbr{W_{t-1,i} \cdot \exp\sbr{\epsilon_{t-1,i} r_{t,i} - \epsilon_{t-1,i}^2 (r_{t,i} - m_{t,i})}}^{\frac{\epsilon_{t,i}}{\epsilon_{t-1,i}}}, \text{ and } W_{0,i} = \frac{1}{N},\\
        p_{t+1,i} \propto \epsilon_{t,i} \cdot \exp(\epsilon_{t,i} m_{t+1,i}) \cdot W_{t,i},
    \end{gathered}
\end{equation}
where $W_{t,i}$ and $\epsilon_{t,i}$ denote the potential variable and learning rate for the $i$-th base learner, respectively.
The feedback loss vector $\ellb_t \in \R^N$ is configured as 
$\ell_{t,i} \define \frac{1}{2GD}\inner{\nabla f_t(\x_t)}{\x_{t,i}} + \frac12 \in [0,1]$, where $r_{t,i} = \inner{\p_t}{\ellb_t} - \ell_{t,i}$ measures the instantaneous regret.
The optimistic vector $\m_t \in \R^{N}$ is designed as\footnote{Though $\x_t$ is unknown when using $m_{t,i}$, we only need the scalar value of $\inner{\nabla f_{t-1}(\x_{t-1})}{\x_t}$, which is bounded and can be efficiently solved via a one-dimensional fixed-point problem: $\inner{\nabla f_{t-1}(\x_{t-1})}{\x_t(z)} = z$.
$\x_t$ is a function of $z$ because $\x_t$ relies on $p_{t,i}$, $p_{t,i}$ relies on $m_{t,i}$ and $m_{t,i}$ relies on $z$.
Interested readers can refer to Section 3.3 of \citet{NIPS'16:Optimistic-ML-Prod} for more details.}
\begin{equation}
    \label{eq:Bregman-optimism}
    m_{t,i} = \tfrac{1}{2GD}\inner{\nabla f_{t-1}(\x_{t-1})}{\x_t - \x_{t,i}} \mbox{ for } i \in [N_{\cvx}], \mbox{ and } m_{t,i} = 0 \mbox{ for } i \in [N_{\exp}] \cup [N_{\scvx}].
\end{equation}
The learning rate $\epsilon_{t,i}$ is set as 
\begin{equation}
    \label{eq:omlprod-lr}
    \epsilon_{t,i} = \min \bbr{\frac{1}{8},\ \sqrt{\frac{\log N}{\sum_{s=1}^t (r_{s,i} - m_{s,i})^2}}}.
\end{equation}
\pref{alg:UniGrad-Bregman} describes the overall update procedures of \Bregman.

\begin{algorithm}[!t]
    \caption{\Bregman: Universal GV Regret by Extracted Bregman Divergence}
    \label{alg:UniGrad-Bregman}
    \begin{algorithmic}[1]
    \REQUIRE Base learner configurations $\{\B_i\}_{i \in [N]}\define \{\B_i^\scvx\}_{ i\in [N_\scvx]} \cup \{\B_i^\exp\}_{i \in [N_\exp]} \cup \B^{\cvx}$
    \STATE \textbf{Initialize}: $\M$~---~meta learner \omlprod with $W_{0,i} = \frac{1}{N}$ for all $i \in [N]$ \\
    \makebox[1.8cm]{} $\{\B_i\}_{i \in [N]}$~---~base learners as specified in \pref{subsec:framework}
    \FOR{$t=1$ {\bfseries to} $T$}
        \STATE Submit $\x_t = \sumN p_{t,i} \x_{t,i}$, suffer $f_t(\x_t)$, and observe $\nabla f_t(\cdot)$
        \STATE $\{\B_i\}_{i=1}^N$ update their own decisions to $\{\x_{t+1,i}\}_{i=1}^N$ using  $\nabla f_t(\cdot)$
        \STATE Calculate $\m_{t+1}$~\eqref{eq:Bregman-optimism} and $\r_t$ using $\{\x_{t,i}\}_{i=1}^N$, $\x_t$, $\nabla f_t(\x_t)$, and $\{\x_{t+1,i}\}_{i=1}^N$, send them to $\M$, and obtain $\p_{t+1} \in \Delta_N$ via \pref{eq:omlprod} and \pref{eq:omlprod-lr}
    \ENDFOR
    \end{algorithmic}
\end{algorithm}

We now present the regret guarantees of \Bregman, demonstrating that the algorithm achieves optimal gradient-variation regret without requiring prior knowledge of curvature information.
The proof is provided in \pref{app:main-Bregman}.
\begin{myThm}
    \label{thm:unigrad-bregman}
    Under Assumptions \ref{assum:domain-boundedness}, \ref{assum:gradient-boundedness}, \ref{assum:smoothness-X+}, by setting the learning rate of meta algorithm as \pref{eq:omlprod-lr}, \Bregman (\pref{alg:UniGrad-Bregman}) achieves the following universal gradient-variation regret guarantees: 
    \begin{equation*}
        \sumT f_t(\x_t) - \min_{\x \in \X} \sumT f_t(\x) \le  
        \begin{cases}
            \O\left(\frac{1}{\lambda}\log V_T\right), & \text{when } \{f_t\}_{t=1}^T \text{ are $\lambda$-strongly convex}, \\[2mm]
            \O\left(\frac{d}{\alpha}\log V_T\right), & \text{when } \{f_t\}_{t=1}^T \text{ are $\alpha$-exp-concave}, \\[2mm]
            \O(\sqrt{V_T}), & \text{when } \{f_t\}_{t=1}^T \text{ are convex}.
        \end{cases}
    \end{equation*}
\end{myThm}
Notably, \Bregman improves the convex regret bound from $\O(\sqrt{V_T \log V_T})$ to the optimal $\O(\sqrt{V_T})$~\citep{MLJ'14:variation-Yang}, at the cost of requiring smoothness over a slightly larger domain. 
Even when the curvature $\alpha$ (or $\lambda$) is smaller than $1/T$, our algorithm still guarantees an $\O(\sqrt{V_T})$ bound, since exp-concave and strongly convex functions are also convex, and therefore our convex bound remains valid.

Finally, we compare \Correct and \Bregman in terms of their implications and applications.
\Correct is applicable to multi-player game settings~\citep{NIPS'15:fast-rate-game,ICML'22:TVgame} due to its stability property, which has been shown to be essential for achieving faster convergence in games~\citep{NIPS'15:fast-rate-game}. 
\Bregman, by contrast, can be extended to the anytime setting, thanks to the simplicity and flexibility of its meta algorithm.
Details of these extensions and applications are provided in \pref{sec:applications}. A more detailed comparison between \Correct and \Bregman will be discussed in \pref{subsec:comparison-correct-Bregman}.
\section{One Gradient Query per Round}
\label{sec:one-gradient}
Though achieving favorable regret guarantees in \pref{sec:method1-correct} and \pref{sec:method2-Bregman}, one caveat is that both \Correct and \Bregman require $\O(\log T)$ gradient queries per round because they need $\nabla f_t(\x_{t,i})$ for all $i \in [N]$, making them computationally inefficient when the gradient evaluation is costly, e.g., in nuclear norm optimization~\citep{ICML'09:Ye} and mini-batch optimization~\citep{KDD'14:Li}.
The same concern also appears in the approach of \citet{ICML'22:universal}, who provided small-loss and worst-case regret guarantees for universal online learning.
By contrast, traditional algorithms such as OGD typically work under the \emph{one-gradient} feedback setup, namely, they only require one gradient $\nabla f_t(\x_t)$ for the update.
In light of this, it is natural to ask \emph{whether we can design a universal algorithm that can maintain the desired regret guarantees while using only one gradient query per round.}

In this section, we provide an affirmative answer for this question via a dedicated surrogate optimization technique that implements base algorithms on carefully designed surrogate functions.
We first present a general framework for the one-gradient algorithm in \pref{subsec:general-one-gradient}, and then instantiate it for \Correct and \Bregman algorithms, in \pref{subsec:correct-one-gradient} and \pref{subsec:bregman-one-gradient}, respectively.

\subsection{A General Idea of Surrogate Optimization}
\label{subsec:general-one-gradient}
In the following, we take $\lambda$-strongly convex functions as an example.
To address this challenge, inspired by \citet{IJCAI'18:Wang}, we propose an effective regret decomposition as follows.
Specifically, let $\xs \in \argmin_{\x \in \X} \sum_{t \in [T]} f_t(\x)$ denote the optimal solution and $\is$ be the index of the best base learner with $\lambda_\is \le \lambda \le 2\lambda_\is$.
We have
\begin{align*}
    & \Reg_T \le \sumT \inner{\nabla f_t(\x_t)}{\x_t - \xs} - \frac{\lambda}{4}\sumT \|\x_t - \xs\|^2 - \frac{1}{2}\sumT \D_{f_t}(\xs, \x_t)\\ 
    \le {} & \sumT \inner{\nabla f_t(\x_t)}{\x_t - \xs} - \frac{\lambda_\is}{4}\sumT \|\x_t - \xs\|^2 - \frac{1}{2}\sumT \D_{f_t}(\xs, \x_t)\\ 
    \le {} & \mbr{\sumT \inner{\nabla f_t(\x_t)}{\x_t - \x_{t,\is}} - \frac{\lambda_\is}{4} \sumT \|\x_t - \x_{t,\is}\|^2} - \frac{1}{2} \sumT \D_{f_t}(\xs, \x_t)\\
    & + \mbr{\sumT \sbr{\inner{\nabla f_t(\x_t)}{\x_{t,\is}} + \frac{\lambda_\is}{4} \|\x_{t,\is} - \x_t\|^2} - \sumT \sbr{\inner{\nabla f_t(\x_t)}{\xs} + \frac{\lambda_\is}{4} \|\xs - \x_t\|^2}}\\
    = {} & \underbrace{\mbr{\sumT r_{t,\is} - \frac{\lambda_\is}{4} \sumT \|\x_t - \x_{t,\is}\|^2}}_{\meta} + \underbrace{\mbr{\sumT \hsc_{t,\is}(\x_{t,\is}) - \sumT \hsc_{t,\is}(\xs)}}_{\base} - \frac{1}{2} \sumT \D_{f_t}(\xs, \x_t).
\end{align*}
The first step follows from $\D_{f_t}(\xs, \x_t) \ge \frac{\lambda}{2} \|\x_t - \xs\|^2$ since $f_t(\cdot)$ is $\lambda$-strongly convex, and it preserves the Bregman divergence linearization-induced negative term for \Bregman.
The second step uses the definition of the best base learner (indexed by $\is$): $\lambda_\is \le \lambda \le 2 \lambda_\is$.
The third step inserts an intermediate term $\frac{\lambda_\is}{4} \sumT \|\x_t - \x_{t,\is}\|^2$ and reorganizes the equation.
The last step rewrites the equation by defining the following surrogate loss:
\begin{equation}
    \label{eq:surrogate-scvx}
    \hsc_{t,i}(\x) \define \inner{\nabla f_t(\x_t)}{\x} + \frac{\lambda_i}{4}\|\x - \x_t\|^2.
\end{equation}
Note that the meta regret here is nearly identical to that in the multi-gradient setup, which can be optimized via algorithms with second-order regret guarantees, as demonstrated in \pref{sec:method1-correct} and \pref{sec:method2-Bregman}, making it as flexible as \citet{ICML'22:universal}.
Furthermore, the surrogate loss function in \pref{eq:surrogate-scvx} requires \emph{only one} gradient $\nabla f_t(\x_t)$, making it as efficient as \citet{NIPS'16:MetaGrad}. 

Similarly, for $\alpha$-exp-concave and convex functions, we define the surrogates
\begin{equation}
    \label{eq:surrogate-exp-cvx}
    \hexp_{t,i}(\x) \define \inner{\nabla f_t(\x_t)}{\x} + \frac{\alpha_i}{4} \inner{\nabla f_t(\x_t)}{\x - \x_t}^2,\quad h^{\text{c}}_t(\x) \define \inner{\nabla f_t(\x_t)}{\x}.
\end{equation}
In \pref{subsec:correct-one-gradient} and \pref{subsec:bregman-one-gradient}, we propose one-gradient improvements of \Correct and \Bregman, respectively, and demonstrate that additional novel analyses are required to achieve gradient-variation guarantees for the base regret, \mbox{\emph{defined on surrogates}}.

As a byproduct, we show that this regret decomposition approach can be used to recover the minimax optimal worst-case universal guarantees using one gradient with a simple approach and analysis, with proof provided in \pref{app:recover-MetaGrad}.
\begin{myProp}
    \label{prop:recover-MetaGrad}
    Under Assumptions~\ref{assum:domain-boundedness} and \ref{assum:gradient-boundedness}, using the surrogate loss functions as defined in \pref{eq:surrogate-scvx} and \pref{eq:surrogate-exp-cvx}, and running \mlprod as the meta algorithm (by choosing $\m_t = \mathbf{0}$ in \pref{eq:omlprod}) guarantees $\O(\frac{1}{\lambda}\log T)$, $\O(\frac{d}{\alpha} \log T)$ and $\O(\sqrt{T})$ regret bounds for strongly convex, exp-concave and convex functions, using one gradient per round.
\end{myProp}
\begin{myRemark}
    This result demonstrates that our surrogate optimization framework not only enables one-gradient universal algorithms but also provides a unified approach to recover classical minimax optimal bounds.
    The simplicity of the analysis compared to existing approaches highlights the power of the surrogate loss technique.
    \endenv
\end{myRemark}

\subsection{\textsf{UniGrad++.Correct}: One-Gradient Improvement of \textsf{UniGrad.Correct}}
\label{subsec:correct-one-gradient}
To begin with, we recall the base regret definition with surrogates (we still take $\lambda$-strongly convex functions as an example):
\begin{equation*}
    \base = \mbr{\sumT \hsc_{t,\is}(\x_{t,\is}) - \sumT \hsc_{t,\is}(\xs)}, \text{ where } \hsc_{t,i}(\x) \define \inner{\nabla f_t(\x_t)}{\x} + \frac{\lambda_i}{4}\|\x - \x_t\|^2.
\end{equation*}
For strongly convex gradient-variation regret minimization, the best known algorithm runs an initialization of the \OOMD:
\begin{equation*}
    \x_{t,i} = \Pi_\X \mbr{\xh_{t,i} - \eta_{t,i} M_{t,i}},\qquad \xh_{t+1,i} = \Pi_\X \mbr{\xh_{t,i} - \eta_{t,i} \nabla \hsc_{t,i}(\x_{t,i})},
\end{equation*}
where $\eta_{t,i}$ represents the step size.
With appropriately chosen step sizes, the base learner achieves an optimistic bound of $\O(\log (\sum_{t \le T} \|\nabla \hsc_{t,i}(\x_{t,i}) - M_{t,i}\|^2))$ (e.g., Theorem 15 of \citet{COLT'12:VT}).
Therefore, choosing the optimism as $M_{t,i} = \nabla \hsc_{t-1,i}(\x_{t-1,i})$ leads to an empirical gradient-variation bound $\O(\frac{1}{\lambda}\log \Vb_{T,i}^{\scvx})$ \emph{defined on surrogates}, where $\Vb_{T,i}^{\scvx} \define \sumTT \|\nabla \hsc_{t,i}(\x_t) - \nabla \hsc_{t-1,i}(\x_{t-1})\|^2$.
To handle this term, we decompose it as
\begin{align*}
    \Vb_{T,i}^{\scvx} = {} & \sumTT \norm{\nabla f_t(\x_t) + \frac{\lambda_i}{2} (\x_{t,i} - \x_t) - \nabla f_{t-1}(\x_{t-1}) - \frac{\lambda_i}{2} (\x_{t-1,i} - \x_{t-1})}^2 \\
    \les {} & \sumTT \norm{\nabla f_t(\x_t) - \nabla f_{t-1}(\x_{t-1})}^2 + \lambda_i^2 \sumTT \norm{\x_t - \x_{t-1}}^2 + \lambda_i^2 \sumTT \norm{\x_{t,i} - \x_{t-1,i}}^2,
\end{align*}
where $\nabla \hsc_{t,i}(\x) = \nabla f_t(\x_t) + \frac{\lambda_i}{2} (\x - \x_t)$ because of the definition of the strongly convex surrogate function~\eqref{eq:surrogate-scvx}.
Notice that the above decomposition not only contains the desired gradient variation, but also includes the positive stability terms of base decisions $\|\x_{t,i} - \x_{t-1,i}\|^2$ and final decisions $\|\x_t - \x_{t-1}\|^2$.
Fortunately, same as \Correct in \pref{sec:method1-correct}, these stability terms can be effectively addressed through our cancellation mechanism within the online ensemble, by adjusting the correction coefficients accordingly.

\begin{algorithm}[!t]
    \caption{\Correctpp: One-Gradient Improvement of \Correct}
    \label{alg:UniGrad-Correct-1grad}
    \begin{algorithmic}[1]
    \REQUIRE Base learner configurations $\{\B_i\}_{i \in [N]}\define \{\B_i^\scvx\}_{ i\in [N_\scvx]} \cup \{\B_i^\exp\}_{i \in [N_\exp]} \cup \B^{\cvx}$, algorithm parameters $\gamma^\Mid$, $\gamma^\Top$, and $C_0$

    \STATE \textbf{Initialize}: $\M$~---~meta learner \msoms as shown in \pref{alg:2layer-msmwc} \\
    \makebox[1.8cm]{} $\{\B_i\}_{i \in [N]}$~---~base learners as specified in \pref{subsec:framework}\\
    \makebox[1.8cm]{} $\{\hsc_{t,i}(\cdot)\}_{i \in [N]}$~---~strongly convex surrogate losses as defined in \pref{eq:surrogate-scvx}

    \FOR{$t=1$ {\bfseries to} $T$}
        \STATE Submit $\x_t = \sumN p_{t,i} \x_{t,i}$, suffer $f_t(\x_t)$, and observe $\nabla f_t(\cdot)$
        \STATE $\{\B_i^\scvx\}_{i \in [N_\scvx]}$, $\{\B_i^\exp\}_{i \in [N_\exp]}$, and $\B^{\cvx}$ update their own decisions to $\{\x_{t+1,i}\}_{i=1}^N$ using surrogate losses of $\{\hsc_{t,i}(\cdot)\}_{\lambda_i \in \H^{\scvx}}$~\eqref{eq:surrogate-scvx}, $\{\hexp_{t,i}(\cdot)\}_{\alpha_i \in \H^{\exp}}$~\eqref{eq:surrogate-exp-cvx}, and $h_t^{\cvx}(\cdot)$~\eqref{eq:surrogate-exp-cvx} \label{line:correct-1grad}

        \STATE Compute $\{\ellb^\Mid_{t,j}, \m^\Mid_{t+1,j}\}_{j=1}^M$ via~\pref{eq:mid-loss}, send to $\M$, get $\{\q_{t+1,j}^\Mid\}_{j=1}^M \in (\Delta_N)^M$ \label{line:meta-mid-update-1grad}

        \STATE Compute $\ellb_t^\Top, \m_{t+1}^\Top$ via~\pref{eq:top-loss-2}, send to $\M$, and obtain $\q_{t+1}^\Top \in \Delta_M$ \label{line:meta-top-update-1grad}

        \STATE Aggregate the final meta weights $\p_{t+1} \in \Delta_N$ via \pref{eq:meta-output}
        \label{line:meta-combine-1grad}
    \ENDFOR
    \end{algorithmic}
\end{algorithm}

This efficient version is concluded in \pref{alg:UniGrad-Correct-1grad}, where the only algorithmic modification from \pref{alg:UniGrad-Correct} is that in \pref{line:correct-1grad}, base learners update on the carefully designed surrogate functions, not the original ones.
A more detailed description of base learners' update rules are deferred to \pref{app:base-update-1grad} for self-containedness.
We provide the regret guarantee below, which achieves the same guarantees as \pref{thm:unigrad-correct} with only one gradient per round.
The proof is in \pref{app:unigrad-correct-1grad}.
\begin{myThm}
    \label{thm:unigrad-correct-1grad}
    Under Assumptions~\ref{assum:domain-boundedness}, \ref{assum:gradient-boundedness}, \ref{assum:smoothness-X}, by setting 
    \begin{gather}
        C_0=\max\bbr{1,8D,4\gamma^{\Top},4D^2C_{11},16D^2C_{10},80D^3L^2+4D^2C_1}, \label{eq:unigrad-correct-1grad}\\
        \gamma^\Top=\max\bbr{2D^2C_{11},8D^2C_{10},40D^3L^2+2D^2C_1},
        \gamma^{\Mid}=\max\bbr{C_{11},4C_{10}, 20DL^2+C_1}, \notag
    \end{gather}
    where $C_1=128(D^2L^2+G^2)$, $C_{10}=4L^2 + 32D^2 G^2 L^2 + 8G^4$, and $C_{11}=128G^2(1+L^2)$, \Correctpp (\pref{alg:UniGrad-Correct-1grad}) achieves the following universal gradient-variation regret guarantees using only one gradient per round:
    \begin{equation*}
        \sumT f_t(\x_t) - \min_{\x \in \X} \sumT f_t(\x) \le  
        \begin{cases}
            \O\left(\frac{1}{\lambda}\log V_T\right), & \text{when } \{f_t\}_{t=1}^T \text{ are $\lambda$-strongly convex}, \\[2mm]
            \O\left(\frac{d}{\alpha}\log V_T\right), & \text{when } \{f_t\}_{t=1}^T \text{ are $\alpha$-exp-concave}, \\[2mm]
            \O(\sqrt{V_T \log V_T}), & \text{when } \{f_t\}_{t=1}^T \text{ are convex}.
        \end{cases}
    \end{equation*}
\end{myThm}

\subsection{\textsf{UniGrad++.Bregman}: One-Gradient Improvement of \textsf{UniGrad.Bregman}}
\label{subsec:bregman-one-gradient}
In this part, we leverage the same idea of surrogate losses to improve the gradient query efficiency of \Bregman.
However, this becomes more challenging than that in \pref{subsec:correct-one-gradient}, because the meta-base regret decomposition of \pref{eq:Zhang-decompose} (we restate it below)
\begin{equation*}
    \Reg_T = \mbr{\sumT f_t(\x_t) - \sumT f_t(\x_{t,\is})} + \mbr{\sumT f_t(\x_{t,\is}) - \sumT f_t(\xs)},
\end{equation*}
is not suitable anymore because this decomposition would require each base learner to access its own gradient $\nabla f_t(\x_{t,i})$ per round, which is not allowed in the one-gradient setup.
Therefore, the decomposition in \pref{lem:emp-VT-de-breg} becomes invalid, since the negative Bregman divergence term $\D_{f_t}(\x_{t,\is}, \x_t)$ in \pref{eq:negative-Bregman-1} becomes vacuous.

\paragraph{Empirical Gradient Variation Decomposition.}
To address this issue, we need to find a new regret decomposition that allows us to use only one gradient $\nabla f_t(\x_t)$ per round.
Recall that the negative Bregman divergence term from linearization is algorithm-independent.
Building on this observation, we propose a new decomposition for the empirical gradient variation by inserting algorithm-independent intermediate terms such as $\nabla f_t(\xs)$ and $\nabla f_{t-1}(\xs)$, where $\xs \in \argmin_{\x \in \X} \sumT f_t(\x)$ and a more detailed derivation with constants is provided in \pref{eq:empirical-VT-de}.
\begin{align}
    \Vb_T \lesssim {} & \sumTT \sbr{\|\nabla f_t(\x_t) - \nabla f_t(\xs)\|^2 + \|\nabla f_t(\xs) - \nabla f_{t-1}(\xs)\|^2 + \|\nabla f_{t-1}(\xs) - \nabla f_{t-1}(\x_{t-1})\|^2} \notag\\
    \overset{\eqref{eq:smoothness-property}}{\lesssim} {} & L \sumTT \D_{f_t}(\xs, \x_t) + V_T + L \sumTT \D_{f_{t-1}}(\xs, \x_{t-1}) \le 2L \sumT \D_{f_t}(\xs, \x_t) + V_T, \label{eq:VbT-de}
\end{align}
Consequently, to cancel the additional $\D_{f_t}(\xs, \x_t)$ while using only one gradient $\nabla f_t(\x_t)$ per round, we use the overall regret linearization below:
\begin{equation}
    \label{eq:negative-Bregman}
    \sumT f_t(\x_t) - \sumT f_t(\xs) = \sumT \inner{\nabla f_t(\x_t)}{\x_t - \xs} - \sumT \D_{f_t}(\xs, \x_t).
\end{equation}

\paragraph{Surrogate Empirical Gradient Variation.}
Furthermore, we show that additional novel analysis is required to handle the empirical gradient variation \mbox{\emph{defined on surrogates}}.
Again, we take $\lambda$-strongly convex functions as an example and provide the following decomposition of the empirical gradient on surrogates:
\begin{align*}
    D_{T,i}^\scvx = {} & \sumTT \|\nabla \hsc_{t,i}(\x_{t,i}) - \nabla \hsc_{t-1,i}(\x_{t-1,i})\|^2\\
    = {} & \sumTT \norm{\nabla f_t(\x_t) - \nabla f_{t-1}(\x_{t-1}) + \frac{\lambda_i}{2}(\x_{t,i} - \x_t) - \frac{\lambda_i}{2}(\x_{t-1,i} - \x_{t-1})}^2\\
    \les {} & \sumTT \norm{\nabla f_t(\x_t) - \nabla f_{t-1}(\x_{t-1})}^2 + \sumTT \|\x_{t,i} - \x_t\|^2 + \sumTT \|\x_{t-1,i} - \x_{t-1}\|^2\\
    \les {} & \sumTT \norm{\nabla f_t(\x_t) - \nabla f_{t-1}(\x_{t-1})}^2 + \sumTT \|\x_{t,i} - \x_t\|^2.
\end{align*}
In the third step, instead of controlling $(\x_{t,i} - \x_t) - (\x_{t-1,i} - \x_{t-1})$ per round, which requires analyzing the stability term $\|\x_t - \x_{t-1}\|^2$ directly, we deal with the additional surrogate-induced terms by \emph{aggregation over the time horizon}, using a similar idea in our unifying optimism design in \pref{lem:universal-optimism}.
Consequently, this term can be canceled out by the curvature-induced negative term from the meta regret, as shown in \pref{eq:meta-regret-exp-concave}.
For this cancellation to occur, appropriate coefficients are needed, which are provided in the detailed proofs and are omitted here for clarity.\footnote{For strongly convex functions, a simpler choice would be $M_{t,i} = \nabla f_{t-1}(\x_{t-1})$ to allow simpler surrogate-induced terms.
We choose the gradient of the last round as the optimism since this is the the only choice at present to achieve a gradient-variation regret for exp-concave functions~\citep{COLT'12:VT}.}

The empirical gradient variation defined on the original functions, i.e., $\|\nabla f_t(\x_t) - \nabla f_{t-1}(\x_{t-1})\|^2$, can be decomposed as shown in \pref{eq:VbT-de} and canceled by the negative Bregman divergence terms from linearization as presented in \pref{eq:negative-Bregman}.

\begin{algorithm}[!t]
    \caption{\Bregmanpp: One-Gradient Improvement of \Bregman}
    \label{alg:UniGrad-Bregman-1grad}
    \begin{algorithmic}[1]
    \REQUIRE Base learner configurations $\{\B_i\}_{i \in [N]}\define \{\B_i^\scvx\}_{ i\in [N_\scvx]} \cup \{\B_i^\exp\}_{i \in [N_\exp]} \cup \B^{\cvx}$ 
    \STATE \textbf{Initialize}: $\M$~---~meta learner \omlprod with $W_{0,i} = \frac{1}{N}$ for all $i \in [N]$ \\
    \makebox[1.8cm]{} $\{\B_i\}_{i \in [N]}$~---~base learners as specified in \pref{subsec:framework}
    \FOR{$t=1$ {\bfseries to} $T$}
        \STATE Submit $\x_t = \sumN p_{t,i} \x_{t,i}$, suffer $f_t(\x_t)$, and observe $\nabla f_t(\cdot)$
        \STATE $\{\B_i^\scvx\}_{i \in [N_{\scvx}]}$, $\{\B_i^\exp\}_{i \in [N_{\exp}]}$, $\B^{\cvx}$ update their own decisions to $\{\x_{t+1,i}\}_{i=1}^N$ using surrogate losses of $\{\hsc_{t,i}(\cdot)\}_{\lambda_i \in \H^{\scvx}}$~\eqref{eq:surrogate-scvx}, $\{\hexp_{t,i}(\cdot)\}_{\alpha_i \in \H^{\exp}}$~\eqref{eq:surrogate-exp-cvx}, and $h_t^{\cvx}(\cdot)$~\eqref{eq:surrogate-exp-cvx} \label{line:bregman-1grad}
        \STATE Calculate $\m_{t+1}$~\eqref{eq:Bregman-optimism} and $\r_t$ using $\{\x_{t,i}\}_{i=1}^N$, $\x_t$, $\nabla f_t(\x_t)$, and $\{\x_{t+1,i}\}_{i=1}^N$, send them to $\M$, and obtain $\p_{t+1} \in \Delta_N$
    \ENDFOR
    \end{algorithmic}
\end{algorithm}

This simple and novel analysis eliminates the need to control the overall algorithmic stability $\|\x_t - \x_{t-1}\|^2$ required by \Correct, and is essential for achieving the optimal regret guarantees, as provided in the following, where the proof is deferred to \pref{app:unigrad-bregman-1grad}.
\begin{myThm}
    \label{thm:unigrad-bregman-1grad}
    Under Assumptions \ref{assum:domain-boundedness}, \ref{assum:gradient-boundedness}, \ref{assum:smoothness-X+}, by setting the learning rate of meta algorithm as \pref{eq:omlprod-lr}, \Bregmanpp (\pref{alg:UniGrad-Bregman-1grad}) achieves the following universal gradient-variation regret guarantees using only one gradient per round:
    \begin{equation*}
        \sumT f_t(\x_t) - \min_{\x \in \X} \sumT f_t(\x) \le  
        \begin{cases}
            \O\left(\frac{1}{\lambda}\log V_T\right), & \text{when } \{f_t\}_{t=1}^T \text{ are $\lambda$-strongly convex}, \\[2mm]
            \O\left(\frac{d}{\alpha}\log V_T\right), & \text{when } \{f_t\}_{t=1}^T \text{ are $\alpha$-exp-concave}, \\[2mm]
            \O(\sqrt{V_T}), & \text{when } \{f_t\}_{t=1}^T \text{ are convex}.
        \end{cases}
    \end{equation*}
\end{myThm}

\section{Implications, Applications, and Extension}
\label{sec:applications}
In this section, we demonstrate the effectiveness of our methods through their implications for small-loss and gradient-variance regret in \pref{subsec:implication}, as well as their applications to the Stochastically Extended Adversarial (SEA) model (in \pref{subsec:applications-SEA}) and online games (in \pref{subsec:applications-games}).
Finally, in \pref{subsec:extension-anytime}, we establish optimal universal regret without requiring prior knowledge of the time horizon $T$ through an anytime variant of our method.

\subsection{Implications to Small-Loss and Gradient-Variance Bounds}
\label{subsec:implication}
In this subsection, we demonstrate that our universal gradient-variation regret bounds naturally yield other problem-dependent quantities such as small-loss~\citep{NIPS'10:smooth,AISTATS'12:Orabona} and gradient-variance~\citep{COLT'08:Hazan-variation,NIPS'09:Hazan-Kale} regret bounds directly through the analysis \emph{without} any algorithmic modifications.
This demonstrates that our methods can capture the complexity of the online learning problem from multiple perspectives, providing a more comprehensive understanding of the problem's behavior.

Specifically, the small-loss and gradient-variance quantities are formally defined as:
\begin{equation}
    \label{eq:FT-WT}
    \begin{gathered}
        F_T \define \min_{\x \in \X} \sumT f_t(\x) - \sumT \min_{\x \in \X_+} f_t(\x), \ \ \X_+ \define \bbr{\x + \z \given \x \in \X, \z \in \tfrac{G}{L} \cdot \Bb}\\
        W_T \define \sup_{\{\x_1,\ldots,\x_T\} \in \X} \bbr{\sumT \|\nabla f_t(\x_t) - \mub_T\|^2}, \ \ \mub_T \define \frac{1}{T} \sumT \nabla f_t(\x_t),
    \end{gathered}
\end{equation}
where $\X_+$ is a superset of the original domain $\X$ defined in \pref{assum:smoothness-X+} and $\mu_T$ represents the average gradient. 

In what follows, we demonstrate that both the small-loss and gradient-variance quantities can be derived from the empirical gradient variation through standard analytical techniques.
Specifically, for the small-loss quantity, we utilize the self-bounding property $\|\nabla f(\x)\|_2^2 \le 4 L (f(\x)  - \min_{\x \in \X_+} f(\x))$ for any $L$-smooth function $f:\X_+ \rightarrow \R$ and any $\x \in \X_+$,\footnote{This more restricted self-bounding property can be derived using the arguments in \pref{app:smooth-relax}.} which yields
\begin{equation}
    \label{eq:to-small-loss}
    \begin{aligned}
        \Vb_T \le {} & 2 \sumTT \|\nabla f_t(\x_t)\|^2 + 2 \sumTT \|\nabla f_{t-1}(\x_{t-1})\|^2 \le 4 \sumT \|\nabla f_t(\x_t)\|^2\\
        \le {} &  16 L \sbr{\sumT f_t(\x_t) - \min_{\x \in \X_+} f_t(\x)}.
    \end{aligned}
\end{equation}
Note that the right-hand side of \pref{eq:to-small-loss} can be directly transformed to the small-loss quantity using standard techniques~\citep{NIPS'10:smooth,AISTATS'12:Orabona}. 

Next, we demonstrate that the gradient-variance quantity can be derived from the empirical gradient variation through a standard analytical technique:
\begin{align}
    \Vb_T = {} & \sumTT \|\nabla f_t(\x_t) - \nabla f_{t-1}(\x_{t-1})\|^2 \le 2 \sumTT \|\nabla f_t(\x_t) - \mub_T\|^2 + 2 \sumTT \|\nabla f_{t-1}(\x_{t-1}) - \mub_T \|^2 \notag\\
    \le {} & 4 \sumT \|\nabla f_t(\x_t) - \mub_T\|^2 \le 4 W_T. \label{eq:to-variance}
\end{align}
To conclude, our universal gradient-variation regret can directly imply universal small-loss and gradient-variance guarantees without any algorithmic modifications.
We present the corresponding guarantees for \Correctpp and \Bregmanpp below.
The proofs are deferred to \pref{app:FT-Correct} and \pref{app:FT-bregman}.

\begin{myCor}
    \label{cor:FT-WT-Correct}
    Under Assumptions \ref{assum:domain-boundedness}, \ref{assum:gradient-boundedness}, \ref{assum:smoothness-X+}, by setting
    \begin{equation}
        \label{eq:FT-WT-Correct}
        \begin{gathered}
            C_0=\max\bbr{1,8D,4\gamma^{\Top},512D^2G^2,128D^2G^4},\\
            \gamma^\Top=\max\bbr{256D^2G^2,64D^2G^4},\quad
            \gamma^{\Mid}=\max\bbr{128G^2,32G^4},
        \end{gathered}
    \end{equation}
    \Correctpp (\pref{alg:UniGrad-Correct-1grad}) achieves the following universal regret guarantees: 
    \begin{equation*}
        \Reg_T \le  
        \begin{cases}
            \O\left(\min\left\{\tfrac{1}{\lambda} \log F_T,\ \tfrac{1}{\lambda} \log W_T\right\}\right), & \text{when } \{f_t\}_{t=1}^T \text{ are $\lambda$-strongly convex}, \\[2mm]
            \O\left(\min\left\{\tfrac{d}{\alpha} \log F_T,\ \tfrac{d}{\alpha} \log W_T\right\}\right), & \text{when } \{f_t\}_{t=1}^T \text{ are $\alpha$-exp-concave}, \\[2mm]
            \O\left(\min\left\{\sqrt{F_T \log F_T},\ \sqrt{W_T \log W_T}\right\}\right), & \text{when } \{f_t\}_{t=1}^T \text{ are convex}.
        \end{cases}
    \end{equation*}
\end{myCor}
\begin{myCor}
    \label{cor:FT-bregman}
    Under Assumptions \ref{assum:domain-boundedness}, \ref{assum:gradient-boundedness}, \ref{assum:smoothness-X+}, by setting the learning rate of meta algorithm as \pref{eq:omlprod-lr}, \Bregmanpp (\pref{alg:UniGrad-Bregman-1grad}) achieves the following universal regret: 
    \begin{equation*}
        \Reg_T \le  
        \begin{cases}
            \O\left(\min\left\{\tfrac{1}{\lambda} \log F_T,\ \tfrac{1}{\lambda} \log W_T\right\}\right), & \text{when } \{f_t\}_{t=1}^T \text{ are $\lambda$-strongly convex}, \\[2mm]
            \O\left(\min\left\{\tfrac{d}{\alpha} \log F_T,\ \tfrac{d}{\alpha} \log W_T\right\}\right), & \text{when } \{f_t\}_{t=1}^T \text{ are $\alpha$-exp-concave}, \\[2mm]
            \O\left(\min\left\{\sqrt{F_T},\ \sqrt{W_T}\right\}\right), & \text{when } \{f_t\}_{t=1}^T \text{ are convex}.
        \end{cases}
    \end{equation*}
\end{myCor}
We would like to emphasize that the hyper-parameter configurations in \pref{eq:FT-WT-Correct} is employed only for the small-loss and gradient-variance regret. To achieve \emph{best-of-three-worlds} guarantees in terms of $\min\{V_T, F_T, W_T\}$, the hyper-parameters of $C_0, \gamma^\Top, \gamma^\Mid$ should be set as the union of configurations in \pref{eq:unigrad-correct-1grad} and \pref{eq:FT-WT-Correct}. Furthermore, because the hyper-parameter setups of \Bregmanpp for this problem are the same for those for the gradient-variation regret (in \pref{thm:unigrad-bregman-1grad}), \Bregmanpp directly enjoys \emph{best-of-three-worlds} guarantees.

\subsection{Application to Stochastically Extended Adversarial (SEA) Model}
\begin{table}[!t]
    \centering
    \caption{\small{Comparisons of our results with existing ones. The second column presents the regret bounds, where $\smash{\sigma_{1:T}^2}$ and $\smash{\Sigma_{1:T}^2}$ represent the stochastic and adversarial statistics of the SEA problem. The last column indicates whether the results can be achieved by a single algorithm (i.e., suitable in the universal setup). \Correctpp suffers and additional logarithmic factor compared with the best known guarantees of \citet{JMLR'24:OMD4SEA}, while \Bregmanpp achieves exactly the same state-of-the-art guarantees using one single algorithm.}}
    \vspace{2mm}
    \label{table:SEA}
    \renewcommand*{\arraystretch}{1.6}
    \resizebox{\textwidth}{!}{
        \begin{tabular}{c|ccc|c}
            \hline

            \hline
            \multirow{2}{*}{\textbf{Method}} & \multicolumn{3}{c|}{\textbf{Regret Bounds}} & \multirow{2}{*}{\makecell[c]{\textbf{Single} \\ \textbf{Algorithm?}}}\\  \cline{2-4}
             & Strongly Convex & Exp-concave & Convex \\ \hline
             \citet{NeurIPS'22:SEA} & $\O\sbr{\frac{1}{\lambda} \sbr{\sigma_{\max}^2 + \Sigma_{\max}^2} \log T}$ & N/A & $\O\sbr{\sqrt{\sigma_{1:T}^2 + \Sigma_{1:T}^2}}$ & \No \\ \hline
             \rule{0pt}{0.5cm}\citet{JMLR'24:OMD4SEA} & $\O\sbr{\frac{1}{\lambda} \sbr{\sigma_{\max}^2 + \Sigma_{\max}^2} \log \sbr{\frac{\sigma_{1:T}^2 + \Sigma_{1:T}^2}{\sigma_{\max}^2 + \Sigma_{\max}^2}}}$ & $\O\sbr{\frac{d}{\alpha} \log \sbr{\sigma_{1:T}^2 + \Sigma_{1:T}^2}}$ & $\O\sbr{\sqrt{\sigma_{1:T}^2 + \Sigma_{1:T}^2}}$ & \No \\[3pt] \hline \hline
            \citet{arXiv'23:SEA-Sachs} & $\O\sbr{\frac{1}{\lambda} \sbr{\sigma_{\max}^2 + \Sigma_{\max}^2 + D^2 L^2} \log^2 T}$ & N/A & $\O\sbr{\sqrt{T \log T}}$ & \Yes\\ \hline 
            \rowcolor{gray!13}\Correctpp & $\O\sbr{\frac{1}{\lambda} \sbr{\sigma_{\max}^2 + \Sigma_{\max}^2} \log \sbr{\sigma_{1:T}^2 + \Sigma_{1:T}^2}}$ & $\O\sbr{\frac{d}{\alpha} \log \sbr{\sigma_{1:T}^2 + \Sigma_{1:T}^2}}$ & $\O\sbr{\sqrt{\sbr{\sigma_{1:T}^2 + \Sigma_{1:T}^2} \log \sbr{\sigma_{1:T}^2 + \Sigma_{1:T}^2}}}$ & \Yes\\
            \hline 
            \rowcolor{gray!13}\rule{0pt}{0.65cm}\Bregmanpp & $\O\sbr{\frac{1}{\lambda} \sbr{\sigma_{\max}^2 + \Sigma_{\max}^2} \log \sbr{\frac{\sigma_{1:T}^2 + \Sigma_{1:T}^2}{\sigma_{\max}^2 + \Sigma_{\max}^2}}}$ & $\O\sbr{\frac{d}{\alpha} \log \sbr{\sigma_{1:T}^2 + \Sigma_{1:T}^2}}$ & $\O\sbr{\sqrt{\sigma_{1:T}^2 + \Sigma_{1:T}^2}}$ & \Yes \\[3pt] 
            \hline

            \hline
        \end{tabular}}
\end{table}

\label{subsec:applications-SEA}
Stochastically extended adversarial (SEA) model \citep{NeurIPS'22:SEA} interpolates between stochastic and adversarial online convex optimization.
Formally, it assumes that the online function $f_t(\cdot)$ is sampled stochastically from an adversarially chosen distribution $\mathfrak{D}_t$.
Denoting by $F_t(\cdot) \define \E_{f_t \sim \mathfrak{D}_t}[f_t(\cdot)]$ the expected function, two terms capture the essential characteristics of SEA model:
\begin{equation*}
    \sigma_{1:T}^2 \define \sum_{t=1}^T \max_{\x \in \X} \E_{f_t \sim \mathfrak{D}_t} \left[\|\nabla f_t(\x) - \nabla F_t(\x)\|^2\right], \ \ 
    \Sigma_{1:T}^2 \define \E\left[\sum_{t=1}^T \sup_{\x \in \X} \|\nabla F_t(\x) - \nabla F_{t-1}(\x)\|^2\right],
\end{equation*}
where $\sigma_{1:T}^2$ is the variance in sampling $f_t(\cdot)$ from $\mathfrak{D}_t(\cdot)$ and $\Sigma_{1:T}^2$ is the variation of $\{F_t(\cdot)\}_{t \in [T]}$.
Accordingly, we define the per-round maximum versions of the above quantities as
\begin{equation*}
    \sigma_{\max}^2 \define \max_{t \in [T]} \max_{\x \in \X} \E_{f_t \sim \mathfrak{D}_t} \left[\|\nabla f_t(\x) - \nabla F_t(\x)\|^2\right],
    \quad
    \Sigma_{\max}^2 \define \max_{t \in [T]} \sup_{\x \in \X} \|\nabla F_t(\x) - \nabla F_{t-1}(\x)\|^2.
\end{equation*}

For the SEA problem, \citet{NeurIPS'22:SEA} pioneered the study of the SEA model. 
For smooth expected functions $\seq{F_t(\cdot)}$, they established the optimal $\O(\sqrt{\sigma_{1:T}^2 + \Sigma_{1:T}^2})$ regret for convex expected functions and $\O(\frac{1}{\lambda}(\sigma_{\max}^2 + \Sigma_{\max}^2) \log T)$ in the strongly convex case.
Subsequently, \citet{JMLR'24:OMD4SEA} enhanced the strongly convex regret to $\O(\frac{1}{\lambda}(\sigma_{\max}^2 + \Sigma_{\max}^2) \log ((\sigma_{1:T}^2 + \Sigma_{1:T}^2)/ (\sigma_{\max}^2 + \Sigma_{\max}^2)))$ and derived a new $\O (\frac{d}{\alpha}\log (\sigma_{1:T}^2 + \Sigma_{1:T}^2))$ regret bound for exp-concave individual functions $\seq{f_t(\cdot)}$.

The gradient variation is essential in connecting the stochastic and adversarial optimization in the SEA problem~{\citep[Lemma 3]{JMLR'24:OMD4SEA}}. To see this, we can decompose the empirical gradient variation as:
\begin{equation}
    \label{eq:SEA-correction}
    \E\mbr{\sumTT \|\nabla f_t(\x_t) - \nabla f_{t-1}(\x_{t-1})\|^2} \le 4L^2 \E\mbr{\sumTT \|\x_t - \x_{t-1}\|^2} + 8 \sigma_{1:T}^2 + 4 \Sigma_{1:T}^2 + \O(1),
\end{equation}
which not only consists of the stochastic variation $\sigma_{1:T}^2$ and the adversarial variation $\Sigma_{1:T}^2$, but also the algorithmic stability $\|\x_t - \x_{t-1}\|^2$.
And the last term can be perfectly handled by \Correctpp as introduced in \pref{sec:method1-correct} and \pref{subsec:correct-one-gradient}.

For \Bregmanpp to solve the SEA problem, since it cannot directly deal with the stability terms of $\|\x_t - \x_{t-1}\|^2$, we provide it a different decomposition of the empirical gradient variation to let negative Bregman divergence terms in the analysis of \Bregmanpp to take effect.
Specifically, we decompose it as 
\begin{equation}
    \label{eq:SEA-bregman}
    \E\mbr{\sumTT \|\nabla f_t(\x_t) - \nabla f_{t-1}(\x_{t-1})\|^2} \le 10 \sigma_{1:T}^2 + 5 \Sigma_{1:T}^2 + 20L E\mbr{\sumT \D_{F_t}(\xs, \x_t)},
\end{equation}
where the positive Bregman divergence terms can be canceled accordingly.
A detailed derivation of this inequality is deferred to \pref{eq:SEA-bregman-detailed} in \pref{app:SEA-bregman}. 

Therefore, universal gradient-variation regret can be applied to in the SEA problem, therefore solving a major open problem from \citet{JMLR'24:OMD4SEA} about whether it is possible to get rid of different parameter configurations and obtain universal guarantees.
In the following, we show that our approaches of \Correctpp and \Bregmanpp can be both directly applied and achieve almost the same guarantees as those in \citet{JMLR'24:OMD4SEA}, with a \mbox{\emph{single}} algorithm.
We conclude our results in \pref{thm:SEA-correct} and \pref{thm:SEA-bregman} below and the proofs can be found in \pref{app:SEA-correct} and \pref{app:SEA-bregman}.
\begin{myThm}
    \label{thm:SEA-correct}
    Under Assumptions \ref{assum:domain-boundedness}, \ref{assum:gradient-boundedness}, \ref{assum:smoothness-X}, 
    by setting
    \begin{gather*}
        C_0=\max\bbr{1,8D,4\gamma^{\Top},8D^2C_{24},8D^2C_{23},8D^2C_{25}},\\
        \gamma^\Top=\max\bbr{4D^2C_{24},8D^2C_{23},4D^2\sbr{20DL^2+64G^2+128D^2L^2}},\\
        \gamma^{\Mid}=\max\bbr{2C_{24},4C_{23}, 20DL^2+64G^2+128D^2L^2},
    \end{gather*}
    where $C_{23}=8L^2 + 64D^2 G^2 L^2 + 8G^4$, $C_{24}=64D^2(1+L^2)^2$, and $C_{25}=20DL^2+\frac{64G^2}{Z}+128D^2L^2$, \Correctpp (\pref{alg:UniGrad-Correct-1grad}) achieves the following universal regret: 
    \begin{equation*}
        \Reg_T \le 
        \begin{cases}
        \O\left(\frac{1}{\lambda}(\sigma_{\max}^2 + \Sigma_{\max}^2) \log \left(\frac{\sigma_{1:T}^2 + \Sigma_{1:T}^2}{\sigma_{\max}^2 + \Sigma_{\max}^2}\right)\right), & \text{when } \{F_t\}_{t=1}^T \text{ are $\lambda$-strongly convex}, \\[2mm]
        \O\left(\frac{d}{\alpha} \log (\sigma_{1:T}^2 + \Sigma_{1:T}^2)\right), & \text{when } \{f_t\}_{t=1}^T \text{ are $\alpha$-exp-concave}, \\[2mm]
        \O\left(\sqrt{\sbr{\sigma_{1:T}^2 + \Sigma_{1:T}^2} \log \sbr{\sigma_{1:T}^2 + \Sigma_{1:T}^2}}\right), & \text{when } \{F_t\}_{t=1}^T \text{ are convex}.
        \end{cases}
    \end{equation*}
\end{myThm}

\begin{myThm}
    \label{thm:SEA-bregman}
    Under Assumptions \ref{assum:domain-boundedness}, \ref{assum:gradient-boundedness}, \ref{assum:smoothness-X+}, by setting the learning rate of meta algorithm as \pref{eq:omlprod-lr}, \Bregmanpp (\pref{alg:UniGrad-Bregman-1grad}) achieves the following universal gradient-variation regret guarantees: 
    \begin{equation*}
        \Reg_T \le 
        \begin{cases}
        \O\left(\frac{1}{\lambda}(\sigma_{\max}^2 + \Sigma_{\max}^2) \log \left(\frac{\sigma_{1:T}^2 + \Sigma_{1:T}^2}{\sigma_{\max}^2 + \Sigma_{\max}^2}\right)\right), & \text{when } \{F_t\}_{t=1}^T \text{ are $\lambda$-strongly convex}, \\[2mm]
        \O\left(\frac{d}{\alpha} \log (\sigma_{1:T}^2 + \Sigma_{1:T}^2)\right), & \text{when } \{f_t\}_{t=1}^T \text{ are $\alpha$-exp-concave}, \\[2mm]
        \O\left(\sqrt{\sigma_{1:T}^2 + \Sigma_{1:T}^2}\right), & \text{when } \{F_t\}_{t=1}^T \text{ are convex}.
        \end{cases}
    \end{equation*}
\end{myThm}
\begin{myRemark}
    \citet{arXiv'23:SEA-Sachs} also considered the problem of universal learning and obtained an $\O(\sqrt{T \log T})$ regret for convex functions and an $\O((\sigma_{\max}^2 + \Sigma_{\max}^2 + D^2 L^2) \log^2 T)$ regret for strongly convex functions simultaneously. We conclude the existing results in \pref{table:SEA}. Our results are better than theirs in two aspects: \textit{(i)} for strongly convex and convex functions, our guarantees are adaptive with the problem-dependent quantities $\sigma_{1:T}$ and $\Sigma_{1:T}$ while theirs depends on the time horizon $T$; and \textit{(ii)} our algorithm achieves an additional guarantee for exp-concave functions.
    \endenv
\end{myRemark}

\begin{myRemark}
    \pref{thm:SEA-correct} and \pref{thm:SEA-bregman} require the exp-concavity of the individual function $f_t(\cdot)$ rather than the expected function $F_t(\cdot)$. This assumption is also used by \citet{JMLR'24:OMD4SEA} and common in the studies of stochastic exp-concave optimization~\citep{COLT'15:Mahdavi,NIPS'15:Koren}.
    \endenv
\end{myRemark}

\subsection{Application to Faster-Rate Convergence in Online Games}
\label{subsec:applications-games}

Multi-player online games~\citep{book'06:PLG-Bianchi} is a versatile model that depicts the interaction of multiple players over time.
Since each player is facing similar players like herself, the theoretical guarantees, e.g., the summation of all players' regret, can be better than the minimax optimal $\O(\sqrt{T})$ in adversarial environments, thus achieving \emph{faster rates}.

The pioneering works of \citet{NIPS'13:optimism-games} and \citet{NIPS'15:fast-rate-game} investigated optimistic algorithms in multi-player online games and illuminated the importance of the gradient variation.
Specifically, \citet{NIPS'15:fast-rate-game} showed that optimistic algorithms, such as \OOMD or optimistic follow the regularized leader \citep{journals/ml/Shalev, thesis:shai2007}, possess a specific property known as ``Regret bounded by Variation in Utilities'' (RVU) property.
\begin{myDef}[RVU Property]
    An algorithm with a parameter $\eta>0$ satisfies the RVU property if there exist constants $\alpha,\beta,\gamma>0$ such that the regret $\Reg_T$ on decision sequence $\{\x_t\}_{t=1}^T$ and gradient sequence $\{\g_t\}_{t=1}^T$ is bounded by
    \begin{align}
        \Reg_T\le \frac{\alpha}{\eta}+ \beta\eta\sumTT \|\g_t-\g_{t-1}\|_\infty^2-\frac{\gamma}{\eta} \sumTT\|\x_t-\x_{t-1}\|_1^2.
    \label{eq:RVU}
    \end{align}
\end{myDef}
To illustrate the usefulness of the RVU property, we consider a simple bilinear zero-sum game of $\x^\top A \y$ where $\x,\y \in \Delta_d$ and $\max_{i,j \in [d]} \abs{A_{i,j}}\le 1$.
In this case, the gradients of the $\x$-player are given by $\g_t^\x=A\y_t$ for $t\in[T]$, which implies $\sumTT\|\g_t^\x-\g_{t-1}^\x\|_{\infty}^2=\sumTT\|A\y_t-A\y_{t-1}\|_\infty^2\le \sumTT\|\y_t-\y_{t-1}\|_1^2$.
Similarly, for the $\y$-player, we have $\sumTT\|\g_t^\y-\g_{t-1}^\y\|_{\infty}^2\le \sumTT\|\x_t-\x_{t-1}\|_1^2$.
By setting both learners' learning rates to $\eta^\x=\eta^\y=\sqrt{\gamma/\beta}$, the sum of the two players' regrets can be bounded by 
\begin{align*}
 \Reg_T^\x+\Reg_T^\y\le {} &  \frac{\alpha}{\eta^\x}+\frac{\alpha}{\eta^\y}+(\beta\eta^\x-\frac{\gamma}{\eta^\y})\sumTT\|\x_t-\x_{t-1}\|_1^2\\
 {}&\qquad\qquad+(\beta\eta^\y-\frac{\gamma}{\eta^\x})\sumTT\|\y_t-\y_{t-1}\|_1^2\le \O(1).
\end{align*}
This bound, in turn, enables the efficient computation of Nash equilibria.

The above results assume that the players are \emph{honest}, i.e., they agree to run the same algorithm. In the \emph{dishonest} case, where there exist players who do not follow the agreed protocol, the problem degenerates to two separate online adversarial convex optimization problems.
At a high level, online games can be regarded as a special instance of adaptive online learning. The goal is to ensure robust performance on hard problems (e.g., when facing a dishonest opponent) while achieving superior performance on easy problems (e.g., when the opponent is honest). In particular, the adaptivity (e.g., gradient variation) can yield faster-rate convergence as a direct consequence of the RVU property.

 \begin{algorithm}[!t]
    \caption{\Correctpp for $\x$-player}
    \label{alg:game}
    \begin{algorithmic}[1]
    \REQUIRE Base learner configurations $\{\B_i\}_{i \in [N]}\define \{\B_i^\scvx\}_{i \in [N_\scvx]}\cup \B^{\cvx}$, algorithm parameters $\gamma^\Mid$, $\gamma^\Top$, and $C_0$
    
    \STATE \textbf{Initialize}: $\M$~---~meta learner \msoms as shown in \pref{alg:2layer-msmwc} \\
    \makebox[1.8cm]{} $\{\B_i\}_{i \in [N]}$~---~base learners as specified in \pref{subsec:framework}
    \FOR{$t=1$ {\bfseries to} $T$}
        \STATE Submit $\x_t = \sumN p_{t,i} \x_{t,i}$, suffer $f_t(\x,\y)$, and observe $\g_t^\x$
        \STATE $\{\B_i^\scvx\}_{i \in [N_\scvx]}$ and $\B^{\cvx}$ update their own decisions to $\{\x_{t+1,i}\}_{i=1}^N$ using surrogate losses of $\{\hsc_{t,i}(\cdot)\}_{\lambda_i \in \H^{\scvx}}$~\eqref{eq:surrogate-scvx} and $h_t^{\cvx}(\cdot)$~\eqref{eq:surrogate-exp-cvx} \label{line:game-base}

        \STATE Compute $\{\ellb^\Mid_{t,j}, \m^\Mid_{t+1,j}\}_{j=1}^M$ via~\pref{eq:mid-loss} with gradients of $\nabla f_t(\cdot) = \g_t^{\x}$, send to $\M$, get $\{\q_{t+1,j}^\Mid\}_{j=1}^M \in (\Delta_N)^M$ \label{line:meta-mid-update-game}

        \STATE Compute $\ellb_t^\Top, \m_{t+1}^\Top$ via~\pref{eq:top-loss-2}, send to $\M$, and obtain $\q_{t+1}^\Top \in \Delta_M$ \label{line:meta-top-update-game}

        \STATE Aggregate the final meta weights $\p_{t+1} \in \Delta_N$ via \pref{eq:meta-output}
        \label{line:meta-combine-game}
    \ENDFOR
    \end{algorithmic}
\end{algorithm}

Since the faster-rate convergence requires the RVU property, in this part, we validate the effectiveness of our proposed \Correctpp in a simple two-player zero-sum game as an illustrating example. The game can be formulated as a min-max optimization problem of $\min_{\x \in \X} \max_{\y \in \Y} f(\x,\y)$, in which the $\x$-player aims to minimize and the $\y$-player aims to maximize the objective. To validate the universality of our method, we consider the case that the game is either \emph{bilinear}, i.e., $f(\x,\y) = \x^\top A \y$, or \emph{strongly-convex-strongly-concave}, i.e., $f(\x,\y)$ is $\lambda$-strongly convex in $\x$ and $\lambda$-strongly concave in $\y$. We denote the two players' gradients by $\g_t^\x = \nabla_\x f(\x,\y)$ and $\g_t^\y = \nabla_\y f(\x,\y)$. Besides, to validate the RVU property of our method, we investigate the players can be either honest or dishonest. We proceed under the following standard assumptions concerning the strategy domains and the gradients of the two players, following previous works \citep{farina2022near}. The second assumption, known as the smoothness assumption, is classical in online games.
\begin{myAssum}
    \label{assum:game}
    We make the following assumptions: 
    \begin{enumerate}[leftmargin=7mm, labelsep=5pt, itemsep=1pt, topsep=1mm, parsep=1mm]
        \item[\rom{1}] The $\x$-player's strategy set is the simplex $\Delta_{d_\x}$, and the $\y$-player's strategy set is the simplex $\Delta_{d_\y}$. Moreover, the gradients of both players are uniformly bounded by $G$, i.e., $\|\g_t^\x\|\le  G$ and $\|\g_t^\y\|\le  G$ for all $t \in [T]$.  
        \item[\rom{2}] For $t\in[T]$, both players' gradients satisfy $\max\{\|\g_t^\x - \g_{t-1}^\x\|, \|\g_t^\y - \g_{t-1}^\y\|\}\le \|\x_t - \x_{t-1}\|+\|\y_t - \y_{t-1}\|$. 
    \end{enumerate}    
\end{myAssum}
 In \pref{alg:game}, we present the online game variant of \Correctpp for the $\x$-player, which ensures regret summation guarantees in the honest case and individual regret guarantees in the dishonest case, \emph{without} requiring prior knowledge of the game type.
 Compared with the algorithm for the single-player setup, \pref{alg:game} leverages additional problem structures for effective learning.
 Specifically, $\x$-player uses $\g_t^\x$ instead of the general $\nabla f_t(\x_t)$ as the feedback.
 The analogous variant for the $\y$-player follows the same design and is omitted for brevity.
 Here, the total number of base learners for both the $\x$-player and the $\y$-player is $N = 1 + \abs{\H^{\scvx}}= \O(\log T)$.
 We conclude our results in \pref{thm:game} and defer the proof to \pref{app:game}.

\begin{table}[!t]
    \centering
    \caption{\small{Comparisons of our results with existing ones. In the honest case, the results are measured by the summation of all players' regret and in the dishonest case, the results are in terms of the individual regret of each player. Bilinear and strongly-convex-strongly-concave games are considered inside each case. $\star$ denotes the best result in each case (row).}}
    \vspace{1mm}
    \label{table:game}
    \renewcommand*{\arraystretch}{1.5}
    \resizebox{0.95\textwidth}{!}{
    \begin{tabular}{c|c|ccc}
    \hline
    
    \hline
    & \textbf{Games} & \citet{NIPS'15:fast-rate-game} & \citet{ICML'22:universal} & \textbf{Ours} \\ \hline
    \multirow{2}{*}{\textbf{Honest}} 
        & bilinear & $\O(1)^\star$ & $\O(\sqrt{T})$ & $\O(1)^\star$ \\ \cline{2-5}
        & \makecell[c]{strongly-convex \\ strongly-concave} & $\O(1)^\star$ & $\O(\log T)$ & $\O(1)^\star$ \\ \hline
    \multirow{2}{*}{\textbf{Dishonest}} 
        & bilinear & $\O(\sqrt{T})^\star$ & $\O(\sqrt{T})^\star$ & $\O(\sqrt{T \log T})^\star$ \\ \cline{2-5}
        & \makecell[c]{strongly-convex \\ strongly-concave} & $\O(\sqrt{T})$ & $\O(\log T)^\star$ & $\O(\log T)^\star$ \\
    \hline
    
    \hline
    \end{tabular}}
\end{table}

\begin{myThm}
    \label{thm:game}
    Under Assumption~\ref{assum:game}, by setting
    \begin{gather*}
        C_0=\max\bbr{1,8D,4\gamma^{\Top},16C_{31},16C_{32}, 4\gamma_{\x}^{\Top}, 4\gamma_{\y}^{\Top}},\\
        \gamma_{\x}^{\Mid}=\gamma_{\y}^{\Mid}=128+128G^2+40\sqrt{2},\quad \gamma_{\x}^{\Top}=\gamma_{\y}^{\Top}=512+512G^2+160\sqrt{2},
    \end{gather*}
where
\begin{gather*}
    C_{31} = \frac{32 + 64G^2}{Z^\x}+\frac{32}{Z^\y}+20\sqrt{2}, \ \ C_{32} = \frac{32 + 64G^2}{Z^\y}+\frac{32}{Z^\x}+20\sqrt{2}, \\
    Z^\x = \max\{GD + \gamma_{\x}^{\Mid} D^2, 1 + \gamma_{\x}^{\Mid} D^2 + 2\gamma_{\x}^{\Top}\}, \\ Z^\y = \max\{GD + \gamma_{\y}^{\Mid} D^2, 1 + \gamma_{\y}^{\Mid} D^2 + 2\gamma_{\y}^{\Top}\},
\end{gather*}
for bilinear and strongly-convex-strongly-concave games, \pref{alg:game} enjoys $\O(1)$ regret summation in the honest case, and achieves $\O(\sqrt{T \log T})$ and $\O(\log T)$ regret bounds for bilinear and strongly-convex-strongly-concave games respectively in the dishonest case.
\end{myThm}
\pref{table:game} compares our approach with \citet{NIPS'13:optimism-games} and \citet{ICML'22:universal}.
Specifically, ours is better than \citet{NIPS'13:optimism-games} in the strongly-convex-strongly-concave games in the dishonest case due to its universality, and better than \citet{ICML'22:universal} in the honest case since our approach enjoys gradient-variation bounds that are essential in achieving fast rates for regret summation.

\subsection{Extension to Anytime Setting: Dynamic Online Ensemble}
\label{subsec:extension-anytime}
We start by observing that previous methods require the knowledge of the time horizon $T$ in advance, which is often unavailable in practice. To this end, we propose a \emph{dynamic online ensemble}~---~an anytime framework that avoids dependence on the time horizon $T$ and adjusts its candidate pools dynamically during the online learning process.

As established in \pref{subsec:notation-assumption-definition}, all base learners are \OOMD-type algorithms with adaptive step sizes, which means the base learners are inherently anytime given the choice of the curvature parameter. Therefore, we consider the dependence on the time horizon $T$ within the meta algorithm design as well as the scheduling of the curvature parameter.

\paragraph{Scheduling.}
To begin with, to make the curvature parameter scheduling suitable for the anytime setting, we consider maintaining \emph{infinite} candidates of the curvature coefficient and only activating a finite set of them at each round. This idea is inspired by \citet{COLT'19:Lipschitz-MetaGrad}.
Specifically, we define two candidate coefficient pools for the exp-concave and strongly convex cases, respectively:
\begin{equation}
    \label{eq:candidate-anytime}
    \H^\exp \define \{\alpha_i = 2^{-i}, \text{ for } i \in \{0\} \cup \mathbb{N} \},\quad  \H^\scvx \define \{\lambda_i = 2^{-i}, \text{ for } i \in \{0\} \cup \mathbb{N} \}.
\end{equation}
Intuitively, if the time horizon $T$ is known, we only need to discretize the possible range of $[1/T, 1]$ to get the candidate pool, as in \pref{eq:candidate-pool}.
In contrast, when $T$ is unknown, it may grow arbitrarily large, causing the lower bound of the feasible range to converge to zero.
Therefore, we deploy $2^{-i}$ for all $i \in \{0\} \cup \mathbb{N}$ as all possible candidates of the curvature coefficient.
For convex functions, we still maintain a single base learner $\B^\cvx$ as there is no unknown-curvature-coefficient issues. 

Furthermore, as we cannot implement infinite base learners in the actual running of the algorithm, we define two \emph{active} versions (denoted by $\H^\exp_t$ and $\H^\scvx_t$) of them which indicates that when $\lambda_i \in \H^\scvx_t$, the $i$-th base learner is active.
Formally, 
\begin{equation}
    \label{eq:candidate-active-anytime}
    \begin{gathered}
        \H^\scvx_t \define \bbr{\lambda_i = 2^{-i} \text{ and } t \ge \sbr{s^\scvx_i \define \frac{1}{\lambda_i}}, \text{ for } i \in \{0\} \cup \mathbb{N} },\\
        \H^\exp_t \define \bbr{\alpha_i = 2^{-i} \text{ and } t \ge \sbr{s^\exp_i \define \frac{1}{\alpha_i}}, \text{ for } i \in \{0\} \cup \mathbb{N} }.
    \end{gathered}
\end{equation}
We denote their sizes by $N^\scvx_t = |\H^\scvx_t|$ and $N^\exp_t = |\H^\exp_t|$. This means that the base learner with $\alpha_i$ is only activated from $t = s^\exp_i$. For example, the base learner with $\alpha_i = \frac{1}{8}$ is activated from $t = s^\exp_i = 8$.

\paragraph{Meta Algorithm.}
Subsequently, we consider making the meta algorithm anytime. Note that when the aforementioned \emph{infinitely many} base learners, the meta algorithms such as \msoms (used in \Correct and its one-gradient version) and \omlprod (used in \Bregmanpp and its one-gradient version) are \emph{not} applicable because they only support a fixed and number of base learners.

Fortunately, \citet{NeurIPS'24:Xie} proposed a variant of \omlprod that can handle the case of infinite base learners, which is perfectly suitable for our anytime setting.
Specifically, let $\A_t$ be the set of active experts at round $t$ and let $N_t=|\A_t|$ denote its size. 
We initialize $\A_0 = \{\B^\cvx, \B_0^\scvx, \B_0^\exp\}$, where $\B_0^\scvx$ and $\B_0^\exp$ are associated with the initial coefficients $\lambda_0$ and $\alpha_0$, respectively. At $t$-th round, a newly added expert $i$ is initialized with weight $W_{t,i} = 1$ and learning rate $\epsilon_{t,i} = \frac{1}{\sqrt{5}}$. The meta algorithm submits $\p_t \in \Delta_{N_t}$ as
\begin{equation}
    \label{eq:omlprod-variant-p}
    p_{t,i} = \frac{\epsilon_{t,i} W_{t,i} \exp(\epsilon_{t,i} m_{t,i})}{\sum_{j \in [N_t]} \epsilon_{t,j} W_{t,j} \exp(\epsilon_{t,j} m_{t,j})}, \text{ for all } i \in [N_t].
\end{equation}
The optimistic vector $\m_t \in \R^{N_t}$ is designed as 
\begin{equation}
    \label{eq:anytime-optimism}
   m_{t,i} = \inner{\nabla f_{t-1}(\x_{t-1})}{\x_t - \x_{t,i}}  \mbox{ for } i=1, \mbox{ and } m_{t,i} = 0 \mbox{ for } i > 1.
\end{equation} 
After receiving the loss vector $\r_t = (r_{t,1}, \ldots, r_{t,N_t}) \in \R^{N_t}$, where $r_{t,i} \define \inner{\ellb_t}{\p_t - \e_i}$ and $\ellb_{t,i}= \inner{\nabla f_t(\x_t)}{\x_{t,i}} $, for $i\in [N_t]$, it chooses the learning rate as
\begin{equation}
    \label{eq:anytime-lr}
    \epsilon_{t+1,i} = \sqrt{\frac{1}{5 + \sum_{s=s_i^{\{\exp, \scvx\}}}^t (r_{s,i} - m_{s,i})^2}}, \text{ for all } i \in [N_t].
\end{equation}
\begin{algorithm}[!t]
    \caption{Anytime Variant of \Bregmanpp}
    \label{alg:anytime}
    \begin{algorithmic}[1]
    \STATE \textbf{Initialize}: $\M$~---~meta learner \omlprod variant \\
    \makebox[1.8cm]{} $\{\B_i\}$~---~base learners as specified in \pref{subsec:framework}~(new scheduled in \pref{eq:candidate-anytime})\\
    \makebox[1.8cm]{} $\A_0 = \emptyset$~---~active set
    \FOR{$t=1$ {\bfseries to} $T$}
        \STATE Activate the base learners with $\lambda_i \in \H^\scvx_t$ or $\alpha_i \in \H^\exp_t$, initialize weights and learning rates as $W_{t,i} = 1$ and $\epsilon_{t,i} = \frac{1}{\sqrt{5}}$, and add them to $\A_{t-1}$ to obtain $\A_t$ \label{line:anytime-activate}
        \STATE Receive $\{\x_{t,i}\}_{i \in [N_t]}$ from $\{\B_i\}_{i \in [N_t]}$ and $\p_t \in \Delta_{N_t}$ from $\M$ via \pref{eq:omlprod-variant-p}

        \STATE Submit $\x_t = \sum_{i \in [N_t]} p_{t,i} \x_{t,i}$, suffer $f_t(\x_t)$, and observe $\nabla f_t(\x_t)$

        \STATE $\{\B_i^\scvx\}_{i \in [N^\scvx_t]}$, $\{\B_i^\exp\}_{i \in [N^\exp_t]}$ and $\B^{\cvx}$ update their own decisions to $\{\x_{t+1,i}\}_{i \in [N_t]}$ using surrogate losses of $\{\hsc_{t,i}(\cdot)\}_{\lambda_i \in \H_t^{\scvx}}$~\eqref{eq:surrogate-scvx}, $\{\hexp_{t,i}(\cdot)\}_{\alpha_i \in \H_t^{\exp}}$~\eqref{eq:surrogate-exp-cvx}, and $h_t^{\cvx}(\cdot)$~\eqref{eq:surrogate-exp-cvx} 
        \STATE Calculate $\m_{t+1}$~\eqref{eq:anytime-optimism} and $\r_t$ using $\{\x_{t,i}\}_{i=1}^{N_t}$, $\x_t$, $\nabla f_t(\x_t)$, and $\{\x_{t+1,i}\}_{i=1}^{N_t}$, send them to $\M$, and obtain $(W_{t+1,1}, \ldots, W_{t+1,N_t})$ via \pref{eq:omlprod-variant-weight}
        
    \ENDFOR
    \end{algorithmic}
\end{algorithm}
Finally, for each $i \in [N_t]$, the meta algorithm updates the weights as 
\begin{equation}
    \label{eq:omlprod-variant-weight}
    W_{t+1,i} = \sbr{W_{t,i} \exp \sbr{\epsilon_{t,i} r_{t,i} 
     - \epsilon_{t,i}^2 (r_{t,i} - m_{t,i})^2}}^{\frac{\epsilon_{t+1,i}}{\epsilon_{t,i}}}, \text{ for all } i \in [N_t].
\end{equation}
The corresponding guarantee is deferred to \pref{lem:anytime-optimistic-mlprod} in \pref{app:anytime}.

Note that the meta algorithm required here does not enjoy stability-induced negative terms as it falls in the \mlprod family.
Therefore, this extension cannot be applied to \Correct and \Correctpp because correction-based methods require stability-induced negative terms for effective cancellations.
On the contrary, \Bregman and \Bregmanpp can be made anytime by replacing their original meta algorithm, i.e., \omlprod, with this anytime variant in a straightforward manner algorithmically.
For simplicity, we only present the anytime extension of \Bregmanpp here, where the algorithm is concluded in \pref{alg:anytime} and the corresponding guarantee is presented in \pref{thm:anytime}.
The proof is deferred to \pref{app:anytime}.
\begin{myThm}
    \label{thm:anytime}
    Under Assumptions \ref{assum:domain-boundedness}, \ref{assum:gradient-boundedness}, \ref{assum:smoothness-X+}, and without the knowledge of time horizon $T$, by setting the learning rate of meta algorithm as \pref{eq:anytime-lr}, \Bregmanpp in \pref{alg:UniGrad-Bregman-1grad} achieves the following anytime universal gradient-variation regret guarantees using only one gradient per round: for any $\tau\in[T]$, we have
    \begin{equation*}
        \Reg_\tau\define \sumtau f_t(\x_t) - \min_{\x \in \X} \sumtau f_\tau(\x) \le  
        \begin{cases}
            \O\left(\frac{1}{\lambda}\log V_\tau\right), & \text{when } \{f_t\}_{t=1}^T \text{ are $\lambda$-strongly convex}, \\[2mm]
            \O\left(\frac{d}{\alpha}\log V_\tau\right), & \text{when } \{f_t\}_{t=1}^T \text{ are $\alpha$-exp-concave}, \\[2mm]
            \O(\sqrt{V_\tau}), & \text{when } \{f_t\}_{t=1}^T \text{ are convex}.
        \end{cases}
    \end{equation*}
\end{myThm}

\section{Discussions}
\label{sec:comparison-discussion}
This section provides further discussions to better understand our proposed methods, which is organized along two directions: \rom{1} a comprehensive comparison between our two proposed methods (\Correct and \Bregman; as well as their respective one-gradient variants), and \rom{2} a detailed technical comparison of \Correct with its previous conference version~\citep{NeurIPS'23:universal}.

\subsection{Comparison of \textsf{UniGrad.Correct} and \textsf{UniGrad.Bregman}}
\label{subsec:comparison-correct-Bregman}
To achieve the desired gradient-variation universal regret in online learning, we propose two different methods called \Correct and \Bregman, each with its own merits and characteristics.
Here we provide a systematic comparison along several key dimensions.

\begin{figure*}[!t]
  \centering
  \includegraphics[clip, trim=45mm 17mm 35mm 15mm, width=0.8\textwidth]{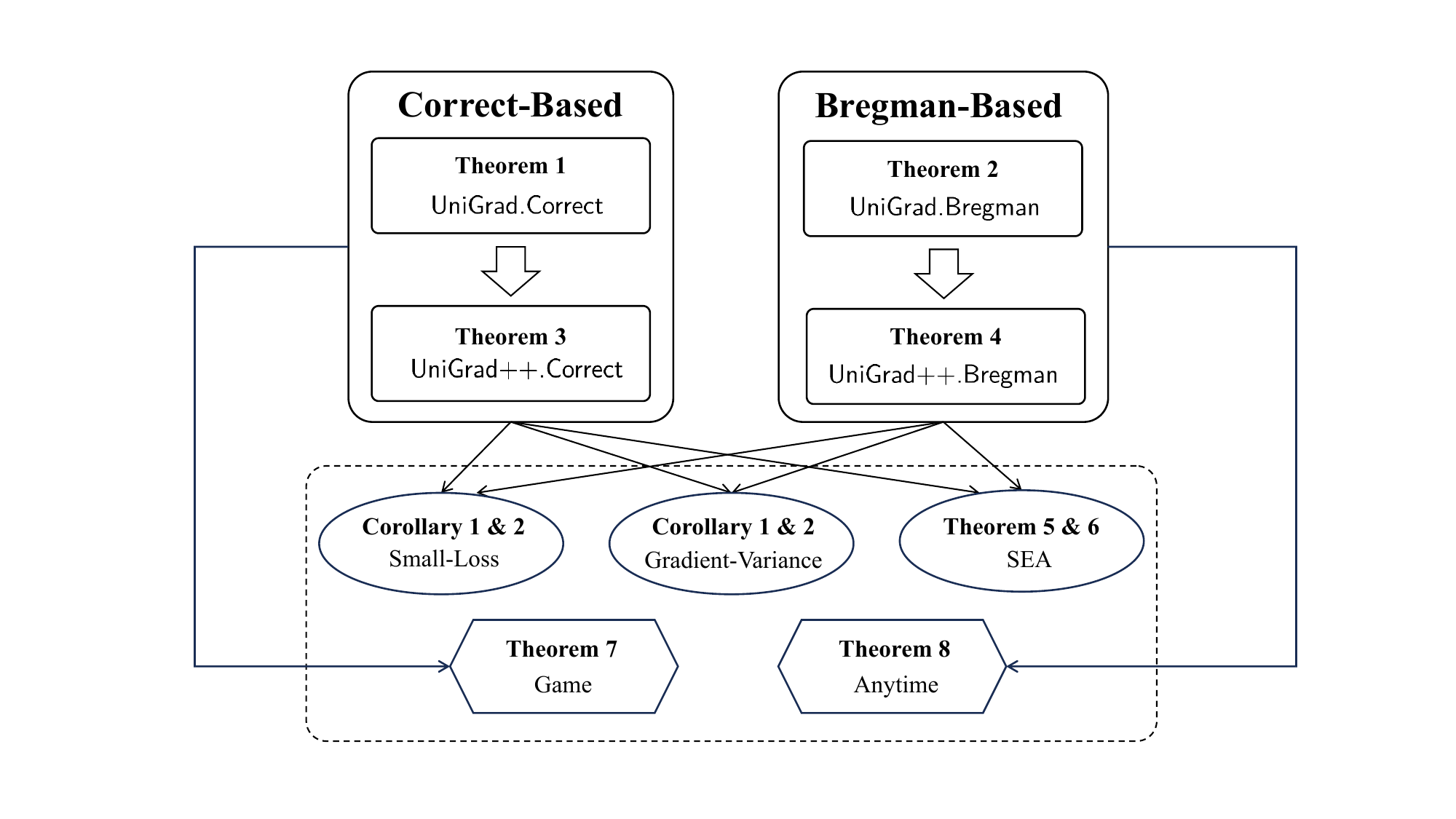}
  \caption{\small{The summary of the theoretical results in our work. Specifically, we propose two methods named \Correct and \Bregman (\pref{thm:unigrad-correct} and \pref{thm:unigrad-bregman}) to achieve gradient-variation universal regret. Both methods can be strengthened to the one-gradient feedback scenario (\pref{thm:unigrad-correct-1grad} and \pref{thm:unigrad-bregman-1grad}). Besides, our results find important implications in small-loss and gradient-variance problem-dependent regret (\pref{cor:FT-WT-Correct} and \pref{cor:FT-bregman}), stochastically extended adversarial (SEA) model (\pref{thm:SEA-correct} and \pref{thm:SEA-bregman}), and game theory (\pref{thm:game}). Furthermore, our results can be extended to the anytime setup (without knowing the time horizon $T$) in \pref{thm:anytime}.}}
  \label{fig:theorems}
\end{figure*}

\paragraph{Overview of Results.}
The theoretical contributions of this paper can be organized into two distinct methodological approaches.
As demonstrated in \pref{fig:theorems}, we introduce two principal algorithms: \Correct and \Bregman, both of which achieve gradient-variation universal regret bounds (see~\pref{thm:unigrad-correct} and \pref{thm:unigrad-bregman}).
Each method can be further extended to the one-gradient feedback setting (\pref{thm:unigrad-correct-1grad} and \pref{thm:unigrad-bregman-1grad}), requiring only a single gradient query per round.
Beyond these core regret guarantees, both methods support a range of important implications and applications.
These include problem-dependent bounds for small-loss and gradient-variance settings (\pref{cor:FT-WT-Correct} and \pref{cor:FT-bregman}), adaptive guarantees for the SEA model (\pref{thm:SEA-correct} and \pref{thm:SEA-bregman}).
\Correct is particularly well-suited for faster convergence in online games as it can preserve the its RVU property (\pref{thm:game}).
In contrast, \Bregman features a simpler structure and is more amenable to extension to the anytime setting, where the time horizon $T$ is unknown in advance (\pref{thm:anytime}).

\paragraph{Regret Bounds and Algorithmic Structures.}
While both methods achieve gradient-variation universal regret, they differ in the regret bound for convex functions.
\Correct provides an $\O(\sqrt{V_T \log V_T})$ regret bound for convex functions, while \Bregman can achieve the optimal $\O(\sqrt{V_T})$ bound for convex functions, matching the lower bound established in~\citet{COLT'12:VT}.
A concrete comparison of the regret bounds can be found in \pref{table:main}.
The most fundamental difference lies in  algorithmic structures due to their different methodologies.
As summarized in Table~\ref{table:unigrad-comparison}, \Correct employs a three-layer online ensemble with $N = \Theta(\log T)$ base learners, while \Bregman uses a more streamlined two-layer structure with the same number of base learners.
For both \Correct and \Bregman, they require $N$ gradient queries per round, and their one-gradient variants successfully reduce the number of gradient queries to $1$ per round.

\begin{table}[!t]
  \centering
  \caption{Comparison of the algorithmic structures of different methods. \Correct and \Bregman are two different methods to achieve the gradient-variation universal regret, and \Correctpp and \Bregmanpp are their one-gradient variants.}
  \label{table:unigrad-comparison}
  \vspace{2mm}
  \renewcommand*{\arraystretch}{1.25}
  \resizebox{0.98\textwidth}{!}{
  \begin{tabular}{@{}l|lll@{}}
  \hline
  
  \hline
  \textbf{Method} & \textbf{Meta Algorithm} & \textbf{Base Algorithm} & \textbf{Remark} \\
  \hline
  \textsf{UniGrad.Correct} & MoM (2-layer) & 
  \begin{tabular}[c]{@{}l@{}} \OOMD with $\{\nabla f_t(\x_{t,i})\}_{i=1}^N$ \end{tabular} &
  \begin{tabular}[c]{@{}l@{}}three layers;\\ $N = \Theta(\log T)$ base learners,\\ $N$ gradient queries per round\end{tabular} \\
  \hline 
  \textsf{UniGrad++.Correct} & MoM (2-layer) & 
  \begin{tabular}[c]{@{}l@{}} \OOMD with $\nabla f_t(\x_t)$ \end{tabular} &
  \begin{tabular}[c]{@{}l@{}}three layers;\\ $N = \Theta(\log T)$ base learners,\\ 1 gradient query per round\end{tabular} \\
  \hline 
  \textsf{UniGrad.Bregman} & Optimistic-Adapt-ML-Prod & 
  \begin{tabular}[c]{@{}l@{}} \OOMD with $\{\nabla f_t(\x_{t,i})\}_{i=1}^N$ \end{tabular} &
  \begin{tabular}[c]{@{}l@{}}two layers;\\ $N = \Theta(\log T)$ base learners,\\ $N$ gradient queries per round\end{tabular} \\
  \hline 
  \textsf{UniGrad++.Bregman} & Optimistic-Adapt-ML-Prod & 
  \begin{tabular}[c]{@{}l@{}} \OOMD with $\nabla f_t(\x_t)$ \\ \end{tabular} &
  \begin{tabular}[c]{@{}l@{}}two layers;\\ $N = \Theta(\log T)$ base learners,\\ 1 gradient query per round\end{tabular} \\
  \hline
  
  \hline
  \end{tabular}}
\end{table}

\paragraph{Technical Differences.}
The two methods differ significantly in their core technical approaches, particularly in how they convert the empirical gradient variation term $\bar{V}_T$ to the gradient variation term $V_T$ (see \pref{lemma:empirical-GV-stability} and \pref{lem:emp-VT-de-breg}).
\begin{itemize}[leftmargin=*]
    \item \textbf{UniGrad.Correct:} The most important technical feature of \Correct is the \emph{cancellation argument} based on the (positive and negative) stability term and the curvature-induced negative term.
    By carefully exploiting these negative terms together with the cascade correction scheme, \Correct attains the desired universal regret with a three-layer online ensemble structure.
    The development significantly advances the adaptivity of the online ensemble framework, providing a principled basis for analyzing the stability of multi‑layer online ensembles.
    \item \textbf{UniGrad.Bregman:} 
    The key innovation of \Bregman is to eliminate the meta level stability term when converting empirical gradient variation to the desired gradient variation (see \pref{lem:emp-VT-de-breg}).
    This is achieved by extracting the negative Bregman divergence arising from the linearization of the regret from the beginning.
    This mechanism greatly simplifies the algorithmic analysis and design and provides a useful tool for future research in adaptive online learning.
\end{itemize}

\begin{figure}[!t]
    \centering
        \subfigure[{Method of \citet{NeurIPS'23:universal}}]{ 
        \label{fig:left}
        \vspace{-2mm}
        \includegraphics[clip, trim=2mm 23mm 5mm 15mm,height=0.23\textwidth]{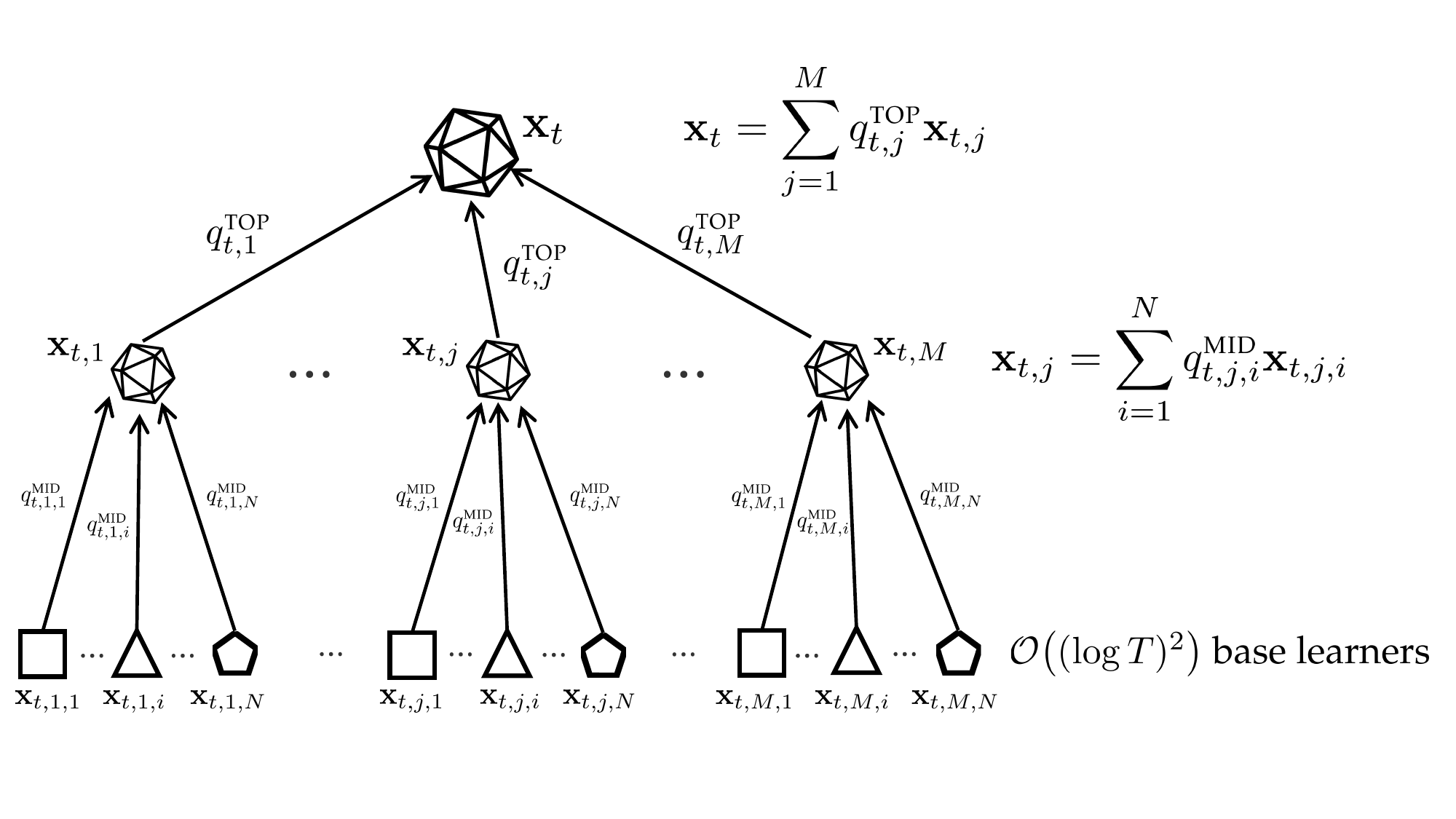}}
        \subfigure[{\Correct}]{ 
        \label{fig:right} \hspace{-4mm}
        \includegraphics[clip, trim=20mm 20mm 40mm 15mm,height=0.22\textwidth]{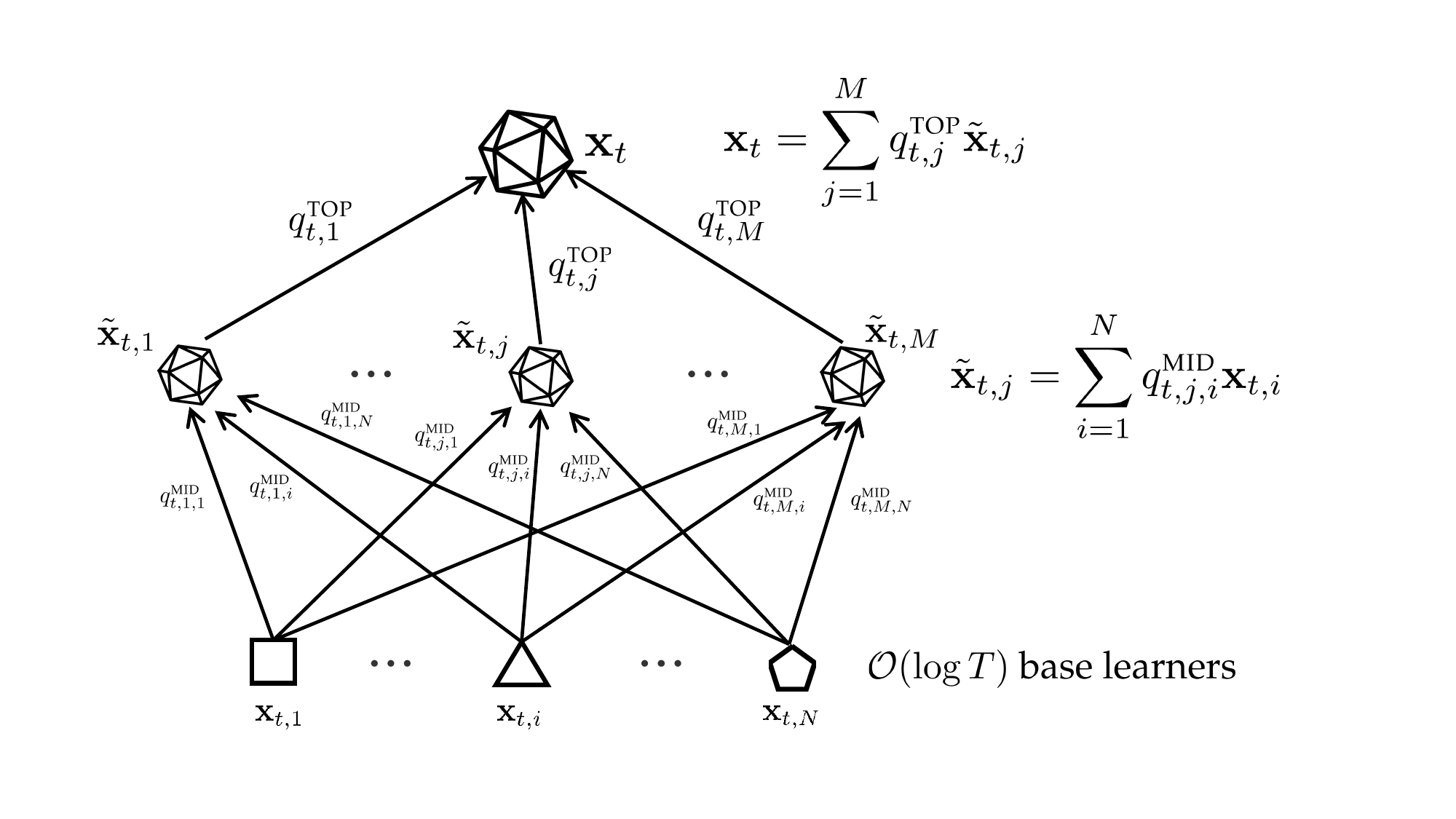}}
        \vspace{-2mm}
    \caption{Comparison of the three-layer online ensemble structures between the conference version~\citep{NeurIPS'23:universal} and \Correct. The key difference lies in how base learners are managed: \citet{NeurIPS'23:universal} maintain a separate group of base learners for each \MSMWCmid, whereas \Correct employs \emph{shared} base learners across all \MSMWCmid's, thereby reducing the total number of base learners from $\O((\log T)^2)$ to $\O(\log T)$.}
    \label{fig:comparison-correction}
\end{figure}

\subsection{Comparison of \textsf{UniGrad.Correct} with Conference Version}
\label{subsec:comparison-correct-conference}
The design of \Correct builds upon and significantly improves the correction-based method presented in our conference version~\citep{NeurIPS'23:universal}.
While both share the fundamental idea of using correction terms to handle positive stability terms, our current methods achieve superior theoretical guarantees and computational efficiency.

\paragraph{Conference Version.}
The algorithm of~\citet{NeurIPS'23:universal} employs a three-layer online ensemble, with the final output at each round computed as:
\begin{equation*}
  \x_t = \sumM q^\Top_{t,j} \x_{t,j},\quad \x_{t,j} = \sumN q^\Mid_{t,j,i} \x_{t,j,i}.
\end{equation*}
Here, $\q_t^\Top$ is the decision of \MSMWCtop, which connects with $M = \O(\log T)$ \MSMWCmid.
Similarly, $\q^\Mid_{t,j}$ denotes the decisions of \MSMWCmid, which further connects with $N = \O(\log T)$ base learners.
As a result, the algorithm requires maintaining $\O(MN) = \O((\log T)^2)$ base learners in total, as illustrated in \pref{fig:left}.
To handle the positive stability term $\norm{\x_t - \x_{t-1}}^2$, \citet{NeurIPS'23:universal} leveraged the following \emph{cascaded} stability decomposition:
\begin{gather}
    \|\x_t - \x_{t-1}\|^2 \lesssim \|\q_t^\Top - \q^\Top_{t-1}\|_1^2 + \sumM q^\Top_{t,j} \|\x_{t,j} - \x_{t-1,j}\|^2, \label{eq:sta-de-23-top}\\
    \|\x_{t,\js} - \x_{t-1,\js}\|^2 \les \|\q^\Mid_{t,\js} - \q^\Mid_{t-1,\js}\|_1^2 + \sumN q^\Mid_{t,\js,i} \|\x_{t,\js,i} - \x_{t-1,\js,i}\|^2. \label{eq:sta-de-23-mid}
\end{gather}
According to the second term in the right-hand side of~\pref{eq:sta-de-23-top}, the top-layer correction term is set as $q^\Top_{t,j} \|\x_{t,j} - \x_{t-1,j}\|^2$, which generates additional positive term $\|\x_{t,\js} - \x_{t-1,\js}\|^2$.
This term is further decomposed using~\pref{eq:sta-de-23-mid}, whose second term in the right-hand side motives the injection of the middle-layer correction term $q^\Mid_{t,\js,i} \|\x_{t,\js,i} - \x_{t-1,\js,i}\|^2$ to $\js$-th \MSMWCmid to ensure a property cancellation.
While this three-layer ensemble and correction scheme is intuitive and conceptually straightforward, it has a limitation: each \MSMWCmid in the middle layer requires maintaining its own set of base learners, leading to a total complexity of $\O(\log T) \times \O(\log T) = \O((\log T)^2)$ base learners.
This design introduces redundancy and unnecessary complexity.

\paragraph{Improved Version in Current Paper.}
\Correct reduces the number of base learners to $\O(\log T)$ by carefully restructuring the framework, though the three-layer structure and cascade corrections are still necessary.
As shown  in \pref{fig:right}, the proposed \Correct algorithm  produces the final output at each round as:
\begin{equation*}
    \x_t = \sumN p_{t,i} \x_{t,i}, \quad \p_t = \sumM q_{t,j}^\Top \q_{t,j}^\Mid,
\end{equation*}
where $\{\x_{t,i}\}_{i=1}^N$ are the local decisions returned by the base learners, with $N = \O(\log T)$.
The meta combination weight $p_{t,i}$ is calculated based on a two-layer \msmwc.
In fact, this update can be equivalently understood as follows:
\begin{equation}
    \label{eq:new-understd}
    \x_t = \sumN p_{t,i} \x_{t,i} = \sumN \Big(\sumM q_{t,j}^\Top q_{t,j,i}^\Mid\Big) \x_{t,i} = \sumM q_{t,j}^\Top \Big(\sumN q_{t,j,i}^\Mid \x_{t,i}\Big) \define \sumM q_{t,j}^\Top \xt_{t,j},
\end{equation}
where the last equality defines new hidden aggregation nodes $\xt_{t,j} \define \sumN q_{t,j,i}^\Mid \x_{t,i}$.
Compared with the aggregation of $\x_t = \sumM q^\Top_{t,j} \x_{t,j}$ and $\x_{t,j} = \sumN q^\Mid_{t,j,i} \x_{t,j,i}$ in the conference version~\citep{NeurIPS'23:universal}, it can be seen that hidden aggregations $\{\xt_{t,j}\}_{j=1}^M$ use the \emph{shared} base learner group, namely, $\x_{t,j,i} = \x_{t,i}$ for any $j \in [M]$.
This ``base learner sharing'' design is the key improvement over the conference version: by requiring the middle-layer meta algorithm to use a shared set of base learners across all \MSMWCmid, \Correct reduces the number of base learners from $\O((\log T)^2)$ to $\O(\log T)$, significantly enhancing efficiency.

The new understanding in \pref{eq:new-understd} can also benefit and simplify the stability analysis for \Correct in a similar cascade way to the conference version~\citep{NeurIPS'23:universal}:
\begin{gather*}
  \|\x_t - \x_{t-1}\|^2 \lesssim \|\q_t^\Top - \q^\Top_{t-1}\|_1^2 + \sumM q^\Top_{t,j} \|\xt_{t,j} - \xt_{t-1,j}\|^2, \\
  \|\xt_{t,\js} - \xt_{t-1,\js}\|^2 \les \|\q^\Mid_{t,\js} - \q^\Mid_{t-1,\js}\|_1^2 + \sumN q^\Mid_{t,\js,i} \|\x_{t,i} - \x_{t-1,i}\|^2.
\end{gather*}
This decomposition not only provides a conceptually more straightforward proof for \pref{lem:decompose-three-layer}, but also establishes a principled framework for analyzing even deeper online ensemble structures.
We believe these insights could be of interest to the community.

\section{Experiments}
\label{sec:experiments}
This section provides empirical studies to validate the effectiveness of our algorithms. Through empirical evaluations, we aim to answer the following three questions:
\begin{itemize}[leftmargin=*]
    \item \textbf{Universality}: Can our methods automatically adapt to the unknown curvature of online functions and achieve comparable performance with the optimal algorithm specifically designed for each problem instance?
    \item \textbf{Adaptivity}: Can our methods adapt to the gradient variation $V_T$ and achieve better performance than the methods that are not fully gradient-variation adaptive, e.g., the method of \citet{ICML'22:universal}, when $V_T$ is small?
    \item \textbf{Efficiency}: Can our one-gradient improvements \textsf{UniGrad++} (\textsf{Correct/Bregman}) achieve comparable performance to their vanilla versions \textsf{UniGrad} (\textsf{Correct/Bregman}) while with significantly reduced gradient query cost?
\end{itemize}

\begin{figure}[!t]
    \centering
    \subfigure[\texttt{ijcnn1} - convex]{
        \label{fig:universality-cvx-ijcnn}
        \includegraphics[clip, trim=2mm 3mm 2mm 2mm, height=2.80cm]{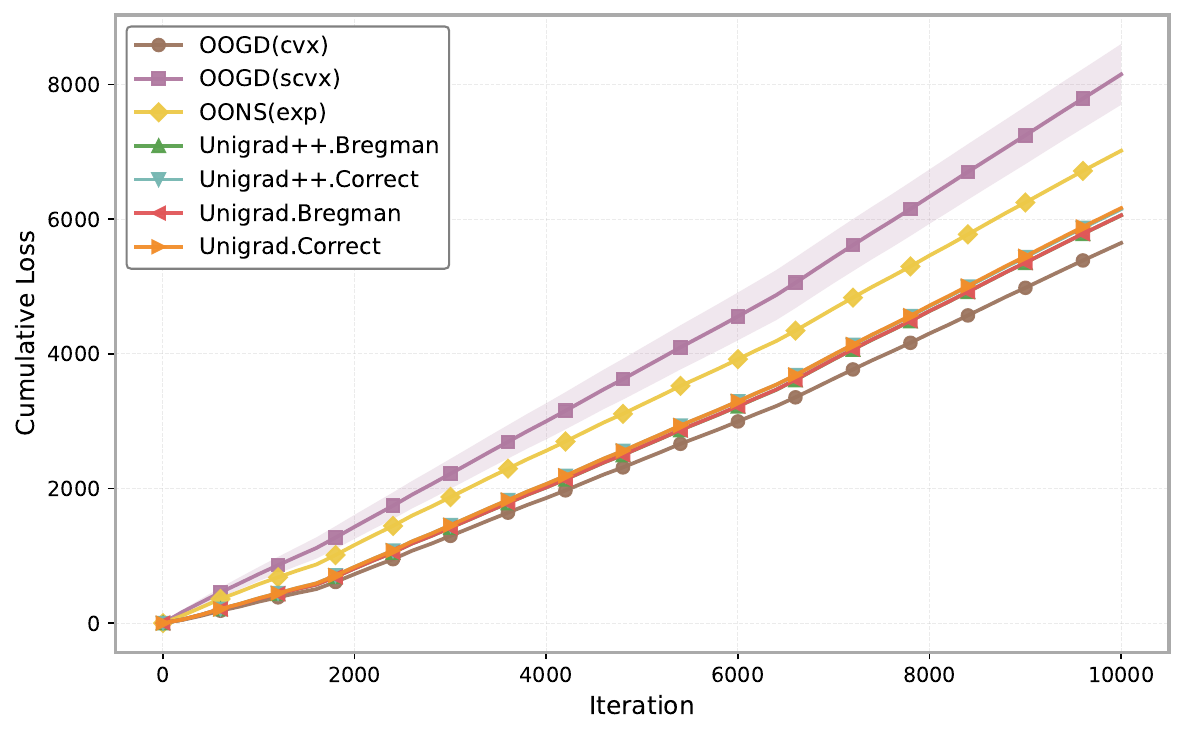}}
    \hfill
    \subfigure[\texttt{ijcnn1} - exp-concave]{
        \label{fig:universality-exp-ijcnn}
        \includegraphics[clip, trim=2mm 3mm 2mm 2mm, height=2.80cm]{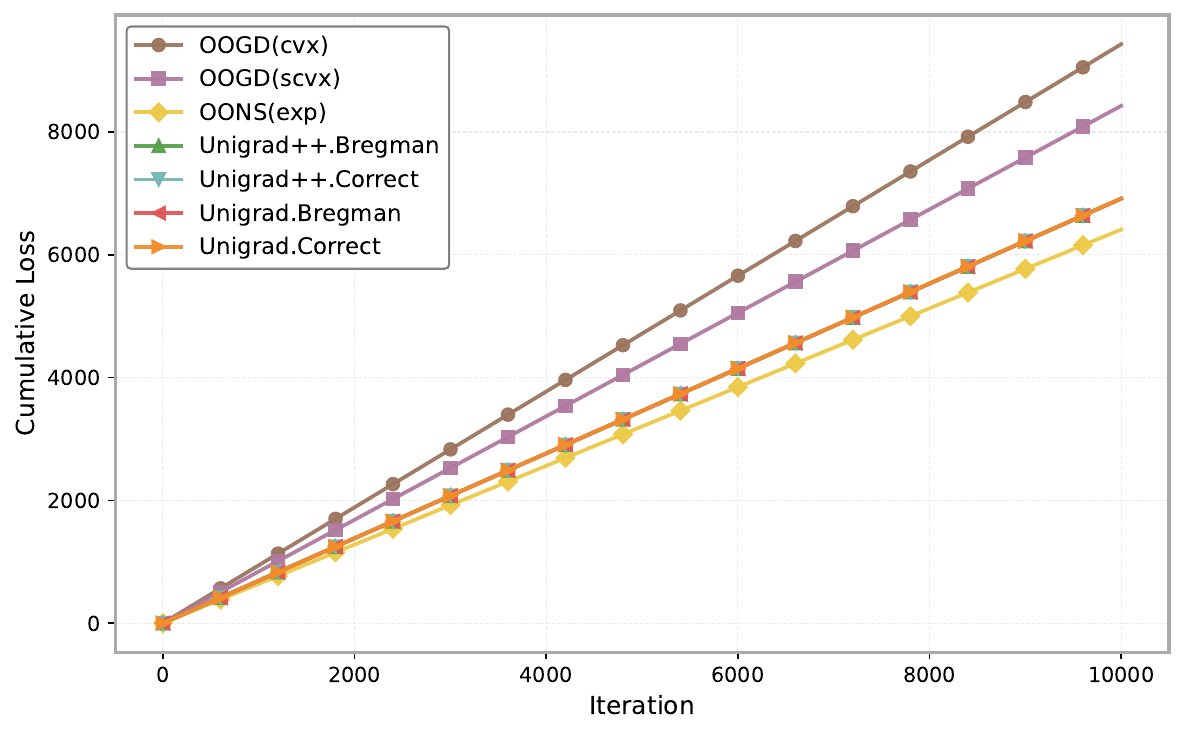}}
    \hfill
    \subfigure[\texttt{ijcnn1} - str-convex]{
        \label{fig:universality-scvx-ijcnn}
        \includegraphics[clip, trim=2mm 3mm 2mm 2mm, height=2.80cm]{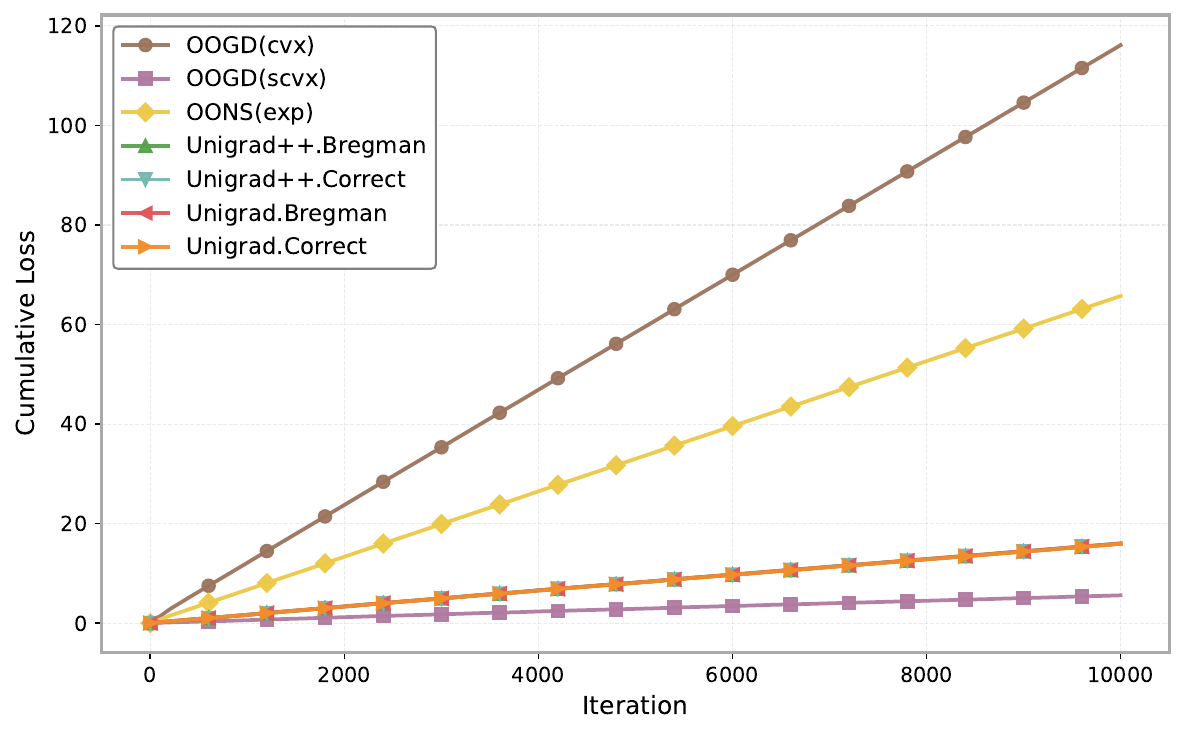}} \\
    \vspace{0.1cm}
    \subfigure[\texttt{svmguide1} - convex]{
        \label{fig:universality-cvx-poker}
        \includegraphics[clip, trim=2mm 3mm 2mm 2mm, height=2.80cm]{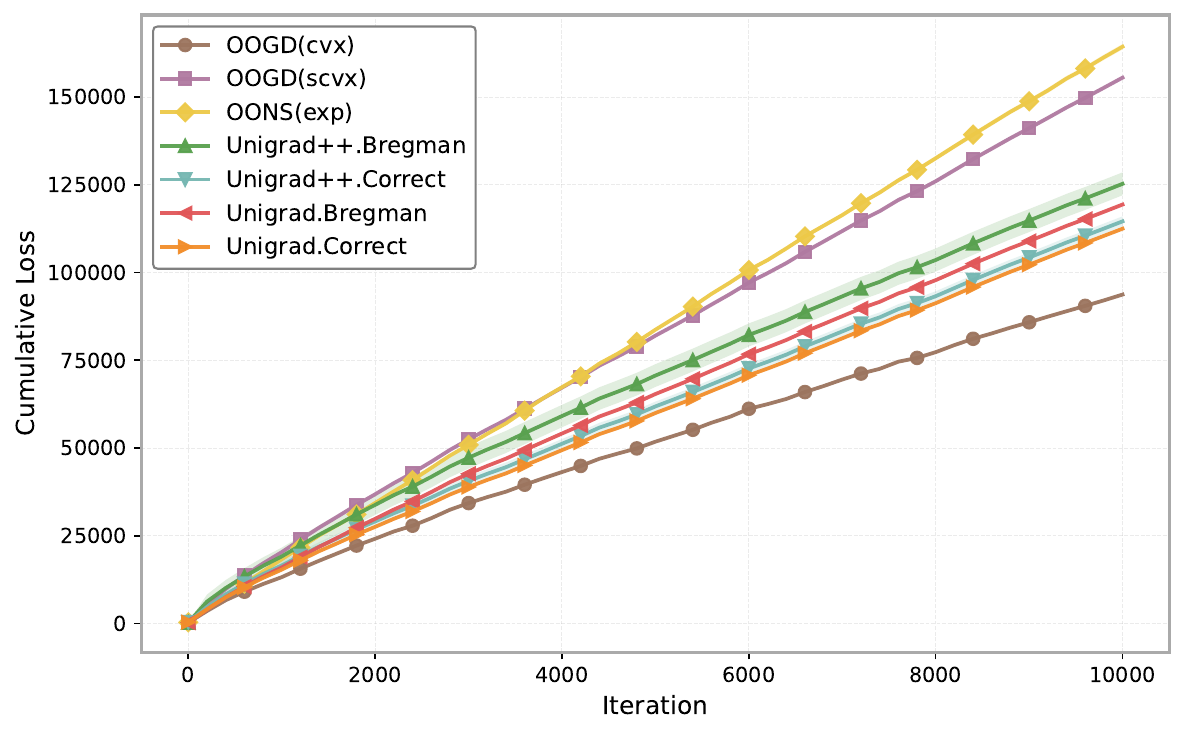}}
    \hfill
    \subfigure[\texttt{svmguide1} - exp-concave]{
        \label{fig:universality-exp-poker}
        \includegraphics[clip, trim=2mm 3mm 2mm 2mm, height=2.80cm]{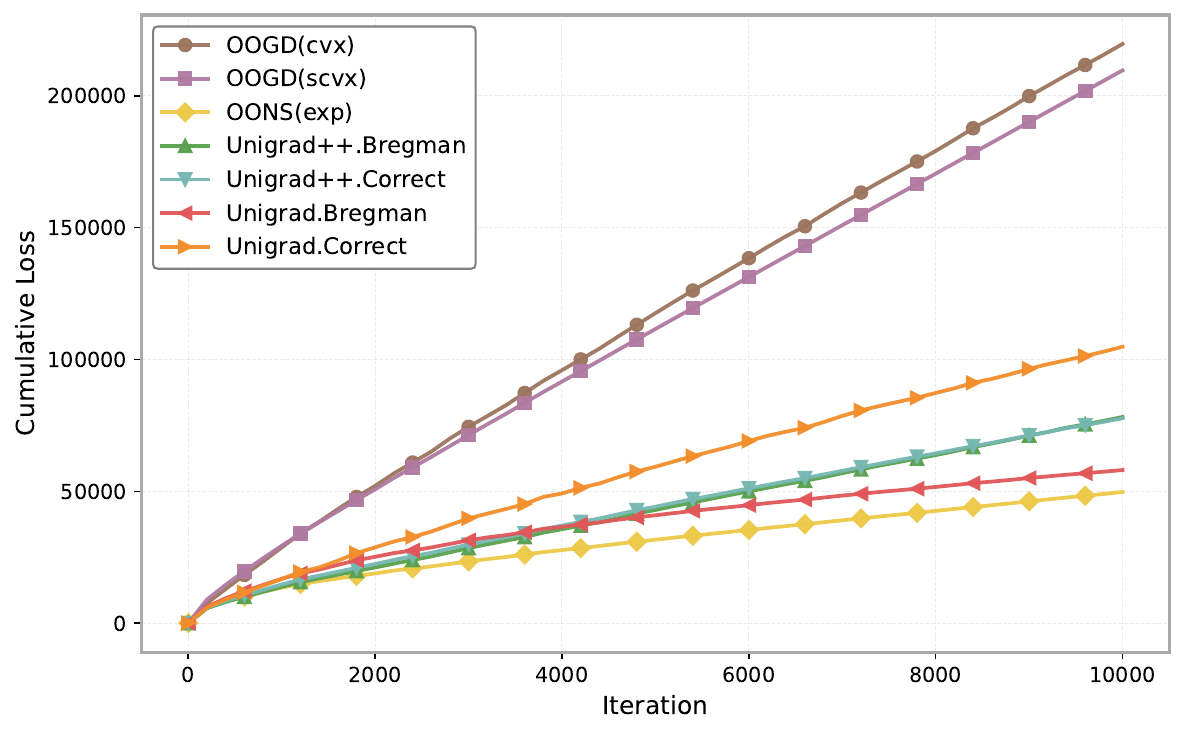}}
    \hfill
    \subfigure[\texttt{svmguide1} - str-convex]{
        \label{fig:universality-scvx-poker}
        \includegraphics[clip, trim=2mm 3mm 2mm 2mm, height=2.80cm]{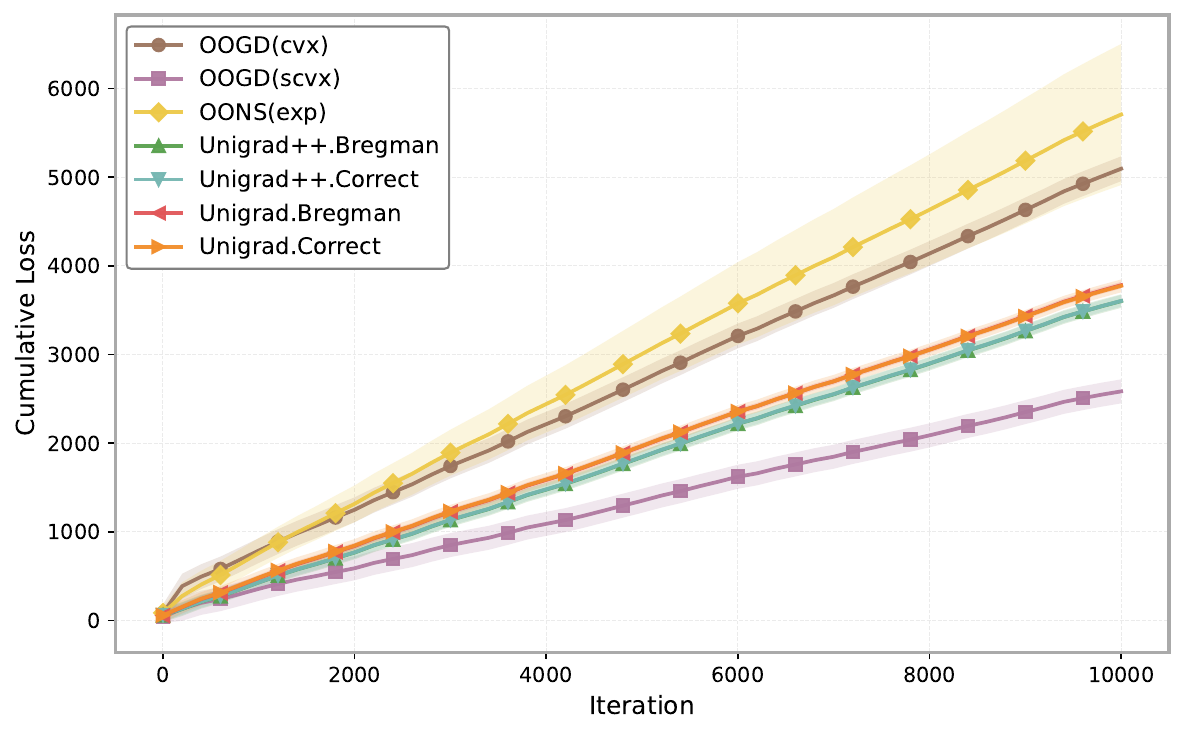}} \\
    \vspace{0.1cm}
    \subfigure[\texttt{skin}\_\texttt{nonskin} - convex]{
        \label{fig:universality-cvx-shuttle}
        \includegraphics[clip, trim=2mm 3mm 2mm 2mm, height=2.80cm]{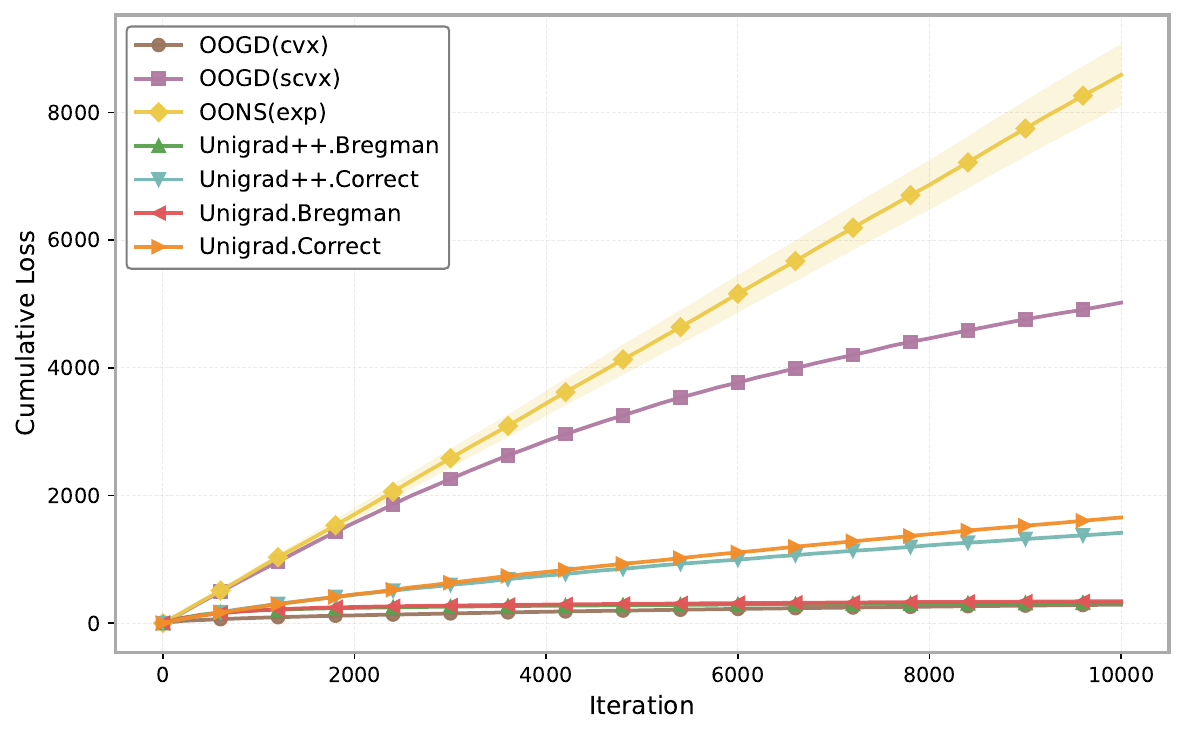}}
    \hfill
    \subfigure[\texttt{skin}\_\texttt{nonskin} - exp-concave]{
        \label{fig:universality-exp-shuttle}
        \includegraphics[clip, trim=2mm 3mm 2mm 2mm, height=2.80cm]{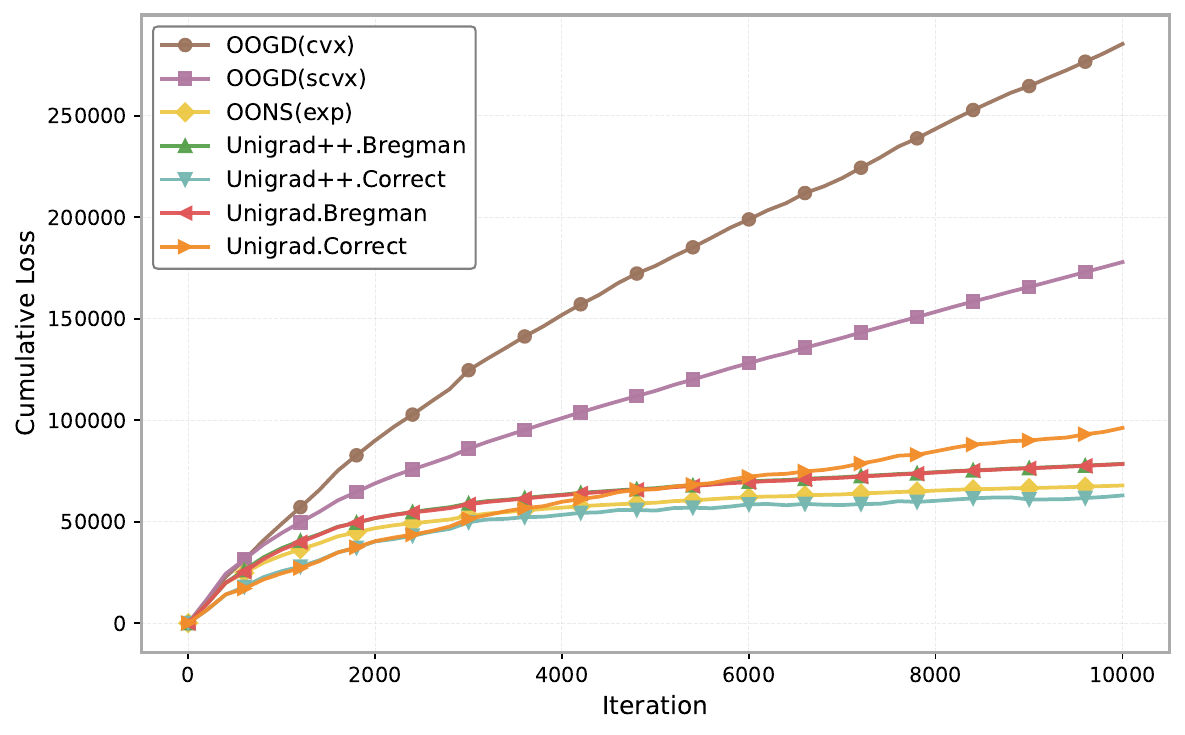}}
    \hfill
    \subfigure[\texttt{skin}\_\texttt{nonskin} - str-convex]{
        \label{fig:universality-scvx-shuttle}
        \includegraphics[clip, trim=2mm 3mm 2mm 2mm, height=2.80cm]{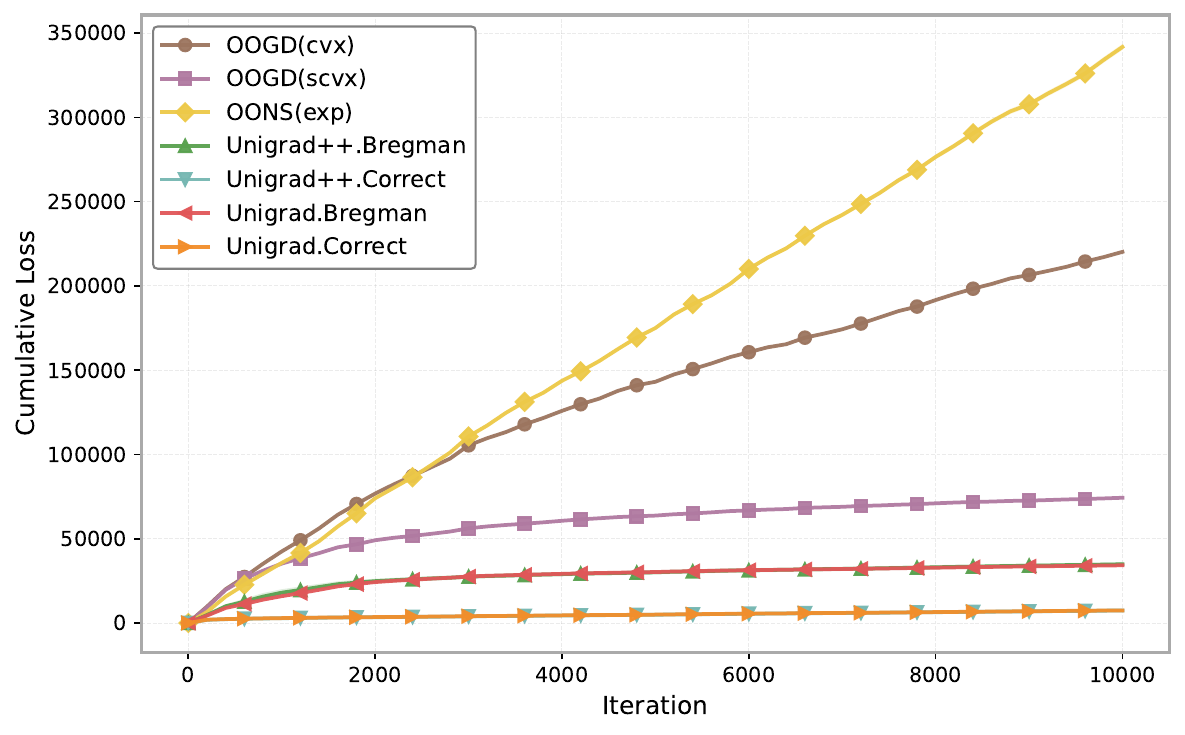}}
    \caption{\small{\textbf{Universality}: comparisons on three problem classes---convex, exp-concave, and strongly convex---across three datasets (\texttt{ijcnn1}, \texttt{svmguide1}, \texttt{skin}\_\texttt{nonskin}). Rows correspond to datasets, columns correspond to problem classes. Our methods are evaluated against the optimal algorithm specifically designed for each class, showing comparable regret performance.}}
    \label{fig:universality}
\end{figure}

\paragraph{Contenders and Configurations.}
To validate the universality, we compare our methods with the optimal algorithm specifically designed for each problem instance, as specified in \pref{subsec:framework}. For the adaptivity validation, we compare our methods with the USC algorithm~\Citep{ICML'22:universal}, which enjoys the universal regret of $\O(\sqrt{T})$ for convex functions, $\O(\frac{d}{\alpha} \log V_T)$ for exp-concave functions, and $\O(\frac{1}{\lambda} \log V_T)$ for strongly convex functions, respectively. Finally, for efficiency validation, we compare the one-gradient methods (\Correctpp and \Bregmanpp) with their vanilla versions (\Correct and \Bregman), which require $\O(\log T)$ gradient queries per round.

In all experiments, we set the total time horizon to $T = 10{,}000$ and choose the decision domain $\mathcal{X}$ as the unit ball.
All hyper-parameters are set to be theoretically optimal.

\begin{itemize}[leftmargin=*]
    \item All algorithm hyper-parameters are set to be theoretically optimal.
    To validate the universality of our approach, we conduct experiments on three datasets from the \texttt{LIBSVM} repository~\Citep{CC01a}: \texttt{ijcnn1}, \texttt{svmguide1}, and \texttt{skin}\_\texttt{nonskin}. All of these are binary classification datasets, and we transform the labels to $\{-1,1\}$.
    At each round $t\in[T]$, we randomly sample a data point $(\mathbf{a}_t, y_t)$ from the chosen dataset to construct the loss function $f_t(\mathbf{x})$.
    Specifically, for the convex setting, we choose the linear loss function $f_t(\x) =\max(0,1-y_t\cdot \abf_t^\top\x) $.
    For the exp-concave setting, we use the logistic loss function $f_t(\x) = \log(1 + \exp(-y_t\cdot\abf_t^\top \x))$.
    For the strongly convex setting, we choose the loss function $f_t(\x) = \max(0,1-y_t\cdot \abf_t^\top\x)+\frac{1}{2}\|\x\|^2_2$.
    \item To validate the adaptivity of our approach, we first compare our method with \textsf{USC} \citep{ICML'22:universal} on the \texttt{ijcnn1} dataset, where the gradient variation satisfies $V_T = \O(T)$.
    We then construct an online function sequence with $V_T = \O(1)$ and perform the same comparison on this sequence.
    Specifically, we choose the loss functions as $f_t(\x)= \abf_{i_t}^\top\x+b_{i_t}$, where $\mathbf{a}_0 = [0.2, 0.2]^\top$, $b_0=0$, and $i_t = \lfloor 10t /T\rfloor$. 
    The parameters evolve gradually as $\abf_i=\abf_{i-1}+0.1 \cdot \epsilon_i$ and $b_i=b_{i-1}+0.1 \cdot \xi_i$ for $i\in[10]$, with noise $\epsilon_i\sim \mathcal{N}(\mathbf{0},I_2)$ and $\xi_i\sim\mathcal{N}(0,1)$, such that the total gradient variation of the online functions sequence can be treated as a constant. 
    \item For efficiency evaluation, we compare the total running time of the one-gradient variants \textsf{UniGrad++}.(\textsf{Correct/Bregman})  with their vanilla counterparts \textsf{UniGrad}.(\textsf{Correct/Bregman}) across all three datasets (\texttt{ijcnn1}, \texttt{svmguide1}, and \texttt{skin}\_\texttt{nonskin}) in the exp-concave case. 
\end{itemize}
We report the average cumulative losses with standard deviations of 5 independent runs to obtain convincing results.
Only the randomness of the initialization is preserved.

\paragraph{Numeric Results.}
\pref{fig:universality} shows the universality comparison results. Our methods are compared with the optimal algorithm specifically designed for each problem instance, validating the universality of our methods, and show comparable performance across different problem classes and datasets.

\pref{fig:adaptivity} presents the adaptivity comparison results. Our methods outperform the USC algorithm~\citep{ICML'22:universal} when the gradient variation $V_T$ is small, e.g., $V_T = \O(1) $ in \pref{fig:adaptivity-vt-o1}, and show comparable performance when $V_T = \O(T)$ in \pref{fig:adaptivity-vt-ot}.

\pref{fig:time-comparison} shows the efficiency comparison results. Our one-gradient variants are more efficient than their vanilla multi-gradient versions in time complexity, while maintaining comparable performance, as shown in \pref{fig:universality} and \pref{fig:adaptivity}.

\begin{figure}[!t]
    \centering
    \subfigure[$V_T = \O(1)$]{
    \label{fig:adaptivity-vt-o1}
    \includegraphics[clip, trim=2mm 3mm 2mm 2mm, height=3.90cm]{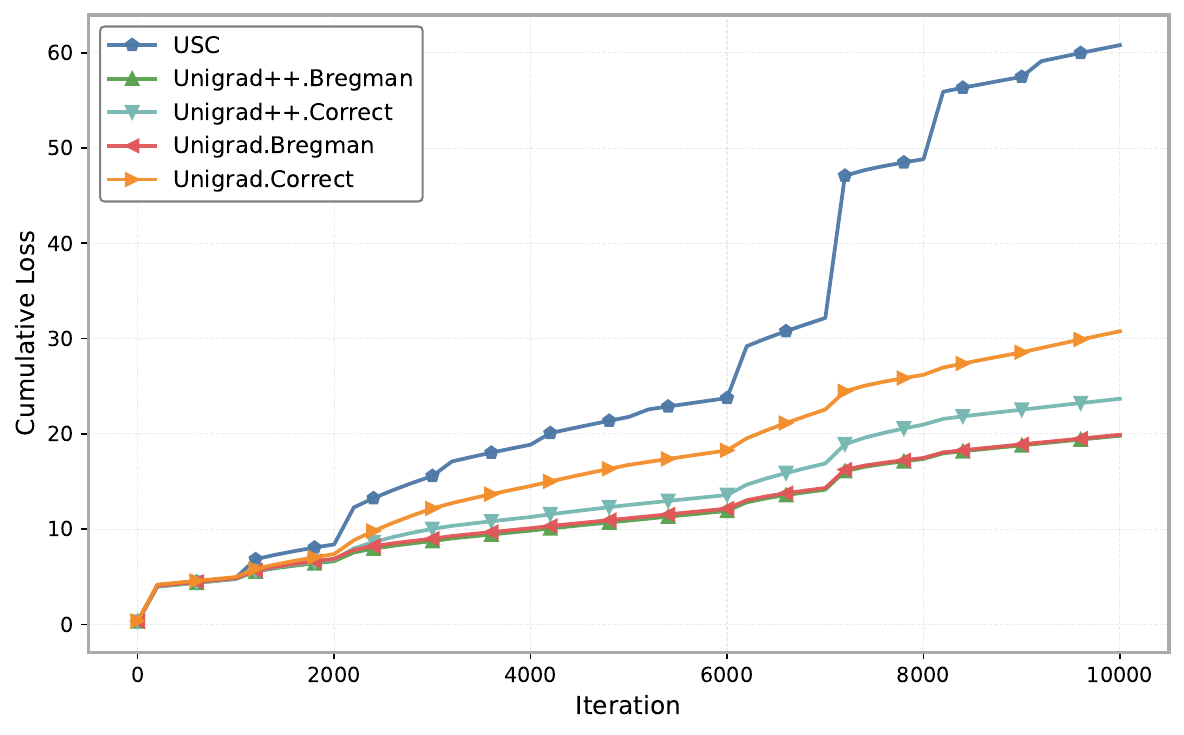}}
    \subfigure[$V_T = \O(T)$]{
    \label{fig:adaptivity-vt-ot}
    \includegraphics[clip, trim=2mm 3mm 2mm 2mm, height=3.90cm]{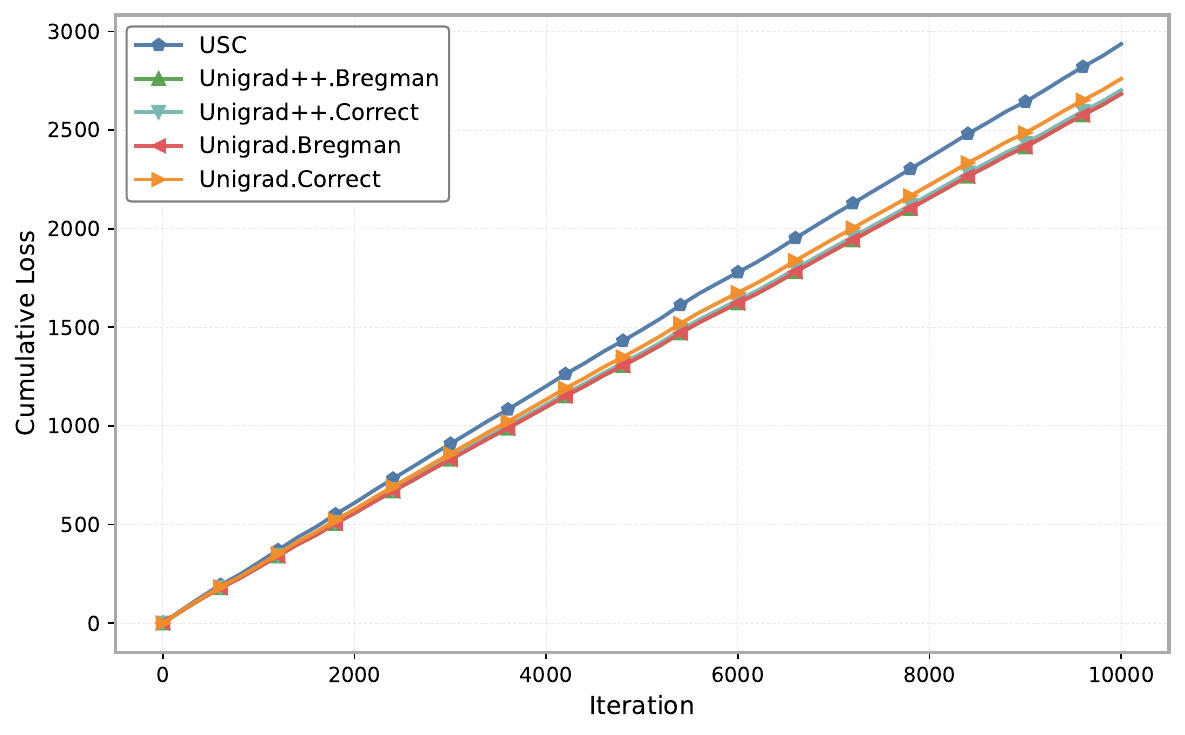}}
    \caption{\small{\textbf{Adaptivity}: comparisons on the adaptivity of our methods against \textsf{USC} of \citet{ICML'22:universal}. Our methods outperform \textsf{USC} when the gradient variation $V_T$ is small, e.g., $V_T = \O(1) $ in \pref{fig:adaptivity-vt-o1}, and show comparable performance when $V_T = \O(T)$ in \pref{fig:adaptivity-vt-ot}.}}
    \label{fig:adaptivity}
\end{figure}

\begin{figure}[!t]
    \centering
    \hspace*{15mm}\includegraphics[clip, trim=2mm 3mm 2mm 2mm, height=3.90cm]{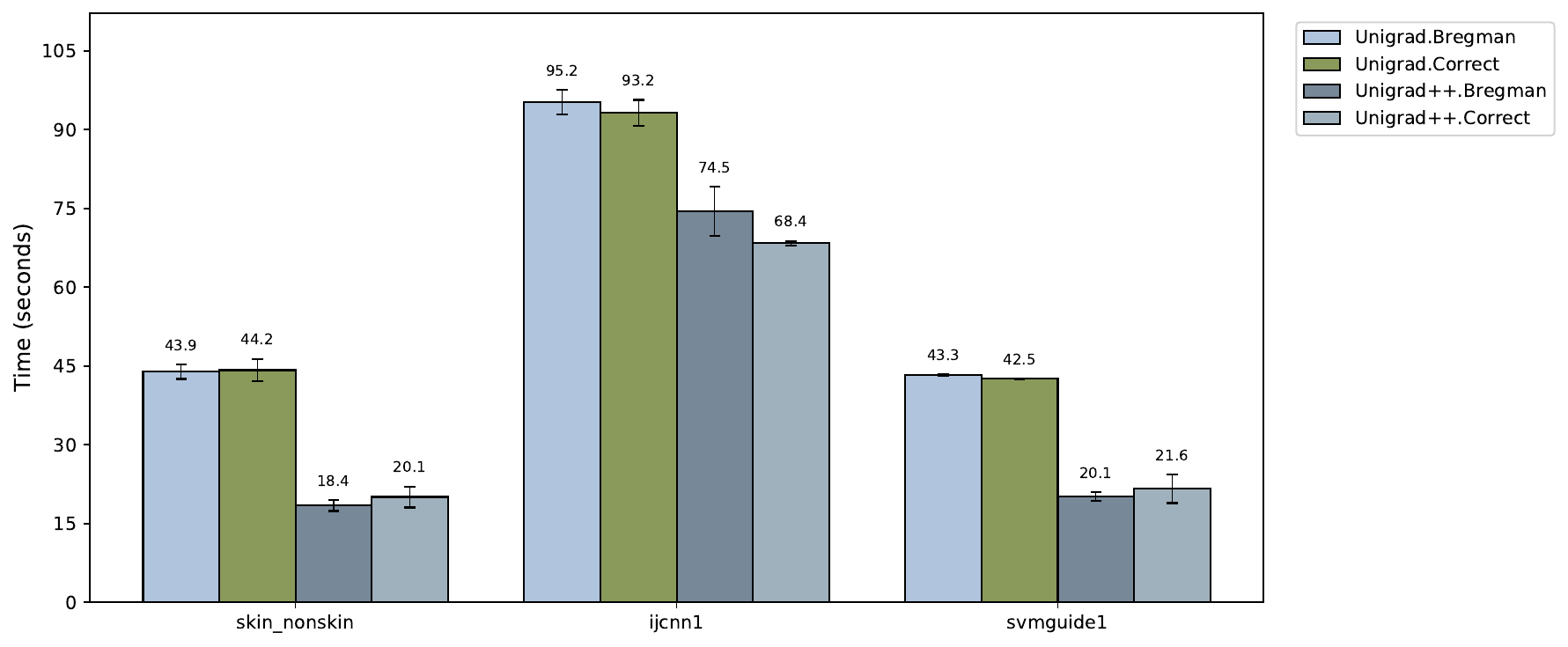}
    \caption{\small{\textbf{Efficiency}: comparisons on the efficiency of our one-gradient improvements \textsf{UniGrad++} (\textsf{Correct/Bregman}) against their vanilla versions \textsf{UniGrad} (\textsf{Correct/Bregman}). Our one-gradient improvements are more efficient than their vanilla versions in terms of the time complexity.}}
    \label{fig:time-comparison}
\end{figure}

\section{Conclusion}
\label{sec:conclusion}
In this paper, we addressed the fundamental challenge of achieving both universality and adaptivity in online learning by introducing \UniGrad, a new approach that obtains gradient-variation universal regret guarantees.
We proposed two distinct methods, \Correct and \Bregman, each with its own technical innovations.
UniGrad.Correct employs a three-layer online ensemble with cascaded correction terms, achieving regret bounds of $\O(\tfrac{1}{\lambda} \log V_T)$ for $\lambda$-strongly convex functions, $\O(\tfrac{d}{\alpha} \log V_T)$ for $\alpha$-exp-concave functions, and $\O(\sqrt{V_T \log V_T})$ for convex functions, while preserving the RVU property crucial for fast convergence in online games.
In contrast, UniGrad.Bregman leverages a novel Bregman divergence analysis to achieve the same bounds for strongly convex and exp-concave functions, but improves upon the convex case with the optimal $\O(\sqrt{V_T})$ regret.
Both methods maintain $\O(\log T)$ base learners and require $\O(\log T)$ gradient queries per round.

Building on these results, we further developed \UniGradpp, which preserves the same regret guarantees while reducing the gradient query cost to just one per round via a surrogate optimization technique.
We also extended our method to an anytime variant that removes the need to know the horizon $T$ in advance, using a dynamic online ensemble framework that adjusts the number of base learners based on monitoring metrics.
Importantly, our results lead to broader implications and applications, including optimal small-loss and gradient-variance bounds, novel guarantees for the stochastically extended adversarial model, and faster convergence in online games.

There are several interesting future directions worthy investigating.
The first is to explore whether the computational overhead can be further reduced by requiring only $1$ projection per round for gradient-variation universal regret~\citep{NeurIPS'22:efficient, NeurIPS'24:universalProjection}.
A second direction is to extend our results to unconstrained domains in order to achieve parameter-free online learning~\citep{COLT'18:black-box-reduction}, thereby broadening its applicability.
Finally, current universal online learning methods assume a homogeneous setting where all online functions share the same curvature class (convex, $\lambda$-strongly convex, or $\alpha$-exp-concave).
A more challenging and realistic goal is to handle heterogeneous environments where the curvature may vary over time. One possible starting point is the recently proposed contaminated OCO setting~\citep{TMLR'24:contaminated-OCO}, which assumes that the objective functions are mostly uniform but may be contaminated by a small fraction of rounds -- up to some unknown $k$ -- where the curvature class differs.

\newpage

\bibliography{online}

\newpage
\clearpage 
\setcounter{tocdepth}{2}
\tableofcontents 
\clearpage 

\appendix
\section{Omitted Proofs for Section~\ref{sec:method1-correct}}
\label{appendix:method1-proofs}
In this section, we provide the proofs for the results in Section~\ref{sec:method1-correct}, including \pref{lem:MsMwC-refine}, \pref{lem:universal-optimism} and \pref{lem:two-layer-MsMwC}.
For simplicity, we introduce the following notations denoting the stability of the final and intermediate decisions of the algorithm.
Specifically, for any $j \in [M], i \in [N]$, we define
\begin{equation}
  \label{eq:short}
\begin{gathered}
  S^\x_T \define \sumTT \|\x_t - \x_{t-1}\|^2,\ S^\x_{T,i} \define \sumTT \|\x_{t,i} - \x_{t-1,i}\|^2,\\
  S_T^\Top \define \sumTT \|\q_t^\Top - \q_{t-1}^\Top\|_1^2,\quad \text{and}\quad S_{T,j}^\Mid \define \sumTT \|\q_{t,j}^\Mid - \q_{t-1,j}^\Mid\|_1^2. 
\end{gathered}
\end{equation}

\subsection{Proof of Lemma~\ref{lem:MsMwC-refine}}
\label{app:msmwc}
In this part, we analyze the negative stability terms in the \msmwc algorithm~\citep{COLT'21:impossible-tuning}.
For self-containedness, we restate its update rule in the following general form:
\begin{equation*}
  \p_t = \argmin_{\p \in \Delta_d}\ \bbr{\inner{\m_t}{\p} + \D_{\psi_t}(\p, \pbh_t)},\quad \pbh_{t+1} = \argmin_{\p \in \Delta_d}\ \bbr{\inner{\ellb_t + \b_t}{\p} + \D_{\psi_t}(\p, \pbh_t)},
\end{equation*}
where $\Delta_d$ denotes a $d$-dimension simplex, $\psi_t(\p) = \sum_{i=1}^d \epsilon_{t,i}^{-1} p_i \log p_i$ is the weighted negative entropy regularizer with time-coordinate-varying learning rate $\epsilon_{t,i}$, and the bias term $a_{t,i} = 16 \epsilon_{t,i} (\ell_{t,i} - m_{t,i})^2$.
Below, we give a detailed proof of \pref{lem:MsMwC-refine}, following a similar logic flow as Lemma 1 of \citet{COLT'21:impossible-tuning}, while illustrating the negative stability terms.
Moreover, for generality, we investigate a more general setting of an arbitrary comparator $\uu \in \Delta_d$ and changing step sizes $\epsilon_{t,i}$.
This was done hoping that the negative stability term analysis would be comprehensive enough for readers interested solely in the MsMwC algorithm. 

\begin{proof}[of \pref{lem:MsMwC-refine}]
  To begin with, the regret with correction can be analyzed as follows:
  \begin{align*}
      & \sumT \inner{\ellb_t + \b_t}{\p_t - \uu} \le \sumT \sbr{\D_{\psi_t}(\uu,\pbh_t) - \D_{\psi_t}(\uu,\pbh_{t+1})} + \sumT \inner{\ellb_t + \b_t - \m_t}{\p_t - \pbh_{t+1}}\\
      & \makebox[3.5cm]{} - \sumT \sbr{\D_{\psi_t}(\pbh_{t+1}, \p_t) + \D_{\psi_t}(\p_t, \pbh_t)}\\
      & \le \underbrace{\sumT \sbr{\D_{\psi_t}(\uu,\pbh_t) - \D_{\psi_t}(\uu,\pbh_{t+1})}}_{\term{a}} + \underbrace{\sumT \sbr{\inner{\ellb_t + \b_t - \m_t}{\p_t - \pbh_{t+1}} - \frac{1}{2} \D_{\psi_t}(\pbh_{t+1}, \p_t)}}_{\term{b}}\\
      & \makebox[3.5cm]{} - \frac{1}{2} \underbrace{\sumT \sbr{\D_{\psi_t}(\pbh_{t+1}, \p_t) + \D_{\psi_t}(\p_t, \pbh_t)}}_{\term{c}},
  \end{align*}
  where the first step follows the standard analysis of \OOMD, e.g., Theorem 1 of~\citet{JMLR'24:Sword++}.
  One difference of our analysis from the previous one lies in the second step, where previous work dropped the $\D_{\psi_t}(\p_t, \pbh_t)$ term while we keep it for negative terms.

  To begin with, we require an upper bound of $\epsilon_{t,i} \le \frac{1}{32}$ for the step sizes.
  To give a lower bound for $\term{c}$, we notice that for any $\a, \b \in \Delta_d$,
  \begin{align}
      \D_{\psi_t}(\a, \b) = {} & \sum_{i=1}^d \frac{1}{\epsilon_{t,i}} \sbr{a_i \log \frac{a_i}{b_i} -a_i + b_i} = \sum_{i=1}^d \frac{b_i}{\epsilon_{t,i}} \sbr{\frac{a_i}{b_i} \log \frac{a_i}{b_i} - \frac{a_i}{b_i} + 1}\notag\\
      \ge {} & \min_{t, i} \frac{1}{\epsilon_{t,i}}  \sum_{i=1}^d \sbr{a_i \log \frac{a_i}{b_i} -a_i + b_i} \ge 32 \KL(\a, \b), \label{eq:lemma2 eq1}
  \end{align}
  where the first inequality is due to $x\log x - x + 1 \ge 0$ for all $x > 0$ and the last step is by $\epsilon_{t,i} \le \frac{1}{32}$.
  Thus, we have 
  \begin{align*}
      \term{c} \ge {} & 32 \sumT (\KL(\pbh_{t+1}, \p_t) + \KL(\p_t, \pbh_t)) \ge \frac{32}{2\log 2} \sumT \sbr{\|\pbh_{t+1} - \p_t\|_1^2 + \|\p_t - \pbh_t\|_1^2}\\
      \ge {} & 16 \sumTT \sbr{\|\pbh_t - \p_{t-1}\|_1^2 + \|\p_t - \pbh_t\|_1^2} \ge 8 \sumTT \|\p_t - \p_{t-1}\|_1^2,
  \end{align*}
  where the first step is from the above derivation, the second step is due to the Pinsker's inequality~\citep{Pinsker}: $\KL(\a, \b) \ge \frac{1}{2 \log 2} \|\a - \b\|_1^2$ for any $\a, \b \in \Delta_d$.

  For $\term{b}$, the proof is similar to the previous work, where only some constants are modified.
  For self-containedness, we give the analysis below.
  Treating $\pbh_{t+1}$ as a free variable and defining
  \begin{equation*}
    \p^\star \in \argmax_{\p} \inner{\ellb_t + \b_t - \m_t}{\p_t - \p} - \frac{1}{2} \D_{\psi_t}(\p, \p_t),
  \end{equation*}
  by the optimality of $\p^\star$, we have 
  \begin{equation*}
    \ellb_t + \b_t - \m_t = \frac{1}{2}(\nabla \psi_t(\p_t) - \nabla \psi_t(\p^\star)).
  \end{equation*}
  Since $[\nabla \psi_t(\p)]_i = \frac{1}{\epsilon_{t,i}} (\log p_i + 1)$, it holds that
  \begin{equation*}
    \ell_{t,i} - m_{t,i} + b_{t,i} = \frac{1}{2 \epsilon_{t,i}} \log \frac{p_{t,i}}{p^\star_i} \Leftrightarrow p_i^\star = p_{t,i} \exp(-2 \epsilon_{t,i} (\ell_{t,i} - m_{t,i} + b_{t,i})).
  \end{equation*}
  Therefore we have 
  \begin{align*}
    & \inner{\ellb_t + \b_t - \m_t}{\p_t - \pbh_{t+1}} - \frac{1}{2} \D_{\psi_t}(\pbh_{t+1}, \p_t) \le \inner{\ellb_t + \b_t - \m_t}{\p_t - \p^\star} - \frac{1}{2} \D_{\psi_t}(\p^\star, \p_t)\\
    = {} & \frac{1}{2} \inner{\nabla \psi_t(\p_t) - \nabla \psi_t(\p^\star)}{\p_t - \p^\star} - \frac{1}{2} \D_{\psi_t}(\p^\star, \p_t) = \frac{1}{2} \D_{\psi_t}(\p_t, \p^\star) \tag*{(by definition)}\\
    = {} & \frac{1}{2} \sum_{i=1}^d \frac{1}{\epsilon_{t,i}} \sbr{p_{t,i} \log \frac{p_{t,i}}{p^\star_i} - p_{t,i} + p^\star_i}\\
    = {} & \frac{1}{2} \sum_{i=1}^d \frac{p_{t,i}}{\epsilon_{t,i}} \big( 2 \epsilon_{t,i} (\ell_{t,i} - m_{t,i} + b_{t,i}) - 1 + \exp(-2 \epsilon_{t,i} (\ell_{t,i} - m_{t,i} + b_{t,i}))\big)\\
    \le {} & \frac{1}{2} \sum_{i=1}^d \frac{p_{t,i}}{\epsilon_{t,i}} 4 \epsilon_{t,i}^2 (\ell_{t,i} - m_{t,i} + b_{t,i})^2 = 2 \sum_{i=1}^d \epsilon_{t,i} p_{t,i} (\ell_{t,i} - m_{t,i} + b_{t,i})^2,
  \end{align*}
    where the first and second steps use the optimality of $\p^\star$, the last inequality uses $e^{-x} - 1 + x \le x^2$ for all $x \ge -1$, requiring $|2 \epsilon_{t,i} (\ell_{t,i} - m_{t,i} + b_{t,i})| \le 1$.
    It can be satisfied by $\epsilon_{t,i} \le 1/32$ and $|\ell_{t,i} - m_{t,i} + b_{t,i}| \le 16$, where the latter requirement can be satisfied by setting $b_{t,i} = 16 \epsilon_{t,i} (\ell_{t,i} - m_{t,i})^2$:
    \begin{equation*}
      |\ell_{t,i} - m_{t,i} + b_{t,i}| \le 2 + 16 \cdot \frac{1}{32} (2 \ell_{t,i}^2 + 2 m_{t,i}^2) \le 4 \le 16.
    \end{equation*}
    As a result, we have 
    \begin{equation*}
      (\ell_{t,i} - m_{t,i} + b_{t,i})^2 = \sbr{\ell_{t,i} - m_{t,i} + 16 \epsilon_{t,i} (\ell_{t,i} - m_{t,i})^2}^2 \le 4 (\ell_{t,i} - m_{t,i})^2,
    \end{equation*}
    where the last step holds because $|\ell_{t,i}|, |m_{t,i}| \le 1$ and $\epsilon_{t,i} \le 1/32$.
    Finally, it holds that 
    \begin{equation*}
        \term{b} \le 2 \sumT \sum_{i=1}^d \epsilon_{t,i} p_{t,i} (\ell_{t,i} - m_{t,i} + b_{t,i})^2 \le 8 \sumT \sum_{i=1}^d \epsilon_{t,i} p_{t,i} (\ell_{t,i} - m_{t,i})^2.
    \end{equation*}
    As for $\term{a}$, following the same argument as Lemma 1 of \citet{COLT'21:impossible-tuning}, we have 
    \begin{equation*}
        \term{a} \le \sum_{i=1}^d \frac{1}{\epsilon_{1,i}} f_\KL(u_i, \ph_{1,i}) + \sumTT \sum_{i=1}^d \sbr{\frac{1}{\epsilon_{t,i}} - \frac{1}{\epsilon_{t-1,i}}} f_\KL(u_i, \ph_{t,i}),
    \end{equation*}
    where $f_\KL(a,b) \define a \log (a/b) - a + b$.
    Combining all three terms, we have 
    \begin{align*}
        \sumT \inner{\ellb_t + \b_t}{\p_t - \uu} \le {} & \sum_{i=1}^d \frac{1}{\epsilon_{1,i}} f_\KL(u_i, \ph_{1,i}) + \sumTT \sum_{i=1}^d \sbr{\frac{1}{\epsilon_{t,i}} - \frac{1}{\epsilon_{t-1,i}}} f_\KL(u_i, \ph_{t,i}) \\
        & + 8 \sumT \sum_{i=1}^d \epsilon_{t,i} p_{t,i} (\ell_{t,i} - m_{t,i})^2 - 4 \sumTT \|\p_t - \p_{t-1}\|_1^2.
    \end{align*}
    Moving the correction term $\sumT \inner{\b_t}{\p_t - \uu}$ to the right-hand side gives: 
    \begin{align*}
        & \sumT \inner{\ellb_t}{\p_t - \uu} \le \sum_{i=1}^d \frac{1}{\epsilon_{1,i}} f_\KL(u_i, \ph_{1,i}) + \sumTT \sum_{i=1}^d \sbr{\frac{1}{\epsilon_{t,i}} - \frac{1}{\epsilon_{t-1,i}}} f_\KL(u_i, \ph_{t,i}) \\
        & - 8 \sumT \sum_{i=1}^d \epsilon_{t,i} p_{t,i} (\ell_{t,i} - m_{t,i})^2 + 16 \sumT \sum_{i=1}^d \epsilon_{t,i} u_i (\ell_{t,i} - m_{t,i})^2 - 4 \sumTT \|\p_t - \p_{t-1}\|_1^2.
    \end{align*}
    Finally, choosing $\uu = \e_\is$ and $\epsilon_{t,i} = \varepsilon_i$ for all $t \in [T]$ finishes the proof.
\end{proof}

\subsection{Proof of Lemma~\ref{lem:universal-optimism}}
\label{app:universal-optimism}
\begin{proof}
  For simplicity, we introduce the notation $\g_t = \nabla f_t(\x_t)$ to denote the gradient at each round.
  Below we consider the three function families respectively.

  For \emph{exp-concave} and \emph{strongly convex} functions, we have
  \begin{align*}
    \sumT (r_{t, \is} - m_{t, \is})^2 = {} & \sumT (\inner{\g_t}{\x_t - \x_{t,\is}} - \inner{\g_{t-1}}{\x_{t-1} - \x_{t-1,\is}})^2\\
    \le {} & 2 \sumT \inner{\g_t}{\x_t - \x_{t,\is}}^2 + 2 \sumT \inner{\g_{t-1}}{\x_{t-1} - \x_{t-1,\is}}^2\\
    \le {} & 4 \sumT \inner{\g_t}{\x_t - \x_{t,\is}}^2 + \O(1).
  \end{align*}
  For \emph{strongly convex} functions, using the boundedness of gradients, we further obtain
  \begin{equation*}
    \sumT (r_{t, \is} - m_{t, \is})^2 = 4 \sumT \inner{\g_t}{\x_t - \x_{t,\is}}^2 + \O(1) \le 4 G^2 \sumT \|\x_t - \x_{t,\is}\|^2 + \O(1).
  \end{equation*}
  For \emph{convex} function, it holds that
  \begin{align}
    & \sumT (r_{t, \is} - m_{t, \is})^2 = \sumT \Big(\inner{\g_t}{\x_t - \x_{t,\is}} - \inner{\g_{t-1}}{\x_{t-1} - \x_{t-1,\is}}\Big)^2\notag\\
    \le {} & 2 \sumT \inner{\g_t - \g_{t-1}}{\x_t - \x_{t,\is}}^2 + 2 \sumT \inner{\g_{t-1}}{\x_t - \x_{t-1} + \x_{t,\is} - \x_{t-1,\is}}^2\notag\\
    \le {} & 2 D^2 \sumT \|\g_t - \g_{t-1}\|^2 + 2 G^2 \sumT \|\x_t - \x_{t-1} + \x_{t,\is} - \x_{t-1,\is}\|^2\label{eq:lem1 eq1}\\
    \le {} & 4D^2 V_T + 4 (D^2 L^2 + G^2) S_T^{\x} + 4 G^2 S_{T,\is}^{\x}\notag,
  \end{align}
  where the fourth step is by Assumption~\ref{assum:domain-boundedness} and Assumption~\ref{assum:gradient-boundedness} and the last step is due to the definition of the gradient variation, finishing the proof.
\end{proof}

\subsection{Proof of Lemma~\ref{lem:two-layer-MsMwC}}
\label{app:two-layer-MsMwC}
\begin{proof}
  By \eqref{eq:lemma2 eq1}, the regret of \MSMWCtop can be bounded as
  \begin{align*}
    \sumT \inner{\ellb_t^\Top}{\q_t^\Top - \e_\js} \le {} & \sbr{\frac{1}{\varepsilon_{\js}^\Top} \log \frac{1}{\qh_{1, \js}^\Top} + \sumM \frac{\qh_{1,j}^\Top}{\varepsilon_j^\Top}} + 16 \varepsilon_\js^\Top \sumT (\ell^\Top_{t,\js} - m^\Top_{t,\js})^2 - \min_{j \in [M]} \frac{1}{4\varepsilon_j^\Top} S_T^\Top,
  \end{align*}
  where the first step comes from $f_\KL(a,b) = a \log (a/b) - a + b \le a \log (a/b) + b$ for $a,b > 0$.
  The first term above can be further bounded as
  \begin{equation*}
    \frac{1}{\varepsilon_{\js}^\Top} \log \frac{1}{\qh_{1, \js}^\Top} + \sumM \frac{\qh_{1,j}^\Top}{\varepsilon_j^\Top} = \frac{1}{\varepsilon_{\js}^\Top} \log \frac{\sumM (\varepsilon_j^\Top)^2}{(\varepsilon_{\js}^\Top)^2} + \frac{\sumM \varepsilon_j^\Top}{\sumM (\varepsilon_j^\Top)^2} \le \frac{1}{\varepsilon_{\js}^\Top} \log \frac{1}{3C_0^2 (\varepsilon_{\js}^\Top)^2} + 4C_0,
  \end{equation*}
  where the first step is due to the initialization of $\qh_{1,j}^\Top = (\varepsilon_j^\Top)^2/\sumM (\varepsilon_j^\Top)^2$.
  Plugging in the setting of $\varepsilon_j^\Top = 1/(C_0 \cdot 2^j)$, the second step holds since
  \begin{equation*}
      \frac{1}{4C_0^2}=\sumM (\varepsilon_j^\Top)^2 \le \sumM \frac{1}{C_0^2 \cdot 4^j} \le \frac{1}{3C_0^2}, \quad \sumM \varepsilon_j^\Top =\sumM\frac{1}{C_0 \cdot 2^j} \le \frac{1}{C_0}.
  \end{equation*}
  Since $1/\varepsilon_j^\Top = C_0 \cdot 2^j \ge 2 C_0$, the regret of \MSMWCtop can be bounded by 
  \begin{equation*}
    \sumT \inner{\ellb_t^\Top}{\q_t^\Top - \e_\js} \le \frac{1}{\varepsilon_{\js}^\Top} \log \frac{1}{3 C_0^2 (\varepsilon_{\js}^\Top)^2} + 16 \varepsilon_\js^\Top \sumT (\ell^\Top_{t,\js} - m^\Top_{t,\js})^2 - \frac{C_0}{2} S_T^\Top + \O(1).
  \end{equation*}
  Next, using \pref{lem:MsMwC-refine} again, the regret of the $\js$-th \MSMWCmid, whose step size is $\varepsilon^\Mid_{t,j,i} = 2 \varepsilon_j^\Top$ for all $t \in [T]$ and $i \in [N]$, can be bounded as
  \begin{align*}
    \sumT \inner{\ellb_{t,\js}^\Mid}{\q_{t,\js}^\Mid - \e_\is} \le {} & \frac{\log N}{2 \varepsilon_{\js}^\Top} + 32 \varepsilon_{\js}^\Top \sumT (\ell_{t,\js,\is}^\Mid - m_{t,\js,\is}^\Mid)^2 - \frac{C_0}{4} S_{T,\js}^\Mid\\
    & - 16 \varepsilon_{\js}^\Top \sumT \sumN q^\Mid_{t,\js,i} (\ell_{t,j,i}^\Mid - m_{t,j,i}^\Mid)^2,
  \end{align*}
  where the first step is due to the initialization of $\ph_{1,j,i}^\Mid = 1/N$.
  Based on the observation of
  \begin{equation*}
    (\ell^\Top_{t,j} - m^\Top_{t,j})^2 = \inner{\ellb^\Mid_{t,j} - \m^\Mid_{t,j}}{\q^\Mid_{t,j}}^2 \le \sumN q^\Mid_{t,\js,i} (\ell_{t,j,i}^\Mid - m_{t,j,i}^\Mid)^2,
  \end{equation*}
  where the last step uses the Cauchy-Schwarz inequality, combining the regret of \MSMWCtop and the $\js$-th \MSMWCmid finishes the proof.
\end{proof}

\subsection{Proof of Theorem~\ref{thm:unigrad-correct}}
\label{app:unigrad-correct}
\begin{proof}
  The proof proceeds in three steps: we first decompose the total regret into meta and base regret, then analyze the meta regret and base regret separately, and finally combines them to achieve the final regret guarantees.

  \paragraph{Regret Decomposition.}
  For simplicity, we let $\xs=\argmin_\X f_t(\x)$ and $\g_t = \nabla f_t(\x_t)$.
  For \emph{$\lambda$-strongly convex} functions, we decompose the regret as
    \begin{equation*}
        \Reg_T \le \underbrace{\sumT \inner{\g_t}{\x_t - \x_{t,\is}} - \frac{\lambda}{2} \sumT \|\x_t - \x_{t,\is}\|^2}_{\meta} + \underbrace{\sumT f_t(\x_{t,\is})-\sumT f_t(\xs)}_{\base},
    \end{equation*}
    where $\lambda_{\is}\le  \lambda\le  2\lambda_{\is}$.
    For \emph{$\alpha$-exp-concave} functions, we decompose the regret as
    \begin{align*}
      \Reg_T
          \le {} \underbrace{\sumT \inner{\g_t}{\x_t - \x_{t,i}}-\frac{\alpha}{2}\sumT \inner{\g_t}{\x_t - \x_{t,\is}}^2}_{\meta} +\underbrace{\sumT f_t(\x_{t,\is})-\sumT f_t(\xs)}_{\base},
      \end{align*}
      where $\alpha_{\is}\le  \alpha\le  2\alpha_{\is}$.
    For \emph{convex} functions, we decompose the regret as
    \begin{align*}
      \Reg_T=\underbrace{\sumT f_t(\x_{t})-\sumT f_t(\x_{t,\is})}_{\meta}+\underbrace{\sumT f_t(\x_{t,\is})-\sumT f_t(\xs)}_{\base}.
    \end{align*}

    \paragraph{Meta Regret Analysis.}
    
    Recall that the normalization factor $Z=\max\{GD+\gamma^\Mid D^2, 1+\gamma^\Mid D^2+2\gamma^\Top\}$. We focus on the linearized term $\sumT \inner{\g_t}{\x_t - \x_{t,\is}}$ and let $\Vs=\sumTT(\ell^\Mid_{t,\js,\is} - m^\Mid_{t,\js,\is})^2$.  Specifically, 
    \begin{align}
        & \sumT \inner{\g_t}{\x_t - \x_{t,\is}} = Z\sumT \inner{\ellb_t}{\p_t - \e_\is} = Z\sumT\inner{\ellb_t}{\p_t - \q^\Mid_{t,\js}} + Z\sumT \inner{\ellb_t}{\q^\Mid_{t,\js} - \e_\is}\notag\\
        = {} & Z\sumT \inner{\ellb^\Top_t}{\q^\Top_t - \e_{\js}} + Z\sumT \inner{\ellb^\Mid_t}{\q^\Mid_{t,\js} - \e_\is} - \gamma^\Top \sumTT \sumM q_{t,j}^\Top \|\q_{t,j}^\Mid - \q_{t-1,j}^\Mid\|_1^2\notag\\
        & \qquad + \gamma^\Top S_{T,\js}^\Mid - \gamma^\Mid \sumTT \sumM q_{t,j}^\Top \sumN q_{t,j,i}^\Mid\norm{\x_{t,i} - \x_{t-1,i}}^2 + \gamma^\Mid S_{T,\is}^{\x}\notag\\
        & \qquad + \gamma^\Mid \sumTT \sumN q_{t,\js,i}^\Mid \norm{\x_{t,i} - \x_{t-1,i}}^2 - \gamma^\Mid \sumTT \sumN q_{t,\js,i}^\Mid \|\x_{t,i} - \x_{t-1,i}\|^2 \notag\\
        \le {} & \frac{Z}{\varepsilon^\Top_{\js}} \log \frac{N}{3 C_0^2 (\varepsilon^\Top_{\js})^2} + 32Z \varepsilon^\Top_{\js} \Vs - \frac{C_0}{2} S_T^\Top + \sbr{\gamma^\Top - \frac{C_0}{4}} S_{T,\js}^\Mid + \gamma^\Mid S_{T,\is}^{\x}\notag\\
        & \qquad - \gamma^\Mid \sumTT \sumM q_{t,j}^\Top \sumN q_{t,j,i}^\Mid\norm{\x_{t,i} - \x_{t-1,i}}^2  - \gamma^\Top \sumTT \sumM q_{t,j}^\Top \|\q_{t,j}^\Mid - \q_{t-1,j}^\Mid\|_1^2\notag\\
        \le {} & \frac{Z}{\varepsilon^\Top_{\js}} \log \frac{N}{3 C_0^2 (\varepsilon^\Top_{\js})^2} + 32Z \varepsilon^\Top_{\js} \Vs - \frac{C_0}{2} S_T^\Top + \gamma^\Mid S_{T,\is}^{\x} \tag*{(requiring $C_0 \ge 4 \gamma^\Top$)}\\
        & \quad - \gamma^\Mid \sumTT \sumM q_{t,j}^\Top \sumN q_{t,j,i}^\Mid\norm{\x_{t,i} - \x_{t-1,i}}^2  - \gamma^\Top \sumTT \sumM q_{t,j}^\Top \|\q_{t,j}^\Mid - \q_{t-1,j}^\Mid\|_1^2,\label{eq:thm1 eq}
    \end{align}
    where the first step is due to $\x_t = \sumN p_{t,i} \x_{t,i}$ and defines $\ell_{t,i} \define \frac{1}{Z}\inner{\nabla f_t(\x_t)}{\x_{t,i}}$.
    The third step is due to the definition of $\ellb^\Top_t$ and $\ellb^\Mid_t$ as defined in \pref{eq:top-loss-2} and \pref{eq:mid-loss}.
    The fourth step uses the analysis of \msoms as show in in \pref{lem:two-layer-MsMwC}. 

For \emph{$\lambda$-strongly convex} functions, applying \pref{eq:thm1 eq} and omitting the stability and curvature-induced negative terms, we bound the meta regret by
\begin{align}
    & \meta \le \sumT \inner{\g_t}{\x_t - \x_{t,\is}} - \frac{\lambda}{2} \sumT \|\x_t - \x_{t,\is}\|^2 \notag\\
    \le {} & \frac{Z}{\varepsilon^\Top_{\js}} \log \frac{N}{3 C_0^2 (\varepsilon^\Top_{\js})^2} + 32Z\varepsilon^\Top_{\js} \Vs+ \gamma^\Mid S_{T,\is}^{\x} - \frac{\lambda}{2} \sumT \|\x_t - \x_{t,\is}\|^2 \notag\\
    \le {} & \frac{Z}{\varepsilon^\Top_{\js}} \log \frac{N}{3 C_0^2 (\varepsilon^\Top_{\js})^2} + \sbr{\frac{128 G^2}{Z} \varepsilon^\Top_{\js} - \frac{\lambda}{2}} \sumT \|\x_t - \x_{t,\is}\|^2 + \gamma^\Mid S_{T,\is}^{\x} \notag\\
    \le {} & 2ZC_0\log\frac{4N}{3}+\frac{512ZG^2}{\lambda}\log \frac{2^{20}G^2N}{3C_0^2\lambda^2}+ \gamma^\Mid S_{T,\is}^{\x} \label{eq:meta-scvx}, 
\end{align}
    where the third step leverages the property of our universal optimism design as given in \pref{lem:universal-optimism} and the last step again follows from \pref{lem:tune-eta-2} and requires $\epsilon^\Top_\js\le\epsilon^\Top_{\star}\define \frac{\lambda Z}{256G^2}$.

    For \emph{$\alpha$-exp-concave} functions, applying \pref{eq:thm1 eq} and omitting the stability and curvature-induced negative terms, we bound the meta regret by
    \begin{align}
      & \meta \le \sumT \inner{\g_t}{\x_t - \x_{t,\is}} - \frac{\alpha}{2} \sumT \inner{\g_t}{\x_t - \x_{t,\is}}^2 \notag\\
      \le {} & \frac{Z}{\varepsilon^\Top_{\js}} \log \frac{N}{3 C_0^2 (\varepsilon^\Top_{\js})^2} + 32Z\varepsilon^\Top_{\js} \Vs + \gamma^\Mid S_{T,\is}^{\x} - \frac{\alpha}{2} \sumT \inner{\g_t}{\x_t - \x_{t,\is}}^2 \notag\\
      \le {} & \frac{Z}{\varepsilon^\Top_{\js}} \log \frac{N}{3 C_0^2 (\varepsilon^\Top_{\js})^2} + \sbr{\frac{128\varepsilon^\Top_{\js}}{Z} - \frac{\alpha}{2}} \sumT \inner{\g_t}{\x_t - \x_{t,\is}}^2 + \gamma^\Mid S_{T,\is}^{\x} \notag\\
      \le {} & 2ZC_0\log\frac{4N}{3}+\frac{512Z}{\alpha}\log \frac{2^{20}N}{3C_0^2\alpha^2}+ \gamma^\Mid S_{T,\is}^{\x}, \label{eq:meta-exp}
  \end{align}
  where the third step leverages the property of our universal optimism design as given in \pref{lem:universal-optimism} and the last step follows from \pref{lem:tune-eta-2} and requires $\epsilon^\Top_\js\le\epsilon^\Top_{\star}\define \frac{\alpha Z}{256}$. 

For \emph{convex} functions, applying \pref{eq:thm1 eq} while \emph{retaining} the crucial stability and curvature-induced negative terms, the meta regret can be bounded by
  \begin{align}
    & \meta \le \sumT \inner{\g_t}{\x_t - \x_{t,\is}}\notag \\
    \overset{\eqref{eq:thm1 eq}}{\le} {} & \frac{Z}{\varepsilon^\Top_{\js}} \log \frac{N}{3 C_0^2 (\varepsilon^\Top_{\js})^2} + 32Z\varepsilon^\Top_{\js} \Vs+ \gamma^\Mid S_{T,\is}^{\x}-\frac{C_0}{2}S_T^{\Top}\notag\\
    {}&-\gamma^\Mid \sumTT \sumM q_{t,j}^\Top \sumN q_{t,j,i}^\Mid\norm{\x_{t,i} - \x_{t-1,i}}^2  - \gamma^\Top \sumTT \sumM q_{t,j}^\Top \|\q_{t,j}^\Mid - \q_{t-1,j}^\Mid\|_1^2\label{eq:thm1 eq0}\\
    \le {} & \frac{Z}{\varepsilon^\Top_{\js}} \log \frac{N}{3 C_0^2 (\varepsilon^\Top_{\js})^2} +\frac{128D^2 \varepsilon^\Top_{\js}}{Z} V_T+\sbr{\frac{64G^2}{Z}+\gamma^\Mid} S_{T,\is}^{\x}+\frac{64(D^2L^2+G^2)}{Z}S_T^{\x}\notag\\
    {}&-\frac{C_0}{2}S_T^{\Top}- \gamma^\Mid \sumTT \sumM q_{t,j}^\Top \sumN q_{t,j,i}^\Mid\norm{\x_{t,i} - \x_{t-1,i}}^2  - \gamma^\Top \sumTT \sumM q_{t,j}^\Top \|\q_{t,j}^\Mid - \q_{t-1,j}^\Mid\|_1^2\notag \\
    \le {} & 2ZC_0\log \frac{4N}{3 }+32D\sqrt{2V_T\log\sbr{512ND^2V_T/Z^2}}+\sbr{\frac{64G^2}{Z}+\gamma^\Mid} S_{T,\is}^{\x}+\frac{64(D^2L^2+G^2)}{Z}S_T^{\x}\notag\\
    {}&-\frac{C_0}{2}S_T^{\Top}- \gamma^\Mid \sumTT \sumM q_{t,j}^\Top \sumN q_{t,j,i}^\Mid\norm{\x_{t,i} - \x_{t-1,i}}^2  - \gamma^\Top \sumTT \sumM q_{t,j}^\Top \|\q_{t,j}^\Mid - \q_{t-1,j}^\Mid\|_1^2\label{eq:thm1 eq2} \\
    \le{}&2ZC_0\log \frac{4N}{3 }+32D\sqrt{2V_T\log\sbr{512ND^2V_T/Z^2}}+\sbr{\frac{64G^2}{Z}+\gamma^\Mid}S_{T,\is}^{\x}\notag\\
    {}&+\sbr{\frac{C_1}{Z}-\gamma^{\Mid}}\sumTT \sumM q_{t,j}^\Top \sumN q_{t,j,i}^\Mid\norm{\x_{t,i} - \x_{t-1,i}}^2\notag\\
    {}&+\sbr{\frac{2D^2C_1}{Z}-\frac{C_0}{2}}S_T^{\Top}+\sbr{\frac{2D^2C_1}{Z}-\gamma^{\Top}}\sumTT \sumM q_{t,j}^\Top \|\q_{t,j}^\Mid - \q_{t-1,j}^\Mid\|_1^2\tag*{(denoting $C_1=128(D^2L^2+G^2)$)}\\
    \le{}& 2ZC_0\log \frac{4N}{3 }+32D\sqrt{2V_T\log\sbr{512ND^2V_T/Z^2}}+\sbr{\frac{64G^2}{Z}+\gamma^\Mid}S_{T,\is}^{\x}, \tag*{$\sbr{\text{requiring }\gamma^{\Top} \ge  2D^2C_1/Z,  \gamma^{\Mid} \ge  C_1/Z\text{, and }C_0 \ge  4D^2C_1
    /Z}$}
\end{align}
where the third step is due to \pref{lem:universal-optimism}, and $\epsilon_{\js}^{\Top}\le  1/2$ under the requirement $C_0 \ge 1$. The fourth step is by \pref{lem:tune-eta-1} and requiring $C_0 \ge  8D$ and the fifth step is by \pref{lem:decompose-three-layer}.

    \paragraph{Base Regret Analysis.}
    For \emph{$\lambda$-strongly convex} functions,  the $\is$-th base learner guarantees the following:
    \begin{align*}
      &\base
    \le \frac{16G^2}{\lambda_{\is}} \log \sbr{1 + 2\lambda_{\is} V_T+2\lambda_{\is}L^2 S_{T,\is}^{\x} } + \frac{1}{4}\Bottomcoef D^2  - \frac{1}{8} \Bottomcoef S_{T,\is}^{\x} + \O(1)\\
    \le {}& \frac{32G^2}{\lambda} \log \sbr{1 +  2\lambda V_T}+\sbr{32L^2G^2-\frac{1}{8}\Bottomcoef}S_{T,\is}^{\x}+\frac{1}{4}\Bottomcoef D^2+\O(1),
    \end{align*}
    where the first step is due to \pref{lem:str-convex-base} and  $\sumTT\norm{\nabla f_t(\x_{t,\is})-\nabla f_{t-1}(\x_{t-1,\is})}^2\le  2V_T+2L^2 S_{T,\is}^{\x}$ and the last step follows from $\log (1+x)\le  x$ for $x \ge  0$ and $\lambda_{\is}\le  \lambda \le  2\lambda_{\is}$.

  For \emph{$\alpha$-exp-concave} functions, the $\is$-th base learner guarantees the following:
  \begin{align*}
    &\base
  \le \frac{16d}{\alpha_{\is}} \log \sbr{1 + \frac{\alpha_{\is}}{4\Bottomcoef d}V_T+\frac{\alpha_{\is}L^2}{4\Bottomcoef d} S_{T,\is}^{\x} } + \frac{1}{2}\Bottomcoef D^2  - \frac{1}{4} \Bottomcoef S_{T,\is}^{\x} + \O(1)\\
  \le {}& \frac{32d}{\alpha} \log \sbr{1 + \frac{\alpha}{4\Bottomcoef d}V_T}+\sbr{\frac{4L^2}{\Bottomcoef}-\frac{1}{4}\Bottomcoef}S_{T,\is}^{\x}+\frac{1}{2}\Bottomcoef D^2+\O(1),
  \end{align*}
  where the first step is due to \pref{lem:exp-concave-base} and $\sumTT\norm{\nabla f_t(\x_{t,\is})-\nabla f_{t-1}(\x_{t-1,\is})}^2\le  2V_T+2L^2S_{T,\is}^{\x}$ and the last step follows from $\log (1+x)\le  x$ for $x \ge  0$ and $\alpha_{\is}\le  \alpha \le  2\alpha_{\is}$.

For \emph{convex} functions, by \pref{lem:convex-base}, the $\is$-th base learner guarantees the following:
\begin{align*}
  \base\le& 5D\sqrt{1+2 V_T+2L^2 S^{\x}_{T,\is}} + \Bottomcoef D^2  - \frac{1}{4} \Bottomcoef S_{T,\is}^{\x} + \O(1)\\
  \le {}&5D\sqrt{1+2 V_T}+\Bottomcoef D^2+\sbr{ 10DL^2-\frac{1}{4} \Bottomcoef} S_{T,\is}^{\x} + \O(1),
\end{align*}
where the first step is due to $\sumTT\norm{\nabla f_t(\x_{t,\is})-\nabla f_{t-1}(\x_{t-1,\is})}^2\le  2V_T+2L^2 S_{T,\is}^{\x}$.

\paragraph{Overall Regret Analysis.}
  For \emph{$\lambda$-strongly convex} functions, the overall regret can be bounded by
  \begin{align*}
    \Reg_T \le {} & 2ZC_0\log\frac{4N}{3}+\frac{512ZG^2}{\lambda}\log \frac{2^{20}G^2N}{3C_0^2\lambda^2}+\frac{32G^2}{\lambda} \log \sbr{1 +  2\lambda V_T}\\
    & \qquad +\sbr{\gamma^\Mid+32L^2G^2-\frac{\Bottomcoef}{4}}S_{T,\is}^{\x}+\frac{\Bottomcoef D^2}{4}+\O(1) \le \O\sbr{\frac{1}{\lambda}\log V_T},
  \end{align*}
  where the last step requires $\Bottomcoef \ge  4\gamma^{\Mid}+128L^2G^2 $.

  For \emph{$\alpha$-exp-concave} functions, the overall regret can be bounded by
  \begin{align*}
    \Reg_T \le {} & 2ZC_0\log\frac{4N}{3}+\frac{512Z}{\alpha}\log \frac{2^{20}N}{3C_0^2\alpha^2}+\frac{32d}{\alpha} \log \sbr{1 + \frac{\alpha}{4\Bottomcoef d}V_T}\\
    & \qquad +\sbr{\gamma^\Mid+\frac{4L^2}{\Bottomcoef}-\frac{1}{4}\Bottomcoef}S_{T,\is}^{\x}+\frac{1}{2}\Bottomcoef D^2+\O(1) \le \O\sbr{\frac{d}{\alpha}\log V_T},
  \end{align*}
  where the last step requires $\Bottomcoef \ge  4\gamma^{\Mid}+8L^2$.

  For \emph{convex} functions, the overall regret can be bounded by
  \begin{align*}
    \Reg_T \le {} & 2ZC_0\log \frac{4N}{3 }+32D\sqrt{2V_T\log\sbr{512ND^2V_T/Z^2}}+5D\sqrt{1+2 V_T}\\
    & \qquad +\sbr{\frac{64G^2}{Z}+\gamma^\Mid+10DL^2-\frac{1}{4}\Bottomcoef}S_{T,\is}^{\x}+\Bottomcoef D^2+\O(1) \le \O\sbr{\sqrt{V_T\log V_T}},
  \end{align*}
  where the last step requires $\Bottomcoef \ge  4\gamma^{\Mid}+40DL^2+256G^2/Z$.

At last, we determine the specific values of $C_0$, $\gamma^{\Top}$, and $\gamma^{\Mid}$. These parameters need to satisfy the following requirements:
\begin{gather*}
  C_0 \ge  1,\ C_0 \ge  8D,\ C_0 \ge  4\gamma^{\Top},\ C_0 \ge  4D^2C_1/Z,\ \gamma^{\Top} \ge  C_1/Z,\ \text{and}\ \gamma^{\Mid} \ge  2D^2C_1/Z.
\end{gather*}
As a result, we set 
\begin{equation}
    \label{eq:Thm1-params}
    C_0=\max\bbr{1,8D,4\gamma^{\Top},4D^2C_1},\ 
    \gamma^\Top=C_1,\ 
    \gamma^{\Mid}=2D^2C_1,
\end{equation}
where $Z=\max\{GD+\gamma^\Mid D^2, 1+\gamma^\Mid D^2+2\gamma^\Top\}$ and $C_1=128(D^2L^2+G^2)$.
\end{proof}

\section{Omitted Proofs for Section~\ref{sec:method2-Bregman}}
\label{app:method2-proofs}
In this section, we provide the omitted details for \pref{sec:method2-Bregman}, including the proofs of \pref{lem:emp-VT-de-breg}, the correctness of \pref{lem:emp-VT-de-breg} under \pref{assum:smoothness-X+}, and \pref{thm:unigrad-bregman}.

\subsection{Proof of Lemma~\ref{lem:emp-VT-de-breg}}
\label{app:emp-VT-de-breg}
\begin{proof}
    By inserting intermediate terms, we have
    \begin{align}
        \Vb_T \le {} & 4 \sumTT \|\nabla f_t(\x_t) - \nabla f_t(\x_{t,\is})\|^2 + 4 \sumTT \|\nabla f_t(\x_{t,\is}) - \nabla f_{t-1}(\x_{t,\is})\|^2 \notag\\
        & \quad + 4 \sumTT \|\nabla f_{t-1}(\x_{t,\is}) - \nabla f_{t-1}(\x_{t-1,\is})\|^2 + 4 \sumTT \|\nabla f_{t-1}(\x_{t-1,\is}) - \nabla f_{t-1}(\x_{t-1})\|^2 \notag\\
        \overset{\eqref{eq:smoothness-property}}{\le} {} & 8L \sumTT \D_{f_t}(\x_{t,\is}, \x_t) + 4 V_T + 4 L^2 \sumTT \|\x_{t,\is} - \x_{t-1,\is}\|^2 + 8 L \sumTT \D_{f_{t-1}}(\x_{t-1,\is}, \x_{t-1}) \notag\\
        \le {} & 4V_T + 16L \sumT \D_{f_t}(\x_{t,\is}, \x_t) + 4L^2 \sumTT \|\x_{t,\is} - \x_{t-1,\is}\|^2,
    \end{align}
    where the first step introduces intermediate terms $\nabla f_t(\x_{t,\is})$, $\nabla f_{t-1}(\x_{t,\is})$, and $\nabla f_{t-1}(\x_{t-1,\is})$, the second step uses \mbox{\pref{prop:smoothness}} and the standard smoothness \pref{assum:smoothness}, and the last step combines two summations into one by shifting the indexes of $t$. 
\end{proof}

\subsection{Proof about Relaxed Smoothness}
\label{app:smooth-relax}
In this part, we show that \pref{lem:emp-VT-de-breg} also holds under the relaxed \pref{assum:smoothness-X+}.
To begin with, we present the following lemma, which shows that \pref{assum:smoothness-X+} is a sufficient condition for \pref{eq:smoothness-property} on $\X$.
Using \pref{lem:smooth-relax}, we can see that \pref{lem:emp-VT-de-breg} also holds under \pref{assum:smoothness-X+}.
\begin{myLemma}
    \label{lem:smooth-relax}
    Under \pref{assum:smoothness-X+}, for any online function $f(\cdot)$ satisfying \pref{assum:smoothness-X+}, it holds that $\|\nabla f(\x) - \nabla f(\y)\|^2 \le 2 L \D_f(\y,\x)$ for any $\x, \y \in \X$.
\end{myLemma}
\begin{proof}
    To begin with, we present the self-bounding property~\citep{NIPS'10:smooth}, which is useful in proving our result --- if a function $f: \R^d \rightarrow \R$ is $L$-smooth and bounded from below, then for any $\x \in \R^d$, it holds that
    \begin{equation}
        \label{eq:self-bounding}
        \|\nabla f(\x)\|^2 \le 2L \sbr{f(\x) - \inf_{\y \in \R^d} f(\y)}.
    \end{equation}
    Next, we aim to prove that if we only need \eqref{eq:self-bounding} on a bounded domain $\X$, we require smoothness only on a slightly larger domain than $\X$.
    To see this, we delve into the proof of the self-bounding property.
    Specifically, for any $\x,\v \in \R^d$, it holds that 
    \begin{equation*}
        \inner{-\nabla f(\x)}{\v} - \frac{L}{2} \|\v\|^2 \le f(\x) - f(\x + \v) \le f(\x) - \inf_{\y \in \R^d} f(\y),
    \end{equation*}
    where the first step requires smoothness on $\x$ and $\x+\v$.
    Consequently, by taking maximization over $\v$, it holds that
    \begin{equation*}
        f(\x) - \inf_{\y \in \R^d} f(\y) \ge \sup_{\v \in \R^d} \inner{-\nabla f(\x)}{\v} - \frac{L}{2} \|\v\|^2 = \frac{1}{2L} \|\nabla f(\x)\|^2,
    \end{equation*}
    which leads to the self-bounding property~\eqref{eq:self-bounding} by taking $\v = -\frac{1}{L} \nabla f(\x)$.
    The above proof is from Theorem 4.23 of \citet{book'22:FO-book}.
    This means that for the self-bounding property, we only require the smoothness to hold for any $\x \in \X$ and $\x -\frac{1}{L} \nabla f(\x)$.
    Under \pref{assum:smoothness-X}, this can be satisfied by requiring smoothness on a slightly larger domain than $\X$, namely, $\X_+ \define \{\x + \b \given \x \in \X, \b \in G/L \cdot \Bb\}$.

    Now we are ready to prove the final result.
    To begin with, we define a surrogate function of $g(\x) \define f(\x) - \inner{\nabla f(\x_0)}{\x}$ for any $\x \in \X$, where $\x_0 \in \X$.
    Due to the above property we have just proved, by requiring smoothness on $\X_+$, we have
    \begin{equation*}
        \|\nabla g(\x)\|^2 \le 2L \sbr{g(\x) - \inf_{\y \in \R^d} g(\y)}.
    \end{equation*}
    Denoting by $\y^\star \in \argmin_{\y \in \R^d} g(\y)$, the above inequality equals to
    \begin{align*}
        \|\nabla f(\x) - \nabla f(\x_0)\|^2 \le {} & 2L \sbr{f(\x) - \inner{\nabla f(\x_0)}{\x} - f(\y^\star) + \inner{\nabla f(\x_0)}{\y^\star}}\\
        = {} & 2L (f(\x) - f(\y^\star) - \inner{\nabla f(\x_0)}{\x - \y^\star}),
    \end{align*}
    due to the definition of $g(\cdot)$.
    The proof using the self-bounding property is from Theorem 2.1.5 of \citet{book'18:Nesterov-Convex}.
    Finally, we note that $g(\cdot)$ is minimized at $\y^\star = \x_0$, leading to $\|\nabla f(\x) - \nabla f(\x_0)\|^2 \le 2L \D_f(\x_0, \x)$ for any $\x, \x_0 \in \X$, which finishes the proof.
\end{proof}

\subsection{Proof of Theorem~\ref{thm:unigrad-bregman}}
\label{app:main-Bregman}
\begin{proof}
    For simplicity, we denote by $\g_t \define \nabla f_t(\x_t)$.
    We start by decomposing the total regret into the meta regret and the base regret.
    We then analyze the meta regret separately, followed by tailored proofs for different classes of loss functions. 
    
    To start with, the meta empirical gradient variation $\Vb_T$ can be bounded as 
    \begin{align}
        &\Vb_T=\sumTT \norm{\nabla f_t(\x_t)-\nabla f_{t-1}(\x_{t-1})}^2\notag\\
         \le{}& 4 \sumTT \|\nabla f_t(\x_t) - \nabla f_t(\x_{t,i})\|^2+4\sumTT \|\nabla f_t(\x_{t,i}) - \nabla f_{t-1}(\x_{t,i})\|^2  \notag\\
         &+ 4 \sumTT \|\nabla f_{t-1}(\x_{t,i}) - \nabla f_{t-1}(\x_{t-1,i})\|^2+4\sumTT \|\nabla f_{t-1}(\x_{t-1,i}) - \nabla f_{t-1}(\x_{t-1})\|^2 \notag\\
        \le{}& 8L \sumTT \D_{f_t}(\x_{t,i},\x_t) + 4 V_T + 4L^2 \sumTT \|\x_{t,i}-\x_{t-1,i}\|^2+8L \sumTT \D_{f_t}(\x_{t-1,i},\x_{t-1})\notag \\
        \le{}& 4V_T +  4L^2 S_{T,i}^{\x}+16L \sumT \D_{f_t}(\x_{t,i},\x_t).\label{eq:meta-decompose} 
    \end{align}
   Similarly, we denote by $\Vb_{T,i} = \sumTT \norm{\nabla f_t(\x_{t,i}) - \nabla f_{t-1}(\x_{t-1,i})}^2$ the empirical gradient variation of the $i$-th expert, for $i \in [N]$. Then $\Vb_{T,i}$ can be bounded as
    \begin{align}
    \Vb_{T,i}\le{}& 3 \sumTT \|\nabla f_t(\x_{t,i}) - \nabla f_t(\xs)\|^2 + 3 \sumTT \|\nabla f_t(\xs) - \nabla f_{t-1}(\xs)\|^2 \notag\\
    & + 3 \sumTT \|\nabla f_{t-1}(\xs) - \nabla f_{t-1}(\x_{t-1,i})\|^2\notag\\
    \le{}& 6L \sumTT \D_{f_t}(\xs, \x_{t,i}) + 3 V_T + 6L \sumTT \D_{f_{t-1}}(\xs, \x_{t-1,i})\notag\\
    \le{}& 3V_T + 12L \sumT \D_{f_t}(\xs, \x_{t,i}). \label{eq:base-decompose}
    \end{align}

    \paragraph{Regret Decomposition.} For \emph{$\lambda$-strongly convex} functions, we decompose the regret as
    \begin{align}
        \Reg_T \le {} & \underbrace{\sumT \inner{\g_t}{\x_t - \x_{t,\is}} - \frac{\lambda}{2} \sumT \|\x_t - \x_{t,\is}\|^2}_{\meta}\notag \\
         & \quad+ \underbrace{\sumT \inner{\nabla f_t(\x_{t,\is})}{\x_{t,\is} - \xs}- \frac{\lambda_{\is}}{4} \sumT \|\xs - \x_{t,\is}\|^2}_{\base} - \frac{1}{2}\sumT \D_{f_t}(\xs, \x_{t,\is}),
        \label{eq:multi-sc-reg-de}
    \end{align}
    where $\lambda_{\is}\le  \lambda\le  2\lambda_{\is}$. For \emph{$\alpha$-exp-concave} functions, we decompose the regret as 
    \begin{align}
    \Reg_T
       \le {} & \underbrace{\sumT \inner{\g_t}{\x_t - \x_{t,\is}}-\frac{\alpha}{2}\sumT \inner{\g_t}{\x_t - \x_{t,\is}}^2}_{\meta} \notag\\
        & + \underbrace{\sumT \inner{\nabla f_t(\x_{t,\is})}{\x_{t,\is} - \xs}-\frac{\alpha_{\is}}{4}\sumT \inner{\nabla f_t(\x_{t,\is})}{\xs - \x_{t,\is}}^2}_{\base} - \frac{1}{2}\sumT \D_{f_t}(\xs, \x_{t,\is}),
        \label{eq:multi-exp-reg-de}
    \end{align}
    where $\alpha_{\is}\le  \alpha\le  2\alpha_{\is}$. For \emph{convex} functions, we decompose the regret as 
    \begin{align}
        &\Reg_T=\sumT f_t(\x_t)-\sumT f_t(\x_{t,\is})+\sumT f_t(\x_{t,\is})-\sumT f_t(\xs) \label{eq:multi-cvx-reg-de}\\
        = {} & \underbrace{\sumT\inner{\g_t}{\x_t - \x_{t,\is}}}_{\meta} + \underbrace{\sumT \inner{\nabla f_t(\x_{t,\is})}{\x_{t,\is} - \xs}}_{\base}-\sumT \D_{f_t}(\x_{t,\is}, \x_t) - \sumT \D_{f_t}(\xs, \x_{t,\is}). \notag        
    \end{align}

    \paragraph{Meta Regret Analysis.}
    We adopt \omlprod~\citep{NIPS'16:Optimistic-ML-Prod} as the meta learner, and present its regret analysis below for self-containedness.
    \begin{myLemma}[Theorem 3.4 of \citet{NIPS'16:Optimistic-ML-Prod}]
        \label{lem:optimistic-mlprod}
        Denote by $\p_t \in \Delta_N$ the algorithm's weights, $\ellb_t \in [0,1]^N$ the loss vector, and $m_{t,i}$ the optimism. With the learning rate in \eqref{eq:omlprod-lr}, the regret of \omlprod~\eqref{eq:omlprod} with respect to any expert $i \in [N]$ satisfies
        \begin{equation*}
            \sumT \inner{\ellb_t}{\p_t - \e_i} \le C_0 \sqrt{1 + \sumT (r_{t,i} - m_{t,i})^2} + C_2,
        \end{equation*}
        where $\e_i$ is the $\ith$ standard basis vector, $C_0 = \sqrt{\log N} + \log (1 + \frac{N}{e} (1 + \log (T+1))) / \sqrt{\log N}$, and $C_2 = \frac{1}{4}(\log N + \log (1 + \frac{N}{e} (1 + \log (T+1)))) + 2 \sqrt{\log N} + 16 \log N$.
    \end{myLemma}
    Here we adopt $\ell_{t,i} = \frac{1}{2GD}\inner{\g_t}{\x_{t,i}} + \frac12 \in[0,1]$ such that $\inner{\ellb_t}{\p_t - \e_i} = \frac{1}{2GD} \inner{\g_t}{\x_t - \x_{t,i}}$.
    Besides, since the number of base learners $N = \O(\log T)$ as explained in \pref{sec:preliminary}, the constants $C_0$ and $C_2$ are in the order of $\O(\log \log T)$ and can be ignored, following previous convention~\citep{COLT'15:Luo-AdaNormalHedge,COLT'14:ML-Prod}.

    For \emph{$\lambda$-strongly convex} functions, according to \pref{eq:Bregman-optimism}, we have $m_{t,i} = 0$ where $\lambda_i \in \H^{\scvx}$.
    By \pref{lem:optimistic-mlprod}, the meta regret in \eqref{eq:multi-sc-reg-de} can be bounded as 
    \begin{align}
        & \meta =2GD\sumT \inner{\ellb_t}{\p_t - \e_i} - \frac{\lambda}{2} \sumT \|\x_t - \x_{t,\is}\|^2 \notag\\
        \le{}& C_0 \sqrt{4G^2D^2+ \sumT \inner{\g_t}{\x_t - \x_{t,\is}}^2} - \frac{\lambda}{2} \sumT \|\x_t - \x_{t,\is}\|^2 + 2GDC_2  \notag\\
        \le {} & C_0 \sqrt{4G^2D^2+G^2 \sumT \|\x_t - \x_{t,\is}\|^2} - \frac{\lambda}{2} \sumT \|\x_t - \x_{t,\is}\|^2 + 2GDC_2 \notag\\
        \le {} & \O(C_3) + \sbr{\frac{C_0 G^2}{2 C_3} - \frac{\lambda}{2}} \sumT \|\x_t - \x_{t,\is}\|^2, \label{eq:scvx-meta}
    \end{align}
    where the last step omits the ignorable additive $C_0$ or $C_2$ terms and is due to AM-GM inequality (\pref{lem:AM-GM}).
    $C_3$ is a constant to be specified.

    For  \emph{$\alpha$-exp-concave} functions, according to \pref{eq:Bregman-optimism}, we have $m_{t,i} = 0$ where $\alpha_i \in \H^{\exp}$.
    By \pref{lem:optimistic-mlprod}, the meta regret in \eqref{eq:multi-exp-reg-de} can be bounded as 
    \begin{align}
        \meta \le {} & C_0 \sqrt{4G^2D^2 + \sumT \inner{\g_t}{\x_t - \x_{t,\is}}^2} - \frac{\alpha}{2} \sumT \inner{\g_t}{\x_t - \x_{t,\is}}^2 + 2GDC_2 \notag\\
        \le {} & \O(C_4) + \sbr{\frac{C_0}{2 C_4} - \frac{\alpha}{2}} \sumT \inner{\g_t}{\x_t - \x_{t,\is}}^2, \label{eq:exp-meta}
    \end{align}
    where the last step omits the ignorable additive $C_0$ or $C_2$ terms and is due to AM-GM inequality (\pref{lem:AM-GM}).
    $C_4$ is a constant to be specified.

    For \emph{convex} functions, according to \pref{eq:Bregman-optimism}, we have $m_{t,i}= \inner{\g_{t-1}}{\x_t - \x_{t,i}}/(2GD)$ for the convex base learner.
    As explained in \pref{sec:method2-Bregman}, although $\x_t$ is unknown for now, we only require the scalar value of $\inner{\g_{t-1}}{\x_t}$.
    Denoting by $z = \inner{\g_{t-1}}{\x_t}$, it actually forms a fixed-point problem of $z = \inner{\g_{t-1}}{\x_t(z)}$, where $\x_t$ is a function of $z$ since $\x_t$ depends on $p_{t,i}$, $p_{t,i}$ relies on $m_{t,i}$, and $m_{t,i}$ depends on $z$.
    Such a one-dimensional fixed-point problem can be solved with an $\O(1/T)$ approximation error through $\O(\log T)$ binary searches, and aggregating the approximate error over the whole time horizon will only incur an additive constant to the final regret.
    As a result, such an optimism setup is valid.
    Consequently, the meta regret in \eqref{eq:multi-cvx-reg-de} can be bounded as 
    \begin{align}
        \meta \le {}& C_0 \sqrt{4G^2D^2 + \sumT \inner{\g_t - \g_{t-1}}{\x_t - \x_{t,\is}}^2} + 2GDC_2 \le C_0 \sqrt{1 + D^2 \Vb_T} + C_2 \notag\\
        \overset{\eqref{eq:meta-decompose}}{\le} {}&C_0 \sqrt{4G^2D^2 + 4D^2 V_T + 4L^2D^2 S_{T,\is}^{\x}+ 16LD^2 \sumT \D_{f_t}(\x_{t,\is}, \x_t)} + 2GDC_2\notag\\
        \le {} & \O(\sqrt{V_T}) + C_0\sqrt{4L^2D^2 S_{T,\is}^{\x}+ 16LD^2 \sumT \D_{f_t}(\x_{t,\is}, \x_t)} \notag\\
        \le{}& \O(\sqrt{V_T}) + \O(C_5) + \frac{C_0}{2C_5} \sumT \D_{f_t}(\x_{t,\is},\x_t)+\frac{LC_0}{8C_5} S_{T,\is}^{\x},
        \label{eq:cvx-meta}
    \end{align}
    where the second step adopts \pref{assum:domain-boundedness}.
    Note that $C_5$ is used to ensure the positive Bregman divergence term to be canceled and will be specified in the end.

    \paragraph{Base Regret Analysis.}

    For \emph{$\lambda$-strongly convex} functions, according to \pref{lem:str-convex-base}, the base regret can be bounded by
    \begin{align}
        \base\le{}&\frac{16G^2}{\lambda_\is} \log \sbr{1 + \lambda_{\is}G^2\Vb_{T,\is}}+\O(1)\notag\\
        \le{}&\frac{16G^2}{\lambda_\is} \log \sbr{1 + 3\lambda_{\is}G^2V_T+12\lambda_{\is}LG^2\sumT \D_{f_t}(\xs, \x_{t,\is})}+\O(1)\notag\\
        \le{}&\O\sbr{\frac{1}{\lambda}\log V_T}+\O(\log C_6)+ \frac{192LG^2}{C_6}\sumT\D_{f_t}(\x^{\star},\x_{t,\is}),
    \label{eq:sc-base}
    \end{align}
    where the last step uses $\log(1+x) \le x$ for any $x > -1$. 

    For \emph{$\alpha$-exp-concave} functions, by \pref{lem:exp-concave-base}, the base regret can be bounded by
    \begin{align}
        \base\le{}&\frac{16d}{\alpha_{\is}} \log \sbr{1 + \frac{\alpha_{\is}}{d}\Vb_{T,\is}}+ \O(1)\notag\\
        \le{}&\frac{16d}{\alpha_{\is}} \log \sbr{1 + \frac{3\alpha_{\is}}{d}V_T+\frac{12\alpha_{\is}L}{ d}\sumT\D_{f_t}(\x^{\star},\x_{t,\is})}+ \O(1)\notag\\
        \le{}& \O\sbr{\frac{d}{\alpha}\log {V_T}}+\O(\log C_7)+\frac{192L}{C_7}\sumT\D_{f_t}(\x^{\star},\x_{t,\is}),\label{eq:exp-base1}
    \end{align}
   where the last step is by $\log(1+x)\leq x$.

    For \emph{convex} functions, by \pref{lem:convex-base}, the base regret can be bounded by
    \begin{align}
        \base\le& 5D\sqrt{1+\Vb_{T,\is} }-\frac{\Bottomcoef}{4}S_{T,\is}^{\x}+\O(1)\notag\\
        \le& 5D\sqrt{1+3V_T+12L\sumT\D_{f_t}(\x^{\star},\x_{t,\is}) }-\frac{\Bottomcoef}{4}S_{T,\is}^{\x} + \O(1)\notag\\
        \le {}&\O(\sqrt{V_T})+\O(C_8)+\frac{5DL}{2C_8}\sumT\D_{f_t}(\x^{\star},\x_{t,\is})-\frac{\Bottomcoef}{4}S_{T,\is}^{\x}\label{eq:cvx-base}.
    \end{align}

\paragraph{Overall Regret Analysis.}
For \emph{$\lambda$-strongly convex} functions, plugging \pref{eq:scvx-meta} and \pref{eq:sc-base} into \pref{eq:multi-sc-reg-de}, we obtain
    \begin{align*}
        \Reg_T\le{}& \O\sbr{\frac{1}{\lambda}\log {V_T}}+\O(C_3+\log C_6)+\sbr{\frac{192LG^2}{C_6}-\frac{1}{2}} \sumT \D_{f_t}(\xs, \x_{t,\is})\\
        &+\sbr{\frac{C_0D^2}{2 C_3} - \frac{\lambda}{2}}\sumT \norm{\x_t - \x_{t,\is}}^2
        \le\O\sbr{\frac{1}{\lambda}\log {V_T}},
     \end{align*}
    by choosing $C_3={2C_0}/{\lambda}$ and $C_6=384LG^2$.
    
    For \emph{$\alpha$-exp-concave} functions, plugging \pref{eq:exp-meta} and \pref{eq:exp-base1} into \pref{eq:multi-exp-reg-de}, we obtain
    \begin{align*}
        \Reg_T\le{}& \O\sbr{\frac{d}{\alpha}\log {V_T}}+\O(C_4+\log C_7)+\sbr{\frac{192L}{C_7}-\frac{1}{2}} \sumT \D_{f_t}(\xs, \x_{t,\is})\\
        &+\sbr{\frac{C_0}{2 C_4} - \frac{\alpha}{2}} \sumT \inner{\g_t}{\x_t - \x_{t,\is}}^2
        \le\O\sbr{\frac{d}{\alpha}\log {V_T}},
    \end{align*}
    by choosing $C_4={2C_0}/{\alpha}$ and $C_7=384L$. 
    
    For \emph{convex} functions, plugging \pref{eq:cvx-meta} and \pref{eq:cvx-base} into \pref{eq:multi-cvx-reg-de}, we obtain
    \begin{align*}
    \Reg_T\le{}& \O(\sqrt{V_T})+\O(C_5+C_8)+\sbr{\frac{5DL}{2C_8}-1} \sumT \D_{f_t}(\xs, \x_{t,\is})+\sbr{\frac{LC_0}{8C_5}-\frac{\Bottomcoef}{4}}S_{T,\is}^{\x}\\
    &+\sbr{\frac{C_0}{2C_5}-1}\sumT \D_{f_t}(\x_{t,\is},\x_t) \le  \O(\sqrt{V_T}),
    \end{align*}
    by choosing $C_5=C_0$, $C_8=5DL$, and $\Bottomcoef \ge  L/2$.

Note that the constants $C_3, C_4, C_5, C_6, C_7, C_8$ appear only in the analysis, and hence our choices of them are feasible.
\end{proof}
\section{Omitted Proofs for Section~\ref{sec:one-gradient}}
\label{appendix:proof-one-gradient}

In this section, we provide the proofs for the regret guarantees of the efficient algorithms presented in Section~\ref{sec:one-gradient}, including the details of base learners' updates on surrogates, the proofs of \pref{prop:recover-MetaGrad}, \pref{thm:unigrad-correct-1grad}, and \pref{thm:unigrad-bregman-1grad}.

\subsection{Details of Base Learners' Update}
\label{app:base-update-1grad}
\paragraph{Base Learners.}
To begin with, we duplicate the candidate coefficient pool \eqref{eq:candidate-pool} for both the exp-concave coefficient $\alpha$ and the strongly convex coefficient $\lambda$, denoted by $\H^\exp \define \H$ and $\H^\scvx \define \H$.
Consequently, denoting by $N^\exp = N^\scvx \define |\H|$ the size of candidate pool, for each $\alpha_i \in \H^\exp$ and $\lambda_j \in \H^\scvx$, where $i \in [N^\exp]$ and $j \in [N^\scvx]$, we define corresponding groups of base learners for optimizing exp-concave and strongly convex functions.
Specifically, for \emph{$\alpha$-exp-concave} functions, we define a group of base learners $\{\B_i^\exp\}_{i \in [N^]}$, where the $i$-th base learner runs the algorithm below:
\begin{equation}
    \label{eq:base-exp}
    \begin{gathered}
        \x_{t,i} = \argmin_{\x \in \X} \bbr{\inner{\nabla \hexp_{t-1,i}(\x_{t-1,i})}{\x} + \D_{\psi_{t,i}}(\x, \xh_{t,i})},\\
        \xh_{t+1,i} = \argmin_{\x \in \X} \bbr{\inner{\nabla \hexp_{t,i}(\x_{t,i})}{\x} + \D_{\psi_{t,i}}(\x, \xh_{t,i})},
    \end{gathered}
\end{equation}
where $\psi_{t,i}(\x) \define \frac{1}{2} \x^\top U_{t,i} \x$, $U_{t,i} = (1+\frac{\alpha_i G^2}{2}) I + \frac{\alpha_i}{2} \sum_{s=1}^{t-1} \nabla \hexp_{s,i}(\x_{s,i}) \hexp_{s,i}(\x_{s,i})^\top$, $\alpha_i$ is the $\ith$ element in $\H^\exp$, and $\hexp_{t,i}(\cdot)$ is a surrogate loss function for $\B_i^\exp$, defined as 
\begin{equation*}
    \hexp_{t,i}(\x) \define \inner{\nabla f_t(\x_t)}{\x} + \frac{\alpha_i}{4} \inner{\nabla f_t(\x_t)}{\x - \x_t}^2.
\end{equation*}
Similarly, for \emph{$\lambda$-strongly convex} functions, we define a group of base learners $\{\B_i^\scvx\}_{i \in [N^\scvx]}$, where the $i$-th base learner runs the algorithm below:
\begin{equation}
    \label{eq:base-sc}
    \x_{t,i} = \Pi_\X [\xh_{t,i} - \eta_{t,i} \nabla \hsc_{t-1,i}(\x_{t-1,i})],\quad 
    \xh_{t+1,i} = \Pi_\X [\xh_{t,i} - \eta_{t,i} \nabla \hsc_{t,i}(\x_{t,i})],
\end{equation}
where $\eta_{t,i} = 2/(1 + \lambda_i t)$, $\lambda_i$ is the $\ith$ element in $\H^\scvx$, and $\hsc_{t,i}(\cdot)$ is a surrogate loss function for $\B_i^\scvx$, defined as 
\begin{equation*}
    \hsc_{t,i}(\x) \define \inner{\nabla f_t(\x_t)}{\x} + \frac{\lambda_i}{4}\|\x - \x_t\|^2.
\end{equation*}
For \emph{convex} functions, we only have to define one base learner $\B^{\text{c}}$, which updates as 
\begin{equation}
    \label{eq:base-cvx}
    \x_{t,i} = \Pi_\X [\xh_{t,i} - \eta_{t,i} \nabla f_{t-1}(\x_{t-1})],\quad 
    \xh_{t+1,i} = \Pi_\X [\xh_{t,i} - \eta_{t,i} \nabla f_t(\x_t)],
\end{equation}
where $\eta_{t,i} = \min\{D / \sqrt{1 + \sum_{s=2}^{t-1} \|\nabla f_t(\x_{t,i}) - \nabla f_{t-1}(\x_{t-1,i})\|^2}, 1\}$.
Finally, we conclude the configurations of base learners.
Specifically, we deploy 
\begin{equation}
    \label{eq:base-total}
    \{\B_i\}_{i \in [N]} \define \{\B_i^\exp\}_{i \in [N^\exp]} \cup \{\B_i^\scvx\}_{i \in [N^\scvx]} \cup \{\B^{\text{c}}\}, \text{ where } N \define N^\exp + N^\scvx + 1,
\end{equation}
as the total set of base learners.

\subsection{Proof of Proposition~\ref{prop:recover-MetaGrad}}
\label{app:recover-MetaGrad}
\begin{proof}
    For simplicity, we use $\g_t \define \nabla f_t(\x_t)$.
    For the meta regret, we use \mlprod~\citep{COLT'14:ML-Prod} to optimize the linear loss $\ellb_t = (\ell_{t,1}, \ldots, \ell_{t,N})$, where $\ell_{t,i} \define \frac{1}{2GD}\inner{\g_t}{\x_{t,i}} + \frac12 \in[0,1]$, and obtain the following second-order bound by Corollary 4 of \citet{COLT'14:ML-Prod},
    \begin{align*}
      \sumT \inner{\g_t}{\x_t - \x_{t,\is}}={}&2GD\sumT \inner{\ellb_t}{\p_t - \e_i}\\
       \lesssim{}&\sqrt{\log \log T \sumT \inner{\ellb_t}{\p_t - \e_i}^2}\lesssim \sqrt{\log \log T \sumT \inner{\g_t}{\x_t - \x_{t,\is}}^2}.
    \end{align*}
    For \emph{$\alpha$-exp-concave} functions, it holds that
    \begin{align*}
        \meta \lesssim {} & \sqrt{\log \log T \sumT \inner{\g_t}{\x_t - \x_{t,\is}}^2} - \frac{\alpha_\is}{2} \sumT \inner{\g_t}{\x_t - \x_{t,\is}}^2\\
        \le {} & \frac{\log \log T}{2 \alpha_\is} \le \frac{\log \log T}{\alpha} \tag*{(by $\alpha_\is \le \alpha \le 2 \alpha_\is$)},
    \end{align*}
    where the second step uses AM-GM inequality (\pref{lem:AM-GM}) with $a = \alpha_\is$.
    To handle the base regret, by optimizing the surrogate loss function $\hexp_{t, \is}$ using Online Newton Step (ONS), it holds that
    \begin{equation*}
        \base \lesssim \frac{d D G_\exp}{\alpha_\is} \log T \le \frac{2 d D (G + GD)}{\alpha} \log T,
    \end{equation*}
    where $G_\exp \define \max_{\x \in \X, t \in [T], i \in [N]} \|\nabla \hexp_{t,i}(\x)\| \le G + GD$ represents the maximum gradient norm the last step is because $\alpha \le 2 \alpha_\is$.
    Combining the meta and base regret, the regret can be bounded by $\O(d\log T)$. 
    
    For \emph{$\lambda$-strongly convex} functions, since it is also $\alpha = \lambda/G^2$ exp-concave under \pref{assum:gradient-boundedness}~{\citep[Section 2.2]{MLJ'07:Hazan-logT}}, the above meta regret analysis is still applicable.
    To optimize the base regret, by optimizing the surrogate loss function $\hsc_{t,\is}$ using Online Gradient Descent (OGD), it holds that
    \begin{equation*}
      \base \le \frac{\Gsc^2}{\lambda_\is} (1+\log T) \le \frac{2 (G+D)^2}{\lambda} (1+\log T),
    \end{equation*}
    where $\Gsc \define \max_{\x \in \X, t \in [T], i \in [N]} \|\nabla \hsc_{t,i}(\x)\| \le G+D$ represents the maximum gradient norm and the last step is because $\lambda \le 2 \lambda_\is$.
    Thus the overall regret can be bounded by $\O(\log T)$.
    
    For \emph{convex} functions, the meta regret can be bounded by $\O(\sqrt{T \log \log T})$, where the $\log \log T$ factor can be omitted in the $\O(\cdot)$-notation, and the base regret can be bounded by $\O(\sqrt{T})$ using OGD, resulting in an $\O(\sqrt{T})$ regret overall, which completes the proof.  
\end{proof}

\subsection{Proof of Theorem~\ref{thm:unigrad-correct-1grad}}
\label{app:unigrad-correct-1grad}
\begin{proof}
    Recall that we denote by $\g_t = \nabla f_t(\x_t)$ for simplicity.
    We first give different regret decompositions, then analyze the meta regret, and finally provide the proofs for different kinds of loss functions.
    Some abbreviations of the stability terms are defined in \eqref{eq:short}.

  \paragraph{Regret Decomposition.} 
    For \emph{$\lambda$-strongly convex} functions, we have
    \begin{align*}
        \Reg_T \le {} & \sumT f_t(\x_t) - \sumT f_t(\xs) \le \sumT \inner{\g_t}{\x_t - \xs} - \frac{\lambda_\is}{2} \sumT  \|\x_t - \xs\|^2 \\
        = {} & \underbrace{\sumT \inner{\g_t}{\x_t - \x_{t,\is}} - \frac{\lambda_\is}{2} \|\x_t - \x_{t,\is}\|^2}_{\meta} + \underbrace{\sumT \hsc_{t,\is}(\x_{t,\is}) - \sumT \hsc_{t,\is}(\xs)}_{\base},
    \end{align*}
  where the first step is by $\lambda_\is \le \lambda \le 2 \lambda_\is$ and the last step holds by defining surrogate loss functions $\hsc_{t,i}(\x) \define \inner{\g_t}{\x} + \frac{\lambda_i}{2} \|\x - \x_t\|^2$.
  Similarly, for \emph{$\alpha$-exp-concave} functions, the regret can be upper-bounded by  
  \begin{equation*}
    \Reg_T \le \underbrace{\sumT \inner{\g_t}{\x_t - \x_{t,\is}} - \frac{\alpha_\is}{2} \inner{\g_t}{\x_t - \xs}^2}_{\meta} + \underbrace{\sumT \hexp_{t,\is}(\x_{t,\is}) - \sumT \hexp_{t,\is}(\xs)}_{\base},
  \end{equation*}
  by defining surrogate loss functions $\hexp_{t,i}(\x) \define \inner{\g_t}{\x} + \frac{\alpha_i}{2}\inner{\g_t}{\x - \x_t}^2$. 

  For \emph{convex} functions, the regret can be decomposed as:
  \begin{equation*}
    \Reg_T \le \underbrace{\sumT \inner{\g_t}{\x_t - \x_{t,\is}}}_{\meta} + \underbrace{\sumT \inner{\g_t}{\x_{t,\is} - \xs}}_{\base}.
  \end{equation*}

  \paragraph{Meta Regret Analysis.}
  The structure of the meta regret analysis in this part parallels that in \pref{app:unigrad-correct}.
  Recall that, we let $\Vs=\sumTT(\ell^\Mid_{t,\js,\is} - m^\Mid_{t,\js,\is})^2$ for simplicity.

For \emph{$\lambda$-strongly convex} functions, applying \pref{eq:meta-scvx} with the stability and correction-induced negative terms, the meta regret satisfies 
\begin{align}
    & \meta \overset{\eqref{eq:meta-scvx}}{\le}  2ZC_0\log\frac{4N}{3}+\frac{512ZG^2}{\lambda}\log \frac{2^{20}G^2N}{3C_0^2\lambda^2}+ \gamma^\Mid S_{T,\is}^{\x}-\frac{C_0}{2}S_T^{\Top} \notag\\
    {}&- \gamma^\Mid \sumTT \sumM q_{t,j}^\Top \sumN q_{t,j,i}^\Mid\norm{\x_{t,i} - \x_{t-1,i}}^2  - \gamma^\Top \sumTT \sumM q_{t,j}^\Top \|\q_{t,j}^\Mid - \q_{t-1,j}^\Mid\|_1^2, \label{eq:thm3 scvx meta}
\end{align}
using \pref{lem:tune-eta-2} and requiring $\epsilon^\Top_{\js}\le\epsilon^\Top_{\star}\define \frac{\lambda_\is Z}{256G^2}$. 

For \emph{$\alpha$-exp-concave} functions, applying \pref{eq:meta-exp} with the stability and correction-induced negative terms, the meta regret can be similarly bounded by
\begin{align}
  & \meta \overset{\eqref{eq:meta-exp}}{\le} 2ZC_0\log\frac{4N}{3}+\frac{512Z}{\alpha}\log \frac{2^{20}N}{3C_0^2\alpha^2}+ \gamma^\Mid S_{T,\is}^{\x}-\frac{C_0}{2}S_T^{\Top} \notag\\
  {}&- \gamma^\Mid \sumTT \sumM q_{t,j}^\Top \sumN q_{t,j,i}^\Mid\norm{\x_{t,i} - \x_{t-1,i}}^2  - \gamma^\Top \sumTT \sumM q_{t,j}^\Top \|\q_{t,j}^\Mid - \q_{t-1,j}^\Mid\|_1^2,
  \label{eq:thm3 exp-concave meta}
\end{align}
using \pref{lem:tune-eta-2} and requiring $\epsilon^\Top_{\js}\le\epsilon^\Top_{\star}\define \frac{\alpha_\is Z}{512}$.

  For \emph{convex} functions, according to \pref{eq:thm1 eq2} in \pref{app:unigrad-correct}, the meta regret satisfies
  \begin{align}
    & \meta \le \sumT \inner{\g_t}{\x_t - \x_{t,\is}}\notag \\
    \le{}& 2ZC_0\log \frac{4N}{3 }+32D\sqrt{2V_T\log\sbr{512ND^2V_T/Z^2}}+\sbr{\frac{64G^2}{Z}+\gamma^\Mid} S_{T,\is}^{\x}+\frac{C_1}{2Z}S_T^{\x} \label{eq:thm3 cvx meta}\\
    - {} & \frac{C_0}{2}S_T^{\Top}- \gamma^\Mid \sumTT \sumM q_{t,j}^\Top \sumN q_{t,j,i}^\Mid\norm{\x_{t,i} - \x_{t-1,i}}^2  - \gamma^\Top \sumTT \sumM q_{t,j}^\Top \|\q_{t,j}^\Mid - \q_{t-1,j}^\Mid\|_1^2, \notag
\end{align}
where $C_1=128(D^2L^2+G^2)$.

\paragraph{Base Regret Analysis.}
In this part, we first provide different decompositions of the empirical gradient variation defined on surrogates for strongly convex, exp-concave, and convex functions, respectively, and then analyze the base regret in the corresponding cases.

For \emph{$\lambda$-strongly convex} functions, we bound the empirical gradient variation on surrogates, i.e., $\Vb_{T,\is}^{\scvx} \define \sumTT \|\nabla \hsc_{t,\is}(\x_{t,\is}) - \nabla \hsc_{t-1,\is}(\x_{t-1,\is})\|^2$, by
\begin{align*}
  \Vb_{T,\is}^{\scvx} = {} & \sumTT \|(\g_t + \lambda_\is (\x_{t,\is} - \x_t)) - (\g_{t-1} + \lambda_\is (\x_{t-1,\is} - \x_{t-1}))\|^2\\
  \le {} & 2 \sumTT \|\g_t - \g_{t-1}\|^2 + 2 \lambda_\is^2 \sumTT \|(\x_{t,\is} - \x_t) - (\x_{t-1,\is} - \x_{t-1})\|^2\\
  \le {} & 4V_T + (4 + 4L^2) S_T^\x + 4 S^\x_{T,\is} \tag*{(by~$\lambda \in [1/T, 1]$)},
\end{align*}
where the first step uses the definition of $\nabla \hsc_{t,i}(\x) = \g_t + \lambda_i (\x-\x_{t})$.
For \emph{$\alpha$-exp-concave} functions, we control the $\Vb^{\exp}_{T,\is} \define \sumTT \|\nabla \hexp_{t,\is}(\x_{t,\is}) - \nabla \hexp_{t-1,\is}(\x_{t-1,\is})\|^2$ as 
\begin{align*}
    \Vb^{\exp}_{T,\is} = {} & \sumTT \|(\g_t + \alpha_\is \g_t \g_t^\top (\x_{t,\is} - \x_t)) - (\g_{t-1} + \alpha_\is \g_{t-1} \g_{t-1}^\top (\x_{t-1,\is} - \x_{t-1}))\|^2\\
    \le {} & 2 \sumTT \|\g_t - \g_{t-1}\|^2 + 2 \alpha_\is^2 \sumTT \|\g_t \g_t^\top (\x_{t,\is} - \x_t) - \g_{t-1} \g_{t-1}^\top (\x_{t-1,\is} - \x_{t-1})\|^2\\
    \le {} & 4V_T + 4L^2 S_T^\x + 4 D^2 \sumTT \|\g_t \g_t^\top - \g_{t-1} \g_{t-1}^\top\|^2 \tag*{(by~\eqref{eq:empirical-gradient-variation})}\\
    & \qquad + 4 G^4 \sumTT \|(\x_{t,\is} - \x_t) - (\x_{t-1,\is} - \x_{t-1})\|^2 \tag*{(by~$\alpha \in [1/T, 1]$)}\\
    \le {} & C_9 V_T + C_{10} S_T^\x + 8G^4 S^\x_{T,\is},
  \end{align*}
  where the first step uses the definition of $\nabla \hexp_{t,i}(\x) = \g_t + \alpha_i \g_t \g_t^\top (\x - \x_t)$ and the last step holds by setting $C_9 = 4 + 32D^2 G^2$ and $C_{10} = 4L^2 + 32D^2 G^2 L^2 + 8G^4$. For \emph{convex} functions, the empirical gradient variation $\Vb^{\cvx}_{T,\is} \define \sumTT \|\nabla f_t(\x_t) - \nabla f_{t-1}(\x_{t-1})\|^2$ can be bounded by $\Vb^{\cvx}_{T,\is} \le 2V_{T}+2L^2S_T^{\x}$. To conclude, for different curvature types, we provide correspondingly different analysis of the empirical gradient variation on surrogates:
  \begin{equation}
    \label{eq:thm3 VbT-to-VT}
    \Vb_{T,\is}^{\{\scvx,\exp,\cvx\}} \le 
    \begin{cases}
        4V_T + (4 + 4L^2) S_T^\x + 4 S^\x_{T,\is}, & \text{when } \{f_t\}_{t=1}^T \text{ are $\lambda$-strongly convex}, \\[2mm]
        C_9 V_T + C_{10} S_T^\x + 8G^4 S^\x_{T,\is}, & \text{when } \{f_t\}_{t=1}^T \text{ are $\alpha$-exp-concave}, \\[2mm]
        2V_{T}+2L^2S_T^{\x}, & \text{when } \{f_t\}_{t=1}^T \text{ are convex}.
    \end{cases}
    \end{equation}
    In the following, we analyze the base regret for different curvature types.
    For \emph{$\lambda$-strongly convex} functions, by \pref{lem:str-convex-base}, the $\is$-th base learner guarantees the following:
    \begin{align}
        & \base \le \frac{16G^2}{\lambda_{\is}} \log \sbr{1 +\lambda_{\is} \Vb_{T,\is}^{\scvx} } + \frac{1}{4}\Bottomcoef D^2  - \frac{1}{8} \Bottomcoef S_{T,\is}^{\x} + \O(1)\label{eq:thm3 scvx}\\
        \le {} & \frac{16G^2}{\lambda_{\is}}\log (1 + 4\lambda_\is V_T + (4+ 4L^2) \lambda_\is S_T^\x + 4\lambda_\is S^\x_{T,\is})+ \frac{1}{4}\Bottomcoef D^2  - \frac{1}{8} \Bottomcoef S_{T,\is}^{\x}  \notag\\
        \le {} & \frac{32G^2}{\lambda}\log (1 + 4 \lambda V_T) + (64 + 64L^2)G^2 S_T^\x + 64G^2 S^\x_{T,\is}+ \frac{1}{4}\Bottomcoef D^2  - \frac{1}{8} \Bottomcoef S_{T,\is}^{\x} ,\label{eq:thm3 scvx base}
    \end{align}
  where the constant $\O(1)$ is omitted from the second step and the last step is due to $\log (1+x) \le x$ for $x \ge 0$. 
  
  For \emph{$\alpha$-exp-concave} functions, by \pref{lem:exp-concave-base}, the $\is$-th base learner guarantees:
  \begin{align}
    & \base \le \frac{16d}{\alpha_{\is}} \log \sbr{1 + \frac{\alpha_{\is}}{8\Bottomcoef d}\Vb^{\exp}_{T,\is} } + \frac{1}{2}\Bottomcoef D^2  - \frac{1}{4} \Bottomcoef S_{T,\is}^{\x} + \O(1)\label{eq: thm3 exp-concave}\\
    \le{}& \frac{16d}{\alpha_{\is}}\log \sbr{1 + \frac{\alpha_\is C_9}{8 \Bottomcoef d} V_T + \frac{\alpha_\is C_{10}}{8 \Bottomcoef d} S_T^\x + \frac{\alpha_\is G^4}{\Bottomcoef d} S^\x_{T,\is}} + \frac{1}{2}\Bottomcoef D^2  - \frac{1}{4} \Bottomcoef S_{T,\is}^{\x} + \O(1)\notag\\
    \le{}& \frac{32d}{\alpha} \log \sbr{1 + \frac{\alpha C_9}{8 \Bottomcoef d} V_T} + \frac{2C_{10}}{\Bottomcoef} S_T^\x + \sbr{\frac{16G^4}{\Bottomcoef} - \frac{\Bottomcoef}{4}} S^\x_{T,\is} + \frac{1}{2}\Bottomcoef D^2 + \O(1),\label{eq:thm3 exp-concave base}
  \end{align}
  where the last step is due to $\log (1+x) \le x$ for $x \ge 0$. 
  
  For \emph{convex} functions, by \pref{lem:convex-base}, the convex base learner guarantees the following:
  \begin{align}
    &\base\le 5D\sqrt{1+ \Vb^{\cvx}_{T,\is}} + \Bottomcoef D^2  - \frac{1}{4} \Bottomcoef S_{T,\is}^{\x} + \O(1)\notag\\
    \le {} &  5D\sqrt{1+ 2V_{T}+2L^2S_T^{\x}} + \Bottomcoef D^2  - \frac{1}{4} \Bottomcoef S_{T,\is}^{\x} + \O(1)\notag\\
    &\le5D\sqrt{1+2 V_{T}}+10DL^2S_T^{\x} + \Bottomcoef D^2  - \frac{1}{4} \Bottomcoef S_{T,\is}^{\x} + \O(1).\label{eq:thm3 cvx base}
  \end{align}

  \paragraph{Overall Regret Analysis.}
  For \emph{$\lambda$-strongly convex} functions, combining \pref{eq:thm3 scvx meta} and \pref{eq:thm3 scvx base}  and denoting  $C_{11}=128G^2(1+L^2)$, we obtain
  \begin{align*}
    &\Reg_T\le \O\sbr{\frac{1}{\lambda}\log V_T} + (64 + 64L^2)G^2 S_T^\x + \sbr{64G^2+\gamma^\Mid-\frac{1}{8} \Bottomcoef }S^\x_{T,\is}+ \frac{1}{4}\Bottomcoef D^2 \\
    &-\frac{C_0}{2}S_T^{\Top}- \gamma^\Mid \sumTT \sumM q_{t,j}^\Top \sumN q_{t,j,i}^\Mid\norm{\x_{t,i} - \x_{t-1,i}}^2  - \gamma^\Top \sumTT \sumM q_{t,j}^\Top \|\q_{t,j}^\Mid - \q_{t-1,j}^\Mid\|_1^2\\
    \le {} &  \O\sbr{\frac{1}{\lambda}\log V_T}  + \sbr{64G^2+\gamma^\Mid-\frac{1}{8} \Bottomcoef }S^\x_{T,\is}+ \frac{1}{4}\Bottomcoef D^2  \\
    & \qquad +\sbr{2D^2C_{11}-\frac{C_0}{2}}S_T^{\Top}+\sbr{C_{11}-\gamma^\Mid} \sumTT \sumM q_{t,j}^\Top \sumN q_{t,j,i}^\Mid\norm{\x_{t,i} - \x_{t-1,i}}^2 \\
    & \qquad +\sbr{ 2D^2C_{11}-\gamma^\Top} \sumTT \sumM q_{t,j}^\Top \|\q_{t,j}^\Mid - \q_{t-1,j}^\Mid\|_1^2\le \O\sbr{\frac{1}{\lambda}\log V_T},
  \end{align*}
  where the second step follows from \pref{lem:decompose-three-layer} and the last step requires $\gamma^{\Top} \ge  2D^2C_{11}$, $\gamma^{\Mid} \ge  C_{11}$, $\Bottomcoef \ge  512G^2+8\gamma^{\Mid}$, and $C_0 \ge  4D^2C_{11}$.

  For \emph{$\alpha$-exp-concave} functions, combining \pref{eq:thm3 exp-concave meta} and \pref{eq:thm3 exp-concave base}, we obtain
  \begin{align*}
    &\Reg_T\le \O\sbr{\frac{d}{\alpha}\log V_T}+ \frac{2C_{10}}{\Bottomcoef} S_T^\x + \sbr{\frac{16G^4}{\Bottomcoef}+\gamma^{\Mid} - \frac{\Bottomcoef}{4}} S^\x_{T,\is} + \frac{1}{2}\Bottomcoef D^2   \\
    &-\frac{C_0}{2}S_T^{\Top}- \gamma^\Mid \sumTT \sumM q_{t,j}^\Top \sumN q_{t,j,i}^\Mid\norm{\x_{t,i} - \x_{t-1,i}}^2  - \gamma^\Top \sumTT \sumM q_{t,j}^\Top \|\q_{t,j}^\Mid - \q_{t-1,j}^\Mid\|_1^2\\
    \le {} &  \O\sbr{\frac{d}{\alpha}\log V_T} +  \sbr{\frac{16G^4}{\Bottomcoef}+\gamma^{\Mid} - \frac{\Bottomcoef}{4}}S^\x_{T,\is}+ \frac{1}{2}\Bottomcoef D^2  \\
    &+\sbr{\frac{8D^2C_{10}}{\Bottomcoef}-\frac{C_0}{2}}S_T^{\Top}+\sbr{\frac{4C_{10}}{\Bottomcoef}-\gamma^\Mid} \sumTT \sumM q_{t,j}^\Top \sumN q_{t,j,i}^\Mid\norm{\x_{t,i} - \x_{t-1,i}}^2 \\
    & +\sbr{ \frac{8D^2C_{10}}{\Bottomcoef}-\gamma^\Top} \sumTT \sumM q_{t,j}^\Top \|\q_{t,j}^\Mid - \q_{t-1,j}^\Mid\|_1^2\le\O\sbr{\frac{d}{\alpha}\log V_T},
  \end{align*}
  where the second step follows from \pref{lem:decompose-three-layer} and the last step requires $\gamma^{\Top} \ge  8D^2C_{10}$, $\gamma^{\Mid} \ge  4C_{10}$, $\Bottomcoef \ge  64G^4+4\gamma^{\Mid}$, and $C_0 \ge  16D^2C_{10}$.

  For \emph{convex} functions, combining \pref{eq:thm3 cvx meta} and \pref{eq:thm3 cvx base}, we obtain
  \begin{align*}
    &\Reg_T\le \O\sbr{\sqrt{V_T\log V_T}}+\sbr{10DL^2+\frac{C_1}{2Z}}S_T^{\x}+\sbr{\frac{64G^2}{Z}+\gamma^{\Mid}-\frac{\Bottomcoef}{4}}S_{T,\is}^{\x}+\Bottomcoef D^2\\
    &-\frac{C_0}{2}S_T^{\Top}- \gamma^\Mid \sumTT \sumM q_{t,j}^\Top \sumN q_{t,j,i}^\Mid\norm{\x_{t,i} - \x_{t-1,i}}^2  - \gamma^\Top \sumTT \sumM q_{t,j}^\Top \|\q_{t,j}^\Mid - \q_{t-1,j}^\Mid\|_1^2\\ 
    \le {} & \O\sbr{\sqrt{V_T\log V_T}}+\sbr{\frac{64G^2}{Z}+\gamma^{\Mid}-\frac{\Bottomcoef}{4}}S_{T,\is}^{\x}+\sbr{40D^3L^2+\frac{2D^2C_1}{Z}-\frac{C_0}{2}}S_T^{\Top}\\
    &+\Bottomcoef D^2+\sbr{20DL^2+\frac{C_1}{Z}-\gamma^\Mid}\sumTT \sumM q_{t,j}^\Top\sumN q_{t,j,i}^\Mid\norm{\x_{t,i} - \x_{t-1,i}}^2\\
    & +\sbr{40D^3L^2+\frac{2D^2C_1}{Z}- \gamma^\Top} \sumTT \sumM q_{t,j}^\Top \|\q_{t,j}^\Mid - \q_{t-1,j}^\Mid\|_1^2\le\O\sbr{\sqrt{V_T\log V_T}},
  \end{align*}
  where the second step follows from \pref{lem:decompose-three-layer} and the last step requires $\gamma^{\Top} \ge  40D^3L^2+\frac{2D^2C_1}{Z}$, $\gamma^{\Mid} \ge  20DL^2+\frac{C_1}{Z}$, $\Bottomcoef \ge  256G^2+4\gamma^{\Mid}$, and $C_0 \ge  80D^3L^2+\frac{4D^2C_1}{Z}$.

At last, we determine the specific values of $C_0$, $\gamma^{\Top}$, and $\gamma^{\Mid}$.
These parameters need to satisfy the following requirements:
\begin{gather*}
  C_0 \ge  1,\ C_0 \ge  8D,\ C_0 \ge  4\gamma^{\Top},\ C_0 \ge  4D^2C_{11}, \ C_0 \ge  16D^2C_{10}, \ C_0 \ge  80D^3L^2+\frac{4D^2C_1}{Z},\\ 
  \gamma^{\Top} \ge  2D^2C_{11},\ \gamma^{\Top} \ge  8D^2C_{10},\ \gamma^{\Top} \ge  40D^3L^2+\frac{2D^2C_1}{Z},\ \gamma^{\Mid} \ge  C_{11},\\ \gamma^{\Mid} \ge  4C_{10},\text{ and } \gamma^{\Mid} \ge  20DL^2+\frac{C_1}{Z}.
\end{gather*}
As a result, we set 
\begin{gather*}
  C_0=\max\bbr{1,8D,4\gamma^{\Top},4D^2C_{11},16D^2C_{10},80D^3L^2+4D^2C_1},\\
  \gamma^\Top=\max\bbr{2D^2C_{11},8D^2C_{10},40D^3L^2+2D^2C_1},
  \gamma^{\Mid}=\max\bbr{C_{11},4C_{10}, 20DL^2+C_1},
\end{gather*}
where $Z=\max\{GD+\gamma^\Mid D^2, 1+\gamma^\Mid D^2+2\gamma^\Top\}$, $C_1=128(D^2L^2+G^2)$, $C_{10}=4L^2 + 32D^2 G^2 L^2 + 8G^4$, and $C_{11}=128G^2(1+L^2)$.

\end{proof}

\subsection{Proof of Theorem~\ref{thm:unigrad-bregman-1grad}}
\label{app:unigrad-bregman-1grad}
\begin{proof}
    Recall that, for simplicity, we denote by $\g_t \define \nabla f_t(\x_t)$.
    As shown in \pref{subsec:bregman-one-gradient}, the empirical gradient variation can bounded as 
    \begin{align}
        & \Vb_T \le 3 \sumTT \|\nabla f_t(\x_t) - \nabla f_t(\xs)\|^2 + 3 \sumTT \|\nabla f_t(\xs) - \nabla f_{t-1}(\xs)\|^2 \notag\\
        & + 3 \sumTT \|\nabla f_{t-1}(\xs) - \nabla f_{t-1}(\x_{t-1})\|^2 \le 6L \sumTT \D_{f_t}(\xs, \x_t) + 3 V_T + 6L \sumTT \D_{f_{t-1}}(\xs, \x_{t-1})\notag\\
        & \le 3V_T + 12L \sumT \D_{f_t}(\xs, \x_t). \label{eq:empirical-VT-de}
    \end{align}

    \paragraph{Regret Decomposition.} 
    For \emph{$\lambda$-strongly convex} functions, we decompose the regret as
    \begin{align}
        \Reg_T \le {} & \underbrace{\sumT \inner{\g_t}{\x_t - \x_{t,\is}} - \frac{\lambda_\is}{4} \sumT \|\x_t - \x_{t,\is}\|^2}_{\meta} \tag*{(by $\lambda_\is \le \lambda \le 2 \lambda_\is$)}\\
        & \qquad\qquad + \underbrace{\sumT \hsc_{t,\is}(\x_{t,\is}) - \sumT \hsc_{t,\is}(\xs)}_{\base} - \frac{1}{2}\sumT \D_{f_t}(\xs, \x_t), \label{eq:thm4 sc-reg-de}
    \end{align}
    due to the definition of the surrogate $\hsc_{t,i}(\x) \define \inner{\g_t}{\x} + \frac{\lambda_i}{4} \|\x - \x_t\|^2$, where $\lambda_i \in \H$ in~\eqref{eq:candidate-pool}.
    For \emph{$\alpha$-exp-concave} functions, we decompose the regret as 
    \begin{align}
        \Reg_T = {} & \sumT \inner{\g_t}{\x_t - \xs} - \frac{1}{2}\sumT \D_{f_t}(\xs, \x_t) - \frac{1}{2}\sumT \D_{f_t}(\xs, \x_t) \tag*{(by (\ref{eq:negative-Bregman}))}\\
        \le {} & \sumT \inner{\g_t}{\x_t - \xs} - \frac{\alpha}{4}\sumT \inner{\g_t}{\x_t - \xs}^2 - \frac{1}{2}\sumT \D_{f_t}(\xs, \x_t) \notag\\
        \le {} & \underbrace{\sumT \inner{\g_t}{\x_t - \x_{t,\is}} - \frac{\alpha_\is}{4} \sumT \inner{\g_t}{\x_t - \x_{t,\is}}^2}_{\meta} \tag*{(by $\alpha_\is \le \alpha \le 2 \alpha_\is$)}\\
        & \qquad\qquad + \underbrace{\sumT \hexp_{t,\is}(\x_{t,\is}) - \sumT \hexp_{t,\is}(\xs)}_{\base} - \frac{1}{2}\sumT \D_{f_t}(\xs, \x_t), \label{eq:thm4 exp-reg-de}
    \end{align}
    where the second step is due to the definitions of exp-concavity and Bregman divergence and the last step is due to the definition of the surrogate $\hexp_{t,i}(\x) \define \inner{\g_t}{\x} + \frac{\alpha_i}{4} \inner{\nabla f_t(\x_t)}{\x - \x_t}^2$, where $\alpha_i \in \H$, defined in~\eqref{eq:candidate-pool}.
    For \emph{convex} functions, by defining $\hc_{t,i}(\x) \define \inner{\g_t}{\x}$, we have
    \begin{align}
        \Reg_T = {} & \sumT \inner{\g_t}{\x_t - \xs} - \sumT \D_{f_t}(\xs, \x_t) \tag*{(by (\ref{eq:negative-Bregman}))}\\
        = {} & \underbrace{\sumT \inner{\g_t}{\x_t - \x_{t,\is}}}_{\meta} + \underbrace{\sumT \hc_{t,\is}(\x_{t,\is}) - \sumT \hc_{t,\is}(\xs)}_{\base} - \sumT \D_{f_t}(\xs, \x_t). \label{eq:thm4 cvx-reg-de}
    \end{align}

    \paragraph{Meta Regret Analysis.} 
    For \emph{strongly convex} functions and \emph{exp-concave} functions, we adopt a similar meta regret analysis as used in \pref{app:main-Bregman}.

    For \emph{convex} functions, similar to \pref{eq:cvx-meta}, we obtain
    \begin{align}
      & \meta \le C_0 \sqrt{4G^2D^2 + \sumT \inner{\g_t - \g_{t-1}}{\x_t - \x_{t,\is}}^2} + 2GDC_2\notag\\
      \le{}& C_0 \sqrt{1 + D^2 \Vb_T} + 2GDC_2 \le C_0 \sqrt{4G^2D^2 + 3D^2 V_T + 12LD^2 \sumT \D_{f_t}(\xs, \x_t)} + 2GDC_2\notag\\
      \le {} & \O(\sqrt{V_T}) + C_0\sqrt{12LD^2 \sumT \D_{f_t}(\xs, \x_t)} \le \O(\sqrt{V_T}) + \O(C_{12}) + \frac{C_0}{2C_{12}} \sumT \D_{f_t}(\xs,\x_t),
      \label{eq:thm4 cvx-meta}
  \end{align}
  where the second step adopts \pref{assum:domain-boundedness} and the third step is by \pref{eq:meta-decompose}.
  $C_{12}$ is used to ensure the positive Bregman divergence term to be canceled and will be specified finally.
  
    \paragraph{Base Regret Analysis.} 
    Following the analysis structure of \pref{app:unigrad-correct-1grad}, we first provide different decompositions of the empirical gradient variation defined on surrogates for strongly convex, exp-concave, and convex functions, respectively, and then analyze the base regret in the corresponding cases.

    For \emph{$\lambda$-strongly convex} functions, we bound the empirical gradient variation on surrogates, i.e.,  $\Vb_{T,\is}^{\scvx} \define \sumTT \|\nabla \hsc_{t,\is}(\x_{t,\is}) - \nabla \hsc_{t-1,\is}(\x_{t-1,\is})\|^2$, by
    \begin{align}
      \Vb_{T,\is}^{\scvx}= {} & \sumTT \norm{\g_t + \frac{\lambda_\is}{2} (\x_{t,\is} - \x_t) - \g_{t-1} - \frac{\lambda_\is}{2} (\x_{t-1,\is} - \x_{t-1})}^2 \notag\\
      \le {} & 3\Vb_T + 3 \sumTT \norm{\frac{\lambda_\is}{2} (\x_{t,\is} - \x_t)}^2 + 3 \sumTT \norm{\frac{\lambda_\is}{2} (\x_{t-1,\is} - \x_{t-1})}^2 \label{eq:scvx-base-2}\\
      \le {} & 9 V_T + 36L \sumT \D_{f_t}(\xs, \x_t) + 2 \lambda_\is^2 \sumTT \norm{\x_{t,\is} - \x_t}^2, \tag*{(by \eqref{eq:empirical-VT-de})}
    \end{align}
    where the first step is due to the property of the surrogate: $\nabla \hsc_{t,i}(\x_{t,i}) = \g_t + \frac{\lambda_i}{2} (\x_{t,i} - \x_t)$, and the second step is due to the Cauchy-Schwarz inequality.
    For \emph{$\alpha$-exp-concave} functions, we control the $\Vb^{\exp}_{T,\is} \define \sumTT \|\nabla \hexp_{t,\is}(\x_{t,\is}) - \nabla \hexp_{t-1,\is}(\x_{t-1,\is})\|^2$ as 
    \begin{align}
      \Vb^{\exp}_{T,\is} 
      = {} & \sumTT \norm{\g_t + \frac{\alpha_\is}{2} \g_t \inner{\g_t}{\x_t - \x_{t,\is}} - \g_{t-1} - \frac{\alpha_\is}{2} \g_{t-1} \inner{\g_{t-1}}{\x_{t-1} - \x_{t-1,\is}}}^2 \notag\\
      \le {} & 3 \Vb_T + 3 \sumTT \norm{\frac{\alpha_\is}{2} \g_t \inner{\g_t}{\x_t - \x_{t,\is}}}^2 + 3 \sumTT \norm{\frac{\alpha_\is}{2} \g_{t-1} \inner{\g_{t-1}}{\x_{t-1} - \x_{t-1,\is}}}^2 \notag\\
      \le {} & 3 \Vb_T + 6 \sumT \norm{\frac{\alpha_\is}{2} \g_t \inner{\g_t}{\x_t - \x_{t,\is}}}^2 \label{eq:exp-base-2}\\
      \overset{\eqref{eq:empirical-VT-de}}{\le} {} & 9 V_T + 36L \sumT \D_{f_t}(\xs, \x_t) + 2 \alpha_\is^2 G^2 \sumT \inner{\g_t}{\x_t - \x_{t,\is}}^2, \tag*{(by \pref{assum:gradient-boundedness})}
    \end{align}
      where the first step is due to the property of the surrogate function: $\nabla \hexp_{t,i}(\x_{t,i}) = \g_t + \frac{\alpha_i}{2} \g_t \inner{\g_t}{\x_t - \x_{t,i}}$, the second step is by the Cauchy-Schwarz inequality.
      For \emph{convex} functions, the empirical gradient variation $\Vb^{\cvx}_{T,\is} \define \sumTT \|\nabla f_t(\x_t) - \nabla f_{t-1}(\x_{t-1})\|^2$ can be bounded by $\Vb^{\cvx}_{T,\is} \le 2V_{T}+2L^2S_T^{\x}$.
      To conclude, for different curvature types, we provide correspondingly different analysis of the empirical gradient variation on surrogates:
\begin{equation}
  \label{eq:thm4 VbT-to-VT}
  \Vb_{T,\is}^{\{\scvx,\exp,\cvx\}} \le
  \begin{cases}
      9 V_T + 36L \sumT \D_{f_t}(\xs, \x_t) + 2 \lambda_\is^2 \sumTT \norm{\x_{t,\is} - \x_t}^2,
      & \text{ ($\lambda$-strongly convex)}, \\[2.5ex]
      
      9 V_T + 36L \sumT \D_{f_t}(\xs, \x_t) + 2 \alpha_\is^2 G^2 \sumT \inner{\g_t}{\x_t - \x_{t,\is}}^2,
      & \text{ ($\alpha$-exp-concave)}, \\[2.5ex]

      3V_{T} + 12L\sumT \D_{f_t}(\xs, \x_t),
      & \text{(convex)}.
  \end{cases}
  \end{equation}

        In the following, we analyze the base regret for different curvature types.
    For \emph{$\lambda$-strongly convex} functions, when using the update rule~\eqref{eq:base-sc}, according to \pref{lem:str-convex-base}, the base regret can be bounded as 
    \begin{align}
      &\base \le \frac{16G^2}{\lambda_\is} \log \sbr{1 + \lambda_\is \Vb_{T,\is}^{\scvx}} + \O(1)\notag\\
      \le {} & \frac{16G^2}{\lambda_\is} \log \sbr{1 + 9 \lambda_\is V_T + 36L \lambda_\is \sumT \D_{f_t}(\xs, \x_t) + 2 \lambda_\is^3 \sumT \norm{\x_{t,\is} - \x_t}^2}\notag\\
      \le {} & \O\sbr{\frac{1}{\lambda}\log (C_{13} V_T)} + \frac{16 G^2}{C_{13} \lambda_\is} \sbr{36L \lambda_\is \sumT \D_{f_t}(\xs, \x_t) + 2 \lambda_\is^3 \sumT \norm{\x_{t,\is} - \x_t}^2}\notag\\
      \le {} &   \O\sbr{\frac{1}{\lambda}\log V_T} + \frac{576 G^2 L}{C_{13}} \sumT \D_{f_t}(\xs, \x_t) + \frac{32G^2}{C_{13}} \sumT \norm{\x_{t,\is} - \x_t}^2 + \O(\log C_{13}),\label{eq:thm4 scvx base}
    \end{align}
    where the third step requires $C_{13} \ge 1$ by \pref{lem:ln-de} and uses the property of the best base learner, i.e., $\lambda_\is \le \lambda \le 2 \lambda_\is$.
    The last step is due to $\lambda_i \le 1$.
    For \emph{exp-concave} functions, when using the update rule~\eqref{eq:base-exp}, by \pref{lem:exp-concave-base}, the base regret can be bounded as
      \begin{align}
        &\base\le \frac{16d}{\alpha_\is} \log \sbr{1 + \frac{\alpha_\is}{8d} \Vb_{T,\is}^{\exp}} + \O(1) \notag\\
        \le {} & \frac{16d}{\alpha_\is} \log \sbr{1 + \frac{9\alpha_\is}{8d} V_T + \frac{9\alpha_\is L}{2d} \sumT \D_{f_t}(\xs, \x_t) + \frac{\alpha_\is^3 G^2}{4d} \sumT \inner{\g_t}{\x_t - \x_{t,\is}}^2}\notag\\
        \le {} & \O\sbr{\frac{d}{\alpha} \log (C_{14} V_T)} + \frac{16d}{C_{14} \alpha_\is} \sbr{\frac{9\alpha_\is L}{2d} \sumT \D_{f_t}(\xs, \x_t) + \frac{\alpha_\is^3 G^2}{4d} \sumT \inner{\g_t}{\x_t - \x_{t,\is}}^2}\notag\\
        \le {} & \O\sbr{\frac{d}{\alpha}\log V_T} + \frac{72L}{C_{14}} \sumT \D_{f_t}(\xs, \x_t) + \frac{4G^2}{C_{14}} \sumT \inner{\g_t}{\x_t - \x_{t,\is}}^2 + \O(\log C_{14}).\label{eq:thm4 exp base}
      \end{align}
      The third step requires $C_{14} \ge 1$ by \pref{lem:ln-de} and uses the property of the best base learner, i.e., $\alpha_\is \le \alpha \le 2 \alpha_\is$.
      The last step is because of $\alpha_i \le 1$.
      For \emph{convex} functions, when using the update rule~\eqref{eq:base-cvx}, by \pref{lem:convex-base}, we obtain 
    \begin{align}
        & \base \le 5D \sqrt{1 + \Vb_{T,\is}^\cvx } + \O(1)\notag\\
        = {} & 5D \sqrt{1 + \Vb_T} + \O(1) \le \O(\sqrt{V_T}) + 5D \sqrt{12L \sumT \D_{f_t}(\xs, \x_t)} \tag*{(by \eqref{eq:empirical-VT-de})}\\
        \le {} & \O(\sqrt{V_T}) + \O(C_{15}) + \frac{5DL}{2C_{15}} \sumT \D_{f_t}(\xs, \x_t),\label{eq:thm4 cvx base}
    \end{align}
    where the  second step is due to the property of the surrogate function: $\nabla \hc_{t,i}(\x_{t,i}) = \g_t$, and the last step uses AM-GM inequality (\pref{lem:AM-GM}).
    $C_{15}$ is a constant to be specified.
    
    \paragraph{Overall Regret Analysis.} 
    For \emph{$\lambda$-strongly convex} functions, plugging \pref{eq:scvx-meta} and \pref{eq:thm4 scvx base} into \pref{eq:thm4 sc-reg-de}, we obtain
    \begin{align*}
        \Reg_T\le{}& \O\sbr{\frac{1}{\lambda}\log {V_T}}+\O(C_3+\log C_{13})+\sbr{\frac{576G^2L}{C_{13}}-\frac{1}{2}} \sumT \D_{f_t}(\xs, \x_{t})\\
        &+\sbr{\frac{C_0D^2}{2 C_3}+\frac{32G^2}{C_{13}} - \frac{\lambda}{4}}\sumT \norm{\x_t - \x_{t,\is}}^2
        \le\O\sbr{\frac{1}{\lambda}\log {V_T}},
     \end{align*}
    by choosing $C_3={4C_0D^2}/{\lambda}$ and $C_{13}=\max\{1,256G^2/\lambda, 1152G^2L\}$.
    For \emph{$\alpha$-exp-concave} functions, plugging \pref{eq:exp-meta} and \pref{eq:thm4 exp base} into \pref{eq:thm4 exp-reg-de}, we obtain
    \begin{align*}
        \Reg_T\le{}& \O\sbr{\frac{d}{\alpha}\log {V_T}}+\O(C_4+\log C_{14})+\sbr{\frac{72L}{C_{14}}-\frac{1}{2}} \sumT \D_{f_t}(\xs, \x_{t})\\
        &+\sbr{\frac{C_0}{2 C_4}+\frac{4G^2}{C_{14}} - \frac{\alpha}{4}} \sumT \inner{\g_t}{\x_t - \x_{t,\is}}^2
        \le\O\sbr{\frac{d}{\alpha}\log {V_T}},
    \end{align*}
    by choosing $C_4={4C_0}/{\alpha}$ and $C_{14}=\max\{1,144L,32G^2/\alpha\}$.
    For \emph{convex} functions, plugging \pref{eq:thm4 cvx-meta} and \pref{eq:thm4 cvx base} into \pref{eq:thm4 cvx-reg-de}, we obtain
    \begin{align*}
    \Reg_T\le{}& \O(\sqrt{V_T})+\O(C_{12}+C_{15})+\sbr{\frac{5DL}{2C_{15}}+\frac{C_0}{2C_{12}}-1} \sumT \D_{f_t}(\xs, \x_{t}) \le  \O(\sqrt{V_T}),
    \end{align*}
    by choosing $C_{12}=C_0$ and $C_{15}=5DL$.

    Note that the constants $C_3, C_4, C_{12}, C_{13}, C_{14},C_{15}$ only exist in analysis and hence our choices of them are feasible.
\end{proof}
\section{Omitted Proofs for Section~\ref{sec:applications}}
\label{appendix:proof-applications}

In this section, we present the formal proofs supporting the theoretical implications of our methods, specifically regarding the small-loss and gradient-variance bounds discussed in Section~\ref{subsec:implication}.
We also provide complete proofs for the applications of our results to the SEA model (Section~\ref{subsec:applications-SEA}) and to online games (Section~\ref{subsec:applications-games}), including detailed proofs of \pref{cor:FT-WT-Correct}, \pref{cor:FT-bregman}, \pref{thm:SEA-correct}, \pref{thm:SEA-bregman}, and \pref{thm:game}.
Finally, we provide the proof of the extension to anytime variant (\pref{subsec:extension-anytime}).

\subsection{Proof of Corollary~\ref{cor:FT-WT-Correct}}
\label{app:FT-Correct}
    We prove the small-loss regret guarantees of \Correctpp in \pref{app:small-loss-correct-proof} and the gradient-variance regret guarantees of \Correctpp in \pref{app:gradient-variance-correct-proof}.

    \subsubsection{Small-Loss Regret}
    \label{app:small-loss-correct-proof}
    \begin{proof}
    For simplicity, we define $F_T^{\x} \triangleq \sumT f_t(\x_t)-\sumT \min_{\x\in\X_+}f_t(\x)$.
    We adopt the same regret decomposition strategy as utilized in \pref{app:unigrad-correct-1grad}.

    \paragraph{Meta Regret Analysis.}
    Recall that the normalization factor $Z=\max\{GD+\gamma^\Mid D^2, 1+\gamma^\Mid D^2+2\gamma^\Top\}$.
    For \emph{strongly convex} and \emph{exp-concave} functions, we follow the same meta regret analysis as used in \pref{app:unigrad-correct-1grad}.

    For \emph{convex} functions, we give a different analysis for $\Vs=\sumTT (\ell^\Mid_{t,\js,\is} - m^\Mid_{t,\js,\is})^2$. From \pref{lem:universal-optimism}, it holds that
    \begin{align}
    \Vs \le{}& \frac{2 D^2}{Z^2} \sumT\norm{\g_t-\g_{t-1}}^2+\frac{2 G^2}{Z^2} \sumT\norm{\mathbf{x}_{t, i^{\star}}-\mathbf{x}_{t-1, i^{\star}}+\mathbf{x}_{t-1}-\mathbf{x}_t}^2\notag\\
    \le{}&\frac{8D^2}{Z^2}\sumT \norm{\g_t}^2 + \frac{4G^2}{Z^2} S_{T,\is}^{\x} +\frac{4G^2}{Z^2} S_T^{\x} \le \frac{32D^2L}{Z^2} F_T^{\x} + \frac{4G^2}{Z^2} S_{T,\is}^{\x} +\frac{4G^2}{Z^2} S_T^{\x},\label{eq:cor eq1}
    \end{align}
    where the last step is by the self-bounding property of $\|\nabla f(\x)\|_2^2 \le 4 L (f(\x)  - \min_{\x \in \X_+} f(\x))$ for any $\x \in \X$.
    Thus, the meta regret is bounded as
    \begin{align}
        & \meta 
        \le \frac{Z}{\varepsilon^\Top_{\js}} \log \frac{N}{3 C_0^2 (\varepsilon^\Top_{\js})^2} + 32 Z\varepsilon^\Top_{\js} \Vs+ \gamma^\Mid S_{T,\is}^{\x}-\frac{C_0}{2}S_T^{\Top}\notag\\
        {}&-\frac{C_0}{2}S_T^{\Top}- \gamma^\Mid \sumTT \sumM q_{t,j}^\Top \sumN q_{t,j,i}^\Mid\norm{\x_{t,i} - \x_{t-1,i}}^2  - \gamma^\Top \sumTT \sumM q_{t,j}^\Top \|\q_{t,j}^\Mid - \q_{t-1,j}^\Mid\|_1^2\notag \\
        \le {} &  \frac{Z}{\varepsilon^\Top_{\js}} \log \frac{N}{3 C_0^2 (\varepsilon^\Top_{\js})^2} +\frac{1024D^2L}{Z} \varepsilon^\Top_{\js} F_T^{\x}+\sbr{\frac{64G^2}{Z}+\gamma^\Mid} S_{T,\is}^{\x}+\frac{64G^2}{Z}S_T^{\x}\notag\\
        {}&-\frac{C_0}{2}S_T^{\Top}- \gamma^\Mid \sumTT \sumM q_{t,j}^\Top \sumN q_{t,j,i}^\Mid\norm{\x_{t,i} - \x_{t-1,i}}^2  - \gamma^\Top \sumTT \sumM q_{t,j}^\Top \|\q_{t,j}^\Mid - \q_{t-1,j}^\Mid\|_1^2\notag \\
        \le{}& \frac{Z}{\varepsilon^\Top_{\js}} \log \frac{N}{3 C_0^2 (\varepsilon^\Top_{\js})^2} +\frac{1024D^2L}{Z} \varepsilon^\Top_{\js} F_T^{\x}+\sbr{\frac{64G^2}{Z}+\gamma^\Mid}S_{T,\is}^{\x}\notag\\
        {}&+\sbr{\frac{256D^2G^2}{Z}-\frac{C_0}{2}}S_T^{\Top}+\sbr{\frac{128G^2}{Z}-\gamma^{\Mid}}\sumTT \sumM q_{t,j}^\Top \sumN q_{t,j,i}^\Mid\norm{\x_{t,i} - \x_{t-1,i}}^2\notag\\
        {}&+\sbr{\frac{256D^2G^2}{Z}-\gamma^{\Top}}\sumTT \sumM q_{t,j}^\Top \|\q_{t,j}^\Mid - \q_{t-1,j}^\Mid\|_1^2\notag\\
        \le{}&  \frac{Z}{\varepsilon^\Top_{\js}} \log \frac{N}{3 C_0^2 (\varepsilon^\Top_{\js})^2} +\frac{1024D^2L}{Z} \varepsilon^\Top_{\js} F_T^{\x} +\sbr{\frac{64G^2}{Z}+\gamma^\Mid}S_{T,\is}^{\x}, \label{eq:cor1 cvx meta F_T}
    \end{align}
    where the second step is by \pref{lem:tune-eta-1} and requiring $C_0 \ge  8D$, the third step is by \pref{lem:decompose-three-layer}, and the final step requires $\gamma^{\Top} \ge  \frac{256D^2G^2}{Z},  \gamma^{\Mid} \ge  \frac{128G^2}{Z}\text{, and }C_0 \ge  \frac{512D^2G^2}{Z}$.
    \paragraph{Base Regret Analysis.}
    We first provide different decompositions of the empirical gradient variation defined on surrogates for strongly convex, exp-concave, and convex functions, respectively, and then analyze the base regret in the corresponding cases.
    For \emph{$\lambda$-strongly convex} functions, $\Vb_{T,i}^{\scvx} \define \sumTT \|\nabla \hsc_{t,i}(\x_{t,\is}) - \nabla \hsc_{t-1,i}(\x_{t-1,\is})\|^2$ can be bounded by
    \begin{align*}
        \Vb_{T,\is}^{\scvx} = {} & \sumTT \|(\g_t + \lambda_\is (\x_{t,\is} - \x_t)) - (\g_{t-1} + \lambda_\is (\x_{t-1,\is} - \x_{t-1}))\|^2\\
        \le {} & 2 \sumTT \|\g_t - \g_{t-1}\|^2 + 2 \lambda_\is^2 \sumTT \|(\x_{t,\is} - \x_t) - (\x_{t-1,\is} - \x_{t-1})\|^2\\
        \le {} & 32LF_T^{\x} + 4S_T^\x + 4 S^\x_{T,\is} \tag*{(by~$\lambda \in [1/T, 1]$)},
    \end{align*}
    where the first step uses the definition of $\nabla \hsc_{t,i}(\x)$.
    For \emph{$\alpha$-exp-concave} functions, $\Vb^{\exp}_{T,\is} \define \sumTT \|\nabla \hexp_{t,\is}(\x_{t,\is}) - \nabla \hexp_{t-1,\is}(\x_{t-1,\is})\|^2$ can be bounded by
    \begin{align*}
        \Vb^{\exp}_{T,\is} = {} & \sumTT \|(\g_t + \alpha_\is \g_t \g_t^\top (\x_{t,\is} - \x_t)) - (\g_{t-1} + \alpha_\is \g_{t-1} \g_{t-1}^\top (\x_{t-1,\is} - \x_{t-1}))\|^2\notag\\
        \le {} & 2 \sumTT \|\g_t - \g_{t-1}\|^2 + 2 \alpha_\is^2 \sumTT \|\g_t \g_t^\top (\x_{t,\is} - \x_t) - \g_{t-1} \g_{t-1}^\top (\x_{t-1,\is} - \x_{t-1})\|^2\notag\\
        \le {} & 32L F_T^{\x} +  4 D^2 \sumTT \|\g_t \g_t^\top - \g_{t-1} \g_{t-1}^\top\|^2 
        + 4 G^4 \sumTT \|(\x_{t,\is} - \x_t) - (\x_{t-1,\is} - \x_{t-1})\|^2 \tag*{(by~$\alpha \in [1/T, 1]$)}\notag\\
        \le {} & C_{29}F_T^{\x} + 8G^4 S_T^\x + 8G^4 S^\x_{T,\is},
    \end{align*}
    where the first step uses the definition of $\nabla \hexp_{t,i}(\x) = \g_t + \alpha_i \g_t \g_t^\top (\x - \x_t)$ and the last step holds by setting $C_{29} = 32L+256D^2G^2L$.
    
    For \emph{convex} functions, $\Vb_{T,\is}^{\cvx}$ can be bounded by
    $\Vb^{\cvx}_{T,\is} \define \sumTT \|\g_t - \g_{t-1}\|^2\le  16LF_T^{\x}$. 
    To conclude, for different curvature types, we provide correspondingly different analysis of the empirical gradient variation on surrogates:
    \begin{equation}
      \label{eq:cor1 VbT-to-VT}
      \Vb_{T,\is}^{\{\scvx,\exp,\cvx\}} \le 
      \begin{cases}
        32LF_T^{\x} + 4S_T^\x + 4 S^\x_{T,\is}, & \text{when } \{f_t\}_{t=1}^T \text{ are $\lambda$-strongly convex}, \\[2mm]
        C_{29}F_T^{\x} + 8G^4 S_T^\x + 8G^4 S^\x_{T,\is}, & \text{when } \{f_t\}_{t=1}^T \text{ are $\alpha$-exp-concave}, \\[2mm]
        16LF_T^{\x}, & \text{when } \{f_t\}_{t=1}^T \text{ are convex}.
      \end{cases}
      \end{equation}

      In the following, we analyze the base regret for different curvature types. For \emph{$\lambda$-strongly convex} functions, by \pref{lem:str-convex-base}, the $\is$-th base learner guarantees the following:
    \begin{align}
        &\base \le \frac{16G^2}{\lambda_{\is}}\log (1 + 32\lambda_\is L F_T^{\x} + 4 \lambda_\is S_T^\x + 4\lambda_\is S^\x_{T,\is})+ \frac{1}{4}\Bottomcoef D^2  - \frac{1}{8} \Bottomcoef S_{T,\is}^{\x} + \O(1)\notag\\
        \le {} & \frac{32G^2}{\lambda}\log (1 + 32\lambda L F_T^{\x}) + 64G^2 S_T^\x + 64G^2 S^\x_{T,\is}+ \frac{1}{4}\Bottomcoef D^2  - \frac{1}{8} \Bottomcoef S_{T,\is}^{\x} + \O(1),\label{eq:cor1 scvx base F_T}
    \end{align}
    where the last step is due to $\log (1+x) \le x$ for $x \ge 0$.
    For \emph{$\alpha$-exp-concave} functions, by \pref{lem:exp-concave-base}, the $\is$-th base learner guarantees the following:
    \begin{align}
        &\base \le \frac{16d}{\alpha_{\is}}\log \sbr{1 + \frac{\alpha_\is C_{29}}{8 \Bottomcoef d} F_T^{\x} + \frac{\alpha_\is G^4}{ \Bottomcoef d} S_T^\x + \frac{\alpha_\is G^4}{\Bottomcoef d} S^\x_{T,\is}} + \frac{1}{2}\Bottomcoef D^2  - \frac{1}{4} \Bottomcoef S_{T,\is}^{\x} + \O(1)\notag\\
        \le{}& \frac{32d}{\alpha} \log \sbr{1 + \frac{\alpha C_{29}}{8 \Bottomcoef d} F_T^{\x}} + \frac{16G^4}{\Bottomcoef} S_T^\x + \sbr{\frac{16G^4}{\Bottomcoef} - \frac{\Bottomcoef}{4}} S^\x_{T,\is} + \frac{1}{2}\Bottomcoef D^2 + \O(1),\label{eq:cor1 exp-concave base F_T}
    \end{align}
    where the last step is due to $\log (1+x) \le x$ for $x \ge 0$.
    
    For \emph{convex} functions, by \pref{lem:convex-base}, the convex base learner guarantees the following:
    \begin{align}
        \base\le{}& 5D\sqrt{1+ \Vb^{\cvx}_{T,\is}} + \Bottomcoef D^2  - \frac{1}{4} \Bottomcoef S_{T,\is}^{\x} + \O(1)\notag\\
        \le{}& 5D\sqrt{1+ 16LF_T^{\x}} + \Bottomcoef D^2  - \frac{1}{4} \Bottomcoef S_{T,\is}^{\x} + \O(1).\label{eq:cor1 cvx base F_T}
    \end{align}

    \paragraph{Overall Regret Analysis.}
    For \emph{$\lambda$-strongly convex} functions, combining \pref{eq:thm3 scvx meta} and \pref{eq:cor1 scvx base F_T}, we obtain
    \begin{align*}
        &\Reg_T\le \O\sbr{\frac{1}{\lambda}\log F_T^{\x}} + 64G^2 S_T^\x + \sbr{64G^2+\gamma^\Mid-\frac{1}{8} \Bottomcoef }S^\x_{T,\is}+ \frac{1}{4}\Bottomcoef D^2 \\
        &-\frac{C_0}{2}S_T^{\Top}- \gamma^\Mid \sumTT \sumM q_{t,j}^\Top \sumN q_{t,j,i}^\Mid\norm{\x_{t,i} - \x_{t-1,i}}^2  - \gamma^\Top \sumTT \sumM q_{t,j}^\Top \|\q_{t,j}^\Mid - \q_{t-1,j}^\Mid\|_1^2\\
        \le {} &  \O\sbr{\frac{1}{\lambda}\log F_T^{\x}}  + \sbr{64G^2+\gamma^\Mid-\frac{1}{8} \Bottomcoef }S^\x_{T,\is}+ \frac{1}{4}\Bottomcoef D^2  \\
        & \qquad +\sbr{256D^2G^2-\frac{C_0}{2}}S_T^{\Top}+\sbr{128G^2-\gamma^\Mid} \sumTT \sumM q_{t,j}^\Top \sumN q_{t,j,i}^\Mid\norm{\x_{t,i} - \x_{t-1,i}}^2 \\
        & \qquad +\sbr{ 256D^2G^2-\gamma^\Top} \sumTT \sumM q_{t,j}^\Top \|\q_{t,j}^\Mid - \q_{t-1,j}^\Mid\|_1^2\le \O\sbr{\frac{1}{\lambda}\log F_T^{\x}}\le \O\sbr{\frac{1}{\lambda}\log F_T},
    \end{align*}
    where the second step follows from \pref{lem:decompose-three-layer}, the third step requires $\gamma^{\Top} \ge  256D^2G^2$, $\gamma^{\Mid} \ge  128G^2$, $\Bottomcoef \ge  512G^2+8\gamma^{\Mid}$, and $C_0 \ge  512D^2G^2$, and the last step uses \pref{lem:small-loss-log} by choosing $a,b,c$ as some $T$-independent constants and setting
    \begin{equation*}
        x=\sumT f_t(\x_t)-\sumT\min_{\x\in\X_+}f_t(\x), \text{ and } d=\min_{\x\in\X_+}\sumT f_t(\x)-\sumT\min_{\x\in\X_+}f_t(\x).
    \end{equation*}
    For \emph{$\alpha$-exp-concave} functions, combining \pref{eq:thm3 exp-concave meta} and \pref{eq:cor1 exp-concave base F_T}, we obtain
    \begin{align*}
        &\Reg_T\le \O\sbr{\frac{d}{\alpha}\log F_T^{\x}}+ \frac{16G^4}{\Bottomcoef} S_T^\x + \sbr{\frac{16G^4}{\Bottomcoef}+\gamma^{\Mid} - \frac{\Bottomcoef}{4}} S^\x_{T,\is} + \frac{1}{2}\Bottomcoef D^2   \\
        &-\frac{C_0}{2}S_T^{\Top}- \gamma^\Mid \sumTT \sumM q_{t,j}^\Top \sumN q_{t,j,i}^\Mid\norm{\x_{t,i} - \x_{t-1,i}}^2  - \gamma^\Top \sumTT \sumM q_{t,j}^\Top \|\q_{t,j}^\Mid - \q_{t-1,j}^\Mid\|_1^2\\
        \le {} &  \O\sbr{\frac{d}{\alpha}\log F_T^{\x}} +  \sbr{\frac{16G^4}{\Bottomcoef}+\gamma^{\Mid} - \frac{\Bottomcoef}{4}}S^\x_{T,\is}+ \frac{1}{2}\Bottomcoef D^2  \\
        &+\sbr{\frac{64D^2G^4}{\Bottomcoef}-\frac{C_0}{2}}S_T^{\Top}+\sbr{\frac{32G^4}{\Bottomcoef}-\gamma^\Mid} \sumTT \sumM q_{t,j}^\Top \sumN q_{t,j,i}^\Mid\norm{\x_{t,i} - \x_{t-1,i}}^2 \\
        & +\sbr{ \frac{64D^2G^4}{\Bottomcoef}-\gamma^\Top} \sumTT \sumM q_{t,j}^\Top \|\q_{t,j}^\Mid - \q_{t-1,j}^\Mid\|_1^2\le\O\sbr{\frac{d}{\alpha}\log F_T^{\x}}\le\O\sbr{\frac{d}{\alpha}\log F_T},
    \end{align*}
    where the second step follows from \pref{lem:decompose-three-layer} and the third step requires $\gamma^{\Top} \ge  64D^2G^4$, $\gamma^{\Mid} \ge  32G^4$, $\Bottomcoef \ge  64G^4+4\gamma^{\Mid}$, and $C_0 \ge  128D^2G^4$. Similar to the \emph{strongly convex} case, the final step follows from \pref{lem:small-loss-log}, where we choose $a,b,c$ as some $T$-independent constants, and set $x$ and $d$ to the same values as in the \emph{strongly convex} case.

    For \emph{convex} functions, combining \pref{eq:cor1 cvx meta F_T} and \pref{eq:cor1 cvx base F_T}, we obtain
    \begin{align*}
        &\Reg_T\le \frac{Z}{\varepsilon^\Top_{\js}} \log \frac{N}{3 C_0^2 (\varepsilon^\Top_{\js})^2} +\frac{1024D^2L \varepsilon^\Top_{\js}}{Z} F_T^{\x}+5D\sqrt{1+ 16LF_T^{\x}}\\
        & +\sbr{\frac{64G^2}{Z}+\gamma^{\Mid}-\frac{\Bottomcoef}{4}}S_{T,\is}^{\x}+\Bottomcoef D^2 \\
        \le {} &   \frac{Z}{\varepsilon^\Top_{\js}} \log \frac{N}{3 C_0^2 (\varepsilon^\Top_{\js})^2} +\frac{1024D^2L \varepsilon^\Top_{\js}}{Z} F_T^{\x}+5D\sqrt{1+ 16LF_T^{\x}}+\Bottomcoef D^2 \\
        \le {} &   \frac{Z}{\varepsilon^\Top_{\js}} \log \frac{Ne^5}{3 C_0^2 (\varepsilon^\Top_{\js})^2} +\sbr{\frac{1024}{Z}+20} D^2L \varepsilon^\Top_{\js} F_T^{\x}+\Bottomcoef D^2 +\O(1) \\
        \le {} &   \frac{2Z}{\varepsilon^\Top_{\js}} \log \frac{Ne^5}{3 C_0^2 (\varepsilon^\Top_{\js})^2} +\sbr{\frac{2048}{Z}+40}D^2L \varepsilon^\Top_{\js} F_T+\Bottomcoef D^2 +\O(1) \le\O\sbr{\sqrt{F_T\log F_T}},
    \end{align*}
    where the second step requires $\Bottomcoef \ge  \frac{256G^2}{Z}+4\gamma^{\Mid}$, the third step follows from the AM-GM inequality: $\sqrt{ab} \le \frac{a}{2} + \frac{b}{2}$ for any $a,b> 0$ with $a=1/(2D \varepsilon^\Top_{\js})$ and $b=2D\varepsilon^\Top_{\js}LF_T^{\x}$, the fourth step follows from
    \begin{equation}
        x-d\le  c(x-b)+e\Rightarrow x-d\le  \frac{c(d-b)+e}{1-c}\label{eq:FT cvx}
    \end{equation}
    for $x,b,d,e \ge 0$ and $0<c<1$  where we choose $x=\sumT f_t(\x_t)$, $d=\min_{\x\in\X}\sumT f_t(\x)$, $b=\sumT \min_{\x\in\X_+} f_t(\x)$, $ c=\sbr{\frac{1024}{Z}+20} D^2L \varepsilon^\Top_{\js} \le  1/2$, and $e=\O(1)$, and the final step is due to \pref{lem:tune-eta-1}.
\end{proof}

\subsubsection{Gradient-variance Regret}
\label{app:gradient-variance-correct-proof}
\begin{proof}
    We adopt the same regret decomposition strategy as utilized in \pref{app:unigrad-correct-1grad}.

    \paragraph{Meta Regret Analysis.}
     Recall that the normalization factor $Z=\max\{GD+\gamma^\Mid D^2, 1+\gamma^\Mid D^2+2\gamma^\Top\}$. For \emph{strongly convex} and \emph{exp-concave} functions, we follow the same meta regret analysis as used in \pref{app:unigrad-correct-1grad}.

    For \emph{convex} functions, we give a different analysis for $\Vs =\sumTT (\ell^\Mid_{t,\js,\is} - m^\Mid_{t,\js,\is})^2$. From \pref{lem:universal-optimism}, it holds that
    \begin{align*}
    \Vs \le{}& \frac{2 D^2}{Z^2} \sumT\norm{\g_t-\g_{t-1}}^2+\frac{2 G^2}{Z^2} \sumT\norm{\x_{t,\is}-\x_{t-1, \is}+\x_{t-1}-\x_t}^2\\
    \le{}& \frac{8D^2}{Z^2}W_T + \frac{4G^2}{Z^2} S_{T,\is}^{\x} +\frac{4G^2}{Z^2} S_T^{\x}.
    \end{align*}
    Similar to \pref{eq:cor1 cvx meta F_T}, the meta regret can be bounded as
    \begin{align}
        & \meta 
        \le \frac{Z}{\varepsilon^\Top_{\js}} \log \frac{N}{3 C_0^2 (\varepsilon^\Top_{\js})^2} + 32Z\varepsilon^\Top_{\js} \Vs+ \gamma^\Mid S_{T,\is}^{\x}-\frac{C_0}{2}S_T^{\Top}\notag\\
        {}&-\gamma^\Mid \sumTT \sumM q_{t,j}^\Top \sumN q_{t,j,i}^\Mid\norm{\x_{t,i} - \x_{t-1,i}}^2  - \gamma^\Top \sumTT \sumM q_{t,j}^\Top \|\q_{t,j}^\Mid - \q_{t-1,j}^\Mid\|_1^2\notag\\
        \le {} & 2ZC_0\log \frac{4N}{3}+64D\sqrt{W_T\log\sbr{\frac{1024ND^2W_T}{Z^2}}}+\sbr{\frac{64G^2}{Z}+\gamma^\Mid}S_{T,\is}^{\x}\notag\\
        {}&+\sbr{\frac{256D^2G^2}{Z}-\frac{C_0}{2}}S_T^{\Top}+\sbr{\frac{128G^2}{Z}-\gamma^{\Mid}}\sumTT \sumM q_{t,j}^\Top \sumN q_{t,j,i}^\Mid\norm{\x_{t,i} - \x_{t-1,i}}^2\notag\\
        {}&+\sbr{\frac{256D^2G^2}{Z}-\gamma^{\Top}}\sumTT \sumM q_{t,j}^\Top \|\q_{t,j}^\Mid - \q_{t-1,j}^\Mid\|_1^2\notag\\
        \le{}& 2ZC_0\log \frac{4N}{3} + \O\sbr{\sqrt{W_T\log W_T}} +\sbr{\frac{64G^2}{Z}+\gamma^\Mid}S_{T,\is}^{\x}, \label{eq:cor1 cvx meta W_T}
    \end{align}
    where the final step requires $\gamma^{\Top} \ge  \frac{256D^2G^2}{Z},  \gamma^{\Mid} \ge  \frac{128G^2}{Z}$, and $C_0 \ge  \frac{512D^2G^2}{Z}$.
    \paragraph{Base Regret Analysis.}
    We first provide different decompositions of the empirical gradient variation defined on surrogates for strongly convex, exp-concave, and convex functions, respectively, and then analyze the base regret in the corresponding cases.

    For \emph{$\lambda$-strongly convex} functions, $\Vb_{T,i}^{\scvx} \define \sumTT \|\nabla \hsc_{t,\is}(\x_{t,\is}) - \nabla \hsc_{t-1,\is}(\x_{t-1,\is})\|^2$ satisfies
    \begin{align*}
        \Vb_{T,\is}^{\scvx} = {} & \sumTT \|(\g_t + \lambda_\is (\x_{t,\is} - \x_t)) - (\g_{t-1} + \lambda_\is (\x_{t-1,\is} - \x_{t-1}))\|^2\\
        \le{}&  8W_T + 4S_T^\x + 4 S^\x_{T,\is}.
    \end{align*}
    For \emph{$\alpha$-exp-concave} functions, $\Vb^{\exp}_{T,\is} \define \sumTT \|\nabla \hexp_{t,\is}(\x_{t,\is}) - \nabla \hexp_{t-1,\is}(\x_{t-1,\is})\|^2$ satisfies
    \begin{align*}
        \Vb^{\exp}_{T,\is} 
        \le {} & 2 \sumTT \|\g_t - \g_{t-1}\|^2 + 2 \alpha_\is^2 \sumTT \|\g_t \g_t^\top (\x_{t,\is} - \x_t) - \g_{t-1} \g_{t-1}^\top (\x_{t-1,\is} - \x_{t-1})\|^2\notag\\
        \le {} & 8 W_T +  4 D^2 \sumTT \|\g_t \g_t^\top - \g_{t-1} \g_{t-1}^\top\|^2 
        + 4 G^4 \sumTT \|(\x_{t,\is} - \x_t) - (\x_{t-1,\is} - \x_{t-1})\|^2 \tag*{(by~$\alpha \in [1/T, 1]$)}\notag\\
        \le {} & C_{27}W_T + 8G^4 S_T^\x + 8G^4 S^\x_{T,\is},
    \end{align*}
    where the first step uses the definition of $\nabla \hexp_{t,i}(\x) = \g_t + \alpha_i \g_t \g_t^\top (\x - \x_t)$ and the last step holds by setting $C_{27} = 8+64D^2G^2$.   For \emph{convex} functions, $\Vb_{T,\is}^{\cvx}$ can be bounded by $\Vb^{\cvx}_{T,\is} \define \sumTT \|\g_t - \g_{t-1}\|^2\le  4W_T$. To conclude, for different curvature types, we provide correspondingly different analysis of the empirical gradient variation on surrogates:
    \begin{equation}
      \label{eq:cor1 VbT-to-VT W_T}
      \Vb_{T,\is}^{\{\scvx,\exp,\cvx\}} \le 
      \begin{cases}
        8W_T + 4S_T^\x + 4 S^\x_{T,\is}, & \text{when } \{f_t\}_{t=1}^T \text{ are $\lambda$-strongly convex}, \\[2mm]
        C_{27}W_T + 8G^4 S_T^\x + 8G^4 S^\x_{T,\is}, & \text{when } \{f_t\}_{t=1}^T \text{ are $\alpha$-exp-concave}, \\[2mm]
        4W_T, & \text{when } \{f_t\}_{t=1}^T \text{ are convex}.
      \end{cases}
      \end{equation}

      In the following, we analyze the base regret for different curvature types. For \emph{$\lambda$-strongly convex} functions, by \pref{lem:str-convex-base}, the $\is$-th base learner guarantees the following:
    \begin{align}
        &\base \le \frac{16G^2}{\lambda_{\is}}\log (1 + 8\lambda_\is  W_T + 4 \lambda_\is S_T^\x + 4\lambda_\is S^\x_{T,\is})+ \frac{1}{4}\Bottomcoef D^2  - \frac{1}{8} \Bottomcoef S_{T,\is}^{\x} + \O(1)\notag\\
        \le {} & \frac{32G^2}{\lambda}\log (1 + 8\lambda W_T) + 64G^2 S_T^\x + 64G^2 S^\x_{T,\is}+ \frac{1}{4}\Bottomcoef D^2  - \frac{1}{8} \Bottomcoef S_{T,\is}^{\x} + \O(1),\label{eq:cor1 scvx base W_T}
    \end{align}
    where the last step is due to $\log (1+x) \le x$ for $x \ge 0$.
    For \emph{$\alpha$-exp-concave} functions, by \pref{lem:exp-concave-base}, the $\is$-th base learner guarantees the following:
    \begin{align}
        &\base \le \frac{16d}{\alpha_{\is}}\log \sbr{1 + \frac{\alpha_\is C_{27}}{8 \Bottomcoef d} W_T + \frac{\alpha_\is G^4}{ \Bottomcoef d} S_T^\x + \frac{\alpha_\is G^4}{\Bottomcoef d} S^\x_{T,\is}} + \frac{1}{2}\Bottomcoef D^2  - \frac{1}{4} \Bottomcoef S_{T,\is}^{\x} + \O(1)\notag\\
        \le{}& \frac{32d}{\alpha} \log \sbr{1 + \frac{\alpha C_{27}}{8 \Bottomcoef d} W_T} + \frac{16G^4}{\Bottomcoef} S_T^\x + \sbr{\frac{16G^4}{\Bottomcoef} - \frac{\Bottomcoef}{4}} S^\x_{T,\is} + \frac{1}{2}\Bottomcoef D^2 + \O(1),\label{eq:cor1 exp-concave base W_T}
    \end{align}
    where the last step is due to $\log (1+x) \le x$ for $x \ge 0$.
    For \emph{convex} functions, by \pref{lem:convex-base}, the convex base learner guarantees the following:
    \begin{align}
        \base\le{}& 5D\sqrt{1+ \Vb^{\cvx}_{T,\is}} + \Bottomcoef D^2  - \frac{1}{4} \Bottomcoef S_{T,\is}^{\x} + \O(1)\notag\\
        \le{}& 5D\sqrt{1+ 4W_{T}} + \Bottomcoef D^2  - \frac{1}{4} \Bottomcoef S_{T,\is}^{\x} + \O(1).\label{eq:cor1 cvx base W_T}
    \end{align}

    \paragraph{Overall Regret Analysis.}
    For \emph{$\lambda$-strongly convex} functions, combining \eqref{eq:thm3 scvx meta} and \eqref{eq:cor1 scvx base W_T},
    \begin{align*}
        &\Reg_T\le \O\sbr{\frac{1}{\lambda}\log W_T} + 64G^2 S_T^\x + \sbr{64G^2+\gamma^\Mid-\frac{1}{8} \Bottomcoef }S^\x_{T,\is}+ \frac{1}{4}\Bottomcoef D^2 \\
        &-\frac{C_0}{2}S_T^{\Top}- \gamma^\Mid \sumTT \sumM q_{t,j}^\Top \sumN q_{t,j,i}^\Mid\norm{\x_{t,i} - \x_{t-1,i}}^2  - \gamma^\Top \sumTT \sumM q_{t,j}^\Top \|\q_{t,j}^\Mid - \q_{t-1,j}^\Mid\|_1^2\\
        \le {} &  \O\sbr{\frac{1}{\lambda}\log W_T}  + \sbr{64G^2+\gamma^\Mid-\frac{1}{8} \Bottomcoef }S^\x_{T,\is}+ \frac{1}{4}\Bottomcoef D^2  \\
        & \qquad +\sbr{256D^2G^2-\frac{C_0}{2}}S_T^{\Top}+\sbr{128G^2-\gamma^\Mid} \sumTT \sumM q_{t,j}^\Top \sumN q_{t,j,i}^\Mid\norm{\x_{t,i} - \x_{t-1,i}}^2 \\
        & \qquad +\sbr{ 256D^2G^2-\gamma^\Top} \sumTT \sumM q_{t,j}^\Top \|\q_{t,j}^\Mid - \q_{t-1,j}^\Mid\|_1^2\le \O\sbr{\frac{1}{\lambda}\log W_T},
    \end{align*}
    where the second step follows from \pref{lem:decompose-three-layer} and the last step requires $\gamma^{\Top} \ge  256D^2G^2$, $\gamma^{\Mid} \ge  128G^2$, $\Bottomcoef \ge  512G^2+8\gamma^{\Mid}$, and $C_0 \ge  512D^2G^2$.

    For \emph{$\alpha$-exp-concave} functions, combining \eqref{eq:thm3 exp-concave meta} and \eqref{eq:cor1 exp-concave base W_T}, we obtain
    \begin{align*}
        &\Reg_T\le \O\sbr{\frac{d}{\alpha}\log W_T}+ \frac{16G^4}{\Bottomcoef} S_T^\x + \sbr{\frac{16G^4}{\Bottomcoef}+\gamma^{\Mid} - \frac{\Bottomcoef}{4}} S^\x_{T,\is} + \frac{1}{2}\Bottomcoef D^2   \\
        &-\frac{C_0}{2}S_T^{\Top}- \gamma^\Mid \sumTT \sumM q_{t,j}^\Top \sumN q_{t,j,i}^\Mid\norm{\x_{t,i} - \x_{t-1,i}}^2  - \gamma^\Top \sumTT \sumM q_{t,j}^\Top \|\q_{t,j}^\Mid - \q_{t-1,j}^\Mid\|_1^2\\
        \le {} &  \O\sbr{\frac{d}{\alpha}\log W_T} +  \sbr{\frac{16G^4}{\Bottomcoef}+\gamma^{\Mid} - \frac{\Bottomcoef}{4}}S^\x_{T,\is}+ \frac{1}{2}\Bottomcoef D^2  \\
        &+\sbr{\frac{64D^2G^4}{\Bottomcoef}-\frac{C_0}{2}}S_T^{\Top}+\sbr{\frac{32G^4}{\Bottomcoef}-\gamma^\Mid} \sumTT \sumM q_{t,j}^\Top \sumN q_{t,j,i}^\Mid\norm{\x_{t,i} - \x_{t-1,i}}^2 \\
        & +\sbr{ \frac{64D^2G^4}{\Bottomcoef}-\gamma^\Top} \sumTT \sumM q_{t,j}^\Top \|\q_{t,j}^\Mid - \q_{t-1,j}^\Mid\|_1^2\le\O\sbr{\frac{d}{\alpha}\log W_T},
    \end{align*}
    where the second step follows from \pref{lem:decompose-three-layer} and the last step requires $\gamma^{\Top} \ge  64D^2G^4$, $\gamma^{\Mid} \ge  32G^4$, $\Bottomcoef \ge  64G^4+4\gamma^{\Mid}$, and $C_0 \ge  128D^2G^4$.

    For \emph{convex} functions, combining \eqref{eq:cor1 cvx meta W_T} and \eqref{eq:cor1 cvx base W_T}, we obtain
    \begin{align*}
        \Reg_T\le \O\sbr{\sqrt{W_T\log W_T}}+\sbr{\frac{64G^2}{Z}+\gamma^{\Mid}-\frac{\Bottomcoef}{4}}S_{T,\is}^{\x}+\Bottomcoef D^2 \le\O\sbr{\sqrt{W_T\log W_T}},
    \end{align*}
    where the second step requires $\Bottomcoef \ge  \frac{256G^2}{Z}+4\gamma^{\Mid}$.

  At last, we determine the specific values of $C_0$, $\gamma^{\Top}$, and $\gamma^{\Mid}$. These parameters need to satisfy the following requirements:
\begin{gather*}
  C_0 \ge  1,\ C_0 \ge  8D,\ C_0 \ge  4\gamma^{\Top},\ C_0 \ge  512D^2G^2, \ C_0 \ge \frac{512D^2G^2}{Z}, \ C_0 \ge 128D^2G^4,\\ 
  \gamma^{\Top} \ge  256D^2G^2,\ \gamma^{\Top} \ge  \frac{256D^2G^2}{Z},\ \gamma^{\Top} \ge  64D^2G^4,\ \gamma^{\Mid} \ge  128G^2,\\ \gamma^{\Mid} \ge  \frac{128G^2}{Z},\text{ and } \gamma^{\Mid} \ge  32G^4.
\end{gather*}
As a result, we set 
\begin{gather*}
  C_0=\max\bbr{1,8D,4\gamma^{\Top},512D^2G^2,128D^2G^4},\\
  \gamma^\Top=\max\bbr{256D^2G^2,64D^2G^4},\quad
  \gamma^{\Mid}=\max\bbr{128G^2,32G^4},
\end{gather*}
where $Z=\max\{GD+\gamma^\Mid D^2, 1+\gamma^\Mid D^2+2\gamma^\Top\}$.
\end{proof}

\subsection{Proof of Corollary~\ref{cor:FT-bregman}}
    \label{app:FT-bregman}
    We prove the small-loss regret guarantees of \Bregmanpp in \pref{app:small-loss-bregman-proof} and the gradient-variance regret guarantees of \Bregmanpp in \pref{app:gradient-variance-bregman-proof}.
    \subsubsection{Small-Loss Regret}
    \label{app:small-loss-bregman-proof}

    \begin{proof}
    We adopt the same regret decomposition strategy as utilized in \pref{app:unigrad-bregman-1grad}.

    \paragraph{Meta Regret Analysis.}
    For \emph{strongly convex} and \emph{exp-concave} functions, we follow the same meta regret analysis as used in \pref{app:unigrad-bregman-1grad}.

    For \emph{convex} functions, we give a different analysis for $\Vs=\sumTT (\ell^\Mid_{t,\js,\is} - m^\Mid_{t,\js,\is})^2$. From \pref{lem:universal-optimism}, it holds that
    \begin{align}
    &\meta \le C_0 \sqrt{4G^2D^2 + \sumT \inner{\nabla f_t(\x_t) - \nabla f_{t-1}(\x_{t-1})}{\x_t - \x_{t,\is}}^2} + 2GDC_2  \notag\\
    \le {} &  C_0 \sqrt{4G^2D^2 + D^2 \Vb_T} + C_2\le C_0 \sqrt{4G^2D^2 + 16D^2L F_T^{\x}} + 2GDC_2,
    \label{eq:cor2 cvx-meta}
    \end{align}
    where the last step is by the self-bounding property of $\|\nabla f(\x)\|_2^2 \le 4 L (f(\x) - \min_{\x \in \X_+} f(\x))$ for any $\x \in \X_+$.

    \paragraph{Base Regret Analysis.} 
    We first provide different decompositions of the empirical gradient variation defined on surrogates for strongly convex, exp-concave, and convex functions, respectively, and then analyze the base regret in the corresponding cases.
    For \emph{$\lambda$-strongly convex} functions, $\Vb_{T,\is}^{\scvx} \define \sumTT \|\nabla \hsc_{t,\is}(\x_{t,\is}) - \nabla \hsc_{t-1,\is}(\x_{t-1,\is})\|^2$ can be bounded by
    \begin{align*}
        \Vb_{T,\is}^{\scvx}= {} & \sumTT \norm{\g_t + \frac{\lambda_\is}{2} (\x_{t,\is} - \x_t) - \g_{t-1} - \frac{\lambda_\is}{2} (\x_{t-1,\is} - \x_{t-1})}^2 \notag\\
      \le {} & 3\Vb_T + 3 \sumTT \norm{\frac{\lambda_\is}{2} (\x_{t,\is} - \x_t)}^2 + 3 \sumTT \norm{\frac{\lambda_\is}{2} (\x_{t-1,\is} - \x_{t-1})}^2 \\
        \le {} & 48L F_T^{\x} + 2 \lambda_\is^2 \sumT \|\x_{t,\is} - \x_t\|^2 \tag*{(by \eqref{eq:to-small-loss})},
    \end{align*}
where the first step is due to the property of the surrogate: $\nabla \hsc_{t,i}(\x_{t,i}) = \g_t + \frac{\lambda_i}{2} (\x_{t,i} - \x_t)$, and the second step is due to the Cauchy-Schwarz inequality.
For \emph{$\alpha$-exp-concave} functions, $\Vb^{\exp}_{T,\is} \define \sumTT \|\nabla \hexp_{t,\is}(\x_{t,\is}) - \nabla \hexp_{t-1,\is}(\x_{t-1,\is})\|^2$ can be bounded by
\begin{align*}
    \Vb^{\exp}_{T,\is} 
    = {} & \sumTT \norm{\g_t + \frac{\alpha_\is}{2} \g_t \inner{\g_t}{\x_t - \x_{t,\is}} - \g_{t-1} - \frac{\alpha_\is}{2} \g_{t-1} \inner{\g_{t-1}}{\x_{t-1} - \x_{t-1,\is}}}^2 \notag\\
        \le {} & 3 \Vb_T + 3 \sumTT \norm{\frac{\alpha_\is}{2} \g_t \inner{\g_t}{\x_t - \x_{t,\is}}}^2 + 3 \sumTT \norm{\frac{\alpha_\is}{2} \g_{t-1} \inner{\g_{t-1}}{\x_{t-1} - \x_{t-1,\is}}}^2 \notag\\
        \le {} & 3 \Vb_T + 6 \sumT \norm{\frac{\alpha_\is}{2} \g_t \inner{\g_t}{\x_t - \x_{t,\is}}}^2\\
    \le {} & 48L F_T^{\x} + 2 \alpha_\is^2 G^2 \sumT \inner{\g_t}{\x_t - \x_{t,\is}}^2, \tag*{(by \pref{assum:gradient-boundedness} and \eqref{eq:to-small-loss})}
\end{align*}
    where the first step uses the definition of $\nabla \hexp_{t,i}(\x) = \g_t + \frac{\alpha_i}{2} \g_t \g_t^\top (\x - \x_t)$, and the second step is due to the Cauchy-Schwarz inequality. For \emph{convex} functions, the empirical gradient variation $\Vb^{\cvx}_{T,\is} \define \sumTT \|\nabla f_t(\x_t) - \nabla f_{t-1}(\x_{t-1})\|^2$ can be bounded by $\Vb^{\cvx}_{T,\is} \le  16L F_T^{\x}$. To conclude, for different curvature types, we provide correspondingly different analysis of the empirical gradient variation on surrogates:
\begin{equation}
  \label{eq:cor2 VbT-to-VT}
  \Vb_{T,\is}^{\{\scvx,\exp,\cvx\}} \le
  \begin{cases}
      \begin{aligned}
         48L F_T^{\x} + 2 \lambda_\is^2 \sumT \|\x_{t,\is} - \x_t\|^2, 
      \end{aligned}
      & \text{when } \seq{f_t} \text{ are $\lambda$-strongly convex}, \\[2.5ex]
      
      \begin{aligned}
        48L F_T^{\x} + 2 \alpha_\is^2 G^2 \sumT \inner{\g_t}{\x_t - \x_{t,\is}}^2, 
      \end{aligned}
      & \text{when} \seq{f_t} \text{ are $\alpha$-exp-concave}, \\[2.5ex]
  
      16L F_T^{\x},
      & \text{when} \seq{f_t} \text{ are convex}.
  \end{cases}
  \end{equation}
  In the following, we analyze the base regret for different curvature types.
  For \emph{$\lambda$-strongly convex} functions, when using the update rule~\eqref{eq:base-sc}, according to \pref{lem:str-convex-base}, the base regret can be bounded as 
    \begin{align*}
        & \base \le \frac{16G^2}{\lambda_\is} \log \sbr{1 + \lambda_\is \Vb_{T,\is}^{\scvx}} + \O(1)\notag\\
        \le{}& \frac{16G^2}{\lambda_\is} \log \sbr{1 + 48L \lambda_\is F_T^{\x} + 2 \lambda_\is^3 \sumT \|\x_{t,\is} - \x_t\|^2} + \O(1)\\
        \le {} & \frac{16G^2}{\lambda_\is} \log \sbr{C_{18} \sbr{1 + 48L \lambda_\is F_T^{\x}}} + \frac{16G^2}{C_{18} \lambda_\is} \sbr{2 \lambda_\is^3 \sumT \|\x_{t,\is} - \x_t\|^2}\\
        \le {} & \frac{32G^2}{\lambda} \log (1 + 48L F_T^{\x}) + \frac{32G^2}{C_{18}} \sumT \|\x_{t,\is} - \x_t\|^2 + \O(\log C_{18}),
    \end{align*}
    where the third step requires $C_{18} \ge 1$ by \pref{lem:ln-de} and uses the property of the best base learner, i.e., $\lambda_\is \le \lambda \le 2 \lambda_\is$. The last step is due to $\lambda_i \le 1$. For \emph{$\alpha$-exp-concave} functions, by \pref{lem:exp-concave-base}, the $\is$-th base learner guarantees the following:
    \begin{align*}
        &\base\le \frac{16d}{\alpha_\is} \log \sbr{1 + \frac{\alpha_\is}{8d} \Vb_{T,\is}^{\exp}} + \O(1) \notag\\
        \le{}& \frac{16d}{\alpha_\is} \log \sbr{1 + \frac{6L \alpha_\is}{d} F_T^{\x} + \frac{\alpha_\is^3 G^2}{4d} \sumT \inner{\g_t}{\x_t - \x_{t,\is}}^2} + \O(1)\\
        \le {} & \frac{16d}{\alpha_\is} \log \sbr{C_{19} \sbr{1 + \frac{6L \alpha_\is}{d} F_T^{\x}}} + \frac{16d}{C_{19} \alpha_\is} \sbr{\frac{\alpha_\is^3 G^2}{4d} \sumT \inner{\g_t}{\x_t - \x_{t,\is}}^2}\\
        \le {} & \frac{32d}{\alpha} \log \sbr{1 + \frac{6L}{d} F_T^{\x}} + \frac{4G^2}{C_{19}} \sumT \inner{\g_t}{\x_t - \x_{t,\is}}^2 + \O(\log C_{19}),
    \end{align*}
    where the third step requires $C_{19} \ge 1$ by \pref{lem:ln-de} and uses the property of the best base learner, i.e., $\alpha_\is \le \alpha \le 2 \alpha_\is$. The last step is due to $\alpha_i \le 1$.
    For \emph{convex} functions, by \pref{lem:convex-base}, the base regret can be bounded as 
    \begin{align*}
        \base \le {} & 5D \sqrt{1 + \Vb_T} + \O(1) \le 5D \sqrt{1 + 16L F_T^{\x}} + \O(1).
    \end{align*}
    
    \paragraph{Overall Regret Analysis.} 
    For \emph{$\lambda$-strongly convex} functions, by combining the meta and base regret, it holds that
    \begin{align*}
        \Reg_T \le {} & \sbr{\frac{C_0 D^2}{2 C_3} + \frac{32G^2}{C_{18}} - \frac{\lambda_\is}{4}} \sumT \|\x_t - \x_{t,\is}\|^2 + \frac{32G^2}{\lambda} \log (1 + 48L F_T^{\x})  + \O(C_3 + \log C_{18})\\
        \le {} & \frac{32G^2}{\lambda} \log (1 + 48L F_T^{\x}) \le \O\sbr{\frac{1}{\lambda} \log F_T},
    \end{align*}
    where the second step is by choosing $C_3 = {4 C_0 D^2}/{\lambda_\is}$ and $C_{18} = \max\{1, {256G^2}/{\lambda_\is}\}$ and the last step is due to \pref{lem:small-loss-log} by choosing $a,b,c$ as some $T$-independent constants and setting $x=\sumT f_t(\x_t)-\sumT\min_{\x\in\X_+}f_t(\x)$ and $d=\min_{\x\in\X}\sumT f_t(\x)-\sumT\min_{\x\in\X_+}f_t(\x)$. Note that such a parameter configuration will only add an $\O(1/\lambda)$ factor to the final regret bound, which can be absorbed.

    For \emph{$\alpha$-exp-concave} functions, by combining the meta and base regret, it holds that
    \begin{align*}
        \Reg_T \le {} & \sbr{\frac{C_0}{2 C_4} + \frac{4G^2}{C_{19}} - \frac{\alpha_\is}{4}} \sumT \inner{\g_t}{\x_t - \x_{t,\is}}^2 + \frac{32d}{\alpha} \log \sbr{1 + \frac{6L}{d} F_T^{\x}} + \O(C_4 + \log C_{19})\\
        \le {} & \frac{32d}{\alpha} \log \sbr{1 + \frac{6L}{d} F_T^{\x}} + \O(1) \le \O\sbr{\frac{d}{\alpha} \log F_T}, 
    \end{align*}
    where the second step chooses $C_4 = {4C_0}/{\alpha_\is}$ and $C_{19} = \max\{1, {32G^2}/{\alpha_\is}\}$. Similar to the \emph{strongly convex} case, the final step follows from \pref{lem:small-loss-log}, where we choose $a,b,c$ as some $T$-independent constants, and set $x$ and $d$ to the same values as in the \emph{strongly convex} case. Meanwhile, such a parameter configuration will only add an $\O(1/\alpha)$ factor to the final regret bound, which can be absorbed.

    For \emph{convex} functions, by combining the meta and base regret, it holds that
    \begin{equation*}
        \Reg_T \le C_0 \sqrt{4G^2D^2 + 16 D^2 L F_T^{\x}} + 5D \sqrt{1 + 16L F_T^{\x}} + 2GDC_2+\O(1) \le \O(\sqrt{F_T}),
    \end{equation*}
    where  the final step follows from \pref{lem:small-loss-sqrt} by setting $a,b$ as some $T$-independent constants and choosing $x=\sumT f_t(\x_t)-\sumT \min_{\x\in\X_+} f_t(\x)$ and $d=\min_{\x\in\X}\sumT f_t(\x)-\sumT \min_{\x\in\X_+} f_t(\x)$. 
    \end{proof}

    \subsubsection{Gradient-variance Regret}
    \label{app:gradient-variance-bregman-proof}
    \begin{proof}
    We adopt the same regret decomposition strategy as utilized in \pref{app:unigrad-bregman-1grad}.
    
    \paragraph{Meta Regret Analysis.}
    For \emph{strongly convex} and \emph{exp-concave} functions, we follow the same meta regret analysis as used in \pref{app:unigrad-bregman-1grad}.
    
    For \emph{convex} functions, we give a different analysis for $\Vs=\sumTT (\ell^\Mid_{t,\js,\is} - m^\Mid_{t,\js,\is})^2$. From \pref{lem:universal-optimism}, it holds that
    \begin{align}
      &\meta \le C_0 \sqrt{4G^2D^2 + \sumT \inner{\nabla f_t(\x_t) - \nabla f_{t-1}(\x_{t-1})}{\x_t - \x_{t,\is}}^2} + 2GDC_2  \notag\\
      \le {} &  C_0 \sqrt{1 + D^2 \Vb_T} + 2GDC_2\le C_0 \sqrt{1 + 4D^2 W_T} + 2GDC_2.
      \label{eq:cor2 cvx-meta W_T}
    \end{align}
    
        \paragraph{Base Regret Analysis.} 
        We first provide different decompositions of the empirical gradient variation defined on surrogates for strongly convex, exp-concave, and convex functions, respectively, and then analyze the base regret in the corresponding cases.
        For \emph{$\lambda$-strongly convex} functions, $\Vb_{T,i}^{\scvx} \define \sumTT \|\nabla \hsc_{t,i}(\x_{t,\is}) - \nabla \hsc_{t-1,i}(\x_{t-1,\is})\|^2$ can be bounded by
        \begin{align*}
            \Vb_{T,\is}^{\scvx}= {} & \sumTT \norm{\g_t + \frac{\lambda_\is}{2} (\x_{t,\is} - \x_t) - \g_{t-1} - \frac{\lambda_\is}{2} (\x_{t-1,\is} - \x_{t-1})}^2 \notag\\
          \le {} & 3\Vb_T + 3 \sumTT \norm{\frac{\lambda_\is}{2} (\x_{t,\is} - \x_t)}^2 + 3 \sumTT \norm{\frac{\lambda_\is}{2} (\x_{t-1,\is} - \x_{t-1})}^2 \\
            \le {} & 12W_T + 2 \lambda_\is^2 \sumT \|\x_{t,\is} - \x_t\|^2  \tag*{(by \eqref{eq:to-variance})},
        \end{align*}
    where the first step is due to the property of the surrogate: $\nabla \hsc_{t,i}(\x_{t,i}) = \g_t + \frac{\lambda_i}{2} (\x_{t,i} - \x_t)$, and the second step is due to the Cauchy-Schwarz inequality.
    For \emph{$\alpha$-exp-concave} functions, $\Vb^{\exp}_{T,\is} \define \sumTT \|\nabla \hexp_{t,\is}(\x_{t,\is}) - \nabla \hexp_{t-1,\is}(\x_{t-1,\is})\|^2$ can be bounded by
    \begin{align*}
        \Vb^{\exp}_{T,\is} 
        = {} & \sumTT \norm{\g_t + \frac{\alpha_\is}{2} \g_t \inner{\g_t}{\x_t - \x_{t,\is}} - \g_{t-1} - \frac{\alpha_\is}{2} \g_{t-1} \inner{\g_{t-1}}{\x_{t-1} - \x_{t-1,\is}}}^2 \notag\\
            \le {} & 3 \Vb_T + 3 \sumTT \norm{\frac{\alpha_\is}{2} \g_t \inner{\g_t}{\x_t - \x_{t,\is}}}^2 + 3 \sumTT \norm{\frac{\alpha_\is}{2} \g_{t-1} \inner{\g_{t-1}}{\x_{t-1} - \x_{t-1,\is}}}^2 \notag\\
            \le {} & 3 \Vb_T + 6 \sumT \norm{\frac{\alpha_\is}{2} \g_t \inner{\g_t}{\x_t - \x_{t,\is}}}^2\\
        \le {} & 12W_T + 2 \alpha_\is^2 G^2 \sumT \inner{\g_t}{\x_t - \x_{t,\is}}^2,\tag*{(by \pref{assum:gradient-boundedness} and \eqref{eq:to-variance})}
    \end{align*}
        where the first step uses the definition of $\nabla \hexp_{t,i}(\x) = \g_t + \frac{\alpha_i}{2} \g_t \g_t^\top (\x - \x_t)$, and the second step is due to the Cauchy-Schwarz inequality. For \emph{convex} functions, the empirical gradient variation $\Vb^{\cvx}_{T,\is} \define \sumTT \|\nabla f_t(\x_t) - \nabla f_{t-1}(\x_{t-1})\|^2$ can be bounded by $\Vb^{\cvx}_{T,\is} \le  4 W_T$. To conclude, for different curvature types, we provide correspondingly different analysis of the empirical gradient variation on surrogates:
    \begin{equation}
      \label{eq:cor2 VbT-to-VT W_T}
      \Vb_{T,\is}^{\{\scvx,\exp,\cvx\}} \le
      \begin{cases}
          \begin{aligned}
             12W_T + 2 \lambda_\is^2 \sumT \|\x_{t,\is} - \x_t\|^2, 
          \end{aligned}
          & \text{when } \seq{f_t} \text{ are $\lambda$-strongly convex}, \\[2.5ex]
          
          \begin{aligned}
            12W_T+ 2 \alpha_\is^2 G^2 \sumT \inner{\g_t}{\x_t - \x_{t,\is}}^2, 
          \end{aligned}
          & \text{when} \seq{f_t} \text{ are $\alpha$-exp-concave}, \\[2.5ex]
      
          4 W_T,
          & \text{when} \seq{f_t} \text{ are convex}.
      \end{cases}
      \end{equation}

      In the following, we analyze the base regret for different curvature types.
      For \emph{$\lambda$-strongly convex} functions, when using the update rule~\eqref{eq:base-sc}, according to \pref{lem:str-convex-base}, the base regret can be bounded as 
      \begin{align*}
        & \base \le \frac{16G^2}{\lambda_\is} \log \sbr{1 + 12 \lambda_\is W_T + 2 \lambda_\is^3 \sumT \|\x_{t,\is} - \x_t\|^2} + \O(1)\\
        \le {} & \frac{16G^2}{\lambda_\is} \log \sbr{C_{20} \sbr{1 + 12 \lambda_\is W_T}} + \frac{16G^2}{C_{20} \lambda_\is} \sbr{2 \lambda_\is^3 \sumT \|\x_{t,\is} - \x_t\|^2}\\
        \le {} & \frac{32G^2}{\lambda} \log (1 + 12W_T) + \frac{32G^2}{C_{20}} \sumT \|\x_{t,\is} - \x_t\|^2 + \O(\log C_{20}),
    \end{align*}
    where the second step requires $C_{20} \ge 1$ by \pref{lem:ln-de} and uses the property of the best base learner, i.e., $\lambda_\is \le \lambda \le 2 \lambda_\is$. The last step is due to $\lambda_i \le 1$. For \emph{$\alpha$-exp-concave} functions, by \pref{lem:exp-concave-base}, the $\is$-th base learner guarantees the following:
        \begin{align*}
            & \base \le \frac{16d}{\alpha_\is} \log \sbr{1 + \frac{3 \alpha_\is}{2d} W_T + \frac{\alpha_\is^3 G^2}{4d} \sumT \inner{\g_t}{\x_t - \x_{t,\is}}^2} + \O(1)\\
            \le {} & \frac{16d}{\alpha_\is} \log \sbr{C_{21} \sbr{1 + \frac{3 \alpha_\is}{2d} W_T}} + \frac{16d}{C_{21} \alpha_\is} \sbr{\frac{\alpha_\is^3 G^2}{4d} \sumT \inner{\g_t}{\x_t - \x_{t,\is}}^2}\\
            \le {} & \frac{32d}{\alpha} \log \sbr{1 + \frac{3}{2d} W_T} + \frac{4G^2}{C_{21}} \sumT \inner{\g_t}{\x_t - \x_{t,\is}}^2 + \O(\log C_{21}),
        \end{align*}
        where the second step requires $C_{21} \ge 1$ by \pref{lem:ln-de} and uses the property of the best base learner, i.e., $\alpha_\is \le \alpha \le 2 \alpha_\is$. The last step is due to $\alpha_i \le 1$.
        For \emph{convex} functions, by \pref{lem:convex-base}, the base regret can be bounded as 
        \begin{align*}
            \base \le {} & 5D \sqrt{1 + \Vb_T} + \O(1) \le 5D \sqrt{1 + 4W_T} + \O(1).
        \end{align*}

        \paragraph{Overall Regret Analysis.} 
        For \emph{$\lambda$-strongly convex} functions, by combining the meta and base regret, it holds that
        \begin{align*}
            \Reg_T \le {} & \sbr{\frac{C_0 D^2}{2 C_3} + \frac{32G^2}{C_{20}} - \frac{\lambda_\is}{4}} \sumT \|\x_t - \x_{t,\is}\|^2 + \frac{32G^2}{\lambda} \log (1 + 12 W_T)  + \O(C_3 + \log C_{20})\\
            \le {} & \frac{32G^2}{\lambda} \log (1 + 12 W_T) \le \O\sbr{\frac{1}{\lambda} \log W_T}, 
        \end{align*}
        where the second step is by choosing $C_3 = {4 C_0 D^2}/{\lambda_\is}$ and $C_{20} = \max\{1, {256G^2}/{\lambda_\is}\}$. For \emph{$\alpha$-exp-concave} functions, by combining the meta and base regret, it holds that
        \begin{align*}
            \Reg_T \le {} & \sbr{\frac{C_0}{2 C_4} + \frac{4G^2}{C_{21}} - \frac{\alpha_\is}{4}} \sumT \inner{\g_t}{\x_t - \x_{t,\is}}^2 + \frac{32d}{\alpha} \log \sbr{1 + \frac{3W_T}{2d} } + \O(C_4 + \log C_{21})\\
            \le {} & \frac{32d}{\alpha} \log \sbr{1 + \frac{3}{2d} W_T} + \O(1) \le \O\sbr{\frac{d}{\alpha} \log W_T}, 
        \end{align*}
        where the second step chooses $C_4 = {4C_0}/{\alpha_\is}$ and $C_{21} = \max\{1, {32G^2}/{\alpha_\is}\}$. 
        For \emph{convex} functions, by combining the meta and base regret, it holds that
        \begin{equation*}
            \Reg_T \le C_0 \sqrt{4G^2D^2 + 4 D^2 W_T} + 5D \sqrt{4G^2D^2 + 4W_T} + 2GDC_2+\O(1) \le \O(\sqrt{W_T}).
        \end{equation*}
    Finally, we note that the constants $C_3, C_4, C_{18}.C_{19}, C_{20},C_{21}$ only exist in analysis and and hence our choices of them are feasible.
    
\end{proof}

\subsection{Proof of Theorem~\ref{thm:SEA-correct}}
\label{app:SEA-correct}
\begin{proof}
    Recall that we denote by $\g_t=\nabla f_t(\x_t)$ for simplicity. To begin with, we provide a proof for \pref{eq:SEA-correction}:
    \begin{align}
        \E[\Vb_T]&{} = \E\mbr{\sumTT\norm{\nabla f_t(\mathbf{x}_t)-\nabla f_{t-1}(\mathbf{x}_{t-1})}_2^2 }\notag\\
        \le  & 4\E\mbr{\sumTT\|\nabla f_t(\mathbf{x}_t)-\nabla F_t(\mathbf{x}_t)\|_2^2}+4\E\mbr{\sumTT\|\nabla F_t(\mathbf{x}_t)-\nabla F_t(\mathbf{x}_{t-1})\|_2^2} \notag\\
        & +4\E\mbr{\sumTT\|\nabla F_t(\mathbf{x}_{t-1})-\nabla F_{t-1}(\mathbf{x}_{t-1})\|_2^2}+4\E\mbr{\sumTT\|\nabla F_{t-1}(\mathbf{x}_{t-1})-\nabla f_{t-1}(\mathbf{x}_{t-1})\|_2^2} \notag \\
        \le  &8\sigma_{1:T}^2 + 4 \Sigma_{1:T}^2+4L^2 \E\mbr{\sumTT \|\mathbf{x}_t-\mathbf{x}_{t-1}\|_2^2},\label{eq:SEA-correction-detailed}
        \end{align}
  where the second step is due to Cauchy-Schwarz inequality and the last step is because of the definitions of $\sigma_{1:T}^2$ and $\Sigma_{1:T}^2$ (given in \pref{subsec:applications-SEA}).

 In the following, we present regret decompositions tailored to different curvature regimes, proceed to analyze both the meta and base regret components, and finally combine these results to derive the overall regret bounds.

  \paragraph{Regret Decomposition.}    
  For \emph{$\lambda$-strongly convex} functions, similar to the decomposition in \pref{app:unigrad-correct-1grad}, we have
  \begin{align}
      &\E[\Reg_T] \le \E\mbr{\sumT \inner{\nabla F_t(\x_t)}{\x_t - \xs}} - \frac{\lambda}{2} \E\mbr{\sumT \|\x_t - \xs\|^2}\notag\\
      &= \E\mbr{\sumT \inner{\g_t}{\x_t - \xs}} - \frac{\lambda}{2} \E\mbr{\sumT \|\x_t - \xs\|^2}\notag\\
     & \le \underbrace{\E\mbr{\sumT \inner{\g_t}{\x_t - \x_{t,\is}}} - \frac{\lambda_\is}{2}\E\mbr{ \sumT \|\x_t - \x_{t,\is}\|^2}}_{\meta} 
      + \underbrace{\E\mbr{\sumT \hsc_{t,\is}(\x_{t,\is}) - \hsc_{t,\is}(\xs)}}_{\base},\label{eq:thm5 scvx decompose}
  \end{align}
  where the first and second steps rely on the expected loss function \( F_t(\x) = \mathbb{E}[f_t(\x)] \); in particular, the second step additionally requires that \( F_t(\cdot) \) be \emph{strongly convex}.
  
  For \emph{$\alpha$-exp-concave} functions, following the similar decomposition as in the proof of \pref{thm:unigrad-correct-1grad} in \pref{app:unigrad-correct-1grad}, we decompose the regret as 
  \begin{align}
    &\E[\Reg_T] = \E\mbr{\sumT \inner{\g_t}{\x_t - \xs}} - \frac{\alpha}{2} \E\mbr{\sumT \inner{\g_t}{\x_t - \xs}^2}\notag\\
    \le {} & \underbrace{\E\mbr{\sumT \inner{\g_t}{\x_t - \x_{t,\is}}} - \frac{\alpha_\is}{2} \E\mbr{\sumT \inner{\g_t}{\x_t - \xs}^2}}_{\meta} 
    + \underbrace{\E\mbr{\sumT \hexp_{t,\is}(\x_{t,\is}) - \hexp_{t,\is}(\xs)}}_{\base} ,\label{eq:thm5 exp decompose}
\end{align}
where the first step is due to the exp-concavity and defining surrogate loss functions  $\hexp_{t,i}(\x)= \inner{\g_t}{\x} + \frac{\alpha_i}{2} \inner{\g_t}{\x - \x_t}^2$.
For \emph{convex} functions, we decompose the regret as
  \begin{align}
      \E[\Reg_T]\le   \underbrace{\E\mbr{\sumT \inner{\g_t}{\x_t - \x_{t,\is}}}}_{\meta} + \underbrace{\E\mbr{\sumT \hc_{t,\is}(\x_{t,\is}) - \hc_{t,\is}(\xs)}}_{\base} ,
      \label{eq:thm5 cvx decompose}
  \end{align}
  where we have $\hc_{t,i}(\x)=\inner{\g_t}{\x}$.

  \paragraph{Meta Regret Analysis.} 
   Recall that the normalization factor $Z=\max\{GD+\gamma^\Mid D^2, 1+\gamma^\Mid D^2+2\gamma^\Top\}$. Our \pref{alg:UniGrad-Correct-1grad} can be applied to the SEA model without any algorithm modifications. As a result, we directly use the same parameter configurations as in the proof of \pref{thm:unigrad-correct-1grad} (i.e., in \pref{app:unigrad-correct-1grad}).

  For  \emph{strongly convex} and \emph{exp-concave} functions, the meta regret is bounded in a similar way as \eqref{eq:thm3 scvx meta} and \eqref{eq:thm3 exp-concave meta}, and thus omitted here.

  For \emph{convex} functions, by \pref{lem:universal-optimism}, the meta regret can be bounded as 
  \begin{align}
    & \meta 
    \le \frac{Z}{\varepsilon^\Top_{\js}} \log \frac{N}{3 C_0^2 (\varepsilon^\Top_{\js})^2} +\frac{128D^2\varepsilon^\Top_{\js}}{Z} \E[\Vb_T]+\sbr{\frac{64G^2}{Z}+\gamma^\Mid} \E[S_{T,\is}^{\x}]\notag\\
    {}& \qquad\quad+\frac{64G^2}{Z}\E[S_T^{\x}] -\frac{C_0}{2}\E[S_T^{\Top}]- \gamma^\Mid \E\mbr{\sumTT \sumM q_{t,j}^\Top \sumN q_{t,j,i}^\Mid\norm{\x_{t,i} - \x_{t-1,i}}^2}\notag \\
    {}& \qquad\qquad - \gamma^\Top \E\mbr{\sumTT \sumM q_{t,j}^\Top \|\q_{t,j}^\Mid - \q_{t-1,j}^\Mid\|_1^2}\notag\\
    \le {} &  \frac{Z}{\varepsilon^\Top_{\js}} \log \frac{N}{3 C_0^2 (\varepsilon^\Top_{\js})^2} +\frac{512D^2\varepsilon^\Top_{\js}}{Z}\sbr{2\sigma_{1:T}^2 + \Sigma_{1:T}^2}+\sbr{\frac{64G^2}{Z}+\gamma^\Mid} \E[S_{T,\is}^{\x}]\notag\\
    {}& \qquad\qquad -\frac{C_0}{2}\E[S_T^{\Top}]- \gamma^\Mid \E\mbr{\sumTT \sumM q_{t,j}^\Top \sumN q_{t,j,i}^\Mid\norm{\x_{t,i} - \x_{t-1,i}}^2}\notag \\
    {}& \qquad\qquad +\sbr{\frac{64G^2}{Z}+128D^2L^2}\E[S_T^{\x}] - \gamma^\Top \E\mbr{\sumTT \sumM q_{t,j}^\Top \|\q_{t,j}^\Mid - \q_{t-1,j}^\Mid\|_1^2}\notag\\
    \le {} &   \O\sbr{\sqrt{\sbr{\sigma_{1:T}^2 + \Sigma_{1:T}^2}\log \sbr{\sigma_{1:T}^2 + \Sigma_{1:T}^2}}} + \sbr{\frac{64G^2}{Z}+\gamma^\Mid} \E[S_{T,\is}^{\x}]\notag\\
    {}& \qquad\qquad -\frac{C_0}{2}\E[S_T^{\Top}]- \gamma^\Mid \E\mbr{\sumTT \sumM q_{t,j}^\Top \sumN q_{t,j,i}^\Mid\norm{\x_{t,i} - \x_{t-1,i}}^2}\notag \\
    {}& \qquad\qquad +\sbr{\frac{64G^2}{Z}+128D^2L^2}\E[S_T^{\x}] - \gamma^\Top \E\mbr{\sumTT \sumM q_{t,j}^\Top \|\q_{t,j}^\Mid - \q_{t-1,j}^\Mid\|_1^2},
    \label{eq:thm5 cvx meta}
\end{align}
  where the second step is due to the decomposition of $\Vs$, the third step is by \pref{eq:SEA-correction} and the final step follows from \pref{lem:tune-eta-1}.

  \paragraph{Base Regret Analysis.}
  For \emph{$\lambda$-strongly convex} functions, we need to delve into the proof details of the base algorithm, i.e., \OOMD~\eqref{eq:base-sc} for strongly convex functions with step size $\eta_t = 2 / (\Bottomcoef + \lambda_i t)$. For example, from Lemma 12 of \citet{NeurIPS'23:universal}, the base regret can be bounded as 
  \begin{align*}
      \base \le 4 \sumTT \frac{1}{\lambda_\is t} \E\mbr{\norm{\nabla \hsc_{t,\is}(\x_{t,\is}) - \nabla \hsc_{t-1,\is}(\x_{t-1,\is})}^2} -\frac{\Bottomcoef}{8}\E[S_{T,\is}^{\x}] + \O(1).
  \end{align*}
  Subsequently, we analyze the empirical gradient variation defined on surrogates in each round. Denoting by $\sigma_t^2 \define \max_{\x \in \X}\E_{f_t \sim \mathfrak{D}_t} [\|\nabla f_t(\x) - \nabla F_t(\x)\|^2]$ and $\Sigma_t^2 \define \E[\sup_{\x \in \X} \|\nabla F_t(\x) - \nabla F_{t-1}(\x)\|^2]$ for simplicity,
  \begin{align*}
      & \E\mbr{\norm{\nabla \hsc_{t,\is}(\x_{t,\is}) - \nabla \hsc_{t-1,\is}(\x_{t-1,\is})}^2}\\
      = {} & \E\mbr{\norm{\g_t + \lambda_\is (\x_{t,\is} - \x_t) - \g_{t-1} - \lambda_\is (\x_{t-1,\is} - \x_{t-1})}^2}\\
      \le {} & 2 \E \mbr{\|\g_t - \g_{t-1}\|^2} + 2 \norm{\lambda_\is (\x_{t,\is} - \x_t)+\lambda_\is (\x_{t-1,\is} - \x_{t-1})}^2\\
      \le {} &  4\sbr{2\sigma_t^2 + 2\sigma_{t-1}^2 +(1+2L^2)\E\mbr{\norm{\x_t-\x_{t-1}}^2}+\E\mbr{\norm{\x_{t,\is}-\x_{t-1,\is}}^2} + 2\Sigma^2_t}, \tag*{(by \eqref{eq:SEA-correction})}
  \end{align*}
  where the first step is due to the property of the surrogate: $\nabla \hsc_{t,i}(\x_{t,i}) = \g_t + \lambda_i (\x_{t,i} - \x_t)$, and the second step is due to the Cauchy-Schwarz inequality. Plugging the above term back into the base regret and omitting the ignorable $\O(1)$ term, we achieve 
  \begin{align*}
      \base \le {} & \frac{32}{\lambda_\is} \sumTT \frac{\sigma_t^2 + \sigma_{t-1}^2 + \Sigma^2_t}{t} + 16(1+2L^2) \sumTT \frac{\E\mbr{\norm{\x_t-\x_{t-1}}^2}}{\lambda_\is t}\\
      & \qquad + 16\sumTT \frac{ \E\mbr{\|\x_{t,\is} - \x_{t-1,\is}\|^2 }}{\lambda_\is t}.
  \end{align*}
  Using \pref{lem:sum}, we control the base regret as 
  \begin{align}
      & \base \le \O\sbr{\frac{1}{\lambda} \sbr{\sigma_{\max}^2 + \Sigma_{\max}^2} \log \frac{\sigma_{1:T}^2 + \Sigma_{1:T}^2}{\sigma_{\max}^2 + \Sigma_{\max}^2}}\notag \\
      &\quad+\frac{32D^2(L^2+1)}{\lambda_{\is}}\log \sbr{ 1+(1+2L^2)\lambda_{\is}E[S_T^{\x}]+\lambda_{\is}\E[S_{T,\is}^{\x}]}\notag\\
      \le& \O\sbr{\frac{1}{\lambda} \sbr{\sigma_{\max}^2 + \Sigma_{\max}^2} \log \frac{\sigma_{1:T}^2 + \Sigma_{1:T}^2}{\sigma_{\max}^2 + \Sigma_{\max}^2}}+32D^2(L^2+1)\sbr{(1+2L^2)E[S_T^{\x}]+E[S_{T,\is}^{\x}]},\label{eq:thm5 scvx base}
  \end{align}
  where the first term initializes \pref{lem:sum} as $a_t = \sigma_t^2 + \sigma_{t-1}^2 + \Sigma^2_t$ (i.e., $a_{\max} = \O(\sigma_{\max}^2 + \Sigma_{\max}^2)$) and $b = 1 / (\sigma_{\max}^2 + \Sigma_{\max}^2)$, the second term initializes \pref{lem:sum} as $a_t = (1+2L^2)\E\mbr{\norm{\x_{t-1}-\x_t}^2} + \E\mbr{\norm{\x_{t,\is}-\x_{t-1,\is}}^2}$ (i.e., $a_{\max} = \sbr{2+2L^2}D^2$ due to \pref{assum:domain-boundedness}) and $b = \lambda_\is$. The second step is due to $\log(1+x)\le  x$ for $x \ge 0$.

  For \emph{$\alpha$-exp-concave} functions, the base regret is bounded by \eqref{eq: thm3 exp-concave}. Following \eqref{eq:SEA-correction}, we control the empirical gradient variation defined on surrogates as
  \begin{align*}
      & \E\mbr{\sumTT \norm{\nabla \hexp_{t,\is}(\x_{t,\is}) - \nabla \hexp_{t-1,\is}(\x_{t-1,\is})}^2} \\
      \le {} & 2 \E\mbr{\sumTT \|\g_t - \g_{t-1}\|^2} + 2 \alpha_\is^2 \mbr{\sumTT \E\|\g_t \g_t^\top (\x_{t,\is} - \x_t) - \g_{t-1} \g_{t-1}^\top (\x_{t-1,\is} - \x_{t-1})\|^2}\\
      \le {} & 8(2\sigma_{1:T}^2 + \Sigma_{1:T}^2)+8L^2\E[S_T^{\x}]  + 4 D^2\E \mbr{\sumTT \|\g_t \g_t^\top - \g_{t-1} \g_{t-1}^\top\|^2} \\
      & \qquad + 4 G^4 \E\mbr{\sumTT \|(\x_{t,\is} - \x_t) - (\x_{t-1,\is} - \x_{t-1})\|^2} \tag*{(by~$\alpha \in [1/T, 1]$)}\\
      \le {} & C_{22} (2\sigma_{1:T}^2 + \Sigma_{1:T}^2) + C_{23} \E[S_T^\x] + 8G^4 \E[S^\x_{T,\is}],
  \end{align*}
  where the first step uses the definition of $\nabla \hexp_{t,i}(\x) = \g_t + \alpha_i \g_t \g_t^\top (\x - \x_t)$ and the last step holds by setting $C_{22} = 8 + 64D^2 G^2$ and $C_{23} = 8L^2 + 64D^2 G^2 L^2 + 8G^4$. Then we obtain
  \begin{align}
    &\base\notag\\
    \le{}& \frac{16d}{\alpha_{\is}}\log \sbr{1 + \frac{\alpha_\is C_{22}}{8 \Bottomcoef d} (2\sigma_{1:T}^2 + \Sigma_{1:T}^2) + \frac{\alpha_\is C_{23}}{8 \Bottomcoef d}\E[S_T^\x] + \frac{\alpha_\is G^4}{\Bottomcoef d} \E[S^\x_{T,\is}]}\notag\\
    &\qquad + \frac{1}{2}\Bottomcoef D^2  - \frac{1}{4} \Bottomcoef \E[S_{T,\is}^{\x}]+\O(1) \notag\\
    \le{}& \frac{32d}{\alpha} \log \sbr{1 + \frac{\alpha C_{22}}{8 \Bottomcoef d}(2\sigma_{1:T}^2 + \Sigma_{1:T}^2)} + \frac{2C_{23}}{\Bottomcoef} \E[S_T^\x]\notag\\
    &\qquad + \sbr{\frac{16G^4}{\Bottomcoef} - \frac{\Bottomcoef}{4}} \E[S^\x_{T,\is}] + \frac{1}{2}\Bottomcoef D^2 +\O(1),\label{eq:thm5 exp-concave base}
  \end{align}
  where the last step is due to $\log (1+x) \le x$ for $x \ge 0$. 
  
  For \emph{convex} functions, \pref{lem:convex-base} guarantees the following:
    \begin{align}
        &\base\le 5D\sqrt{1+ \E[\Vb^{\cvx}_{T,\is}]} + \Bottomcoef D^2  - \frac{1}{4} \Bottomcoef \E[S_{T,\is}^{\x}] + \O(1)\notag\\
        \le {} &  5D\sqrt{1+ 4(2\sigma_{1:T}^2 + \Sigma_{1:T}^2)+4L^2\E[S_T^{\x}]} + \Bottomcoef D^2  - \frac{1}{4} \Bottomcoef \E[S_{T,\is}^{\x}] + \O(1)\notag\\
        &\le5D\sqrt{1+4(2\sigma_{1:T}^2 + \Sigma_{1:T}^2)}+20DL^2\E[S_T^{\x}] + \Bottomcoef D^2  - \frac{1}{4} \Bottomcoef \E[S_{T,\is}^{\x}] + \O(1),\label{eq:thm5 cvx base}
  \end{align}
  where the second step is due to \pref{eq:SEA-correction}.
  \paragraph{Overall Regret Analysis.}
  For \emph{$\lambda$-strongly convex} functions, plugging \pref{eq:thm3 scvx meta} and \pref{eq:thm5 scvx base} into \pref{eq:thm5 scvx decompose} and letting  $C_{24}=64D^2(1+L^2)^2$, we obtain
  \begin{align*}
    &\Reg_T\le \O\sbr{\frac{1}{\lambda} \sbr{\sigma_{\max}^2 + \Sigma_{\max}^2} \log \frac{\sigma_{1:T}^2 + \Sigma_{1:T}^2}{\sigma_{\max}^2 + \Sigma_{\max}^2}}+C_{24}\E[S_T^{\x}] \\
    &+ \sbr{32D^2(1+L^2)+\gamma^\Mid-\frac{1}{8} \Bottomcoef}\E[S^\x_{T,\is}]- \gamma^\Mid \E\mbr{\sumTT \sumM q_{t,j}^\Top \sumN q_{t,j,i}^\Mid\norm{\x_{t,i} - \x_{t-1,i}}^2} \\
    &+ \frac{1}{4}\Bottomcoef D^2-\frac{C_0}{2}\E[S_T^{\Top}]-\gamma^\Top \E\mbr{\sumTT \sumM q_{t,j}^\Top \|\q_{t,j}^\Mid - \q_{t-1,j}^\Mid\|_1^2}\\
    \le {} &  \O\sbr{\frac{1}{\lambda} \sbr{\sigma_{\max}^2 + \Sigma_{\max}^2} \log \frac{\sigma_{1:T}^2 + \Sigma_{1:T}^2}{\sigma_{\max}^2 + \Sigma_{\max}^2}}  + \sbr{32D^2(1+L^2)+\gamma^\Mid-\frac{1}{8} \Bottomcoef }\E[S^\x_{T,\is}]  \\
    & \quad+\sbr{4D^2C_{24}-\frac{C_0}{2}}\E[S_T^{\Top}]+\sbr{2C_{24}-\gamma^\Mid}\E\mbr{ \sumTT \sumM q_{t,j}^\Top \sumN q_{t,j,i}^\Mid\norm{\x_{t,i} - \x_{t-1,i}}^2} \\
    & \qquad + \frac{1}{4}\Bottomcoef D^2+\sbr{ 4D^2C_{24}-\gamma^\Top} \E\mbr{\sumTT \sumM q_{t,j}^\Top \|\q_{t,j}^\Mid - \q_{t-1,j}^\Mid\|_1^2}\\
    \le {} &   \O\sbr{\frac{1}{\lambda} \sbr{\sigma_{\max}^2 + \Sigma_{\max}^2} \log \frac{\sigma_{1:T}^2 + \Sigma_{1:T}^2}{\sigma_{\max}^2 + \Sigma_{\max}^2}},
  \end{align*}
  where the second step follows from \pref{lem:decompose-three-layer} and the last step requires $\gamma^{\Top} \ge  4D^2C_{24}$, $\gamma^{\Mid} \ge  2C_{24}$, $\Bottomcoef \ge  256D^2(1+L^2)+8\gamma^{\Mid}$, and $C_0 \ge  8D^2C_{24}$.
    For \emph{$\alpha$-exp-concave} functions, plugging \pref{eq:thm3 exp-concave meta} and \pref{eq:thm5 exp-concave base} into \pref{eq:thm5 exp decompose}, we obtain
  \begin{align*}
    &\Reg_T\le \O\sbr{\frac{d}{\alpha}\log \sbr{\sigma_{1:T}^2 + \Sigma_{1:T}^2}}+ \frac{2C_{23}}{\Bottomcoef} \E[S_T^\x] + \sbr{\frac{16G^4}{\Bottomcoef}+\gamma^{\Mid} - \frac{\Bottomcoef}{4}} \E[S^\x_{T,\is}]   \\
    &+ \frac{1}{2}\Bottomcoef D^2 -\frac{C_0}{2}\E[S_T^{\Top}]- \gamma^\Mid \E\mbr{\sumTT \sumM q_{t,j}^\Top \sumN q_{t,j,i}^\Mid\norm{\x_{t,i} - \x_{t-1,i}}^2} \notag\\
    & - \gamma^\Top \E\mbr{\sumTT \sumM q_{t,j}^\Top \|\q_{t,j}^\Mid - \q_{t-1,j}^\Mid\|_1^2}\\
    \le {} &  \O\sbr{\frac{d}{\alpha}\log \sbr{\sigma_{1:T}^2 + \Sigma_{1:T}^2}} +  \sbr{\frac{16G^4}{\Bottomcoef}+\gamma^{\Mid} - \frac{\Bottomcoef}{4}}\E[S^\x_{T,\is}]+ \frac{1}{2}\Bottomcoef D^2  \\
    &+\sbr{\frac{8D^2C_{23}}{\Bottomcoef}-\frac{C_0}{2}}\E[S_T^{\Top}]+\sbr{\frac{4C_{23}}{\Bottomcoef}-\gamma^\Mid} \E\mbr{\sumTT \sumM q_{t,j}^\Top \sumN q_{t,j,i}^\Mid\norm{\x_{t,i} - \x_{t-1,i}}^2} \\
    & +\sbr{ \frac{8D^2C_{23}}{\Bottomcoef}-\gamma^\Top} \E\mbr{\sumTT \sumM q_{t,j}^\Top \|\q_{t,j}^\Mid - \q_{t-1,j}^\Mid\|_1^2}\le\O\sbr{\frac{d}{\alpha}\log \sbr{\sigma_{1:T}^2 + \Sigma_{1:T}^2}},
  \end{align*}
  where the second step follows from \pref{lem:decompose-three-layer} and the last step requires $\gamma^{\Top} \ge  8D^2C_{23}$, $\gamma^{\Mid} \ge  4C_{23}$, $\Bottomcoef \ge  64G^4+4\gamma^{\Mid}$, and $C_0 \ge  16D^2C_{23}$. 
  For \emph{convex} functions, plugging \pref{eq:thm5 cvx meta} and \pref{eq:thm5 cvx base} into \pref{eq:thm5 cvx decompose}, we obtain
  \begin{align*}
    &\Reg_T\le \O\sbr{\sqrt{\sbr{\sigma_{1:T}^2 + \Sigma_{1:T}^2}\log \sbr{\sigma_{1:T}^2 + \Sigma_{1:T}^2}}}+\sbr{\frac{64G^2}{Z}+\gamma^{\Mid}-\frac{\Bottomcoef}{4}}\E[S_{T,\is}^{\x}]\\
    &+C_{25}\E[S_T^{\x}]-\frac{C_0}{2}\E[S_T^{\Top}]- \gamma^\Mid \E\mbr{\sumTT \sumM q_{t,j}^\Top \sumN q_{t,j,i}^\Mid\norm{\x_{t,i} - \x_{t-1,i}}^2}+\Bottomcoef D^2\notag\\
    &- \gamma^\Top \E\mbr{\sumTT \sumM q_{t,j}^\Top \|\q_{t,j}^\Mid - \q_{t-1,j}^\Mid\|_1^2}\\ 
    \le {} & \O\sbr{\sqrt{\sbr{\sigma_{1:T}^2 + \Sigma_{1:T}^2}\log \sbr{\sigma_{1:T}^2 + \Sigma_{1:T}^2}}}+\sbr{\frac{64G^2}{Z}+\gamma^{\Mid}-\frac{\Bottomcoef}{4}}\E[S_{T,\is}^{\x}]\\
    &+\sbr{4D^2C_{25}-\frac{C_0}{2}}\E[S_T^{\Top}]+\Bottomcoef D^2+\sbr{4D^2C_{25}- \gamma^\Top} \E\mbr{\sumTT \sumM q_{t,j}^\Top \|\q_{t,j}^\Mid - \q_{t-1,j}^\Mid\|_1^2}\\
    & +\sbr{2C_{25}-\gamma^\Mid}\E\mbr{\sumTT \sumM q_{t,j}^\Top\sumN q_{t,j,i}^\Mid\norm{\x_{t,i} - \x_{t-1,i}}^2}\\
    \le {} & \O\sbr{\sqrt{\sbr{\sigma_{1:T}^2 + \Sigma_{1:T}^2}\log \sbr{\sigma_{1:T}^2 + \Sigma_{1:T}^2}}},
  \end{align*}
  where the first step is by setting $C_{25}=20DL^2+\frac{64G^2}{Z}+128D^2L^2$, the second step follows from \pref{lem:decompose-three-layer} and the last step requires $\gamma^{\Top} \ge  4D^2C_{25}$, $\gamma^{\Mid} \ge  2C_{25}$, $\Bottomcoef \ge  256G^2+4\gamma^{\Mid}$, and $C_0 \ge  8D^2C_{25}$.

At last, we determine the specific values of $C_0$, $\gamma^{\Top}$, and $\gamma^{\Mid}$. These parameters need to satisfy the following requirements:
\begin{gather*}
  C_0 \ge  1,\ C_0 \ge  8D,\ C_0 \ge  4\gamma^{\Top},\ C_0 \ge  8D^2C_{24}, \ C_0 \ge  8D^2C_{23}, \ C_0 \ge  8D^2C_{25},\\ 
  \gamma^{\Top} \ge  4D^2C_{24},\ \gamma^{\Top} \ge  8D^2C_{23},\ \gamma^{\Top} \ge  4D^2C_{25},\ \gamma^{\Mid} \ge  2C_{24},\\ \gamma^{\Mid} \ge  4C_{23},\text{ and } \gamma^{\Mid} \ge  2C_{25}.
\end{gather*}
As a result, we set 
\begin{gather*}
  C_0=\max\bbr{1,8D,4\gamma^{\Top},8D^2C_{24},8D^2C_{23},8D^2C_{25}},\\
  \gamma^\Top=\max\bbr{4D^2C_{24},8D^2C_{23},4D^2\sbr{20DL^2+64G^2+128D^2L^2}},\\
  \gamma^{\Mid}=\max\bbr{2C_{24},4C_{23}, 20DL^2+64G^2+128D^2L^2},
\end{gather*}
where $Z=\max\{GD+\gamma^\Mid D^2, 1+\gamma^\Mid D^2+2\gamma^\Top\}$, $C_{23}=8L^2 + 64D^2 G^2 L^2 + 8G^4$, $C_{24}=64D^2(1+L^2)^2$, and $C_{25}=20DL^2+\frac{64G^2}{Z}+128D^2L^2$. The proof is finished.
\end{proof}

\subsection{Proof of Theorem~\ref{thm:SEA-bregman}}
\label{app:SEA-bregman}
\begin{proof}
    Recall that we denote by $\g_t=\nabla f_t(\x_t)$ for simplicity. To begin with, we provide a proof for \pref{eq:SEA-bregman}:
    \begin{align}
        & \E[\Vb_T] \le 5 \E\mbr{\sumTT \|\nabla f_t(\x_t) - \nabla F_t(\x_t)\|^2} + 5 \E\mbr{\sumTT \|\nabla F_t(\x_t) - \nabla F_t(\xs)\|^2} \notag\\
        + {} & 5 \E\mbr{\sumTT \|\nabla F_t(\xs) - \nabla F_{t-1}(\xs)\|^2} + 5 \E\mbr{\sumTT \|\nabla F_{t-1}(\xs) - \nabla F_{t-1}(\x_{t-1})\|^2} \label{eq:SEA-bregman-detailed}\\
        + {} & 5 \E\mbr{\sumTT \|\nabla F_{t-1}(\x_{t-1}) - \nabla f_{t-1}(\x_{t-1})\|^2} \le 10 \sigma_{1:T}^2 + 5 \Sigma_{1:T}^2 + 20L \E\mbr{\sumT \D_{F_t}(\xs, \x_t)}, \notag
    \end{align}
  where the first step is due to Cauchy-Schwarz inequality and the last step is because of the definitions of $\sigma_{1:T}^2$ and $\Sigma_{1:T}^2$ (given in \pref{subsec:applications-SEA}) and the analysis proposed in \pref{subsec:bregman-one-gradient}.

  In the following, we first give regret decompositions for different curvature types, then we analyze the meta and base regret, and combine them for the final regret guarantees.

  \paragraph{Regret Decomposition.}    
  For \emph{$\lambda$-strongly convex} functions, similar to the decomposition in \pref{app:unigrad-bregman-1grad}, we have
  \begin{align*}
      \E[\Reg_T] = {} &\E\mbr{\sumT \inner{\nabla F_t(\x_t)}{\x_t - \xs}} - \frac{1}{2} \E\mbr{\sumT \D_{F_t}(\xs, \x_t)} - \frac{1}{2} \E\mbr{\sumT \D_{F_t}(\xs, \x_t)}\\
      \le {} & \E\mbr{\sumT \inner{\g_t}{\x_t - \xs}} - \frac{\lambda}{4} \E\mbr{\sumT \|\x_t - \xs\|^2} - \frac{1}{2} \E\mbr{\sumT \D_{F_t}(\xs, \x_t)}\\
      \le {} & \underbrace{\E\mbr{\sumT \inner{\g_t}{\x_t - \x_{t,\is}}} - \frac{\lambda_\is}{4}\E\mbr{ \sumT \|\x_t - \x_{t,\is}\|^2}}_{\meta} \\
      & \qquad + \underbrace{\E\mbr{\sumT \hsc_{t,\is}(\x_{t,\is}) - \hsc_{t,\is}(\xs)}}_{\base} - \frac{1}{2} \E\mbr{\sumT \D_{F_t}(\xs, \x_t)},
  \end{align*}
  where the first and second steps rely on the expected loss function \( F_t(\x) = \mathbb{E}[f_t(\x)] \); in particular, the second step additionally requires that \( F_t(\cdot) \) be \emph{strongly convex}. The third step follows from the definition of the surrogate function \( \hsc_{t,i}(\x) = \langle \g_t, \x \rangle + \frac{\lambda_i}{4} \|\x - \x_t\|^2 \), where \( \lambda_i \in \mathcal{H} \) is defined in~\eqref{eq:candidate-pool}.

  For \emph{$\alpha$-exp-concave} functions, following the similar decomposition as in the proof of \pref{thm:unigrad-bregman-1grad} in \pref{app:unigrad-bregman-1grad}, we decompose the regret as 
  \begin{align*}
      \E[\Reg_T] = {} &\E\mbr{\sumT \inner{\nabla F_t(\x_t)}{\x_t - \xs}} - \frac{1}{2} \E\mbr{\sumT \D_{F_t}(\xs, \x_t)} - \frac{1}{2} \E\mbr{\sumT \D_{F_t}(\xs, \x_t)}\\
      \le {} & \E\mbr{\sumT \inner{\g_t}{\x_t - \xs}} - \frac{\alpha}{4} \E\mbr{\sumT \inner{\g_t}{\x_t - \xs}^2} - \frac{1}{2} \E\mbr{\sumT \D_{F_t}(\xs, \x_t)}\\
      \le {} & \underbrace{\E\mbr{\sumT \inner{\g_t}{\x_t - \x_{t,\is}}} - \frac{\alpha_\is}{4} \E\mbr{\sumT \inner{\g_t}{\x_t - \x_{t,\is}}^2}}_{\meta} \\
      & \qquad + \underbrace{\E\mbr{\sumT \hexp_{t,\is}(\x_{t,\is}) - \hexp_{t,\is}(\xs)}}_{\base} - \frac{1}{2} \E\mbr{\sumT \D_{F_t}(\xs, \x_t)},
  \end{align*}
  where the first and second steps rely on the expected loss function \( F_t(\x) = \mathbb{E}[f_t(\x)] \); in particular, the second step additionally requires that \( f_t(\cdot) \) be \emph{exp-concave}. The third step follows from the definition of the surrogate function $\hexp_{t,i}(\x)= \inner{\g_t}{\x} + \frac{\alpha_i}{4} \inner{\g_t}{\x - \x_t}^2$, where \( \alpha_i \in \mathcal{H} \) is defined in~\eqref{eq:candidate-pool}.

  For \emph{convex} functions, we decompose the regret as
  \begin{align*}
      \E[\Reg_T] = {} & \E\mbr{\sumT F_t(\x_t) - \sumT F_t(\xs)} = \E\mbr{\sumT \inner{\nabla F_t(\x_t)}{\x_t - \xs}} - \E\mbr{\sumT \D_{F_t}(\xs, \x_t)}\\
      = {} & \E\mbr{\sumT \inner{\g_t}{\x_t - \xs}} - \E\mbr{\sumT \D_{F_t}(\xs, \x_t)}\\
      = {} & \underbrace{\E\mbr{\sumT \inner{\g_t}{\x_t - \x_{t,\is}}}}_{\meta} + \underbrace{\E\mbr{\sumT \hc_{t,\is}(\x_{t,\is}) - \hc_{t,\is}(\xs)}}_{\base} - \E\mbr{\sumT \D_{F_t}(\xs, \x_t)},
  \end{align*}
  where the first and third step use $F_t(\x) = \E[f_t(\x)]$, the second step uses the definition of Bregman divergence, and the fourth step is due to $\hc_{t,i}(\x)=\inner{\g_t}{\x}$.

  \paragraph{Meta Regret Analysis.} 
  Our \pref{alg:UniGrad-Bregman-1grad} can be applied to the SEA model without any algorithm modifications. As a result, we directly use the same parameter configurations as in the proof of \pref{thm:unigrad-bregman-1grad} (i.e., in \pref{app:unigrad-bregman-1grad}).

  For  \emph{strongly convex} and \emph{exp-concave} functions, the meta regret is bounded in a similar way as \eqref{eq:scvx-meta} and \eqref{eq:exp-meta}, and thus omitted here.

  For \emph{convex} functions, the meta regret can be bounded as 
  \begin{align*}
      & \meta \le \E\mbr{C_0 \sqrt{4G^2D^2 + D^2 \Vb_T} + 2GDC_2} \le C_0 \sqrt{4G^2D^2 + D^2 \E[\Vb_T]} + 2GDC_2\\
      \le {} & C_0 \sqrt{4G^2D^2 +  5D^2 (2\sigma_{1:T}^2 + \Sigma_{1:T}^2) + 20 D^2 L \E\mbr{\sumT \D_{F_t}(\xs, \x_t)}} + 2GDC_2 \tag*{(by \eqref{eq:SEA-bregman})}\\
      \le {} & \O\sbr{\sqrt{\sigma_{1:T}^2 + \Sigma_{1:T}^2}} + \O(C_{26}) + \frac{C_0}{2 C_{26}} \E\mbr{\sumT \D_{F_t}(\xs, \x_t)},
  \end{align*}
  where the second step is by Jensen's inequality and the last step is due to AM-GM inequality (\pref{lem:AM-GM}). $C_{26}$ is a constant to be specified.

  \paragraph{Base Regret Analysis.}
  For \emph{$\lambda$-strongly convex} functions, similar to the analysis in \pref{app:SEA-correct}, the base regret can be bounded as 
  \begin{align*}
      \base \le 4 \sumTT \frac{1}{\lambda_\is t} \E\mbr{\norm{\nabla \hsc_{t,\is}(\x_{t,\is}) - \nabla \hsc_{t-1,\is}(\x_{t-1,\is})}^2} + \O(1).
  \end{align*}
  Subsequently, we analyze the empirical gradient variation defined on surrogates in each round, i.e., $\|\nabla \hsc_{t,\is}(\x_{t,\is}) - \nabla \hsc_{t-1,\is}(\x_{t-1,\is})\|^2$. Denoting by $\sigma_t^2 \define \max_{\x \in \X}\E_{f_t \sim \mathfrak{D}_t} [\|\nabla f_t(\x) - \nabla F_t(\x)\|^2]$ and $\Sigma_t^2 \define \E[\sup_{\x \in \X} \|\nabla F_t(\x) - \nabla F_{t-1}(\x)\|^2]$ for simplicity,
  \begin{align*}
      & \E\mbr{\norm{\nabla \hsc_{t,\is}(\x_{t,\is}) - \nabla \hsc_{t-1,\is}(\x_{t-1,\is})}^2}\\
      = {} & \E\mbr{\norm{\g_t + \frac{\lambda_\is}{2} (\x_{t,\is} - \x_t) - \g_{t-1} - \frac{\lambda_\is}{2} (\x_{t-1,\is} - \x_{t-1})}^2}\\
      \le {} & 3 \E \mbr{\|\g_t - \g_{t-1}\|^2} + 3 \norm{\frac{\lambda_\is}{2} (\x_{t,\is} - \x_t)}^2 + 3 \norm{\frac{\lambda_\is}{2} (\x_{t-1,\is} - \x_{t-1})}^2\\
      \le {} & 15 (\sigma_t^2 + \sigma_{t-1}^2 + 2L\E\mbr{ \D_{F_t}(\xs, \x_t)} + 2L \mbr{\D_{F_{t-1}}(\xs, \x_{t-1})} + \Sigma^2_t) \tag*{(by \eqref{eq:SEA-bregman})}\\
      & \qquad + \lambda_\is^2 \E\mbr{\|\x_{t,\is} - \x_t\|^2} + \lambda_\is^2 \E\mbr{\|\x_{t-1,\is} - \x_{t-1}\|^2},
  \end{align*}
  where the first step is due to the property of the surrogate: $\nabla \hsc_{t,i}(\x_{t,i}) = \g_t + \frac{\lambda_i}{2} (\x_{t,i} - \x_t)$, and the second step is due to the Cauchy-Schwarz inequality. Plugging the above term back into the base regret and omitting the ignorable $\O(1)$ term, we achieve 
  \begin{align*}
      \base \le {} & \frac{60}{\lambda_\is} \sumTT \frac{\sigma_t^2 + \sigma_{t-1}^2 + \Sigma^2_t}{t} + 120L \sumTT \frac{\E\mbr{\D_{F_t}(\xs, \x_t) + \D_{F_{t-1}}(\xs, \x_{t-1})}}{\lambda_\is t}\\
      & \qquad + 4\sumTT \frac{\lambda_\is^2 \E\mbr{\|\x_{t,\is} - \x_t\|^2 }+ \lambda_\is^2 \E\mbr{\|\x_{t-1,\is} - \x_{t-1}\|^2}}{\lambda_\is t},
  \end{align*}
  Using \pref{lem:sum}, we control the base regret as 
  \begin{align*}
      & \base \le \O\sbr{\frac{1}{\lambda} \sbr{\sigma_{\max}^2 + \Sigma_{\max}^2} \log \frac{\sigma_{1:T}^2 + \Sigma_{1:T}^2}{\sigma_{\max}^2 + \Sigma_{\max}^2}} \\
      & \ + \frac{480LGD}{\lambda_\is} \log \sbr{1 + 2 \lambda_\is \E\mbr{\sumT \D_{F_t}(\xs, \x_t)}} + \frac{8D^2}{\lambda_\is} \log \sbr{1 + 2 \lambda_\is^3 \E\mbr{\sumT \|\x_{t,\is} - \x_t\|^2}}\\
      \le {} & \O\sbr{\frac{1}{\lambda} \sbr{\sigma_{\max}^2 + \Sigma_{\max}^2} \log \frac{\sigma_{1:T}^2 + \Sigma_{1:T}^2}{\sigma_{\max}^2 + \Sigma_{\max}^2}} + \O(\log C_{27} + \log C_{28})\\
      & \qquad + \frac{960LGD}{C_{27}} \E\mbr{\sumT \D_{F_t}(\xs, \x_t)} + \frac{16D^2}{C_{28}} \E\mbr{\sumTT \|\x_{t,\is} - \x_t\|^2},
  \end{align*}
  where the first term initializes \pref{lem:sum} as $a_t = \sigma_t^2 + \sigma_{t-1}^2 + \Sigma^2_t$ (i.e., $a_{\max} = \O(\sigma_{\max}^2 + \Sigma_{\max}^2)$) and $b = 1 / (\sigma_{\max}^2 + \Sigma_{\max}^2)$, the second term initializes \pref{lem:sum} as $a_t = \E\mbr{\D_{F_t}(\xs, \x_t)} + \E\mbr{\D_{F_{t-1}}(\xs, \x_{t-1})}$ (i.e., $a_{\max} = 4GD$ due to  \pref{assum:domain-boundedness} and \pref{assum:gradient-boundedness}) and $b = \lambda_\is$, the third term initializes \pref{lem:sum} as $a_t = \E\mbr{\lambda_\is^2 \|\x_{t,\is} - \x_t\|^2 + \lambda_\is^2 \|\x_{t-1,\is} - \x_{t-1}\|^2}$ (i.e., $a_{\max} = 2D^2$ due to $\lambda_i \le 1$ and \pref{assum:domain-boundedness}) and $b = \lambda_\is$. The $\O(1)$ term contains ignorable terms like $\O(1/\lambda)$. The second step requires $C_{27}, C_{28} \ge 1$ by \pref{lem:ln-de}.

  For \emph{$\alpha$-exp-concave} functions, the base regret is bounded by \eqref{eq: thm3 exp-concave}. Following \eqref{eq:SEA-bregman}, we control the empirical gradient variation defined on surrogates as
  \begin{align*}
      & \E\mbr{\sumTT \norm{\nabla \hexp_{t,\is}(\x_{t,\is}) - \nabla \hexp_{t-1,\is}(\x_{t-1,\is})}^2} \le 3 \E[\Vb_T] + 6 \sumT \norm{\frac{\alpha_\is}{2} \g_t \inner{\g_t}{\x_t - \x_{t,\is}}}^2\\
      \le {} & 15 (2\sigma_{1:T}^2 + \Sigma_{1:T}^2) + 60L \E\mbr{\sumT \D_{F_t}(\xs, \x_t)} + 2 \alpha_\is^2 G^2 \E\mbr{\sumT \inner{\g_t}{\x_t - \x_{t,\is}}}^2.
  \end{align*}
  Plugging the surrogate's empirical gradient variation back to the base regret, we obtain
  \begin{align*}
      & \base \le \frac{16d}{\alpha_\is} \log \Bigg(1 + \frac{15 \alpha_\is}{8d} (2\sigma_{1:T}^2 + \Sigma_{1:T}^2) + \frac{15L \alpha_\is}{2d} \E\mbr{\sumT \D_{F_t}(\xs, \x_t)}\\
      & + \frac{\alpha_\is^3 G^2}{4d} \E\mbr{\sumT \inner{\g_t}{\x_t - \x_{t,\is}}^2} \Bigg) \le \O\sbr{\frac{d}{\alpha} \log \sbr{\sigma_{1:T}^2 + \Sigma_{1:T}^2}} + \O(\log C_{29})\\
      & + \frac{120L}{C_{29}} \E\mbr{\sumT \D_{F_t}(\xs, \x_t)} + \frac{4G^2}{C_{29}} \E\mbr{\sumT \inner{\g_t}{\x_t - \x_{t,\is}}^2},
  \end{align*}
  where the second step requires $C_{29} \ge 1$ by \pref{lem:ln-de}.

  For \emph{convex} functions, the base regret can be bounded as 
  \begin{align*}
      & \base \le 5D \sqrt{1 + \E[\Vb_T]} \le 5D \sqrt{1 + 10 \sigma_{1:T}^2 + 5 \Sigma_{1:T}^2 + 20L \E\mbr{\sumT \D_{F_t}(\xs, \x_t)}}\\
      \le {} & \O\sbr{\sqrt{\sigma_{1:T}^2 + \Sigma_{1:T}^2}} + \O(C_{30}) + \frac{5D}{2C_{30}} \E\mbr{\sumT \D_{F_t}(\xs, \x_t)},
  \end{align*}
  where the first step is by Jensen's inequality, the second step is due to \eqref{eq:SEA-bregman}, and the last step is because of AM-GM inequality (\pref{lem:AM-GM}). $C_{30}$ is a constant to be specified.
  
  \paragraph{Overall Regret Analysis.}
  For \emph{$\lambda$-strongly convex} functions, by combining the meta and base regret, it holds that
  \begin{align*}
      & \Reg_T \le \O\sbr{\frac{1}{\lambda} \sbr{\sigma_{\max}^2 + \Sigma_{\max}^2} \log \frac{\sigma_{1:T}^2 + \Sigma_{1:T}^2}{\sigma_{\max}^2 + \Sigma_{\max}^2}} + \O(C_3 + \log C_{27} + \log C_{28}) \\
      & + \sbr{\frac{C_0 D^2}{2 C_3} + \frac{16D^2}{C_{28}} - \frac{\lambda_\is}{4}} \E\mbr{\sumT \|\x_t - \x_{t,\is}\|^2} + \sbr{\frac{960LGD}{C_{27}} - \frac{1}{2}} \E\mbr{\sumT \D_{F_t}(\xs, \x_t)}\\
      \le {} & \O\sbr{\frac{1}{\lambda} \sbr{\sigma_{\max}^2 + \Sigma_{\max}^2} \log \frac{\sigma_{1:T}^2 + \Sigma_{1:T}^2}{\sigma_{\max}^2 + \Sigma_{\max}^2}},
  \end{align*}
  by choosing $C_{27} = \max\{1, 1920 LGD\}$, $C_{28} = \max\{1, {128 D^2}/{\lambda_\is}\}$ and $C_3 ={4C_0 D^2}/{\lambda_\is}$. Note that such a parameter configuration will only add an $\O(1/\lambda)$ factor to the final bound, which can be absorbed.

  For \emph{$\alpha$-exp-concave} functions, by combining the meta and base regret, it holds that
  \begin{align*}
      \Reg_T \le {} & \O\sbr{\frac{d}{\alpha} \log \sbr{\sigma_{1:T}^2 + \Sigma_{1:T}^2}} + \O(C_4 + \log C_{29}) + \sbr{\frac{120L}{C_{29}} - \frac{1}{2}} \E\mbr{\sumT \D_{F_t}(\xs, \x_t)}\\
      & \qquad + \sbr{\frac{C_0}{2 C_4} + \frac{4G^2}{C_{29}} - \frac{\alpha_\is}{4}} \E\mbr{\sumT \inner{\g_t}{\x_t - \x_{t,\is}}^2} \le \O\sbr{\frac{d}{\alpha} \log \sbr{\sigma_{1:T}^2 + \Sigma_{1:T}^2}},
  \end{align*}
  by choosing $C_{29} = \max\{1, 240L, {32G^2}/{\alpha_\is}\}$ and $C_4 ={4 C_0}/{\alpha_\is}$. Note that such a parameter configuration will only add an $\O(1/\alpha)$ factor to the final regret bound, which is absorbed.

  For \emph{convex} functions, by combining the meta and base regret, it holds that
  \begin{align*}
      \Reg_T \le {} & \O\sbr{\sqrt{\sigma_{1:T}^2 + \Sigma_{1:T}^2}} + \O(C_{26} + C_{30}) + \sbr{\frac{C_0}{2 C_{26}} + \frac{5D}{2C_{30}} - 1} \E\mbr{\sumT \D_{F_t}(\xs, \x_t)}\\
      \le {} & \O\sbr{\sqrt{\sigma_{1:T}^2 + \Sigma_{1:T}^2}},
  \end{align*}
  by choosing $C_{26} = C_0$ and $C_{30} = 5D$.  

  Note that the constants $C_3, C_4, C_{26}, C_{27}, C_{28}, C_{29}, C_{30}$ only exist in analysis and and hence our choices of them are feasible.  

\end{proof}

\subsection{Proof of Theorem~\ref{thm:game}}
\label{app:game}
\begin{proof}
    For the dishonest case, the two-player game degenerates to two separate online convex optimization problems. Therefore, the results for both bilinear and strongly-convex-strongly-concave games follow directly from \pref{app:unigrad-correct-1grad}. 
    
    For the honest case, we focus on the convex-concave game, where $f(\x, \y)$ is convex in $\x$ and concave in $\y$, which subsumes both bilinear and strongly-convex-strongly-concave games. The proof begins by analyzing the regret of player-$\x$, following the structure in \pref{app:unigrad-correct-1grad}. Leveraging the benign structure of the game, we upper bound the empirical gradient variation by $S_T^\x\define \sumTT \|\x_t - \x_{t-1}\|^2$ and $S_T^\y\define \sumTT \|\y_t - \y_{t-1}\|^2$. By symmetry, we derive a corresponding bound for player-$\y$, and combine the two via cancellation to obtain the final result.

    To start, we denote $\gamma^\Mid,\gamma^\Top,S_T^\Top,q_{t,j}^{\Top},Z$ for player-$\x$ and player-$\y$ by $\gamma_{\x}^{\Mid},\gamma_{\x}^{\Top},$ $S_T^{\Top,\x},  q_{t,j}^{\Top,\x},Z^\x$ and $\gamma_{\y}^{\Mid},\gamma_{\y}^{\Top},S_T^{\Top,\y},q_{t,j}^{\Top,\y},Z^\y$ separately. We choose $D=\sqrt{2}$.
 For the honest case, we starts with the player-$\x$. Following the analysis of the convex functions in \pref{app:unigrad-correct-1grad}, we decompose its regret by
\begin{align}
\Reg_T^\x \le \underbrace{\sumT \inner{\g^\x_t}{\x_t - \x_{t,\is}}}_{\meta^\x} + \underbrace{\sumT \inner{\g^\x_t}{\x_{t,\is} - \xs}}_{\base^\x}.\label{eq:thm7 decompose x}
\end{align}
For the meta regret, combining \pref{eq:lem1 eq1} and \pref{eq:thm1 eq0}, we have
\begin{align}
    & \meta^\x \le \frac{Z^\x}{\varepsilon^\Top_{\js}} \log \frac{N}{3 C_0^2 (\varepsilon^\Top_{\js})^2} + 32Z^\x\varepsilon^\Top_{\js} \Vs+ \gamma_{\x}^{\Mid} S_{T,\is}^{\x}-\frac{C_0}{2}S_T^{\Top,\x}\notag\\
    & \qquad\qquad\qquad -\gamma_{\x}^{\Mid} \sumTT \sumM q_{t,j}^{\Top,\x} \sumN q_{t,j,i}^{\Mid,\x}\norm{\x_{t,i} - \x_{t-1,i}}^2 \notag\\
    & \qquad\qquad\qquad - \gamma_{\y}^{\Top} \sumTT \sumM q_{t,j}^{\Top,\x} \|\q_{t,j}^{\Mid,\x} - \q_{t-1,j}^{\Mid,\x}\|_1^2\notag\\
    \le {} & \frac{Z^\x}{\varepsilon^\Top_{\js}} \log \frac{N}{3 C_0^2 (\varepsilon^\Top_{\js})^2} + \frac{64 \varepsilon^\Top_{\js}}{Z^\x} \sumT \norm{\g^\x_t-\g^\x_{t-1}}^2+ \sbr{\frac{64G^2}{Z^\x}+\gamma_{\x}^{\Mid}} S_{T,\is}^{\x}-\frac{C_0}{2}S_T^{\Top,\x}\notag\\
    & \qquad\qquad\qquad +\frac{64G^2}{Z^\x}S_T^\x-\gamma_{\x}^{\Mid} \sumTT \sumM q_{t,j}^{\Top,\x} \sumN q_{t,j,i}^{\Mid,\x}\norm{\x_{t,i} - \x_{t-1,i}}^2 \notag\\
    & \qquad\qquad\qquad - \gamma_{\y}^{\Top} \sumTT \sumM q_{t,j}^{\Top,\x} \|\q_{t,j}^{\Mid,\x} - \q_{t-1,j}^{\Mid,\x}\|_1^2\notag\\
    \le {} & \frac{Z^\x}{\varepsilon^\Top_{\js}} \log \frac{N}{3 C_0^2 (\varepsilon^\Top_{\js})^2} + \frac{32}{Z^\x} S_T^\x+\frac{32}{Z^\x} S_T^\y+ \sbr{\frac{64G^2}{Z^\x}+\gamma_{\x}^{\Mid}} S_{T,\is}^{\x}-\frac{C_0}{2}S_T^{\Top,\x}\notag\\
    & \qquad\qquad\qquad +\frac{64G^2}{Z^\x}S_T^\x-\gamma_{\x}^{\Mid} \sumTT \sumM q_{t,j}^{\Top,\x} \sumN q_{t,j,i}^{\Mid,\x}\norm{\x_{t,i} - \x_{t-1,i}}^2\notag \\
    & \qquad\qquad\qquad - \gamma_{\x}^{\Top} \sumTT \sumM q_{t,j}^{\Top,\x} \|\q_{t,j}^{\Mid,\x} - \q_{t-1,j}^{\Mid,\x}\|_1^2,\label{eq:thm7 meta x}
\end{align}
where the third step is by $\norm{\g^\x_t-\g^\x_{t-1}}^2\le2\norm{\x_t-\x_{t-1}}^2+2\norm{\y_t-\y_{t-1}}^2$ and $\varepsilon^\Top_{\js}\le  1/2$.

For the base regret, by \pref{lem:convex-base}, we have 
\begin{align}
    \base^\x\le& 5\sqrt{2}\sqrt{1+ \sumT \norm{\g^\x_t-\g^\x_{t-1}}^2} + \gamma_{\x}^{\Top} - \frac{1}{4} \gamma_{\x}^{\Top} S_{T,\is}^{\x} + \O(1)\notag\\
    \le& 5\sqrt{2} \sumT \norm{\g^\x_t-\g^\x_{t-1}}^2+ \gamma_{\x}^{\Top}  - \frac{1}{4} \gamma_{\x}^{\Top} S_{T,\is}^{\x}+\O(1)\notag \\
    \le& 10\sqrt{2}S_T^\x+10\sqrt{2}S_T^\y \gamma_{\x}^{\Top}   - \frac{1}{4} \gamma_{\x}^{\Top} S_{T,\is}^{\x}+\O(1), \label{eq:thm7 base x}
\end{align}
where the second step is by AM-GM inequality (\pref{lem:AM-GM}) and the third step is by  $\norm{\g^\x_t-\g^\x_{t-1}}^2\le 2\norm{\x_t-\x_{t-1}}^2+2\norm{\y_t-\y_{t-1}}^2$.
Plugging \pref{eq:thm7 meta x} and \pref{eq:thm7 base x} into \pref{eq:thm7 decompose x}, we obtain
\begin{align}
    &\Reg^\x_T \le  \sbr{\frac{32}{Z^\x}+10\sqrt{2}}S_T^\y+ \sbr{\frac{64G^2}{Z^\x}+\gamma_{\x}^{\Mid}-\frac{1}{4} \Bottomcoef} S_{T,\is}^{\x}+ 2\Bottomcoef-\frac{C_0}{2}S_T^{\Top,\x}\notag\\
    {}&\quad+\sbr{\frac{32+64G^2}{Z^\x}+10\sqrt{2}}S_T^\x-\gamma_{\x}^{\Mid} \sumTT \sumM q_{t,j}^{\Top,\x} \sumN q_{t,j,i}^{\Mid,\x}\norm{\x_{t,i} - \x_{t-1,i}}^2 \notag \\
    &\quad- \gamma_{\x}^{\Top} \sumTT \sumM q_{t,j}^{\Top,\x} \|\q_{t,j}^{\Mid,\x}- \q_{t-1,j}^{\Mid,\x}\|_1^2+\O(1)\notag\\
    \le {}& \sbr{\frac{32}{Z^\x}+10\sqrt{2}}S_T^\y+ 2\Bottomcoef-\frac{C_0}{2}S_T^{\Top,\x}+\sbr{\frac{32+64G^2}{Z^\x}+10\sqrt{2}}S_T^\x+\O(1)\notag\\
    & \quad -\gamma_{\x}^{\Mid} \sumTT \sumM q_{t,j}^{\Top,\x} \sumN q_{t,j,i}^{\Mid,\x}\norm{\x_{t,i} - \x_{t-1,i}}^2- \gamma_{\x}^{\Top} \sumTT \sumM q_{t,j}^{\Top,\x} \|\q_{t,j}^{\Mid,\x}- \q_{t-1,j}^{\Mid,\x}\|_1^2, \label{eq:thm7 x}
\end{align}
where the last step is by choosing $\Bottomcoef \ge  256G^2+4\gamma_{\x}^{\Mid}$. Symmetrically, for player-$\y$, we can obtain a similar bound 
\begin{align}
&\Reg^\y_T \le  \sbr{\frac{32}{Z^\y}+10\sqrt{2}}S_T^\x+ \sbr{\frac{64G^2}{Z^\y}+\gamma_{\y}^{\Mid}-\frac{1}{4} \Bottomcoef} S_{T,\is}^{\y}+ 2\Bottomcoef-\frac{C_0}{2}S_T^{\Top,\y}\notag\\
{}&\quad+\sbr{\frac{32+64G^2}{Z^\y}+10\sqrt{2}}S_T^\y-\gamma_{\y}^{\Mid} \sumTT \sumM q_{t,j}^{\Top,\y} \sumN q_{t,j,i}^{\Mid,\y}\norm{\y_{t,i} - \y_{t-1,i}}^2 \notag \\&\quad- \gamma_{\x}^{\Top} \sumTT \sumM q_{t,j}^{\Top,\y} \|\q_{t,j}^{\Mid,\y} - \q_{t-1,j}^{\Mid,\y}\|_1^2+\O(1)\notag\\
\le{}&\sbr{\frac{32}{Z^\y}+10\sqrt{2}}S_T^\x+ 2\Bottomcoef-\frac{C_0}{2}S_T^{\Top,\y}+\sbr{\frac{32+64G^2}{Z^\y}+10\sqrt{2}}S_T^\y+\O(1)\notag\\
{}&-\gamma_{\y}^{\Mid} \sumTT \sumM q_{t,j}^{\Top,\y} \sumN q_{t,j,i}^{\Mid,\y}\norm{\y_{t,i} - \y_{t-1,i}}^2- \gamma_{\x}^{\Top} \sumTT \sumM q_{t,j}^{\Top,\y} \|\q_{t,j}^{\Mid,\y} - \q_{t-1,j}^{\Mid,\y}\|_1^2, \label{eq:thm7 y}
\end{align}
where the last step is by choosing $\Bottomcoef \ge  256+4\gamma_{\y}^{\Mid}$.
Combining \pref{eq:thm7 x} and \pref{eq:thm7 y} and letting $C_{31}=\frac{32+64G^2}{Z^\x}+\frac{32}{Z^\y}+20\sqrt{2}$ and $C_{32}=\frac{32+64G^2}{Z^\y}+\frac{32}{Z^\x}+20\sqrt{2}$, we obtain
\begin{align*}
    &\Reg^\x_T + \Reg^\y_T \le  C_{31}S_T^\x+C_{32}S_T^\y-\frac{C_0}{2}(S_T^{\Top,\x}+S_T^{\Top,\y})+\O(1)\\
    &-\gamma_{\x}^{\Mid} \sumTT \sumM q_{t,j}^{\Top,\x} \sumN q_{t,j,i}^{\Mid,\x}\norm{\x_{t,i} - \x_{t-1,i}}^2- \gamma_{\x}^{\Top} \sumTT \sumM q_{t,j}^{\Top,\x} \|\q_{t,j}^{\Mid,\x}- \q_{t-1,j}^{\Mid,\x}\|_1^2\\
    &-\gamma_{\y}^{\Mid} \sumTT \sumM q_{t,j}^{\Top,\y} \sumN q_{t,j,i}^{\Mid,\y}\norm{\y_{t,i} - \y_{t-1,i}}^2- \gamma_{\y}^{\Top} \sumTT \sumM q_{t,j}^{\Top,\y} \|\q_{t,j}^{\Mid,\y} - \q_{t-1,j}^{\Mid,\y}\|_1^2\\
    \le{}&\sbr{8C_{31}-\frac{C_0}{2}}S_T^{\Top,\x}+\sbr{8C_{32}-\frac{C_0}{2}}S_T^{\Top,\y}+\O(1)\\
    &+(2C_{31}-\gamma_{\x}^{\Mid})\sumTT \sumM q_{t,j}^{\Top,\x} \sumN q_{t,j,i}^{\Mid,\x}\norm{\x_{t,i} - \x_{t-1,i}}^2\\
    &+(2C_{32}-\gamma_{\y}^{\Mid})\sumTT \sumM q_{t,j}^{\Top,\y} \sumN q_{t,j,i}^{\Mid,\y}\norm{\y_{t,i} - \y_{t-1,i}}^2\\
    &+(8C_{31}-\gamma_{\x}^{\Top})\sumTT \sumM q_{t,j}^{\Top,\x} \|\q_{t,j}^{\Mid,\x}- \q_{t-1,j}^{\Mid,\x}\|_1^2\\
    &+(8C_{32}-\gamma_{\y}^{\Top})\sumTT \sumM q_{t,j}^{\Top,\y} \|\q_{t,j}^{\Mid,\y} - \q_{t-1,j}^{\Mid,\y}\|_1^2\le \O(1),
\end{align*}
where the first step is due to \pref{lem:decompose-three-layer} and the second step is by choosing  $\gamma_{\x}^{\Mid} \ge  2C_{31}$, $\gamma_{\y}^{\Mid} \ge  2C_{32}$, $\gamma_{\x}^{\Top} \ge  8C_{31}$, $\gamma_{\y}^{\Top} \ge  8C_{32}$, and $C_0 \ge  16\max\{C_{31},C_{32}\}$. 

At last, we determine the specific values of $C_0$, $\gamma_{\x}^{\Top}$, $\gamma_{\x}^{\Mid}$, $\gamma_{\y}^{\Top}$, $\gamma_{\y}^{\Mid}$. These parameters need to satisfy the following requirements:
\begin{gather*}
  C_0 \ge  1,\ C_0 \ge  8D,\ C_0 \ge  4\gamma_{\x}^{\Top},\ C_0 \ge  4\gamma_{\y}^{\Top},\ C_0 \ge  16C_{31},\ C_0 \ge  16C_{32},\ \gamma_{\x}^{\Mid} \ge  2C_{31},\\  \gamma_{\y}^{\Mid} \ge  2C_{32},\  \gamma_{\x}^{\Top} \ge  8C_{31},\ \gamma_{\y}^{\Top} \ge  8C_{32}.
\end{gather*}
As a result, we set 
\begin{gather*}
  C_0=\max\bbr{1,8D,4\gamma^{\Top},16C_{31},16C_{32}, 4\gamma_{\x}^{\Top}, 4\gamma_{\y}^{\Top}},\\
  \gamma_{\x}^{\Mid}=\gamma_{\y}^{\Mid}=128+128G^2+40\sqrt{2},\quad \gamma_{\x}^{\Top}=\gamma_{\y}^{\Top}=512+512G^2+160\sqrt{2},
\end{gather*}
where $C_{31}=\frac{32+64G^2}{Z^\x}+\frac{32}{Z^\y}+20\sqrt{2}$,  $C_{32}=\frac{32+64G^2}{Z^\y}+\frac{32}{Z^\x}+20\sqrt{2}$, $Z^\x=\max\{GD+\gamma_{\x}^{\Mid} D^2, 1+\gamma_{\x}^{\Mid} D^2+2\gamma_{\x}^{\Top}\}$, and $Z^\y=\max\{GD+\gamma_{\y}^{\Mid} D^2, 1+\gamma_{\y}^{\Mid} D^2+2\gamma_{\y}^{\Top}\}$. The proof is finished.
\end{proof}

\subsection{Proof of Theorem~\ref{thm:anytime}}
\label{app:anytime}
\begin{proof}
    In this proof, note that we do not need to analyze the convex case because the convex base learner is naturally anytime by using a self-confident tuning step size, as in \pref{lem:convex-base}. Therefore, in the rest of the proof, we focus on the cases of $\lambda$-strongly convex and $\alpha$-exp-concave functions.
     
    \paragraph{Regret Decomposition.}
    For \emph{$\lambda$-strongly convex} functions, recall that $\is$ denotes the index of the best base learner whose strong convexity coefficient satisfies $\lambda_\is \le \lambda \le 2\lambda_\is$. We consider two cases at the $\tau$-round: \textit{(i)} the $\is$-th base learner is not activated (i.e., $\tau < s_{\is}^\scvx$); and \textit{(ii)} the $\is$-th base learner is activated (i.e., $\tau \ge s_{\is}^\scvx$). For case \textit{(i)}, we obtain
    \begin{equation*}
        \Reg_\tau =\sum_{t=1}^{\tau} f_t(\x_t) - \sum_{t=1}^{\tau} f_t(\xs)\le \tau GD \le (s_{\is}^\scvx-1) GD \le \O\sbr{\frac{1}{\lambda_\is}}\le  \O\sbr{\frac{1}{\lambda}},
    \end{equation*}
    where the third step follows from the fact that $\tau \le s_{\is}^\scvx - 1$, since the corresponding base learner has not been activated at the $\tau$-th round. The fourth step is due to the activation condition of $s_{\is}^\scvx-1 < \frac{1}{\lambda_\is}$. For case \textit{(ii)}, we first decompose the regret into two parts, where the first part corresponds to case \textit{(i)}, and the second one refers to regret after the $\is$-th base learner is activated:
    \begin{equation}
        \Reg_\tau = \underbrace{\sum_{t=1}^{s_{\is}^\scvx-1} f_t(\x_t) - \sum_{t=1}^{s_{\is}^\scvx-1} f_t(\xs)}_{\term{a}} + \underbrace{\sum_{t=s_{\is}^\scvx}^{\tau} f_t(\x_t) - \sum_{t=s_{\is}^\scvx}^{\tau} f_t(\xs)}_{\term{b}}.\label{eq:thm8 scvx de}
    \end{equation}
    \term{a} can be bounded by $\O\sbr{{1}/{\lambda}}$ as in case \textit{(i)}. We then decompose \term{b} into the following two parts:
    \begin{align}
        \term{b}\le {} & \underbrace{ \sum_{t=s_{\is}^\scvx}^{\tau} \inner{\g_t}{\x_t - \x_{t,\is}} - \frac{\lambda_\is}{4}  \sum_{t=s_{\is}^\scvx}^{\tau} \|\x_t - \x_{t,\is}\|^2}_{\meta} \tag*{(by $\lambda_\is \le \lambda \le 2 \lambda_\is$)}\\
        & \qquad\qquad + \underbrace{ \sum_{t=s_{\is}^\scvx}^{\tau} \hsc_{t,\is}(\x_{t,\is}) -  \sum_{t=s_{\is}^\scvx}^{\tau} \hsc_{t,\is}(\xs)}_{\base} - \frac{1}{2} \sum_{t=s_{\is}^\scvx}^{\tau} \D_{f_t}(\xs, \x_t).\label{eq:thm8 scvx decompose}
    \end{align}
    For \emph{$\alpha$-exp-concave} functions, recall that $\is$ denotes the index of the best base learner whose exp-concave coefficient satisfies $\alpha_\is \le \alpha \le 2\alpha_\is$. Similar to the strongly convex functions, we consider two cases at the $\tau$-round: \textit{(i)} the $\is$-th base learner is not activated (i.e., $\tau < s_{\is}^\exp$); and \textit{(ii)} the $\is$-th base learner is activated (i.e., $\tau \ge s_{\is}^\exp$). For case \textit{(i)}, we obtain
    \begin{equation*}
        \Reg_\tau =\sum_{t=1}^{\tau} f_t(\x_t) - \sum_{t=1}^{\tau} f_t(\xs)\le\tau GD \le (s_{\is}^\exp-1) GD \le \O\sbr{\frac{1}{\alpha_\is}}\le  \O\sbr{\frac{d}{\alpha}},
    \end{equation*}
    where the third step follows from the fact that $\tau \le s_{\is}^\exp - 1$, since the corresponding base learner has not been activated at this time. The fourth step is due to the activation condition of $s_{\is}^\exp-1 < \frac{1}{\alpha_\is}$. For case \textit{(ii)}, we first decompose the regret into two parts, where the first part corresponds to case \textit{(i)}, and the second one refers to regret after the $\is$-th base learner is activated:
    \begin{equation}
        \Reg_\tau = \underbrace{\sum_{t=1}^{s_{\is}^\exp-1} f_t(\x_t) - \sum_{t=1}^{s_{\is}^\exp-1} f_t(\xs)}_{\term{a}} + \underbrace{\sum_{t=s_{\is}^\exp}^{\tau} f_t(\x_t) - \sum_{t=s_{\is}^\exp}^{\tau} f_t(\xs)}_{\term{b}}.\label{eq:thm8 exp de}
    \end{equation}
    We bound \term{a} by $\O\sbr{\frac{d}{\alpha}}$ as in case \textit{(i)} and decompose \term{b} into the following two parts:
    \begin{align}
        \term{b}\le {} & \underbrace{ \sum_{t=s_{\is}^\exp}^{\tau} \inner{\g_t}{\x_t - \x_{t,\is}} - \frac{\alpha_\is}{4}  \sum_{t=s_{\is}^\exp}^{\tau} \inner{\g_t}{\x_t - \x_{t,\is}}^2}_{\meta} \tag*{(by $\alpha_\is \le \alpha \le 2 \alpha_\is$)}\\
        & \qquad\qquad + \underbrace{ \sum_{t=s_{\is}^\exp}^{\tau} \hexp_{t,\is}(\x_{t,\is}) -  \sum_{t=s_{\is}^\exp}^{\tau} \hexp_{t,\is}(\xs)}_{\base} - \frac{1}{2} \sum_{t=s_{\is}^\exp}^{\tau} \D_{f_t}(\xs, \x_t).\label{eq:thm8 exp decompose}
    \end{align}
    \paragraph{Meta Regret Analysis.}
    We adopt \omlprod variant as the meta learner, and present its regret analysis below for self-containedness.
    \begin{myLemma}[Simplfied Theorem 7 of \citet{NeurIPS'24:Xie}]
        \label{lem:anytime-optimistic-mlprod}
        Denoting by $\A_t$ the active expert set at time $t$, $\p_t \in \Delta_{|\A_t|}$ the algorithm's weights, $\ellb_t \in \R^{|\A_t|}$ the loss vector, and $\m_t \in \R^{|\A_t|}$ the optimism. Assuming that the $i$-th expert participates in prediction during time $[a,b]$ and choosing the learning rate optimally as \pref{eq:anytime-lr}, the regret of \omlprod variant with respect to expert $i$ satisfies
        \begin{equation*}
            \sum_{t=a}^b \inner{\ellb_t}{\p_t - \e_i} \le C_{33} \sqrt{1 + \sum_{t=a}^b (r_{t,i} - m_{t,i})^2},
        \end{equation*}
        where $\e_i$ denotes the $\ith$ standard basis vector, $C_{33} = \O(\log N_b + \log(1+\log b))$, and $N_b$ represents the total number of base learners initialized till time $b+1$.
    \end{myLemma}

    For \emph{$\lambda$-strongly convex} functions,  the meta regret in \pref{eq:thm8 scvx decompose} can be bounded as 
    \begin{align}
        & \meta \le C_{33} \sqrt{1 +  \sum_{t=s_{\is}^\scvx}^{\tau} \inner{\g_t}{\x_t - \x_{t,\is}}^2} - \frac{\lambda_\is}{4}  \sum_{t=s_{\is}^\scvx}^{\tau} \|\x_t - \x_{t,\is}\|^2  \notag\\
        \le {} & C_{33} \sqrt{1 + G^2  \sum_{t=s_{\is}^\scvx}^{\tau} \|\x_t - \x_{t,\is}\|^2} - \frac{\lambda_\is}{4}  \sum_{t=s_{\is}^\scvx}^{\tau} \|\x_t - \x_{t,\is}\|^2  \notag\\
        \le {} & \O(C_{34}) + \sbr{\frac{C_{33} G^2}{2 C_{34}} - \frac{\lambda_\is}{4}}  \sum_{t=s_{\is}^\scvx}^{\tau} \|\x_t - \x_{t,\is}\|^2, \label{eq:thm8 scvx-meta}
    \end{align}
    where the first step follows from \pref{lem:anytime-optimistic-mlprod} by choosing $a=s_{\is}^\scvx$ and $b=\tau$, the last step uses AM-GM inequality (\pref{lem:AM-GM}) and omits the ignorable additive $C_{33}$ terms. $N_\tau=\O(\log\tau)$ and we omit the $\O(\log \log \tau)$ term. $C_{34}$ is a constant to be specified.

    For \emph{$\alpha$-exp-concave} functions,  the meta regret in \pref{eq:thm8 exp decompose} can be bounded as 
    \begin{align}
        & \meta \le C_{33} \sqrt{1 +  \sum_{t=s_{\is}^\exp}^{\tau} \inner{\g_t}{\x_t - \x_{t,\is}}^2} - \frac{\alpha_\is}{4}  \sum_{t=s_{\is}^\exp}^{\tau} \inner{\g_t}{\x_t - \x_{t,\is}}^2   \notag\\
        \le {} & C_{33} \sqrt{1 + \sum_{t=s_{\is}^\exp}^{\tau} \inner{\g_t}{\x_t - \x_{t,\is}}^2 } - \frac{\alpha_\is}{4}  \sum_{t=s_{\is}^\exp}^{\tau} \inner{\g_t}{\x_t - \x_{t,\is}}^2  \notag\\
        \le {} & \O(C_{36}) + \sbr{\frac{C_{33}}{2 C_{36}} - \frac{\alpha_\is}{4}}  \sum_{t=s_{\is}^\exp}^{\tau} \inner{\g_t}{\x_t - \x_{t,\is}}^2 , \label{eq:thm8 exp-meta}
    \end{align}
    where the first step follows from \pref{lem:anytime-optimistic-mlprod} by choosing $a=s_{\is}^\exp$ and $b=\tau$, the last step uses AM-GM inequality (\pref{lem:AM-GM}) and omits the ignorable additive $C_{33}$ terms. It's because $N_\tau=\O(\log\tau)$ and we omit $\log \log \tau$ terms. $C_{36}$ is a constant to be specified.

    \paragraph{Base Regret Analysis.}  

    Following the analysis structure of \pref{app:unigrad-correct-1grad}, we first provide different decompositions of the empirical gradient variation defined on surrogates for strongly convex and exp-concave functions, respectively, and then analyze the base regret in the corresponding cases.
    For \emph{$\lambda$-strongly convex} functions, since the $\is$-th base learner is activated at time $s_\is^\scvx$, we can directly bound the empirical gradient variation on surrogates, i.e.,  $\Vb^\scvx_{[s_{\is}^\scvx,\tau],\is}\define \sum_{t=s_{\is}^\scvx+1}^\tau \|\nabla \hsc_{t,\is}(\x_{t,\is}) - \nabla \hsc_{t-1,\is}(\x_{t-1,\is})\|^2$ by \eqref{eq:thm4 VbT-to-VT}
    \begin{align*}
        \Vb_{[s_{\is}^\scvx,\tau],\is}^{\scvx}\le {} & 9 V_{[s_{\is}^\scvx,\tau]} + 36L  \sum_{t=s_{\is}^\scvx}^{\tau}\D_{f_t}(\xs, \x_t) + 2 \lambda_\is^2  \sum_{t=s_{\is}^\scvx}^{\tau} \norm{\x_{t,\is} - \x_t}^2,
    \end{align*}
    where $V_{[s_{\is}^\scvx,\tau]}=\sum_{t=s_{\is}^\scvx+1}^\tau\sup_{\x\in\X}\|\nabla f_t(\x)-f_{t-1}(\x)\|^2$.
    
    For \emph{$\alpha$-exp-concave} functions, by \eqref{eq:thm4 VbT-to-VT}, we can similarly bound the empirical gradient variation on surrogates, i.e.,  $\Vb^\exp_{[s_{\is}^\exp,\tau],\is}\define \sum_{t=s_{\is}^\exp+1}^\tau \|\nabla \hexp_{t,\is}(\x_{t,\is}) - \nabla \hexp_{t-1,\is}(\x_{t-1,\is})\|^2$, by
    \begin{align*}
    \Vb_{[s_\is^\exp,\tau],\is}^{\exp}
    \le {}  9 V_{[s_\is^\exp,\tau]} + 36L \sum_{t=s_\is^\exp}^{\tau} \D_{f_t}(\xs, \x_t) + 2 \alpha_\is^2 G^2 \sum_{t=s_\is^\exp}^{\tau} \inner{\g_t}{\x_t - \x_{t,\is}}^2,
    \end{align*}
    where $V_{[s_{\is}^\exp,\tau]}=\sum_{t=s_{\is}^\exp+1}^\tau\sup_{\x\in\X}\|\nabla f_t(\x)-f_{t-1}(\x)\|^2$.
    To conclude, for different curvature types, we provide correspondingly different analysis of the empirical gradient variation on surrogates:
    \begin{equation}
    \label{eq:thm8 VbT-to-VT}
    \Vb_{[s_{\is}^\scvx,\tau],\is}^{\{\scvx,\exp\}} \le \left\{
    \begin{aligned}
        & 9 V_{[s_{\is}^\scvx,\tau]} + 36L  \sum\limits_{t=s_{\is}^\scvx}^{\tau}\D_{f_t}(\xs, \x_t) + 2 \lambda_\is^2  \sum\limits_{t=s_{\is}^\scvx}^{\tau} \norm{\x_{t,\is} - \x_t}^2, \text{ ($\lambda$-strongly convex)}\\
        & 9 V_{[s_\is^\exp,\tau]} + 36L \sum\limits_{t=s_\is^\exp}^{\tau} \D_{f_t}(\xs, \x_t) + 2 \alpha_\is^2 G^2 \sum\limits_{t=s_\is^\exp}^{\tau} \inner{\g_t}{\x_t - \x_{t,\is}}^2.\text{ ($\alpha$-exp-concave)}
    \end{aligned}
    \right.
    \end{equation}
    In the following, we analyze the base regret for different curvature types. Since the empirical gradient variation shares a similar structure to that in \pref{thm:unigrad-bregman-1grad}, we can directly apply the corresponding result.
    For \emph{$\lambda$-strongly convex} functions, according to \pref{eq:thm4 scvx base}, the base regret can be bounded as 
    \begin{align}
        \base\le{}& \O\sbr{\frac{1}{\lambda}\log V_{[s_{\is}^\scvx,\tau]}} + \frac{576 G^2 L}{C_{35}} \sum_{t=s_{\is}^\scvx}^{\tau} \D_{f_t}(\xs, \x_t)\notag\\
        &\qquad\qquad + \frac{32G^2}{C_{35}} \sum_{t=s_{\is}^\scvx}^{\tau} \norm{\x_{t,\is} - \x_t}^2 + \O(\log C_{35}).\label{eq:thm8 scvx base}
    \end{align}
    For \emph{$\alpha$-exp-concave} functions, according to \pref{eq:thm4 exp base}, the base regret can be bounded as 
    \begin{align}
        \base\le{}& \O\sbr{\frac{d}{\alpha}\log V_{[s_\is^\exp,\tau]}} + \frac{72L}{C_{37}}  \sum_{t=s_\is^\exp}^{\tau} \D_{f_t}(\xs, \x_t)\notag\\
        &\qquad \qquad+ \frac{4G^2}{C_{37}}  \sum_{t=s_\is^\exp}^{\tau} \inner{\g_t}{\x_t - \x_{t,\is}}^2 + \O(\log C_{37}).\label{eq:thm8 exp base}
    \end{align}

    \paragraph{Overall Regret Analysis.}
    For \emph{$\lambda$-strongly convex} functions, by combining \pref{eq:thm8 scvx de}, \pref{eq:thm8 scvx decompose}, \pref{eq:thm8 scvx-meta}, and \pref{eq:thm8 scvx base}, we obtain
    \begin{align*}
        \Reg_\tau \le {}& \O\sbr{\frac{1}{\lambda}}+\O\sbr{\frac{1}{\lambda}\log V_{[s_\is^\scvx,\tau]}}+\sbr{\frac{576 G^2 L}{C_{35}}-\frac{1}{2}} \sum_{t=s_{\is}^\scvx}^{\tau} \D_{f_t}(\xs, \x_t)\\
        &+\sbr{\frac{32G^2}{C_{35}}+\frac{C_{33} G^2}{2 C_{34}} - \frac{\lambda_\is}{4}} \sum_{t=s_{\is}^\scvx}^{\tau} \norm{\x_{t,\is} - \x_t}^2+ \O\sbr{C_{34}+\log C_{35}}\le  \O\sbr{\frac{1}{\lambda}\log V_\tau},
    \end{align*}
    where we choose $C_{34}= 4C_{33}G^2/\lambda_\is$ and $C_{35}=\max\{1, 256G^2/\lambda_\is, 1152G^2L\}$.

    For \emph{$\alpha$-exp-concave} functions, combining \pref{eq:thm8 exp de}, \pref{eq:thm8 exp decompose}, \pref{eq:thm8 exp-meta}, and \pref{eq:thm8 exp base},
    \begin{align*}
        \Reg_\tau \le {}& \O\sbr{\frac{d}{\alpha}}+\O\sbr{\frac{d}{\alpha}\log V_{[s_\is^\exp,\tau]}}+\sbr{\frac{72L}{C_{37}}-\frac{1}{2}}  \sum_{t=s_\is^\exp}^{\tau} \D_{f_t}(\xs, \x_t)\\
        &+\sbr{\frac{4G^2}{C_{37}}+\frac{C_{33}}{2 C_{36}} - \frac{\alpha_\is}{4}} \sum_{t=s_\is^\exp}^{\tau}\inner{\g_t}{\x_t - \x_{t,\is}}^2+ \O\sbr{C_{36}+\log C_{37}}\le  \O\sbr{\frac{d}{\alpha}\log V_\tau},
    \end{align*}
    where we choose $C_{36}= 4C_{33}/\alpha_\is$ and $C_{37}=\max\{1, 144L, 32G^2/\alpha_\is\}$. Note that the constants $C_{33}, C_{34}, C_{35}, C_{36},C_{37}$ only exist in analysis and and hence our choices of them are feasible.
\end{proof}
\section{Technical Lemmas}
\label{appendix:technical}
In this section, we present several supporting lemmas used in proving our theoretical results.
In \pref{app:useful}, we provide useful lemmas for the decomposition of two combined decisions and the parameter tuning.
And in \pref{app:negative-base}, we analyze the stability-based negative terms of the base algorithms for different curvature types.

\subsection{Useful Lemmas}
\label{app:useful}
In this part, we conclude some useful lemmas for bounding the gap between two combined decisions (\pref{lem:decompose} and \pref{lem:decompose-simplex}), tuning the parameter (\pref{lem:tune-eta-1} and \pref{lem:tune-eta-2}), and a useful summation (\pref{lem:sum}).

\begin{myLemma}
  \label{lem:decompose}
  Under \pref{assum:domain-boundedness}, if $\x = \sumN p_i \x_i, \y = \sumN q_i \y_i$, where $\p, \q \in \Delta_N, \x_i, \y_i \in \X$ for any $i \in [N]$, then it holds that
  \begin{equation*}
    \|\x - \y\|^2 \le 2 \sumN p_i \|\x_i - \y_i\|^2 + 2D^2 \|\p - \q\|_1^2.
  \end{equation*}
\end{myLemma}

\begin{myLemma}
  \label{lem:decompose-simplex}
  If $\w = \sumN q_{i} \p_i, \w^\prime = \sumN q^\prime_{i} \p^\prime_i$, where $\q, \q^\prime \in \Delta_N$ and $\p_i, \p_i^\prime \in \Delta_d$ for any $i \in [N]$, then it holds that
  \begin{equation*}
    \|\w - \w^\prime\|^2 \le 2 \sumN q_i \|\p_i - \p^\prime_i\|_1^2 + 2 \|\q - \q^\prime\|_1^2.
  \end{equation*}
\end{myLemma}

\begin{myLemma}
  \label{lem:tune-eta-1}
  For a step size pool of $\H_\eta = \{\eta_k\}_{k \in [K]}$, where $\eta_1 = \frac{1}{2C_0} \ge \ldots \ge \eta_K = \frac{1}{2C_0 T}$, if $C_0 \ge \frac{\sqrt{X}}{2T}$, there exists $\eta \in \H_\eta$ such that
  \begin{equation*}
    \frac{1}{\eta} \log \frac{Y}{\eta^2} + \eta X \le 2C_0 \log (4Y C_0^2) + 4 \sqrt{X \log (4XY)}.
  \end{equation*}
\end{myLemma}

\begin{myLemma}
  \label{lem:tune-eta-2}
  Denoting by $\eta_\star$ the optimal step size, for a step size pool of $\H_\eta = \{\eta_k\}_{k \in [K]}$, where $\eta_1 = \frac{1}{2C_0} \ge \ldots \ge \eta_K = \frac{1}{2C_0 T}$, if $C_0 \ge \frac{1}{2 \eta_\star T}$, there exists $\eta \in \H_\eta$ such that
  \begin{equation*}
    \frac{1}{\eta} \log \frac{Y}{\eta^2} \le 2C_0 \log (4Y C_0^2) + \frac{2}{\eta_\star} \log \frac{4Y}{\eta_\star^2}.
  \end{equation*}
\end{myLemma}

\begin{myLemma}
  \label{lem:sum}
  For a sequence of $\{a_t\}_{t=1}^T$ and $b$, where $a_t, b > 0$ for any $t \in [T]$, denoting by $a_{\max} \define \max_t a_t$ and $A \define \ceil{b \sumT a_t}$, we have 
  \begin{equation*}
    \sumT \frac{a_t}{bt} \le \frac{a_{\max}}{b} (1 + \log A) + \frac{1}{b^2}.
  \end{equation*}
\end{myLemma}

\begin{myLemma}
  \label{lem:ln-de}
  For any $a>1, b>0$, it holds that $\log(a+b) \le \log (Ca) + \frac{b}{C}$ for some $C \ge 1$.
\end{myLemma}

\begin{myLemma}[Corollary 5 of \citet{AISTATS'12:Orabona}]
  \label{lem:small-loss-log}
  If $a, b, c, d, x>0$ satisfy $x-d \le  a \log (b x+c)$, then it holds that
  \begin{equation*}
      x-d \le  a \log \sbr{2 a b \log \frac{2 a b}{e} + 2bd + 2c}.
  \end{equation*}
\end{myLemma}

\begin{myLemma}[Lemma 9 of \citet{JMLR'24:Sword++}]
  \label{lem:small-loss-sqrt}
 For any $x, y, a, b>0$ satisfying $x-y \le  \sqrt{a x}+b$, it holds that
  \begin{equation*}
     x-y \le  \sqrt{a y+ab}+a+b.
  \end{equation*}
\end{myLemma}

\begin{myLemma}[AM-GM Inequality]
    \label{lem:AM-GM}
    $\sqrt{xy} \le \frac{ax}{2} + \frac{y}{2a}$ for any $x,y,a > 0$.
\end{myLemma}

\begin{proof}[of \pref{lem:decompose}]
  The term of $\|\x - \y\|^2$ can be decomposed as follows:
  \begin{align*}
    \|\x - \y\|^2 = {} & \norm{\sumN p_i \x_i - \sumN q_i \y_i}^2 = \norm{\sumN p_i \x_i - \sumN p_i \y_i + \sumN p_i \y_i - \sumN q_i \y_i}^2\\
    \le {} & 2 \norm{\sumN p_i (\x_i - \y_i)}^2 + 2 \norm{\sumN (p_i - q_i) \y_i}^2\\
    \le {} & 2 \sbr{\sumN p_i \|\x_i - \y_i\|}^2 + 2 \sbr{\sumN |p_i - q_i| \|\y_i\|}^2\\
    \le {} & 2 \sumN p_i \|\x_i - \y_i\|^2 + 2D^2 \|\p-\q\|_1^2,
  \end{align*}
  where the first inequality is due to $(a+b)^2 \le 2a^2 + 2b^2$ for any $a,b \in \R$, and the last step is due to Cauchy-Schwarz inequality, \pref{assum:domain-boundedness} and the definition of $\ell_1$-norm.
\end{proof}

\begin{proof}[of \pref{lem:decompose-simplex}]
  The proof follows a similar flow as \pref{lem:decompose}.
  Specifically, we first decompose it as 
  \begin{equation*}
      \|\w - \w^\prime\|_1^2 = \norm{\sumN q_i \p_i - \sumN q^\prime_i \p^\prime_i}_1^2 \le 2\underbrace{\norm{\sumN q_i (\p_i - \p^\prime_i)}_1^2}_{\term{a}} + 2 \underbrace{\norm{\sumN (q_i - q^\prime_i) \p^\prime_i}_1^2}_{\term{b}}.
  \end{equation*}
  For \term{a}, we have 
  \begin{align*}
      \term{a} = {} & \norm{\sumN q_i (\p_i - \p^\prime_i)}_1^2 = \sbr{\sumN \abs{\sum_{j=1}^d q_i (p_{i,j} - p^\prime_{i,j})}}^2 \le \sbr{\sumN \sum_{j=1}^d q_i \abs{p_{i,j} - p^\prime_{i,j}}}^2\\
      = {} & \sbr{\sumN q_i \sum_{j=1}^d \abs{p_{i,j} - p^\prime_{i,j}}}^2 \le \sumN q_i \sbr{\sum_{j=1}^d \abs{p_{i,j} - p^\prime_{i,j}}}^2 = \sumN q_i \|\p_i - \p^\prime_i\|_1^2.
  \end{align*}
  For \term{b}, we have
  \begin{align*}
      \term{b} = {} & \norm{\sumN (q_i - q^\prime_i) \p^\prime_i}_1^2 = \sbr{\sumN \abs{\sum_{j=1}^d (q_i - q^\prime_i) p^\prime_{i,j}}}^2 \le \sbr{\sumN \sum_{j=1}^d \abs{q_i - q^\prime_i} p^\prime_{i,j}}^2\\
      = {} & \sbr{\sumN \abs{q_i - q^\prime_i} \sum_{j=1}^d p^\prime_{i,j}}^2 = \sbr{\sumN \abs{q_i - q^\prime_i}}^2 = \|\q - \q^\prime\|_1^2,
  \end{align*}
  where the second last step is due to $\sum_{j=1}^d p^\prime_{i,j} = 1$.
  Combining the bounds for \term{a} and \term{b} finishes the proof.
\end{proof}

\begin{proof}[of \pref{lem:tune-eta-1}]
  Denoting the optimal step size by $\eta_\star \define \sqrt{\log (4XY) / X}$, if the optimal step size satisfies $\eta \le \eta_\star \le 2\eta$, where $\eta \le  \eta_\star$ can be guaranteed if $C_0 \ge \frac{\sqrt{X}}{2T}$, then 
  \begin{equation*}
    \frac{1}{\eta} \log \frac{Y}{\eta^2} + \eta X \le \frac{2}{\eta_\star} \log \frac{4Y}{\eta_\star^2} + \eta_\star X \le 3 \sqrt{X \log (4XY)}.
  \end{equation*}
  Otherwise, if the optimal step size is greater than the maximum step size in the parameter pool, i.e., $\eta_\star \ge (\eta = \eta_1 = \frac{1}{2C_0})$, then we have 
  \begin{equation*}
    \frac{1}{\eta} \log \frac{Y}{\eta^2} + \eta X \le \frac{1}{\eta} \log \frac{Y}{\eta^2} + \eta_\star X \le 2C_0 \log (4Y C_0^2) + \sqrt{X \log (4XY)}.
  \end{equation*}
  Overall, it holds that 
  \begin{equation*}
    \frac{1}{\eta} \log \frac{Y}{\eta^2} + \eta X \le 2C_0 \log (4Y C_0^2) + 4 \sqrt{X \log (4XY)},
  \end{equation*}
  which completes the proof.
\end{proof}

\begin{proof}[of \pref{lem:tune-eta-2}]
  The proof follows the same flow as \pref{lem:tune-eta-1}.
\end{proof}

\begin{proof}[of \pref{lem:sum}]
  This result is inspired by Lemma 5 of \citet{ICML'23:OMD4SEA}, and we generalize it to arbitrary variables for our purpose.
  Specifically, we consider two cases: $A < T$ and $A \ge T$.
  For the first case, if $A<T$, it holds that 
  \begin{align*}
    \sumT \frac{a_t}{bt} = \sum_{t=1}^A \frac{a_t}{bt} + \sum_{A+1}^T \frac{a_t}{bt} \le \frac{a_{\max}}{b} \sum_{t=1}^A \frac{1}{t} + \frac{1}{b(A+1)} \sum_{A+1}^T a_t \le \frac{a_{\max}}{b} (1 + \log A) + \frac{1}{b^2},
  \end{align*}
  where the last step is due to $\sum_{A+1}^T a_t \le \sumT a_t \le A/b$.
  The case of $A < T$ can be proved similarly, which finishes the proof.
\end{proof}

\begin{proof}[of \pref{lem:ln-de}]
  The one-line proof is presented below:
  \begin{equation*}
      \log(a+b) \le \log (Ca+b) \le \log (Ca) + \log \sbr{1 + \frac{b}{Ca}} \le \log (Ca) + \frac{b}{C},
  \end{equation*}
  where the first step is due to $C \ge 1$, and the last step adopts $\log (1+x) \le x$ for any $x \ge 0$.
\end{proof}

\subsection{Stability Analysis of Base Algorithms}
\label{app:negative-base}
In this part, we analyze the negative stability terms in the optimistic OMD analysis, for convex, exp-concave and strongly convex functions, respectively.
For simplicity, we define the \emph{empirical gradient variation} below: 
\begin{equation}
  \label{eq:empirical-VT}
  \Vb_T \define \sumTT \|\g_t- \g_{t-1}\|^2, \text{ where }\g_t \define \nabla f_t(\x_t).
\end{equation}
Next we provide the regret analysis in terms of the empirical gradient-variation $\Vb_T$, for strongly convex (\pref{lem:str-convex-base}),  exp-concave (\pref{lem:exp-concave-base}), and convex (\pref{lem:convex-base}) functions.

\begin{myLemma}
  \label{lem:str-convex-base}
  Under Assumptions~\ref{assum:domain-boundedness}, \ref{assum:gradient-boundedness}, and \ref{assum:smoothness}, if the loss functions are $\lambda$-strongly convex, OOGD~\eqref{eq:OOGD} with $\mb_t = \nabla f_{t-1}(\x_{t-1})$ and $\eta_t = 2/(\Bottomcoef + \lambda t)$, where $\Bottomcoef$ is a parameter to be specified, enjoys the following empirical gradient-variation bound:
  \begin{equation*}
    \sumT f_t(\x_t) - \min_{\x \in \X} \sumT f_t(\x) \le \frac{16G^2}{\lambda} \log \sbr{1 + \lambda \Vb_T} + \frac{1}{4} \Bottomcoef D^2 - \frac{\Bottomcoef}{8} \sumTT \|\x_t - \x_{t-1}\|^2 + \O(1).
  \end{equation*}
\end{myLemma}

\begin{myLemma}
  \label{lem:exp-concave-base}
  Under Assumptions~\ref{assum:domain-boundedness}, \ref{assum:gradient-boundedness}, and \ref{assum:smoothness}, if the loss functions are $\alpha$-exp-concave, \OOMD~\eqref{eq:ONNS} with $U_t=\Bottomcoef I+\frac{\alpha G^2}{2}I+\frac{\alpha}{2} \sum_{s=1}^{t-1} \nabla f_s(\x_s) \nabla f_s(\x_s)^\top$, where $\Bottomcoef$ is a parameter to be specified, enjoys the following empirical gradient-variation bound:
  \begin{equation*}
    \sumT f_t(\x_t) - \min_{\x \in \X} \sumT f_t(\x) \le \frac{16d}{\alpha} \log \sbr{1 + \frac{\alpha}{8 \Bottomcoef d} \Vb_T} + \frac{1}{2} \Bottomcoef D^2 - \frac{\Bottomcoef}{4} \sumTT \|\x_t - \x_{t-1}\|^2 + \O(1).
  \end{equation*}
\end{myLemma}

\begin{myLemma}
  \label{lem:convex-base}
  Under Assumptions~\ref{assum:domain-boundedness}, \ref{assum:gradient-boundedness}, and \ref{assum:smoothness}, if the loss functions are convex, OOGD~\eqref{eq:OOGD} with $\mb_t = \nabla f_{t-1}(\x_{t-1})$ and $\eta_t = \min\bbr{D / \sqrt{1 + \Vb_{t-1}}, 1/\Bottomcoef}$, where $\Bottomcoef$ is a parameter to be specified, enjoys the following empirical gradient-variation bound:
  \begin{equation*}
    \sumT f_t(\x_t) - \min_{\x \in \X} \sumT f_t(\x) \le 5D \sqrt{1 + \Vb_T} + \Bottomcoef D^2 - \frac{\Bottomcoef}{4} \sumTT \|\x_t - \x_{t-1}\|^2 + \O(1).
  \end{equation*}
\end{myLemma}

\begin{proof}[of \pref{lem:str-convex-base}]
  The proof mainly follows Theorem 3 of \citet{ICML'23:OMD4SEA}.
  Following the almost the same regret decomposition in \pref{lem:convex-base}, it holds that
  \begin{align*}
      \sumT f_t(\x_t) - \sumT f_t(\xs) \le{}& \underbrace{\sumT \eta_t \|\nabla f_t(\x_t) - \mb_t\|^2}_{\textsc{Adaptivity}} + \underbrace{\sumT \frac{1}{\eta_t} (\Dpsi(\xs, \xh_t) - \Dpsi(\xs, \xh_{t+1}))}_{\textsc{Opt-Gap}}\notag \\
      & - \underbrace{\sumT \frac{1}{\eta_t} (\Dpsi(\xh_{t+1}, \x_t) + \D_{\psi_t}(\x_t, \xh_t))}_{\textsc{Stability}} - \underbrace{\frac{\lambda}{2} \sumT \norm{\x_t - \xs}^2}_{\textsc{Negativity}}.
  \end{align*}
  First, we analyze the optimality gap,
  \begin{align*}
      \textsc{Opt-Gap} \le {} & \frac{1}{\eta_1} \Dpsi(\xs, \xh_1) + \sumT \sbr{\frac{1}{\eta_{t+1}} - \frac{1}{\eta_t}} \Dpsi(\xs, \xh_{t+1})\\
      \le {} & \frac{1}{4} (\Bottomcoef + \lambda) D^2 + \frac{\lambda}{4} \sumT \|\xs - \xh_{t+1}\|^2.
  \end{align*}
  We handle the last term by leveraging the negative term imported by strong convexity:
  \begin{align*}
      \textsc{Opt-Gap} - \textsc{Negativity} 
      \le {} & \frac{1}{4} (\Bottomcoef + \lambda) D^2 + \frac{\lambda}{4} \sumT \|\xs - \xh_{t+1}\|^2 - \frac{\lambda}{2} \sumT \norm{\x_t - \xs}^2\\
      \le {} & \frac{1}{4} \Bottomcoef D^2 + \frac{\lambda}{2} \sumT \|\x_t - \xh_{t+1}\|^2 + \O(1).
  \end{align*}
  The second term above can be bounded by the stability of optimistic OMD: 
  \begin{equation*}
      \frac{\lambda}{2} \sumT \|\x_t - \xh_{t+1}\|^2 \le \frac{\lambda}{2} \sumT \eta_t^2 \|\nabla f_t(\x_t) - \mb_t\|^2 \le \sumT \eta_t \|\nabla f_t(\x_t) - \mb_t\|^2.
  \end{equation*}
  Finally, we lower-bound the stability term as 
  \begin{align*}
    \textsc{Stability} = {} & \sumT \frac{\Bottomcoef + \lambda t}{4} (\|\xh_{t+1} - \x_t\|^2 + \|\x_t - \xh_t\|^2) \tag*{\Big(by $\eta_t = \frac{2}{\Bottomcoef + \lambda t}$\Big)}\\
    \ge {} & \frac{\Bottomcoef}{4} \sumTT (\|\xh_t - \x_{t-1}\|^2 + \|\x_t - \xh_t\|^2) \ge \frac{\Bottomcoef}{8} \sumTT \|\x_t - \x_{t-1}\|^2.
  \end{align*}
  Choosing the optimism as $\mb_t = \nabla f_{t-1}(\x_{t-1})$, we have
  \begin{equation*}
      \Reg_T \le 2 \sumT \eta_t \|\nabla f_t(\x_t) - \nabla f_{t-1}(\x_{t-1})\|^2 + \frac{1}{4} \Bottomcoef D^2 - \frac{\Bottomcoef}{8} \sumTT \|\x_t - \x_{t-1}\|^2 + \O(1).
  \end{equation*}
  To analyze the first term above, we follow the similar argument of \citet{ICML'23:OMD4SEA}.
  By \pref{lem:sum} with $a_t = \|\nabla f_t(\x_t) - \nabla f_{t-1}(\x_{t-1})\|^2$, $a_{\max} = 4G^2$, $A = \ceil{\lambda \Vb_T}$, and $b = \lambda$, it holds that 
  \begin{equation*}
    \sumT \frac{1}{\lambda t} \|\g_t - \g_{t-1}\|^2 \le \frac{4G^2}{\lambda} \log \sbr{1 + \lambda \Vb_T} + \frac{4G^2}{\lambda} + \frac{1}{\lambda^2}.
  \end{equation*}
  Since $\eta_t = 2/(\Bottomcoef + \lambda t) \le 2 / (\lambda t)$, combining existing results, we have
  \begin{equation*}
    \Reg_T \le \frac{16G^2}{\lambda} \log \sbr{1 + \lambda \Vb_T} + \frac{1}{4} \Bottomcoef D^2 - \frac{\Bottomcoef}{8} \sumTT \|\x_t - \x_{t-1}\|^2 + \O(1),
  \end{equation*}
  which completes the proof.
\end{proof}

\begin{proof}[of \pref{lem:exp-concave-base}]
  The proof mainly follows Theorem 15 of \citet{COLT'12:VT}.
  Denoting by $\xs \in \argmin_{\x \in \X} \sumT f_t(\x)$, it holds that 
  \begin{align*}
    \sumT f_t(\x_t) - \sumT f_t(\xs) \le \underbrace{\sumT \|\nabla f_t(\x_t) - \mb_t\|^2_{U_t^{-1}}}_{\textsc{Adaptivity}} + \underbrace{\sumT (\D_{\psi_t}(\xs, \xh_t) - \D_{\psi_t}(\xs, \xh_{t+1}))}_{\textsc{Opt-Gap}} &\\
    - \underbrace{\sumT (\D_{\psi_t}(\xh_{t+1}, \x_t) + \D_{\psi_t}(\x_t, \xh_t))}_{\textsc{Stability}} - \underbrace{\frac{\alpha}{2} \sumT \norm{\x_t - \xs}^2_{\nabla f_t(\x_t) \nabla f_t(\x_t)^\top}}_{\textsc{Negativity}}&,
  \end{align*}
  where the last term is imported by the definition of exp-concavity.
  First, the optimality gap satisfies
  \begin{align*}
      \textsc{Opt-Gap} = {} & \frac{1}{2} \sumT \|\xs - \xh_t\|_{U_t}^2 - \frac{1}{2} \sumT \|\xs - \xh_{t+1}\|_{U_t}^2\\
      \le {} & \frac{1}{2} \|\xs - \xh_1\|_{V_1}^2 + \frac{1}{2} \sumT (\|\xs - \xh_{t+1}\|_{U_{t+1}}^2 - \|\xs - \xh_{t+1}\|_{U_t}^2)\\
      \le {} & \frac{1}{2} \Bottomcoef D^2 + \frac{\alpha G^2 D^2}{4} + \frac{\alpha}{4} \sumT \|\xs - \xh_{t+1}\|_{\nabla f_t(\x_t) \nabla f_t(\x_t)^\top}^2.
  \end{align*}
  We handle the last term by leveraging the negative term imported by exp-concavity:
  \begin{align*}
      & \textsc{Opt-Gap} - \textsc{Negativity}\\
      \le {} & \frac{1}{2} \Bottomcoef D^2 + \frac{\alpha}{4} \sumT \|\xs - \xh_{t+1}\|_{\nabla f_t(\x_t) \nabla f_t(\x_t)^\top}^2 - \frac{\alpha}{2} \sumT \norm{\x_t - \xs}^2_{\nabla f_t(\x_t) \nabla f_t(\x_t)^\top} + \O(1)\\
      \le {} & \frac{1}{2} \Bottomcoef D^2 + \frac{\alpha}{2} \sumT \|\x_t - \xh_{t+1}\|_{\nabla f_t(\x_t) \nabla f_t(\x_t)^\top}^2 + \O(1),
  \end{align*}
  where the local norm of the second term above can be transformed into $U_t$:
  \begin{equation*}
      \frac{\alpha}{2} \sumT \|\x_t - \xh_{t+1}\|_{\nabla f_t(\x_t) \nabla f_t(\x_t)^\top}^2 \le \frac{\alpha G^2}{2} \sumT \|\x_t - \xh_{t+1}\|^2 \le \sumT \|\x_t - \xh_{t+1}\|^2_{U_t}.
  \end{equation*}
  Using the stability of optimistic OMD~{\citep[Proposition 7]{COLT'12:VT}}, the above term can be further bounded by 
  \begin{equation*}
      \sumT \|\x_t - \xh_{t+1}\|^2_{U_t} \le \sumT \|\nabla f_t(\x_t) - \mb_t\|^2_{U_t^{-1}}.
  \end{equation*}
  By choosing the optimism as $\mb_t = \nabla f_{t-1}(\x_{t-1})$, the above term can be consequently bounded due to Lemma 19 of \citet{COLT'12:VT}:
  \begin{equation*}
    \sumT \|\nabla f_t(\x_t) - \nabla f_{t-1}(\x_{t-1})\|^2_{U_t^{-1}} \le \frac{8d}{\alpha} \log \sbr{1 + \frac{\alpha}{8 \Bottomcoef d} \Vb_T}.
  \end{equation*}
  The last step is to analyze the negative stability term:
  \begin{align*}
    \textsc{Stability} = {} & \sumT (\D_{\psi_t}(\xh_{t+1}, \x_t) + \D_{\psi_t}(\x_t, \xh_t)) = \frac{1}{2} \sumT \|\xh_{t+1} - \x_t\|_{U_t}^2 + \frac{1}{2} \sumT \|\x_t - \xh_t\|_{U_t}^2\\
    \ge {} & \frac{\Bottomcoef}{2} \sumT \|\xh_{t+1} - \x_t\|^2 + \frac{\Bottomcoef}{2} \sumT \|\x_t - \xh_t\|^2 \ge \frac{\Bottomcoef}{4} \sumTT \|\x_t - \x_{t-1}\|^2.
  \end{align*}
  Combining existing results, we have
  \begin{equation*}
    \Reg_T \le \frac{16d}{\alpha} \log \sbr{1 + \frac{\alpha}{8 \Bottomcoef d} \Vb_T} + \frac{1}{2} \Bottomcoef D^2 - \frac{\Bottomcoef}{4} \sumTT \|\x_t - \x_{t-1}\|^2 + \O(1),
  \end{equation*} 
  which completes the proof.
\end{proof}

\begin{proof}[of \pref{lem:convex-base}]
  The proof mainly follows Theorem 11 of \citet{COLT'12:VT}.
  Following the standard analysis of optimistic OMD, e.g., Theorem 1 of~\citet{JMLR'24:Sword++}, 
  \begin{align*}
    \sumT f_t(\x_t) - \sumT f_t(\xs) \le {} & \underbrace{\sumT \eta_t \|\nabla f_t(\x_t) - \mb_t\|^2}_{\textsc{Adaptivity}} + \underbrace{\sumT \frac{1}{\eta_t} (\Dpsi(\xs, \xh_t) - \Dpsi(\xs, \xh_{t+1}))}_{\textsc{Opt-Gap}}\notag \\
    & - \underbrace{\sumT \frac{1}{\eta_t} (\Dpsi(\xh_{t+1}, \x_t) + \D_{\psi_t}(\x_t, \xh_t))}_{\textsc{Stability}},
  \end{align*}
  where $\xs \in \argmin_{\x \in \X} \sumT f_t(\x)$ and $\psi(\cdot) \define \frac{1}{2} \|\cdot\|^2$.
  The adaptivity term satisfies
  \begin{align*}
    \textsc{Adaptivity} = {} & \sumT \eta_t \|\nabla f_t(\x_t) - \mb_t\|^2 \le D \sumT \frac{\|\nabla f_t(\x_t) - \nabla f_{t-1}(\x_{t-1})\|^2}{\sqrt{1 + \sum_{s=1}^{t-1} \|\nabla f_s(\x_s) - \nabla f_{s-1}(\x_{s-1})\|^2}}\\
    \le {} & 4 D \sqrt{1 + \sumT \|\nabla f_t(\x_t) - \nabla f_{t-1}(\x_{t-1})\|^2} + 4DG^2,
  \end{align*}
  where the last step uses $\sumT a_t/\sqrt{1 + \sum_{s=1}^{t-1} a_s} \le 4 \sqrt{1 + \sumT a_t} + \max_{t \in [T]} a_t$~{\citep[Lemma 4.8]{UAI'19:FIRST-ORDER}}.
  Next, we move on to the optimality gap,
  \begin{align*}
      \textsc{Opt-Gap} = {} & \sumT \frac{1}{\eta_t} (\Dpsi(\xs, \xh_t) - \Dpsi(\xs, \xh_{t+1})) = \sumT \frac{1}{2 \eta_t} (\|\xs - \xh_t\|^2 - \|\xs - \xh_{t+1}\|^2)\\
      \le {} & \frac{\|\xs - \xh_1\|^2}{2 \eta_1} + \sumTT \sbr{\frac{1}{2\eta_t} - \frac{1}{2\eta_{t-1}}} \|\xs - \xh_t\|^2\\
      \le {} & \frac{D}{2} (1 + \Bottomcoef D) + D^2 \sumTT \sbr{\frac{1}{2\eta_t} - \frac{1}{2\eta_{t-1}}} \le \frac{D}{2} (1 + \Bottomcoef D) + \frac{D^2}{2 \eta_T}\\
      = {} & \Bottomcoef D^2 + \frac{D}{2} \sqrt{1 + \Vb_T} + \O(1).
  \end{align*}
  Finally, we analyze the stability term,
  \begin{align*}
    \textsc{Stability} = {} & \sumT \frac{1}{2\eta_t} (\|\xh_{t+1} - \x_t\|^2 + \|\x_t - \xh_t\|^2) \ge \sumTT \frac{1}{2\eta_t} (\|\xh_t - \x_{t-1}\|^2 + \|\x_t - \xh_t\|^2)\\
    \ge {} & \sumTT \frac{1}{4\eta_t} \|\x_t - \x_{t-1}\|^2 \ge \frac{\Bottomcoef}{4} \sumTT \|\x_t - \x_{t-1}\|^2.
  \end{align*}
  Combining the above inequalities completes the proof.
\end{proof}

\end{document}